\documentclass[11pt,twoside]{article}

\usepackage{fullpage}
\usepackage{mathtools}

\usepackage{epsf}
\usepackage{fancyhdr}
\usepackage{graphics}
\usepackage{graphicx}
\usepackage{psfrag}

\usepackage{booktabs}

\usepackage[linesnumbered,ruled]{algorithm2e}
\DontPrintSemicolon
\SetAlFnt{\small\sffamily} 

\usepackage{caption,subcaption}

\usepackage{float}
\usepackage{color}

\usepackage{etoc}

\usepackage{enumitem}

\usepackage{amsfonts}
\usepackage{amsmath}
\usepackage{amssymb}

\usepackage{nicefrac}
\usepackage{chngpage}

\usepackage{jmlr2e}

\newcommand{\MWARM}{\ensuremath{\mathcal{P}_M(\target)}}
\newcommand{\warmparam}{\ensuremath{M}}

\newcommand{\UNICON}{\ensuremath{c}}

\newcommand{\MYPOLY}{\ensuremath{f}}

\newcommand{\lovone}{\ensuremath{\rho}}

\newcommand{\lovtwo}{\ensuremath{\Delta}}

\newcommand{\MyTerm}{\ensuremath{S}}

\newcommand{\tvnorm}[1]{\ensuremath{\| #1\|_{\mbox{\tiny{TV}}}}}
\newcommand{\kldiv}[2]{\ensuremath{\operatorname{KL}(#1 \Vert #2)}}

\newcommand{\target}{\ensuremath{\pi^*}}

\newcommand{\myinitial}{\ensuremath{\pi^0}}

\newcommand{\mytrans}{\ensuremath{p}}

\newcommand{\plaintransition}{\mathcal{T}}

\newcommand{\Borel}{\ensuremath{\mathcal{B}}}
\newcommand{\Ball}{\ensuremath{\mathbb{B}}}

\newcommand{\rparam}{\ensuremath{r}}

\newcommand{\myvec}{v} 
\newcommand{\myvectwo}{w} 
\newcommand{\mymat}{B}
\newcommand{\set}{\mathcal{S}}

\newcommand{\defn}{:=}

\newcommand{\polylog}{\text{poly-log}}
\newcommand{\polylogfactor}{\kappa_{\obs, \dims}}
\newcommand{\tmix}{k_\text{mix}}

\DeclareMathOperator{\supp}{supp}
\newcommand{\Pspace}{\ensuremath{\mathcal{K}}}
\newcommand{\intP}{\ensuremath{\operatorname{int}\parenth{\Pspace}}}
\newcommand{\boundary}{\ensuremath{\partial\Pspace}}
\newcommand{\condition}{\ensuremath{\gamma_\Pspace}}

\newcommand{\Poly}{\ensuremath{\Pspace}}
\newcommand{\myset}{\ensuremath{\mathcal{X}}}

\newcommand{\diracdelta}{\mathbf{\delta}}
\newcommand{\proposal}{\mathcal{P}}
\newcommand{\density}{p}
\newcommand{\transition}{\mathbb{T}}
\newcommand{\transdensity}{q}
\newcommand{\targetdensity}{\target}
\newcommand{\stationary}{\target}
\newcommand{\initial}{\mu_0}
\newcommand{\rvg}{\xi} 

\newcommand{\tailconst}{\chi}
\newcommand{\tailconstjohn}{\chi}

\newcommand{\nolazytrans}{\tilde\plaintransition}
\newcommand{\lazytrans}{{\plaintransition}}
\newcommand{\conductance}{\Phi} 

\newcommand{\degree}{k}
\newcommand{\repsquare}{\mathcal{S}}

\newcommand{\logbarr}{\mathcal{F}}
\newcommand{\hesslogbarr}{\ensuremath{\nabla^2\logbarr}}
\newcommand{\hessbarr}{H} 

\newcommand{\vaidyabarr}{\mathcal{V}}

\newcommand{\dikin}{{D}}
\newcommand{\john}{J}
\newcommand{\vaidya}{V}

\newcommand{\tagvaidya}{\text{\tiny V}}
\newcommand{\tagjohn}{\text{\tiny J}}
\newcommand{\fulltagdikin}{\text{\tiny Dikin}}
\newcommand{\fulltagvaidya}{\text{\tiny Vaidya}}
\newcommand{\fulltagjohn}{\text{\tiny John}}
\newcommand{\fulltagimprovedjohn}{\text{\tiny John$^+$}}

\newcommand{\inverserate}{\nu} 
\newcommand{\weights}{w}
\newcommand{\weightmatrix}{W}
\newcommand{\offset}{\beta}  
\newcommand{\slack}{s} 
\newcommand{\slackmatrix}{S} 
\newcommand{\levscores}{\sigma}
\newcommand{\localslack}{\theta}

\newcommand{\VW}[1]{\text{VW($#1$)}}
\newcommand{\vaidyabeta}{\offset_\tagvaidya}
\newcommand{\levvaidya}{\levscores} 
\newcommand{\localslackvaidya}{\localslack}
\newcommand{\vaidyaepsilonconst}{\ensuremath{1/15}}
\newcommand{\vaidyaradiusbound}{f}
\newcommand{\vaidyaradiusconst}{\ensuremath{10^{-4}}}

\newcommand{\logdetvaidya}{\Psi}
\newcommand{\gradlogdetV}{\nabla \logdetvaidya}
\newcommand{\hesslogdetV}{\nabla^2 \logdetvaidya}

\newcommand{\vaidyalevmatrix}{\Sigma} 
\newcommand{\vaidyaprojmatrix}{\Upsilon} 
\newcommand{\vaidyalocalslackmatrix}{\Theta} 
\newcommand{\vaidyasecondthetamatrix}{\Xi} 

\newcommand{\vaidyaphi}{\varphi} 
\newcommand{\projD}{\eta} 

\newcommand{\JW}[1]{\text{JW($#1$)}}\newcommand{\johnweights}{\zeta}
\newcommand{\johnalpha}{\alpha_\tagjohn}
\newcommand{\alphaprod}{\tau_\alpha} 
\newcommand{\johnbeta}{\offset_\tagjohn}
\newcommand{\johnkappa}{\kappa}
\newcommand{\levjohn}{\levscores}
\newcommand{\johnlocalslack}{\localslack}
\newcommand{\johnepsilonconst}{\ensuremath{1/30}}
\newcommand{\johnradiusbound}{h}
\newcommand{\johnradiusconst}{\ensuremath{10^{-5}}}

\newcommand{\logdetJ}{\Psi}
\newcommand{\gradlogdetJ}{\nabla \logdetJ}
\newcommand{\hesslogdetJ}{\nabla^2 \logdetJ}

\newcommand{\johnmod}{\tilde{\john}}
\newcommand{\johnweightmatrix}{G}
\newcommand{\johnlevmatrix}{\Sigma} 
\newcommand{\johnprojmatrix}{\Upsilon} 
\newcommand{\johnlaplacian}{\Lambda} 

\newcommand{\johnphi}{\varphi} 
\newcommand{\projf}{f} 
\newcommand{\projfmatrix}{F} 
\newcommand{\projd}{d} 
\newcommand{\projdmatrix}{D} 
\newcommand{\projz}{\ell} 
\newcommand{\projzmatrix}{L} 
\newcommand{\projzsub}{\rho} 
\newcommand{\johnftdinvmatrix}{\mu} 
\newcommand{\johnftdinvvector}{\nu}
\newcommand{\theinverse}{E} 

\newlength{\widebarargwidth}
\newlength{\widebarargheight}
\newlength{\widebarargdepth}

\makeatletter
\long\def\@makecaption#1#2{
        \vskip 0.8ex
        \setbox\@tempboxa\hbox{\small {\bf #1:} #2}
        \parindent 1.5em  
        \dimen0=\hsize
        \advance\dimen0 by -3em
        \ifdim \wd\@tempboxa >\dimen0
                \hbox to \hsize{
                        \parindent 0em
                        \hfil
                        \parbox{\dimen0}{\def\baselinestretch{0.96}\small
                                {\bf #1.} #2
                                }
                        \hfil}
        \else \hbox to \hsize{\hfil \box\@tempboxa \hfil}
        \fi
        }
\makeatother

\long\def\comment#1{}
\definecolor{battleshipgrey}{rgb}{0.52, 0.52, 0.51}
\definecolor{darkgray}{rgb}{0.66, 0.66, 0.66}
\definecolor{darkgreen}{rgb}{0.0, 0.2, 0.13}
\definecolor{darkspringgreen}{rgb}{0.09, 0.45, 0.27}
\definecolor{dukeblue}{rgb}{0.0, 0.0, 0.61}
\definecolor{olivedrab7}{rgb}{0.24, 0.2, 0.12}
\definecolor{darkblue}{rgb}{0.0, 0.0, 0.55}
\definecolor{darkscarlet}{rgb}{0.34, 0.01, 0.1}
\definecolor{candyapplered}{rgb}{1.0, 0.03, 0.0}
\definecolor{ao(english)}{rgb}{0.0, 0.5, 0.0}
\definecolor{applegreen}{rgb}{0.55, 0.71, 0.0}

\newcommand{\red}[1]{\textcolor{red}{#1}}
\newcommand{\mjwcomment}[1]{{\bf{{\red{{MJW --- #1}}}}}}
\newcommand{\blue}[1]{\textcolor{blue}{#1}}

\newcommand{\rrdcomment}[1]{{\bf{{\blue{{RRD --- #1}}}}}}

\newcommand{\matsnorm}[2]{|\!|\!| #1 | \! | \!|_{{#2}}}
\newcommand{\vecnorm}[2]{\left\| #1\right\|_{#2}}

\newcommand{\Exs}{\ensuremath{{\mathbb{E}}}}
\newcommand{\Prob}{\ensuremath{{\mathbb{P}}}}

\newcommand{\numobs}{\ensuremath{n}}
\newcommand{\usedim}{\ensuremath{d}}
\newcommand{\dims}{\usedim}
\newcommand{\obs}{\numobs}

\DeclareMathOperator{\diag}{diag}

\DeclareMathOperator{\trace}{trace}

\DeclareMathOperator{\vol}{vol}

\DeclareMathOperator{\rank}{rank}

\newcommand{\NORMAL}{\ensuremath{\mathcal{N}}}

\newcommand{\Xspace}{\ensuremath{\mathcal{X}}}

\newcommand{\widgraph}[2]{\includegraphics[keepaspectratio,width=#1]{#2}}

\newcommand{\Ind}{\ensuremath{\mathbb{I}}}

\newcommand{\real}{\ensuremath{\mathbb{R}}}
\newcommand{\realdim}{\ensuremath{\real^\usedim}}

\newcommand{\brackets}[1]{\left[ #1 \right]}
\newcommand{\parenth}[1]{\left( #1 \right)}
\newcommand{\braces}[1]{\left\{ #1 \right \}}
\newcommand{\abss}[1]{\left| #1 \right |}
\newcommand{\angles}[1]{\left\langle #1 \right \rangle}
\newcommand{\tp}{^\top}

\newcommand{\gaussian}[2]{\NORMAL\parenth{#1, #2}}
\newcommand{\order}[1]{\ensuremath{\mathcal{O}\parenth{#1}}}

\usepackage{lastpage}
\jmlrheading{19}{2018}{1-\pageref{LastPage}}{3/18}{9/18}{18-158}{Chen, Dwivedi, Wainwright and Yu}
\ShortHeadings{Fast MCMC Sampling Algorithms on Polytopes}{Chen,
Dwivedi, Wainwright and Yu}

\begin{document}
\etocdepthtag.toc{mtchapter}
\etocsettagdepth{mtchapter}{subsection}
\etocsettagdepth{mtappendix}{none}

\title{Fast MCMC Sampling Algorithms on Polytopes}

\author{\name Yuansi Chen$^{\ast, \diamondsuit}$ \email yuansi.chen@berkeley.edu \\
  \name Raaz Dwivedi$^{\ast, \dagger}$ \email raaz.rsk@berkeley.edu\\
      \name Martin J. Wainwright$^{\diamondsuit, \dagger, \ddagger}$ \email wainwrig@berkeley.edu\\
      \name Bin Yu$^{\diamondsuit,\dagger}$ \email binyu@berkeley.edu\\
        \addr Department of Statistics$^{\diamondsuit}$ \\
        Department of Electrical Engineering and Computer Sciences$^{\dagger}$ \\
        University of California, Berkeley \\
        Voleon Group$^{\ddagger}$, Berkeley
       }

\editor{Alexander Rakhlin}

\maketitle

\begin{abstract}%
We propose and analyze two new MCMC sampling algorithms, the Vaidya walk and the
John walk, for generating samples from the uniform distribution over a
polytope.  Both random walks are sampling
algorithms derived from interior point methods. The former is based on
volumetric-logarithmic barrier introduced by Vaidya whereas the latter
uses John's ellipsoids.  We show that the Vaidya walk mixes in
significantly fewer steps than the logarithmic-barrier based Dikin
walk studied in past work.  For a polytope in $\realdim$ defined by
$\obs > \dims$ linear constraints, we show that the mixing time from a
warm start is bounded as $\order{\obs^{0.5}\dims^{1.5}}$, compared to
the $\order{\obs\dims}$ mixing time bound for the Dikin walk.  The
cost of each step of the Vaidya walk is of the same order as the Dikin
walk, and at most twice as large in terms of constant pre-factors.
For the John walk, we prove an
$\order{\dims^{2.5}\cdot\log^4(\obs/\dims)}$ bound on its mixing time
and conjecture that an improved variant of it could achieve a mixing
time of $\order{\dims^{2}\cdot\polylog(\obs/\dims)}$.  Additionally,
we propose variants of the Vaidya and John walks that mix in
polynomial time from a deterministic starting point. The speed-up of the
Vaidya walk over the Dikin walk are illustrated in numerical examples.
\end{abstract}


{\let\thefootnote\relax\footnotetext{*Yuansi Chen and Raaz Dwivedi contributed equally to this work.}}



\section{Introduction}
\label{sec:introduction}

Sampling from distributions is a core problem in statistics,
probability, operations research, and other areas involving stochastic
models~\citep{Geman84,Bremaud91,ripley2009stochastic,hastings1970monte}.
Sampling algorithms are a prerequisite for applying Monte Carlo
methods to order to approximate expectations and other integrals.
Recent decades have witnessed great success of Markov Chain Monte
Carlo (MCMC) algorithms; for instance, see the handbook
by~\cite{brooks2011handbook} and references therein.  These methods
are based on constructing a Markov chain whose stationary distribution
is equal to the target distribution, and then drawing samples by
simulating the chain for a certain number of steps.  An advantage of
MCMC algorithms is that they only require knowledge of the target
density up to a proportionality constant.  However, the theoretical
understanding of MCMC algorithms used in practice is far from
complete.  In particular, a general challenge is to bound the
\emph{mixing time} of a given MCMC algorithm, meaning the number of
iterations---as a function of the error tolerance $\delta$, problem
dimension $\dims$ and other parameters---for the chain to arrive at a
distribution within distance $\delta$ of the target.

In this paper, we study a certain class of MCMC algorithms designed
for the problem of drawing samples from the uniform distribution over
a polytope.  The polytope is specified in the form $\Pspace \defn \{x
\in \real^\dims \mid Ax \leq b \}$, parameterized by the matrix-vector
pair \mbox{$(A,b) \in \real^{\obs \times \dims} \times \real^\obs$.}
Our goal is to understand the mixing time for obtaining
$\delta$-accurate samples, and how it grows as a function of the pair
$\parenth{\obs, \dims}$.

The problem of sampling uniformly from a polytope is important in
various applications and methodologies.  For instance, it underlies
various methods for computing randomized approximations to polytope
volumes.  There is a long line of work on sampling methods being used
to obtain randomized approximations to the volumes of polytopes and
other convex bodies \cite[see,
  e.g.,][]{lovasz1990ballwalk,lawrence1991polytope,
  belisle1993hit,lovasz1999hit, cousins2014cubic}.  Polytope sampling
is also useful in developing fast randomized algorithms for convex
optimization~\citep{bertsimas2004solving} and sampling contingency
tables~\citep{kannan2012random}, as well as in randomized methods for
approximately solving mixed integer convex
programs~\citep{huang2013empirical,huang2015empirical}.  Sampling from
polytopes is also related to simulations of the hard-disk model in
statistical physics~\citep{KapKra13}, as well as to simulations of
error events for linear programming in
communication~\citep{Feldman05}.

Many MCMC algorithms have been studied for sampling from polytopes, and
more generally, from convex bodies.  Some early examples include the
Ball Walk~\citep{lovasz1990ballwalk} and the hit-and-run
algorithm~\citep{belisle1993hit,lovasz1999hit}, which apply to sampling
from general convex bodies. Although these algorithms can be applied
to polytopes, they do not exploit any special structure of the
problem.  In contrast, the Dikin walk introduced by \cite{kannan2012random} is specialized to polytopes, and thus can achieve faster convergence rates than generic algorithms.
The Dikin walk was the first sampling algorithm based on a connection
to interior point methods for solving linear programs.  More
specifically, as we discuss in detail below, it constructs proposal
distributions based on the standard logarithmic barrier for a
polytope.  In a later paper, \cite{narayanan2016randomized}
extended the Dikin walk to general convex sets equipped with
self-concordant barriers.

For a polytope defined by $\obs$ constraints, \cite{kannan2012random} proved an upper bound on the mixing time of the Dikin walk that scales linearly with $\obs$. In many applications, the number of constraints $\obs$ can be much larger than the number of variables $\dims$. For example, we could imagine one using many hyperplane constraints to approximate complicated convex sets such as sphere or ellipsoid. For such
problems, linear dependence on the number of constraints is not
desirable.  Consequently, it is natural to ask if it is possible to
design a sampling algorithm whose mixing time scales in a sub-linear
manner with the number of constraints.  Our main contribution is to
investigate and answer this question in affirmative---in
particular, by designing and analyzing two sampling algorithms with
provably faster convergence rates than the the Dikin walk while
retaining its advantages over the ball walk and the hit-and-run
methods.


\paragraph{Our contributions:} We introduce and analyze a new random walk, which we refer to as the
\emph{Vaidya walk} since it is based on the
\emph{volumetric-logarithmic barrier} introduced
by~\mbox{\cite{vaidya1989new}.}  We show that for a polytope in
$\realdim$ defined by $\obs$-constraints, the Vaidya walk mixes in
$\order{\obs^{1/2}\dims^{3/2}}$ steps, whereas the Dikin
walk~\citep{kannan2012random} has mixing time bounded as
$\order{\obs\dims}$. So the Vaidya walk is better in the regime $\obs
\gg \dims$.  We also propose the \emph{John walk}, which is based on
the \emph{John ellipsoidal algorithm} in optimization.  We show that
the John walk has a mixing time of $\order{\dims^{2.5} \cdot
  \log^4(\obs/\dims)}$ and conjecture that a variant of it could
achieve $\order{\dims^2 \cdot \polylog(\obs/\dims)}$ mixing time.  We
show that when compared to the Dikin walk, the per-iteration
computational complexities of the Vaidya walk and the John walk are
within a constant factor and a poly-logarithmic in $\obs/\dims$ factor
respectively.  Thus, in the regime $\obs \gg \dims$, the overall upper
bound on the complexity of generating an approximately uniform sample
follows the order \mbox{Dikin walk $\gg$ Vaidya walk $\gg$ John
  walk}.\\

The remainder of the paper is organized as follows.
In Section~\ref{SecBackground}, we discuss many polynomial-time random
walks on convex sets and polytopes, and motivate the starting point
for the new random walks. In
Section~\ref{sec:vaidya_walk_and_convergence}, we introduce the new random walks
and state bounds on their rates of convergence and provide a sketch of the proof in Section~\ref{sub:bound_mixing_time}.
We discuss the computational complexity of the different random walks
and demonstrate the contrast between the random walks for several
illustrative examples in Section~\ref{sec:numerical_experiments}.  We
present the proof of the mixing time for the Vaidya walk in
Section~\ref{sec:proof} and defer the analysis of the John walk to the appendix.
We conclude with possible extensions of our
work in Section~\ref{sec:discussion}.

\paragraph{Notation:} For two sequences $a_\delta$ and $b_\delta$
indexed by $\delta \in I \subseteq \real$, we say that $a_\delta =
\order{b_\delta}$ if there exists a universal constant $C>0$ such that
$a_\delta \leq C b_\delta$ for all $\delta \in I$.  For a set $\Pspace
\subset \realdim$, the sets $\intP$ and $\Pspace^c$ denote the
interior and complement of $\Pspace$ respectively.  We denote the
boundary of the set $\Pspace$ by $\boundary$.  The Euclidean norm of a
vector $x \in \realdim$ is denoted by $\vecnorm{x}{2}$.  For any
square matrix $M$, we use $\det(M)$ and $\trace(M)$ to denote the
determinant and the trace of the matrix $M$ respectively.  For two
distributions $\proposal_1$ and $\proposal_2$ defined on the same
probability space $(\Xspace, \Borel(\Xspace))$, their
total-variation (TV) distance is denoted by $\tvnorm{\proposal_1 -
  \proposal_2}$ and is defined as follows
\begin{align*}
  \tvnorm{\proposal_1 - \proposal_2} = \sup_{A \in \Borel(\Xspace)} \abss{\proposal_1(A) - \proposal_2(A)}.
\end{align*}
Furthermore if $\proposal_1$ is absolutely continuous with respect to
$\proposal_2$, then the Kullback–Leibler divergence from $\proposal_2$ to
$\proposal_1$ is defined as
\begin{align*}
  \kldiv{\proposal_1}{\proposal_2} = \int_\Xspace \log \parenth{\frac{d\proposal_1}{d\proposal_2}}d\proposal_1.
\end{align*}

\section{Background and problem set-up}
\label{SecBackground}

In this section, we describe general MCMC algorithms and review the
rates of convergence of existing random walks on convex sets.  After
introducing several random walks studied in past work, we introduce
the Vaidya and John walks studied in this paper.

\subsection{Markov chains and mixing}

\label{sub:metropolis_hastings_algorithms}

Suppose that we are interested in drawing samples from a \emph{target
  distribution} $\target$ supported on a subset $\myset$ of
$\real^\dims$.  A broad class of methods are based on first
constructing a discrete-time Markov chain that is irreducible and
aperiodic, and whose stationary distribution is equal to $\target$,
and then simulating this Markov chain for a certain number of steps
$k$.  As we describe below, the number of steps $k$ to be taken is
determined by a mixing time analysis.

In this paper, we consider the class of Markov chains that are of the
\emph{Metropolis-Hastings
  type}~\citep{metropolis1953equation,hastings1970monte}; see the
books by~\cite{robert2004monte} and~\cite{brooks2011handbook}, as well
as references therein, for further background.  Any such chain is
specified by an initial density $\myinitial$ over the set $\myset$,
and a \emph{proposal function} $\mytrans: \myset \times \myset \in
\real_+$, where $\mytrans(x, \cdot)$ is a density function for each $x
\in \myset$.  At each time, given a current state $x \in \myset$ of
the chain, the algorithm first proposes a new vector $z \in \myset$ by
sampling from the proposal density $\mytrans(x, \cdot)$.  It then
accepts $z \in \myset$ as the new state of the Markov chain with
probability
\begin{align}
  \label{EqnMHCorrection}
  \alpha(x, z) \defn \displaystyle\min\braces{1,
    \frac{\targetdensity(z)\density(z,
      x)}{\targetdensity(x)\density(x, z)}}.
\end{align}
Otherwise, with probability equal to $1-\alpha(x, z)$, the chain stays
at $x$.  Thus, the overall transition kernel $\density$ for the Markov
chain is defined by the function
\begin{align*}
\transdensity(x, z) & \defn \density(x, z) \alpha(x, z) \qquad \mbox{for
$z \neq x$,}
\end{align*}
and a probability mass at $x$ with weight $1- \int_\myset
\transdensity(x,z) dz$.  It should be noted that the purpose of the
Metropolis-Hastings correction~\eqref{EqnMHCorrection} is that ensure
that the target distribution $\target$ satisfies the \emph{detailed
  balanced condition}, meaning that
\begin{align}
  \label{EqnDetailedBalance}
\transdensity(y, x ) \targetdensity(x) & = \transdensity(x, y)
\targetdensity(y) \quad \mbox{for all $x, y \in \myset$.}
\end{align}
It is straightforward to verify that the detailed balance
condition~\eqref{EqnDetailedBalance} implies that the target density
$\target$ is stationary for the Markov chain.  Throughout this paper, we
analyze the \emph{lazy version} of the Markov chain, defined as
follows: when at state $x$ with probability $1/2$ the walk stays at
$x$ and with probability $1/2$ it makes a transition as per the
original random walk. Given that the Markov chains discussed in this paper
are also irreducible, the laziness ensures uniqueness of the stationary distribution.

Overall, this set-up defines an operator $\plaintransition_\density$
on the space of probability distributions: given an initial
distribution $\initial$ with $\supp(\initial) \subseteq
\supp(\target)$, it generates a new distribution
$\plaintransition_\density(\initial)$, corresponding to the
distribution of the chain at the next step.  Moreover, for any
positive integer $k = 1, 2, \ldots$, the distribution $\mu_k$ of the
chain at time $k$ is given by $\plaintransition_\density^k(\initial)$,
where $\plaintransition_\density^k$ denotes the composition of
$\plaintransition_\density$ with itself $k$ times.  Furthermore, the
transition distribution at any state $x$ is given by
$\plaintransition_\density(\diracdelta_x)$ where $\diracdelta_x$
denotes the dirac-delta distribution with unit mass \mbox{at $x$}.

Given our assumptions and set-up, we are guaranteed that $\lim_{k
  \rightarrow \infty} \plaintransition_\density^k(\initial) =
\target$---that is, if we were to run the chain for an infinite number
of steps, then we would draw a sample from the target distribution
$\target$.  In practice, however, any algorithm will be run only for a
finite number of steps, which suffices to ensure only that the
distribution from which the sample has been drawn is ``close'' to the
target $\target$.  In order to quantify the closeness, for a given
tolerance parameter $\delta \in (0,1)$, we define the
\emph{$\delta$-mixing time} as
\begin{align}
\label{EqnDefnMixingTime}
\tmix(\delta; \initial) & \defn \min \Big \{ k \; \vert \;
\tvnorm{\plaintransition_\density^k(\initial) - \target} \leq \delta \Big \},
  \end{align}
corresponding to the first time that the chain's distribution is
within $\delta$ in TV norm of the target distribution, given that it
starts with distribution $\initial$.

In the analysis of Markov chains, it is convenient to have a rough
measure of the distance between the initial distribution $\initial$
and the stationary distribution.  Warmness is one such measure:
For a finite scalar $M$, the initial distribution $\initial$ is said
to be $M$-warm with respect to the stationary distribution $\target$
if
  \begin{align}
    \label{EqnDefnMstart}
    \tag{Warm-Start}
    \sup_{S} \parenth{\frac{\initial(S)}{\target(S)}}
    \leq M,
  \end{align}
  where the supremum is taken over all measurable sets $S$.
A number of mixing time guarantees from past work~\citep{lovasz1999hit,vempala2005geometric} are stated in terms
of this notion of \mbox{$M$-warmness,} and our results make use of it
as well.  In particular, we provide bounds on the quantity $
\sup \limits_{\initial \in \MWARM} \tmix(\delta ; \initial)$, where $\MWARM$
denotes the set of all distributions that are $M$-warm with respect to
$\target$.  Naturally, as the value of $M$ decreases, the task of
generating samples from the target distribution gets easier.  However,
access to a warm-start may not be feasible for many applications and
thus deriving bounds on mixing time of the Markov chain from a non
warm-start is also desirable.  Consequently, we provide modifications
of our random walks which mix in polynomial time even from
deterministic starting points.

\subsection{Sampling from polytopes}

In this paper, we consider the problem of drawing a sample uniformly
from a polytope.  Given a full-rank matrix $A \in \real^{\obs \times
  \dims}$ with $\obs \geq \dims$, we consider a polytope $\Poly$ in
$\real^\dims$ of the form
\begin{align}
\Poly & \defn \big \{ x \in \real^\dims \mid A x \leq b \big \},
\end{align}
where $b \in \real^\obs$ is a fixed vector. Since the uniform
distribution on the polytope $\Poly$ is the primary target
distribution considered in the paper, in the sequel we use $\target$
exclusively to denote the uniform distribution on the polytope
$\Poly$.  There are various algorithms to sample a vector from the
uniform distribution over $\Poly$, including the ball
walk~\citep{lovasz1990ballwalk} and hit-and-run
algorithms~\citep{lovasz1999hit}.  To be clear, these two algorithms
apply to the more general problem of sampling from a convex set;
Table~\ref{tab:all_rates} shows their complexity, when applied to the
polytope $\Poly$, relative to the Vaidya walk analyzed in this paper.
Most closely related to our paper is the Dikin walk proposed by
\cite{kannan2012random}, and a more general random walk on a
Riemannian manifold studied by \cite{narayanan2016randomized}.  Both
of these random walks, as with the Vaidya and John walks, can be
viewed as randomized versions of the interior point methods used to
solve linear programs, and more generally, convex programs equipped
with suitable barrier functions.

In order to motivate the form of the Vaidya and John walks proposed in
this paper, we begin by discussing the ball walk and then the Dikin
walk. For the sake of completeness, we end the section with a brief description
another popular sampling algorithm Hit-and-run.

\paragraph{Ball walk:} The ball walk of \cite{lovasz1990ballwalk} is
simple to describe: when at a point $x \in \Poly$, it draws a new
point $u$ from a Euclidean ball of radius $r > 0$ centered at $x$.
Here the radius $r$ is a step size parameter in the algorithm.  If the
proposed point $u$ belongs to the polytope $\Poly$, then the walk
moves to $u$; otherwise, the walk stays at $x$.  On the one hand,
unlike the walks analyzed in this paper, the ball walk applies to any
convex set, but on the other, its mixing time depends on the condition
number $\condition$ of the set $\Pspace$, given by
\begin{align}
\label{eq:condition}
\condition = \inf_{R_\text{in}, R_\text{out} > 0} \big \{ \frac {R_\text{out}}{R_\text{in}} \quad  \vert
\quad \Ball(x, R_\text{in}) \subseteq \Pspace \subseteq \Ball(y, R_\text{out})
\quad \mbox{for some $x, y \in \Pspace$} \Big \}.
\end{align}
Mixing time of the ball walk has been improved greatly since it was introduced~\citep{kannan1997random,kannan2006blocking,lee2018stochastic}.
Nonetheless, as shown in Table~\ref{tab:all_rates}, the mixing time of
the ball walk gets slower when the condition of the set is large;
for instance, it scales\footnote{Although, very recently~\cite{lee2018stochastic}
improved the mixing time of the ball walk for isotropic sets which have
$\condition = \mathcal{O}({\sqrt{d}})$ improved from $\order{\dims^3}$ to 
$\order{\dims^{2.5}}$.} as $\dims^6$ for a set with condition number
$\condition= \dims^2$.
One approach to tackle bad conditioning is to use rounding as a pre-processing
step, where the set is rounded to bring it in a near-isotropic position,
i.e., reduce the condition~$\condition$ to near-constant before sampling
from it.
Nonetheless, these algorithms are themselves based on several rounds of
sampling algorithms and the current best
algorithm by~\cite{lovasz2006simulated} puts a convex body into approximately
isotropic position, i.e., $\mathcal{O}^*({\sqrt{\dims}})$ rounding with
a running time of $\mathcal{O}^*(\dims^4)$ where we have omitted the dependence
on log-factors. If one has more information about the structure of the convex
set (and not just oracle access as required by the ball walk), one can
potentially exploit it to design fast sampling
algorithms which are unaffected by the conditioning of the set thereby reducing
the need of the (expensive) pre-processing step. One such algorithm
is the Dikin walk for polytopes which we describe next.

\paragraph{Dikin walk:}
The Dikin walk~\citep{kannan2012random} is similar in spirit to the ball
walk, except
that it proposes a point drawn uniformly from a \emph{state-dependent}
ellipsoid known as the Dikin
ellipsoid~\citep{dikin1967iterative,nesterov1994interior}.
It then applies an accept-reject step to adjust for the difference in the
volumes of these ellipsoids at different states.  The state-dependent
choice of the ellipsoid allows the Dikin walk to adapt to the boundary
structure.  A key property of the Dikin ellipsoid of unit radius---in
contrast to the Euclidean ball that underlies the ball walk---is that
it is always contained within $\Poly$, as is known from classic
results on interior point methods~\citep{nesterov1994interior}.
Furthermore, the Dikin walk is affine invariant, meaning that its
behavior does not change under linear transformations of the problem.
As a consequence, the Dikin mixing time does not depend on the
condition number $\condition$.  In a variant of this random
walk~\citep{narayanan2016randomized}, uniform proposals in the
ellipsoid are replaced by Gaussian proposals with covariance specified
by the ellipsoid, and it is shown that with high probability, the
proposal falls within the polytope.

The Dikin walk is closely related to the interior point methods for
solving linear programs.  In order to understand the Vaidya and John
walks, it is useful to understand this connection in more detail.
Suppose that our goal is to optimize a convex function over the
polytope $\Poly$.  A barrier method is based on converting this
constrained optimization problem to a sequence of unconstrained ones,
in particular by using a barrier to enforce the linear constraints
defining the polytope.  Letting $a_i^\top$ denote the $i$-th row
vector of matrix $A$, the \emph{logarithmic-barrier} for the polytope
$\Pspace$ given by the function
\begin{align}
    \label{eq:log_barrier}
   \logbarr(x) & \defn -\sum_{i=1}^\obs \log(b_i-a_i^T x).
\end{align}
For each $i \in [\obs]$, we define the scalar $\slack_{x, i} \defn
(b_i-a_i^T x)$, and we refer to the vector \mbox{$\slack_{x} \defn
  (\slack_{x, 1}, \ldots, \slack_{x, \obs})\tp$} as the
\emph{slackness at $x$}.

Each step of an interior point algorithm~\citep{boyd2004convex}
involves (approximately) solving a linear system involving the Hessian
of the barrier function, which is given by
\begin{align}
  \label{eq:hess_log}
  \hesslogbarr(x) & \defn \sum_{i=1}^\obs \frac{a_i
    a_i^\top}{\slack_{x, i}^2}.
\end{align}
In the Dikin walk~\citep{kannan2012random}, given a current iterate
$x$, the algorithm chooses a point uniformly at random from the
ellipsoid
\begin{align}
  \label{eq:dikin_covariance}
  \{u \in \real^\dims \, \mid \, (u-x)\tp \dikin_x  (u-x) \leq R \},
\end{align}
where $\dikin_x \defn \hesslogbarr(x)$ is the Hessian of the log
barrier function, and $R > 0$ is a user-defined radius.  In an
alternative form of the Dikin walk~\citep{narayanan2016randomized,sachdeva2016mixing}, the
proposal vector $u \in \real^\dims$ is drawn randomly from a Gaussian
centered at $x$, and with covariance equal to a scaled copy of
$(\dikin_x)^{-1}$.  Note that in contrast to the ball walk, the
proposal distribution now depends on the current state.


\paragraph{Vaidya walk:} For the \emph{Vaidya walk} analyzed in this paper, we instead generate proposals from the ellipsoids defined, for each $x \in \intP$, by the
positive definite matrix
\begin{subequations}
\begin{align}
    \label{eq:vaidya_covariance}
  \vaidya_x & \defn \sum_{i=1}^\obs \parenth{\levvaidya_{x, i} +
    \vaidyabeta} \frac{a_i a_i^\top}{\slack_{x, i}^2}, \qquad \mbox{
    where } \\
    \label{eq:defn_lev_scores}
    \vaidyabeta & \defn {\dims}/{\obs} \quad \mbox{and} \quad
    \levvaidya_{x} \defn \parenth{ \frac{a_1\tp
        (\hesslogbarr_x)^{-1}a_1}{\slack_{x, 1}^2}, \ldots,
      \frac{a_\obs\tp (\hesslogbarr_x)^{-1} a_\obs}{\slack_{x,
          \obs}^2} }\tp.
\end{align}
\end{subequations}
The entries of the the vector $\sigma_x$ are known as the leverage
scores assciated with the matrix $\hesslogbarr_x$ from
equation~\eqref{eq:hess_log}, and are commonly used to measure the
importance of rows in a linear system~\citep{mahoney2011randomized}.
The matrix $\vaidya_x$ is related to the Hessian of the function $x
\mapsto \vaidyabarr_x$ given by
\begin{align}
\label{eq:vaidyabarrier}
  \vaidyabarr_x \defn \log\det{\hesslogbarr_x}+ \vaidyabeta\logbarr_x.
\end{align}
This particular combination of the \emph{volumetric barrier} and the
\emph{logarithmic barrier} was introduced by~\cite{vaidya1989new}
and~\cite{vaidya1993technique} in the context of interior point
methods, hence our name for the resulting random walk.


\paragraph{John walk:} We now describe the John walk.
For any vector $w \in \real^\obs$, let $W \defn \diag(w)$ denote the
diagonal matrix with $W_{ii} = w_i$ for each $i \in [\obs]$.  Let $S_x
= \diag(s_x)$ denote the slackness matrix at $x$.  It is easy to see
that $\slackmatrix_x$ is positive semidefinite for all $x
\in \Pspace$, and strictly positive definite for all $x \in \intP$.
The (scaled) inverse covariance matrix underlying the John walk is
given by
\begin{align}
    \label{eq:john_covariance}
\john_x \defn \sum_{i=1}^\obs \johnweights_{x, i}
\frac{a_ia_i\tp}{\slack_{x, i}^2},
\end{align}
where for each $x \in \intP$, the weight vector $\johnweights_x \in
\real^\obs$ is obtained by solving the convex program
\begin{align}
  \johnweights_x \defn \arg \min_{\weights\in\real^\obs} \Biggr \{
  \sum_{i=1}^\obs \weights_i - \frac{1}{\johnalpha}\log\det (A\tp
  \slackmatrix_x^{-1}\weightmatrix^{\johnalpha} \slackmatrix_x^{-1} A)
  - \johnbeta \sum_{i=1}^\obs \log \weights_i \Biggr \},
  \label{eq:john_weights}
\end{align}
with $\johnbeta \defn \dims/2\obs$ and $\johnalpha \defn
1-1/\log_2(1/\johnbeta)$.
\cite{lee2014path} proposed the
convex program~\eqref{eq:john_weights} associated with the \emph{approximate
John weights} $\johnweights_x$, with the aim of searching for the best
member of  a family of volumetric barrier functions. They analyzed
the use of the John weights in the context of speeding up interior
point methods for solving linear programs; here we consider them for
improving the mixing time of a sampling algorithm.
The convex program~\eqref{eq:john_weights} is closely related to the problem
of finding the largest ellipsoid at any interior point of the polytope,
such that the ellipsoid is contained within the polytope. This problem of
finding the largest ellipsoid was first studied by~\cite{joh48}
who showed that each convex body in $\realdim$ contains a unique ellipsoid of maximal volume.
The convex program~\eqref{eq:john_weights} was used by ~\cite{lee2014path}
to compute approximate John Ellipsoids for solving linear programs.
In a recent work, \cite{gustafson2018john} make use of the exact John
ellipsoids and design a polynomial time sampling algorithm for polytopes.
See Table~\ref{tab:all_rates} for the associated guarantees.


\paragraph{Hit-and-run:} We conclude with a brief discussion with another
popular sampling algorithm: Hit-and-run. It was introduced by~\cite{smith1984efficient}
as a sampling algorithm for general distributions and it was later
shown to have polynomial mixing time for sampling from convex
sets~\citep{lovasz1999hit,lovasz2003hit,lovasz2006hit}. The
algorithm proceeds as follows: when at point $x$, it firsts draws a random
line through $x$ and then samples from the one-dimensional marginal
of the target distribution restricted to this line.
For uniform sampling from convex sets, the second step simplifies to
drawing a uniform point from the line restricted to the convex
set. Mixing time bounds for this random walk are summarized in
Table~\ref{tab:all_rates}.

\subsection{Mixing time comparisons of walks} 
\label{sub:mixing_time_comparisons_of_walks}

Table~\ref{tab:all_rates} provides a summary of the mixing time bounds and
per step complexity and the effective per sample complexity for various
random walks, including the Vaidya and John
walks analyzed in this paper.  In addition to the Ball Walk,
Hit-and-Run, Dikin, Vaidya and John walks, we also show scalings for
the recently introduced Riemannian Hamiltonian Monte Carlo (RHMC) on polytopes
by~\cite{lee2016geodesic} and the John's walk based on exact John ellipsoids
studied by~\cite{gustafson2018john}. The details of per iteration cost for
the new random walks is discussed in Section~\ref{sub:per_iteration_cost}.
We now compare and contrast the complexities of these random walks.

Unlike the Ball Walk or hit-and-run which are useful for general convex
sets, the Dikin, Vaidya, John and RHMC
walks are specialized for polytopes. These latter random walks exploit the
definition of the polytope in a particular way so that the transition
probability from a point
$x$ to $y$ does not change under an affine transformation,
i.e., $\transition(x, y) = \transition(Ax, Ay)$ where $\transition$
denotes the transition kernel for the random walk. Consequently, the
mixing time bounds for these random walks have no dependence on the
condition number of the set $\condition$~\eqref{eq:condition}. We can see from
Table~\ref{tab:all_rates}, that compared to the Ball walk and hit-and-run,
Vaidya walk mixes significantly faster if $\obs \ll \dims\condition^2$.
The condition number $\condition$ of polytopes with polynomially many
faces can not be $\mathcal{O}({\dims^{\frac{1}{2}-\epsilon}})$ for any
$\epsilon > 0$ but can be arbitrarily larger, even exponential in
dimension $\dims$~\citep{kannan2012random}.  For such polytopes, Vaidya
walk mixes faster as long as $\obs \ll \dims^3$ (and even for larger
$\obs$ when $\condition$ is large).  It takes
$\mathcal{O}({\sqrt{\obs/\dims}})$ fewer steps compared to Dikin walk and
thus provides a practical speed up over all range of $\dims$.

From a warm start, the Riemannian Hamiltonian Monte Carlo on polytopes introduced by~\cite{lee2016geodesic} has $\order{\obs \dims^{2/3}}$ mixing
time, and thus mixes faster (up to constants) compared than the Vaidya
walk (respectively the John walk) when the number of constraints
$\obs$ is is bounded as $\obs \ll \dims^{5/3}$ (respectively $\obs \ll
\dims^{11/6}$).  For larger numbers of constraints, the Vaidya and John
walks exhibit faster mixing.  More generally, it is clear that the
rate of John walk has \emph{almost} the best order across all the
walks for reasonably large values of $\obs \gg \dims^2$.

Finally, let us compare the (exact) John walk due
to~\cite{gustafson2018john} with the (approximate) John walk studied
in our paper. A notable feature of their random walk is that its
mixing time is independent of the number of constraints and the per
iteration cost also depends linearly on the number of constraints.
Nonetheless, the dependence on $\dims$, for both the mixing time
($\dims^7$) and the per iteration cost ($\obs\dims^4+\dims^8$) is
quite poor.  In contrast, the per iteration cost for our John walk is
$\obs\dims^2$ and the mixing time has only a poly-logarithmic
dependence on $\obs$.



\begin{table}[h]
  \centering
\resizebox{0.7\textwidth}{!}{
  \begin{tabular}{llll}
    \hline \toprule {\bf Random walk} & {$\bf \tmix(\delta;
      \initial)$} & {\bf Iteration cost} & {\bf Per sample cost}
      \\ \midrule Ball walk$^\#$~\footnotesize{\citep{kannan2006blocking}}
    & ${\dims^2\condition^2
    }$& $\obs \dims$ & $\obs \dims^3\condition^2$ \vspace{2ex} \\
    Hit-and-Run~\footnotesize\citep{lovasz2006hit} &${\dims^2\condition^2}$& ${\obs
    \dims}$  & $\obs \dims^3\condition^2$
             \vspace{2ex}\\
    Dikin walk~\footnotesize\citep{kannan2012random} &
    ${\obs \dims}$& ${\obs
          \dims^{2}}$  & $\obs^2 \dims^3$\vspace{2ex}\\
    RHMC walk~\footnotesize\citep{lee2018convergence} &${\obs \dims^{2/3}}$
    & ${\obs \dims^{2}}$ & $\obs^2 \dims^{2.67}$
    \vspace{2ex}\\
    John's walk$^\dagger$~\footnotesize\citep{gustafson2018john} &${
    \dims^7}$
    & ${\obs \dims^{4}+ \dims^8}$ & $\obs\dims^{11}+ \dims^{15}$
    \vspace{2ex}\\
    Vaidya walk~\footnotesize{(this paper)}&${\obs^{1/2} \dims^{3/2}}$&
    ${\obs \dims^2}$ & $\obs^{1.5} \dims^{3.5}$
             \vspace{2ex}\\

    John walk~\footnotesize{(this paper)}&${\dims^{5/2} \,\log^4
    \parenth{\frac{2\obs}{\dims}}}$& ${\obs \dims^{2} \log^2 \obs}$
    & $\obs \dims^{4.5}$
            \vspace{2ex}\\

    Improved John walk$^\ddagger$~\footnotesize{(this paper)}&${\dims^2 \,\polylogfactor}$& ${\obs \dims^
    {2}\log^2\obs}$& $\obs \dims^{4}$\vspace{2ex}\\
    \bottomrule
    \hline
  \end{tabular}
  }
  \vspace{2ex}
  \caption{Upper bounds on computational complexity of random walks on
    the polytope $\Pspace = \{x \in \realdim \vert Ax \leq b\}$
    defined by the matrix-vector pair $(A, b) \in \real^{\obs \times \dims}
    \times \real^\obs$ with a warm-start. For simplicity, here we ignore the logarithmic dependence on the
    warmness parameter and the tolerance $\delta$. The iteration
    cost terms of order $\obs\dims^2$ arise from linear system
    solving, using standard and numerically stable algorithms, for
    $\obs$ equations in $\dims$ dimensions; algorithms with best
    possible theoretical complexity  $\obs\dims^{\omega}$ for $\omega<1.373$
    are not numerically stable enough for practical use. $^\#$Mixing time of the Ball walk has been improved to $\order{\dims^2\condition}$ for near isotropic convex bodies by~\cite{lee2018stochastic} during the submission period of this paper. While ball walk, Hit-and-run
    are affected by the condition number~$\condition$ of the set, the Dikin
    and RHMC walks have quadratic dependence on the number of constraints
    $\obs$.
    $^\dagger$John's walk by~\cite{gustafson2018john} (based on
    the exact John ellipsoids) has linear dependence on $\obs$ but poor
    dependence on $\dims$. In contrast, the Vaidya walk has sub-quadratic
    dependence on $\obs$ and significantly better dependence on $\dims$.
    Furthermore, the John walk (based on approximate John's ellipsoids)
    analyzed in this paper has linear dependence with reasonable dependence
    on the dimensions $\dims$.
	$^\ddagger$The mixing time bound for the improved John walk with
	poly-logarithmic factor $\polylogfactor$ is conjectured.}
  \label{tab:all_rates}
\end{table}


\subsection{Visualization of three walks' proposal distributions}
\label{sub:visualization_of_three_walks}

In order to gain intuition about the three interior point based
methods---namely, the Dikin, Vaidya and John walks---it is helpful to
discuss how their underlying proposal distributions change as a
function of the current point $x$. All three walks are based on
Gaussian proposal distributions with inverse covariance matrices of
the general form
\begin{align*}
\sum_{i=1}^\obs \weights_{x, i} \frac{a_i a_i \tp}{s_{x, i}^2},
\end{align*}
where $\weights_{x, i} > 0$ corresponds to a state-dependent weight
associated with the $i$-th constraint. The Dikin walk uses the
weights $\weights_{x, i} = 1$; the Vaidya walk uses the weights
\mbox{$\weights_{x, i} = \levscores_{x, i} + \vaidyabeta$;} and the
John walk uses the weights $\weights_{x, i} = \johnweights_{x, i}$.
For simplicity, we refer to these weights as the Dikin, Vaidya and
John weights. The $i$-th weight characterize the importance of the
$i$-th linear constraint in constructing the inverse covariance
matrix. A larger value of the weight $\weights_{x, i}$ relative to
the total weight $\sum_{i=1}^\obs \weights_{x, i}$ signifies more
importance for the $i$-th linear constraint for the point $x$.

Figure~\ref{fig:weights_move} illustrates the difference in three
weights as we move points inside the polytope $[-1, 1]^2$.  When
the point $x$ is in the middle of the unit square formed by the four
constraints, all walks exhibit equal weight for every constraint.
When the point $x$ is closer to the bottom-left boundary, the Vaidya
and John weights assign larger weights to the bottom and the left constraints, while the weights for top and right constraints
decrease. Note that the total sum of Vaidya weights and that of John weights
remains constant independent of the position of the point $x$.
\begin{figure}[h]
  \centering
  \begin{subfigure}{0.49\linewidth}
    \centering
    \includegraphics[width=1\linewidth]{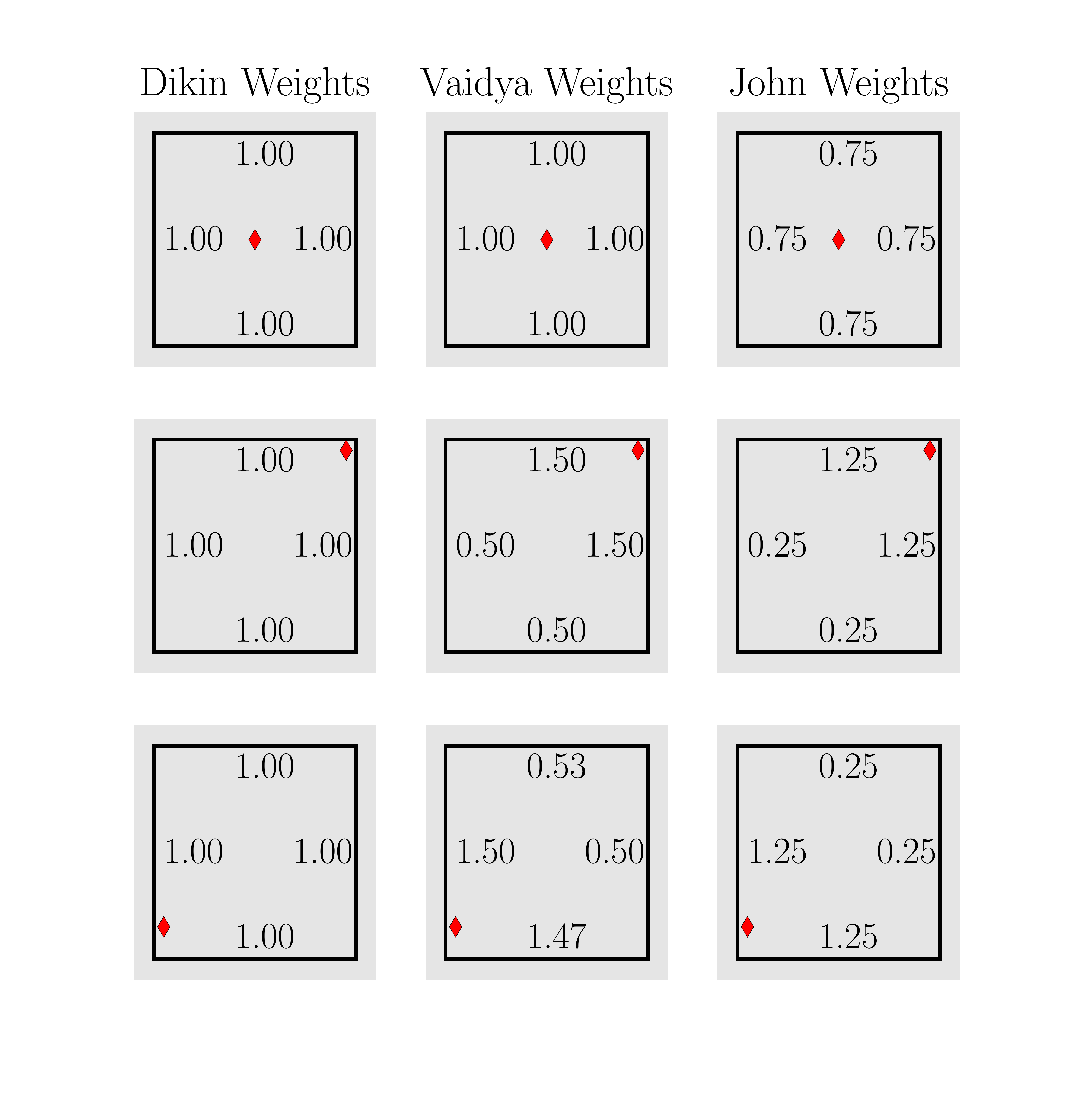}
    \vspace{-2.5\baselineskip}
    \caption{\centering{Weights for different locations and a fixed number of constraints $\obs$.}}
    \label{fig:weights_move}
  \end{subfigure}
  \begin{subfigure}{0.49\linewidth}
    \centering
    \includegraphics[width=1\linewidth]{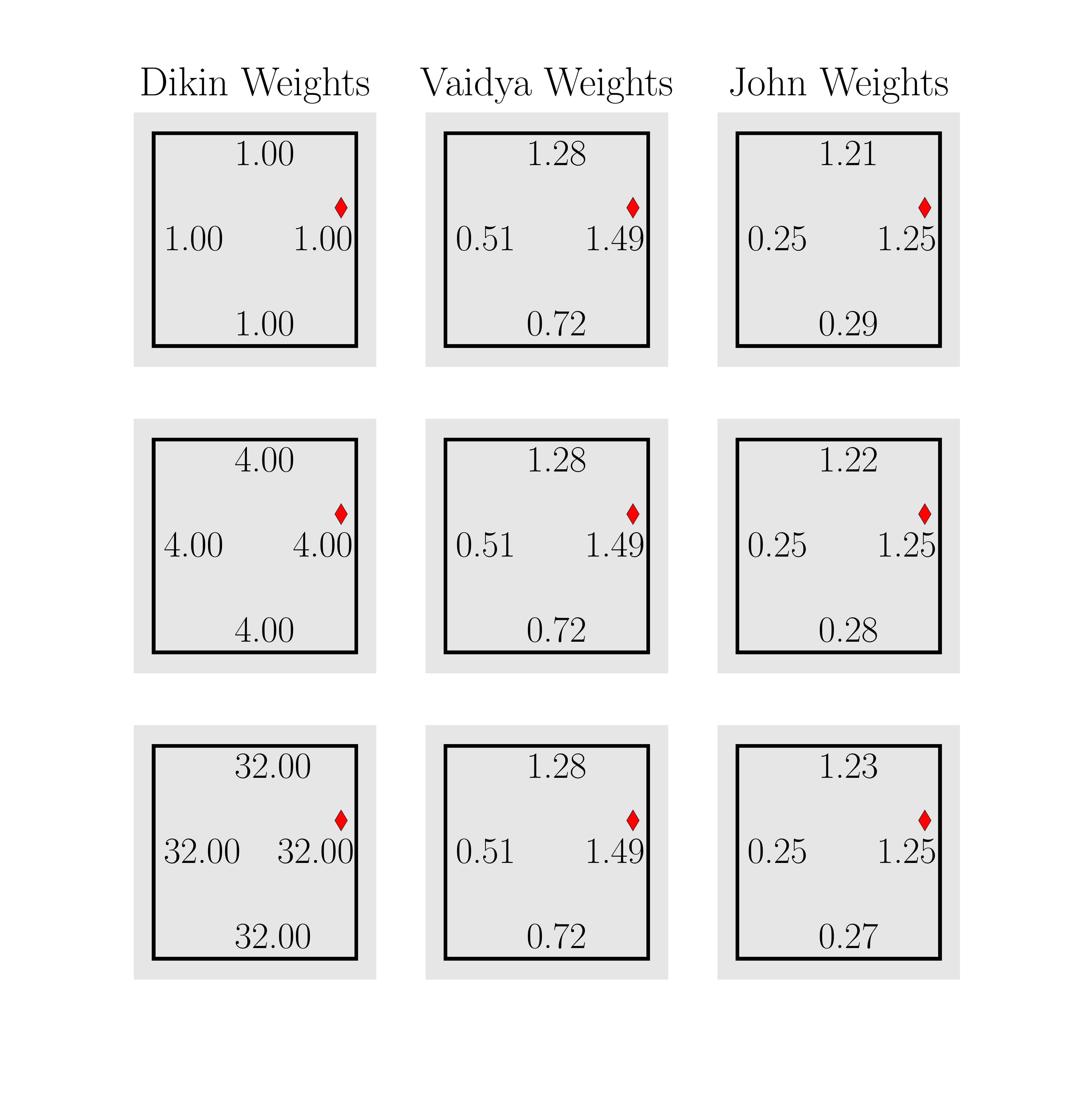}
    \vspace{-2.5\baselineskip}
    \caption{
        \centering{Effective weights for a fixed location and different number of constraints $\obs$}
      }
    \label{fig:weights_repeat}
  \end{subfigure}
\caption{Visualization of the weights on the square with repeated constraints $\repsquare_{\obs/4}$ for the different random walks. The number mentioned next to the boundary lines denotes the effective weight for the location $x$ (denoted by diamond) for the corresponding constraint.
\textbf{(a)}
$\obs = 4$ is common across rows and $x = (0, 0)$ for the top row, $(0.9, 0.9)$ for the middle and $(-0.9, -0.7)$ for the bottom row.
The Dikin weights are independent of $x$, the Vaidya and the John weights for a constraint increase if the location $x$ is closer to it.
  \textbf{(b)}
$x= (0.85, 0.30)$ is common across rows, and $\obs\!=\!4$ for the top row, $\obs = 16$ for the middle and $\obs\!=\!128$ for the bottom row.
The effective Dikin weight for each constraint increases linearly with
$\obs$ but for the Vaidya and John walk adaptively, the weights get adjusted
such that the sum of their weights is always of the order of the dimension $\dims$.
}
\label{fig:weights_new}
\end{figure}

In Figure~\ref{fig:weights_repeat}-\ref{fig:ellipsoid_16384}, we
demonstrate that the Vaidya walk and the John walk are better at
handling repeated constraints.  Note that we can define the square
$[-1, 1]^2$ as
\begin{align}
\label{eq:unit_square}
  [-1, 1]^2 = \braces{x \in \mathbb{R}^2 \Bigg\vert Ax \leq b, A
    = \begin{bmatrix} 1 & 0 \\ 0 & 1 \\ -1 & 0 \\ 0 & -1
  \end{bmatrix},
   b =
   \begin{bmatrix}
     1 \\
     1
   \end{bmatrix}
   }.
\end{align}
Simply repeating the rows of the matrix $A$ several times changes the mathematical formulatiton of the polytope, but does not change the shape of the polytope. We define the square with constraints repeated $\obs/4$ times $\repsquare_{\obs/4}$ as
\begin{align}
  \label{eq:repeated_square}
  \repsquare_{\obs/4} = \braces{x \in \mathbb{R}^2 \Bigg\vert A_{\obs/4} x \leq b_{\obs/4}, A_{\obs/4} = \begin{bmatrix}
    A \\ \vdots \\ \scriptstyle \times \parenth{\obs/4}
  \end{bmatrix},
  b_{\obs/4} = \begin{bmatrix}
    b \\ \vdots \\ \scriptstyle \times \parenth{\obs/4}
  \end{bmatrix},
   }
\end{align}
where $A$ and $b$ were defined above.
We denote effective weight for each distinct constraint as the sum of weights corresponding to the same constraint. Using this definition, the effective Dikin weight, which is $\obs/4$, is thus affected by the repeating of constraints. Consequently, the Dikin ellipsoid is much smaller for polytopes with repeated constraints. However, the Vaidya and John weights do not change as observed in the Figure~\ref{fig:weights_repeat}.  Such a
property of these two weights implies that the Vaidya and John ellipsoids
are not too small even for very large number of constraints.  And we
observe such a phenomenon in
Figures~\ref{fig:ellipsoid_64}-\ref{fig:ellipsoid_16384} where the
repetition of rows in the matrix $A$ leads to very small Dikin
ellipsoid but large Vaidya and John ellipsoid.
A few other numerical computations also suggest that the Vaidya and John
ellipsoids are moder adaptive when compared to Dikin ellipsoids when the
number of constraints is large.
Nonetheless, such a claim is only based on heuristics and is presented simply
to provide an intuition that the new ellipsoids are better behaved than
Dikin ellipsoids and thereby motivated the design of the new random walks.

\begin{figure}[h]
  \centering
  \begin{subfigure}{0.48\linewidth}
    \centering
    \includegraphics[width=1\linewidth]{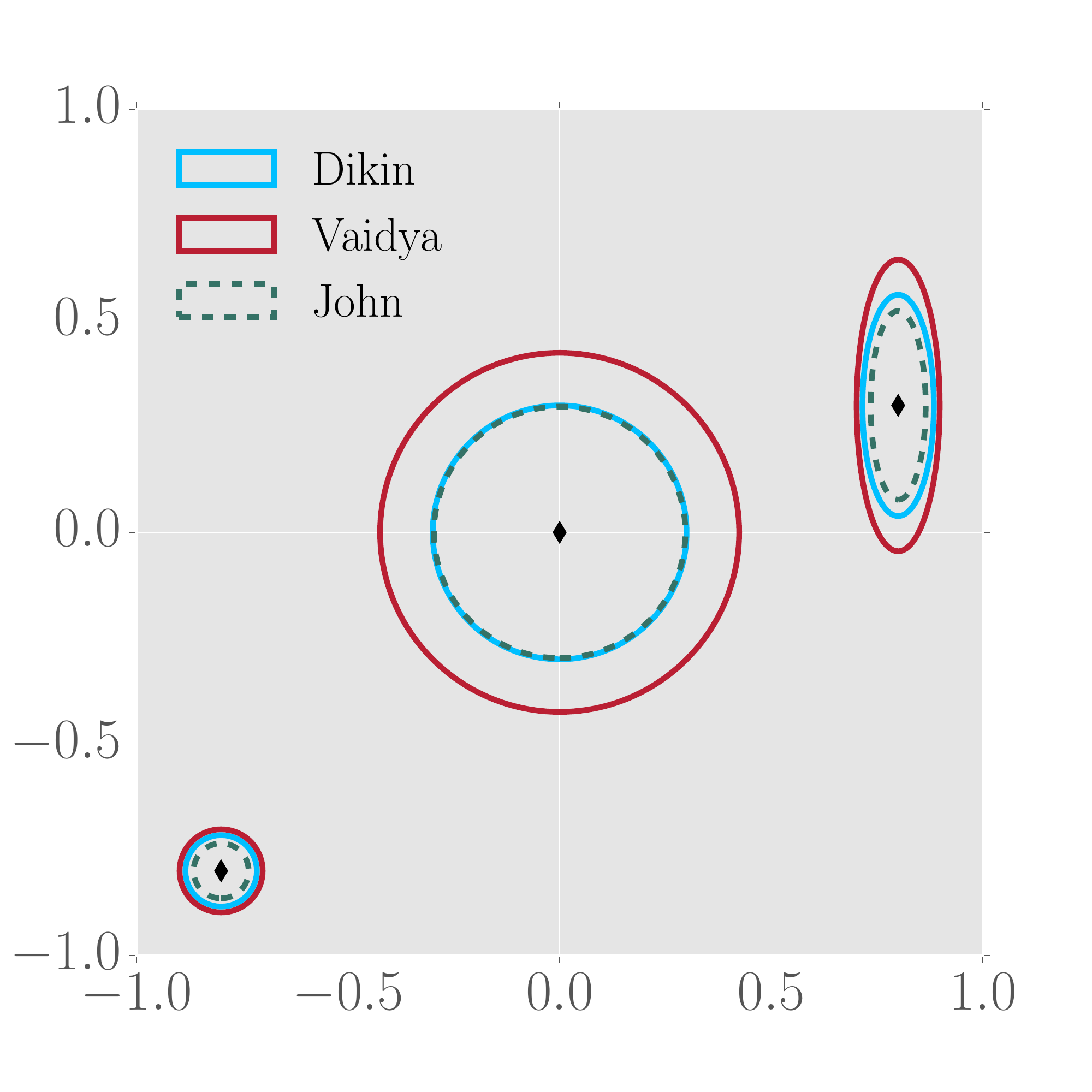}
    \vspace{-1.5\baselineskip}
    \caption{$\obs = 32$}
    \label{fig:ellipsoid_64}
  \end{subfigure}
  \begin{subfigure}[h]{0.48\linewidth}
    \centering
    \includegraphics[width=1\linewidth]{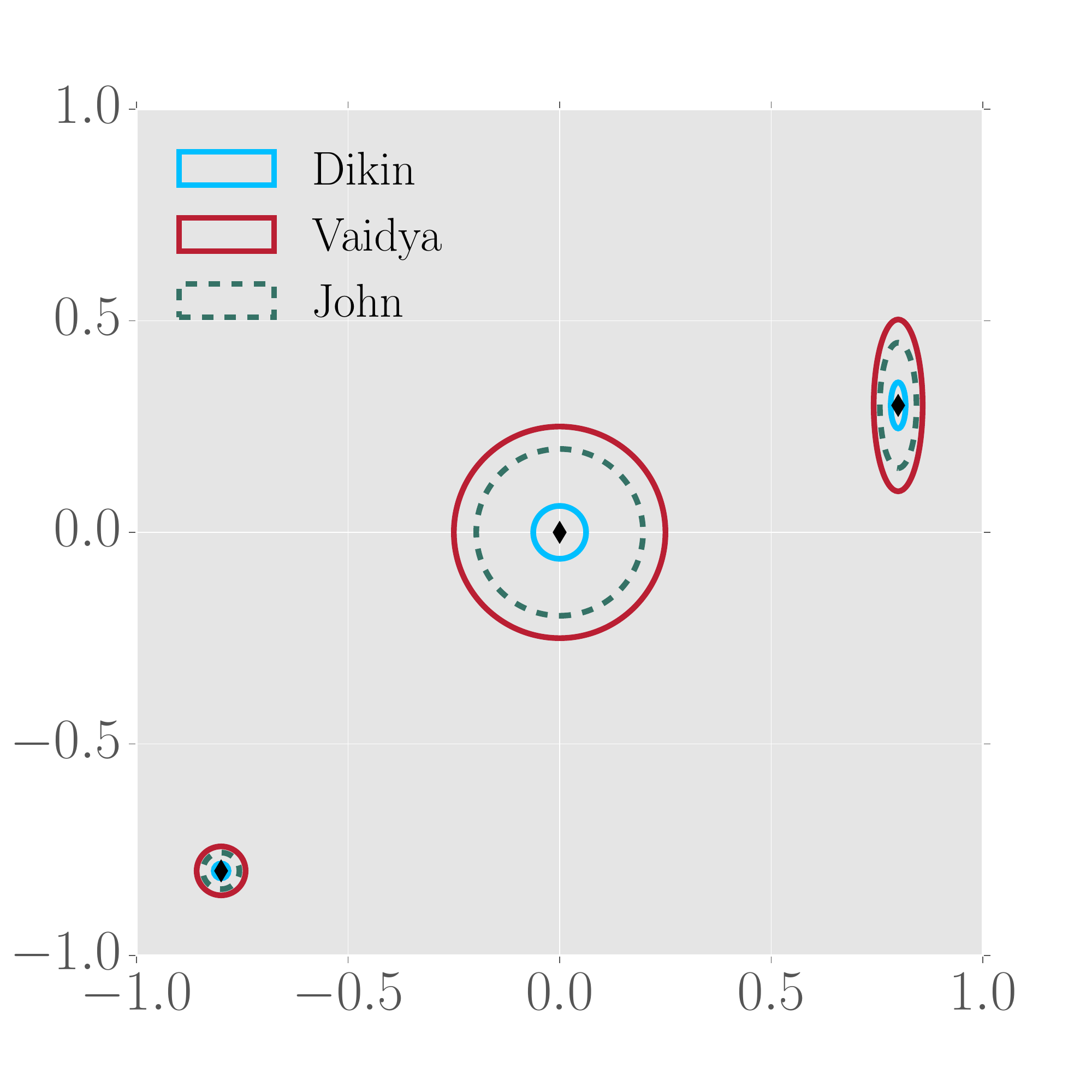}
    \vspace{-1.5\baselineskip}
    \caption{$\obs = 2048$}
    \label{fig:ellipsoid_16384}
  \end{subfigure}\vspace{3mm}
  \caption{Visualization of the proposal distribution on the square with repeated constraints $\repsquare_{\obs/4}$ for the different random walks.   \textbf{(a, b)} Unit ellipsoids associated with the  covariances of the random walks at different states $x$ on the square with repeated constraints $\repsquare_{\obs/4}$.  Clearly, all these ellipsoids adapt to the boundary
    but increasing $\obs$ has a profound impact on the volume of the
    Dikin ellipsoids and comparatively less impact on the Vaidya and
    John ellipsoids.}
  \label{fig:weights}
\end{figure}


\section{Main results}
\label{sec:vaidya_walk_and_convergence}

With the basic background in place, we now describe the algorithms
more precisely and state upper bounds on the mixing time of the Vaidya
and John walks.  In Section~\ref{sub:improved_john_walk}, we propose a
variant of the John walk, known as the \emph{improved John walk}, and
conjecture that it has a better mixing time bound than that of the
John walk.


\subsection{Vaidya and John walks}
\label{sub:vaidya_walk}

In this subsection, we formally define the Vaidya and John walks. In Algorithm~\ref{algo:vaidya_Walk} and Algorithm~\ref{algo:john_Walk}, we summarize the steps of the Vaidya walk and the John walk.


\paragraph{Vaidya walk:}

The Vaidya walk with radius parameter $r > 0$, denoted by $\VW{r}$ for
short, is defined by a Gaussian proposal distribution denoted as $\proposal_x^\tagvaidya$: given a current
state $x \in \intP$, it proposes a new point by sampling from the
multivariate Gaussian distribution
$\gaussian{x}{\frac{r^2}{\sqrt{\obs\dims}} {\vaidya_x}^{-1} }$.  In
analytic terms, the proposal density at $x$ is given by
\begin{align}
\label{eq:proposal_density}
  \density^\tagvaidya_x(z) \defn
  \density_{\fulltagvaidya(\rparam)}(x, z) =
  \sqrt{\det{\vaidya_x}} \parenth{\frac{\obs\dims}{2\pi
      r^2}}^{\dims/2} \exp\parenth{-\frac{\sqrt{\obs\dims}}{2r^2}
    \ (z-x)^\top \vaidya_x (z-x) }.
  \end{align}
As the target distribution for our walk is the uniform distribution on
$\Pspace$, the proposal step is followed by an accept-reject step as described in Section~\ref{sub:metropolis_hastings_algorithms} (equation~\ref{EqnMHCorrection}).
Thus the overall transition distribution for the walk at state $x$ is defined by a density given by
\begin{align*}
q_{\fulltagvaidya(\rparam)}(x, z) = \begin{cases} \min
  \braces{\density^\tagvaidya_x(z), \density^\tagvaidya_z(x)}, &z
  \in \Pspace \text{ and } z \neq x,\\ 0, &z \notin \Pspace,
  \end{cases}
\end{align*}
and a probability mass at $x$, given by $1 -
  \int\displaylimits_{z \in \Pspace} \min \braces{\density_x(z),
    \density_z(x)} dz$. We use $\lazytrans_{\fulltagvaidya(\rparam)}$ to denote the resulting transition operator for the Vaidya walk with parameter $\rparam$.

\begin{algorithm}[h]
  \KwIn{Parameter $r$ and $x_0 \in \intP$}
  \KwOut{Sequence $x_1, x_2, \ldots$}
  \BlankLine
  \For{$i=0, 1, \ldots $}{
    With probability $\frac{1}{2}$ stay at the current state: $x_{i+1}
    \gets x_i$ \quad \% {\footnotesize{\emph{lazy step}}} \\
    With probability $\frac{1}{2}$ perform the following update:\\
     {%
      \quad \textbf{Proposal step}:
      Draw $z_{i+1} \sim
      \mathcal{N}\left(x_{i}, \frac{r^2}{(\numobs\usedim)^
      {1/2}}\vaidya_{x_{i}}^{-1}\right)$ \\
      \quad \textbf{Accept-reject step}:\\
      \quad \quad \emph{if} $z_{i+1} \notin \Pspace$
      \emph{then} {$ x_{i+1} \gets
        x_i $ } \quad \% {\footnotesize{\emph{reject an infeasible
      proposal}}} \\
      \quad \quad \emph{else}\\
      \quad\quad\quad compute
       $\alpha_{i+1} = \displaystyle\min \braces{1,
          {\density_{z_{i+1}}(x_{i+1})}/{\density_{x_{i+1}}(z_{i+1})}
        }$\\
        \quad \quad \quad With probability $\alpha_{i+1}$ accept the proposal:
        $x_{i+1} \gets
        z_{i+1} $\\
        \quad \quad \quad With probability $1-\alpha_{i+1}$ reject the proposal:
        $x_{i+1} \gets
        x_i$}
    }
  \caption{Vaidya Walk with parameter $r$ (\VW{r})}
  \label{algo:vaidya_Walk}
\end{algorithm}


\paragraph{John walk:}

The John walk is similar to the Vaidya walk except that the proposals
at state $x \in \intP$ are generated from the multivariate Gaussian
distribution $\gaussian{x}{\frac{\rparam^2}{\dims^{3/2} \cdot
    \log_2^4(2\obs/\dims)} {\john_x}^{-1} }$, where the matrix
$\john_x$ is defined by equation~\eqref{eq:john_covariance}, and $\rparam >
0$ is a constant. The proposal distribution at $x \in \intP$ is denoted as $\proposal_x^\tagjohn$. The proposal step is then followed by an
accept-reject step similarly defined as in the Vaidya walk.
We use $\lazytrans_{\fulltagjohn(\rparam)}$ to denote the resulting transition operator for the John walk with parameter $\rparam$.

\begin{algorithm}[h]
  \KwIn{Parameter $r$ and $x_0 \in \intP$}
  \KwOut{Sequence $x_1, x_2, \ldots$}
  \BlankLine
  \For{$i=0, 1, \ldots $}{
    With probability $\frac{1}{2}$ stay at the current state: $x_{i+1}
    \gets x_i$ \quad \% {\footnotesize{\emph{lazy step}}} \\
    With probability $\frac{1}{2}$ perform the following update:\\
     {%
      \quad \textbf{Proposal step}:
      Draw $z_{i+1} \sim
      \mathcal{N}\left(x_{i}, \frac{r^2}{\usedim^
      {3/2}}\john_{x_{i}}^{-1}\right)$ \quad \% {\footnotesize{\emph{this
      step is different than the Vaidya walk}}} \\
      \quad \textbf{Accept-reject step}:\\
      \quad \quad \emph{if} $z_{i+1} \notin \Pspace$
      \emph{then} {$ x_{i+1} \gets
        x_i $ } \quad \% {\footnotesize{\emph{reject an infeasible
      proposal}}} \\
      \quad \quad \emph{else}\\
      \quad\quad\quad compute
       $\alpha_{i+1} = \displaystyle\min \braces{1,
          {\density_{z_{i+1}}(x_{i+1})}/{\density_{x_{i+1}}(z_{i+1})}
        }$\\
        \quad \quad \quad With probability $\alpha_{i+1}$ accept the proposal:
        $x_{i+1} \gets
        z_{i+1} $\\
        \quad \quad \quad With probability $1-\alpha_{i+1}$ reject the proposal:
        $x_{i+1} \gets
        x_i$}
    }
  \caption{John Walk with parameter $r$ (\JW{r})}
  \label{algo:john_Walk}
\end{algorithm}


\subsection{Mixing time bounds for warm start}
\label{sub:main_results}

We are now ready to state an upper bound on the mixing time of the
Vaidya walk.  In this and other theorem statements, we use $\UNICON$
to denote a universal positive constant.
Recall that $\target$ denotes the uniform distribution on the polytope $\Poly$,
and, that $\lazytrans_{\fulltagvaidya(\rparam)}$ denotes the operator on
distributions associated with the Vaidya walk.
\begin{theorem}
  \label{thm:mixing_time_bound}
Let $\initial$ be any distribution that is $M$-warm with respect to
$\target$ as defined in equation~\eqref{EqnDefnMstart}.  For any $\delta \in (0, 1]$, the Vaidya walk with
  parameter $\rparam_{\tagvaidya} = \vaidyaradiusconst$ satisfies
\begin{align}
  \label{EqnVaidyaBound}
\tvnorm{\lazytrans_{\fulltagvaidya(\rparam_{\tagvaidya})}^k(\initial) - \target} \leq
\delta \qquad \mbox{for all } \; k \geq \UNICON \obs^{{1}/{2}}
\dims^{{3}/{2}} \, \log\parenth{\frac{\sqrt{\warmparam}}{\delta}}.
\end{align}
\end{theorem}
The proof of
Theorem~\ref{thm:mixing_time_bound} is provided in Section~\ref{sec:proof}.
Theorem~\ref{thm:mixing_time_bound} precisely quantifies the dependence of
mixing time of the Vaidya walk on many parameters of interest such as dimension
$\dims$, number of constraints $\obs$, the error tolerance $\delta$ and
the warmness $\warmparam$.  The specific choice $\rparam_{\tagvaidya} = \vaidyaradiusconst$ is for
theoretical purposes; in practice, we find that substantially larger
values can be used.\footnote{A larger than optimal $\rparam$ leads to an
  undesirable high rejection rate.  In practice, we can fine tune
  $\rparam$ by performing a binary search over the interval
  $[\vaidyaradiusconst, 1]$ and keeping track of the rejection rate of
  the samples during the run of the Markov chain for a given choice of
  $\rparam$. A choice of $\rparam > 1$ is obviously bad because then the
  Vaidya ellipsoid will have poor overlap with polytopes near the boundary,
  causing high rejection rate and slow down of the chain.}
Our upper bound for the mixing time of the Vaidya walk has $\mathcal{O}
({{\sqrt{\obs/\dims}}})$
improvement over the current best upper bound for the mixing time
of the Dikin walk. In Section~\ref{sub:per_iteration_cost}, we show
that the per iteration cost for the two walks is of the same order.
Since $\obs \geq \dims$ for closed polytopes in $\real^\dims$, the
effective cost until convergence (iteration complexity multiplied by
number of iterations required) for the Vaidya walk is at least of the
same order as of the Dikin walk, and significantly smaller when $\obs
\gg \dims$.  Comparing the provable mixing time upper bounds, the Vaidya walk
has an advantage over the Dikin walk for the problems where the number of
constraints is significantly larger than
the number of variables involved. Our simulations also confirm this theoretical
finding.

Let us now state our result for the mixing time of the John walk:
\begin{theorem}
\label{thm:john_walk}
Suppose that $\obs \leq \exp(\sqrt{\dims})$, and let $\initial$ be any
distribution that is $\warmparam$-warm with respect to $\target$.  Then for
any $\delta \in (0, 1]$, the John walk with parameter
  $\rparam_{\tagjohn} = \johnradiusconst$ satisfies
\begin{align*}
    \tvnorm{\lazytrans_{\fulltagjohn(\rparam_{\tagjohn})}^k(\initial) -\target} \leq
    \delta \qquad \text{ for all } k \geq \UNICON \; \dims^{2.5}\,
    \log^4\parenth{\frac{\obs}{\dims}} \, \log\parenth{\frac{\sqrt{\warmparam}}{\delta}}.
  \end{align*}
\end{theorem}
The proof of Theorem~\ref{thm:john_walk} is provided in Appendix~\ref{sec:proof_john_walk}.
Again the specific choice of $\rparam_{\tagjohn} = \johnradiusconst$
is for theoretical purpose; in practice larger choices are possible.
Note that the mixing time bound for the John walk depends only on the
number of constraints $\obs$ via a logarithmic factor, and so is
almost independent of $\obs$.  Consequently, it has a mixing time that
is polynomial in $\dims$ even if the number of constraints $\obs$
scales exponentially in $\sqrt{\dims}$.  Further, we show in
Section~\ref{sub:per_iteration_cost} that the cost to execute one step
of the John walk is of the same order as of the Dikin walk up to a
poly-logarithmic factor in $\obs$.  Thus, using John walk, we obtain improved mixing time bounds for the case when $\obs \gg \dims^2$.


\subsection{Mixing time bounds from deterministic start}
\label{sub:deterministic_start}

The mixing time bounds in Theorem~\ref{thm:mixing_time_bound}
and~\ref{thm:john_walk} depend on the warmness $M$ of the initial
distribution.  In some applications, it may not be easy to find an
$M$-warm initial distribution.  In such cases, we can consider
starting the random walk from a deterministic point \mbox{$x_0 \in
  \intP$} that is not too close to the boundary~$\boundary$.  Indeed,
such a point can be found using standard optimization methods---e.g.,
using a Phase-I method for Newton's algorithm
\citep[see][Section~11.5.4]{boyd2004convex}.

Given such a deterministic initialization, our mixing time guarantees
depend on the distance of the starting point from the boundary.  This
dependence involves the following notion of $s$-centrality:
\begin{definition}
A point $x \in \intP$ is called \emph{$s$-central} if for any chord
$\overline{ef}$ with end points $e, f \in \partial \Pspace$ passing
through $x$, we have $\vecnorm{e-x}{2}/ \vecnorm{f-x}{2} \leq s$.
\end{definition}
Assuming that it is started at an $s$-central point $x_0$, the Dikin
walk \citep[algorithm~in section 2.1]{kannan2012random} has a
polynomial mixing time.  The authors showed that when the
walk moves to a new state for the first time, the distribution of the
iterate is
$\order{(\sqrt{\obs}s)^\dims}$-warm
with respect to
the distribution\footnote{Obtaining a warmness result for the Vaidya
walk from a deterministic start from a central point is non-trivial and
it is quite possible that the warmness does not improve. As a result, we
simply invoke the established result for the Dikin walk.
}
$\target$. Since only constant number of steps is required to get a warm start, for a deterministic start, we
can just use the Dikin walk in the beginning to provide a warm start to the
Vaidya (or John) walk. This motivates us to define the following
hybrid walk.

Given an $s$-central point $x_0$, simulate the Dikin walk until we
observe a new state.  Note that due to \emph{laziness} and the
accept-reject step, the chain can stay at the starting point for
several steps before making the first move a new state.  Let $k_1$
denote the (random) number of steps taken to make the first move to a
new state.  After $k_1$ steps, we run the walk \VW{r} with $x_{k_1}$
as the initial point.  We call such a walk as \emph{$s$-central
Dikin-start-Vaidya-walk} with parameter $r$.  Let
$\plaintransition_{\text{Dikin}}$ denote the transition kernel of the
Dikin walk stated above.  Then, we have the following mixing time
bound for this hybrid walk.
\begin{corollary}
\label{cor:central_start}
Any $s$-central Dikin-start-Vaidya-walk with parameter $r=\vaidyaradiusconst$
satisfies
\begin{align*}
\tvnorm{\lazytrans_{\fulltagvaidya(\rparam)}^k
  \big(\lazytrans_\fulltagdikin^{k_1}(\delta_{x_0}) \big) -
  \target} \leq \delta \qquad \text{ for all } k \geq \UNICON
\obs^{{1}/{2}} \dims^{{5}/{2}} \, \log\parenth{\frac{\obs s}{\delta}},
  \end{align*}
where $k_1$ is a geometric random variable with $\Exs\brackets{k_1}
\leq c'$, and $c, c'>0$ are universal constants.
\end{corollary}
The mixing rate is logarithmic in $\obs s$ and has an extra factor
of $\dims$ compared to the bounds in 
Theorem~\ref{thm:mixing_time_bound}.  However, guaranteeing a warm
start for a general polytope is hard but obtaining a central point
involves only a few steps of optimization.  Consequently, the hybrid
walk and the guarantees from Corollary~\ref{cor:central_start} come in
handy for all such cases.  Once again we observe that the upper bounds for
mixing time are improved by a factor of $\mathcal{O}({\sqrt{\obs/\dims}})$
when compared to the Dikin walk from an $s$-central
start~\citep{kannan2012random,narayanan2016randomized} which had a
mixing time of \order{\obs\dims^2}.  The proof follows immediately
from Theorem~1 by \cite{kannan2012random} and
Theorem~\ref{thm:mixing_time_bound} of this paper and is thereby
omitted.

In a similar fashion, we can provide a polynomial time guarantee for a
modified John walk from a deterministic start.  We can consider a
hybrid random walk that starts at an $s$-central point, simulates the
Dikin walk until it makes the first move to a new state, and from
there onwards simulates the John walk.  Such a chain would have a
mixing time of \mbox{$\order{\dims^{3.5}\polylog(\obs, \dims, s)}$.}  For
brevity, we omit a formal statement of this result.


\subsection{Conjecture on improved John walk}
\label{sub:improved_john_walk}

From our analysis, we suspect that it is possible to improve the
mixing time bound of \order{\dims^{2.5}\polylog(\obs/\dims)} in
Theorem~\ref{thm:john_walk} by considering a variant of the John walk.
In particular, we conjecture that a random walk with proposal
distribution given by $\gaussian{x}{\frac{r^2}{\dims \cdot
    \polylog(\obs/\dims)} {\john_x}^{-1} }$ for a suitable choice of $r$ has
an \order{\dims^2\polylog(\obs/\dims)} mixing time from a warm start.  We
refer to this random walk as the \emph{improved John walk}, and denote
its transition operator by $\plaintransition_\fulltagimprovedjohn$.  Let us
now give a formal statement of our conjecture on its mixing rate.
\begin{conjecture}
  \label{conj:john_walk}
  Let $\mu_0$ be any $M$-warm distribution.  Then for any $\delta \in
  (0, 1]$, the improved John walk with parameter $r = r_0$, satisfies the bound
  \begin{align*}
    \tvnorm{\plaintransition_\fulltagimprovedjohn^k(\initial) -\target}
    \leq \delta \qquad \mbox{ for all } k \geq \UNICON \: \dims^{2}\,
    \log_2^{c'}\parenth{\frac{2\obs}{\dims}} \, \log\parenth{\frac{\sqrt{M}}{\delta}},
  \end{align*}
  where $r_0, c, c'$ are universal constants.
\end{conjecture}

Note that this conjecture involves quadratic (degree two) scaling in
$\dims$; this exponent of two matches the sum of exponents for $\dims$
and $\obs$ in the mixing time bounds for both the Dikin and Vaidya
walks from a warm-start.  Consquently, the improved John walk would
have better performance than the Dikin, Vaidya and John walks for
almost all ranges of $(\obs, \dims)$, apart from possible
poly-logarithmic factors in the ratio $\obs/\dims$.



\subsection{Proof sketch}
\label{sub:bound_mixing_time}

In this subsection, we provide a high-level sketch of the main ingredients of the main proof. It is well-known that mixing of a Markov chain is closely related to its \emph{conductance}. Our main proof relies on the work by \cite{lovasz1999hit}
that characterizes the conductance of Markov chains on a convex set using Hilbert metric. Precisely, \cite{lovasz1999hit} showed that a Markov chain has good conductance if it makes jumps to regions with large overlaps from two nearby
points and the mixing time depends inversely on the maximum Hilbert metric
between such nearby points. Using this argument, it remains to make sure that the ellipsoid radius is chosen properly such that the ellipsoids remain inside the polytope and the ellipsoids corresponding to two different points $x$ and $y$ overlap a lot even if the points $x$ and $y$ are relatively far apart.

The conductance-based argument has been used for analyzing the ball
walk~\citep{lovasz1990ballwalk,lovasz1993random},
Hit-and-run~\citep{lovasz1999hit,lovasz2006hit} and the Dikin
walk~\citep{kannan2012random,narayanan2016randomized,sachdeva2016mixing}.
We refer the reader to the survey by \cite{vempala2005geometric} for a
thorough discussion about the relation between the conductance and
mixing time for Markov chains. Our proof techniques share a few
features with the recent analyses of the Dikin walk by
\cite{kannan2012random} and \cite{sachdeva2016mixing}.  However, new
technical ideas are needed in order to handle the state-dependent
weights $\levscores_x$ and $\johnweights_x$, as defined in
equations~\eqref{eq:defn_lev_scores} and~\eqref{eq:john_weights}
respectively, that underlie the proposal distributions for the Vaidya
and John walks. Note that these techniques are not present in the
analysis of the Dikin walk, which is based on constant weights.

Specifically, we present the proof of
Theorem~\ref{thm:mixing_time_bound} on the mixing time of the Vaidya
walk in Section~\ref{sec:proof} and defer the intermediate technical
results to Appendix~\ref{sec:technical_results_and_useful_properties},
\ref{sec:proof_of_lemma_lemma:close_hessian_eigenvalues} and
\ref{sec:proof_of_lemma_lemma:change_in_log_det_and_local_norm}.  We
present the proof of Theorem~\ref{thm:john_walk} (mixing time bound
for the John walk) in Appendix~\ref{sec:proof_john_walk} and provide
related auxiliary results and their proofs in
Appendices~\ref{sec:technical_lemmas},
\ref{sec:proof_of_lemma_lemma:john_close_hessian_eigenvalues},
\ref{sec:proof_of_lemma_lemma:john_change_in_log_det_and_local_norm},
\ref{sec:proofs_of_technical_lemmas_from_section_} and
\ref{sec:proof_of_lemmas_from_section_sub:tail_bounds}.  As alluded to
earlier, to keep the paper self-contained, we provide the proof of
Lov{\'a}sz's Lemma in
Appendix~\ref{sec:proof_of_lemma_lemma:lovasz_theorem}.


\section{Numerical experiments}
\label{sec:numerical_experiments}

In this section, we first analyze the per-iteration cost to implement
of three walks.  We show that while the Dikin walk has the best
per-iteration cost, the per-iteration cost of the Vaidya walk is only
twice of that of Dikin walk and the per-iteration cost of the John
walk is only of order $\log_2(2\obs/\dims)$ larger.  Second, we
demonstrate the speed-up gained by the Vaidya walk over the Dikin walk
for a warm start on different polytopes.


\subsection{Per iteration cost}
\label{sub:per_iteration_cost}

We now show that the per iteration cost of the Dikin, Vaidya and John
walks is of the same order.  The proposal step of Vaidya walk requires
matrix operations like matrix inversion, matrix multiplication and
singular value decomposition (SVD).  The accept-reject step requires
computation of matrix determinants, besides a few matrix inverses and
matrix-vector products.  The complexity of all aforementioned
operations is \order{\obs \dims^{2}}.
Thus, per
iteration computational complexity for the Vaidya walk is \order{\obs
  \dims^{2}}.\footnote{In theory, the matrix computations for the Dikin walk can be
  carried out in time $\obs \dims^\inverserate$ for an exponent
  $\inverserate < 1.373$, but such algorithms are not numerically stable
  enough for
  practical use.}

Both the Dikin and Vaidya walks requires an SVD computation for
inverting the Hessian of Dikin barrier $\hesslogbarr_x$.  In addition
for the Vaidya walk, we have to invert the matrix $\vaidya_x$, which
leads to almost twice the computation time of the Dikin walk per
step.  This difference can be observed in practice.

For the John walk, we need to compute the weights $\johnweights_x$ at
each point which involves solving the program~\eqref{eq:john_weights}.
\cite{lee2014path} argued that the convex
program~\eqref{eq:john_weights} for obtaining John walk's weights is
strongly convex with a suitably chosen norm.  They proved that solving
this program requires $\log^2\obs$ number of gradient steps, where the
computational complexity of each gradient step is equivalent to that
of solving an $\obs \times \dims$ linear system
($\order{\obs\dims^{2}}$ using a numerically stable routine).  Thus,
the overall cost for the John walk is of the same order as of the
Dikin walk up to a poly-logarithmic factor in the pair $(\obs,
\dims)$.

In practice, for the John walk, the combined effect of logarithmic
factors in the number of steps and the cost to implement each step
cannot be ignored.  This extra factor becomes a bottleneck for the
overall run time for the convergence of the Markov chain.
Consequently, the John walk is not suitable for polytopes with
moderate values of $\obs$ and $\dims$, and its mixing time bounds are
computationally superior to the Dikin and Vaidya walks only for the
polytopes with $\obs \gg \dims \gg 1$.


\subsection{Simulations}
\label{sub:numerical_experiments}

We now present simulation results for the random walks in
$\real^\dims$ for $\dims=2, 10$ and $50$ with initial distribution
$\initial = \NORMAL(0, \sigma_\dims^2\,\Ind_\dims)$ and target distribution
being uniform, on the following polytopes:
\begin{enumerate}[leftmargin=80pt, label=\textbf{Set-up~\arabic*}]
  \itemsep0em
  \item :\label{exp:square} The set $[-1, 1]^2$ defined by different
    number of constraints.
    \item :\label{exp:square-highd} The set $[-1, 1]^\dims$ for $\dims \in \{ 2, 3, 4, 5, 6, 7\}$ for $\obs = \{2\dims, 2\dims^2, 2\dims^3\}$ constraints.
  \item :\label{exp:random} Symmetric polytopes in $\real^2$ with
    $\obs$-randomly-generated-constraints.
  \item :\label{exp:circle} The interior of regular $\obs$-polygons on
    the unit circle.
  \item :\label{exp:random_high} Hyper cube $[-1, 1]^\dims$ for $\dims
    = 10$ and $50$.
\end{enumerate}
We choose $\sigma_\dims$ such that the warmness parameter $M$ is bounded by $100$.
We provide implementations of the Dikin, Vaidya and John walks in
python and a jupyter notebook at the github repository
\mbox{\url{https://github.com/rzrsk/vaidya-walk}}.

We use the following three ways to compare the convergence rate of the Dikin and the Vaidya walks: (1) comparing the approximate mixing time of a particular subset of the polytope---smaller value is associated with a faster mixing chain;
(2) comparing the plot of the empirical distribution of samples from multiple runs of the Markov chain after $k$ steps---if it appears \emph{more uniform} for smaller $k$, the chain is deemed to be faster; and (3) contrasting the sequential plots of one dimensional projection of samples for a single long run of the chain---\emph{less smooth} plot is associated with effective and fast exploration leading to a faster mixing~\citep{yu1998looking}.
Note that MCMC convergence diagnostics is a hard problem, especially in high dimensions, and since the methods outlined above are heuristic in nature we expect our experiments to not fully match our theoretical results.


In \ref{exp:square}, we consider the polytope $[-1, 1]^2$ which can be
represented by exactly $4$ linear constraints (see
Section~\ref{sub:visualization_of_three_walks}).  Suppose that we
repeat the rows of the matrix $A$, and then run the Dikin and Vaidya
walks with the new $A$.
Given the larger number of constraints, our
theory predicts that the random walks should mix more slowly.
In Figure~\ref{fig:square_64} and \ref{fig:square_2048}, we plot the empirical distribution obtained by
the Dikin walk and Vaidya walk, starting from $200$ i.i.d initial
samples, for $\obs = 64$ and $2048$. The empirical
distribution plot shows that having large $\obs$ significantly slows the
mixing rate of the Dikin walk, while the effect on the Vaidya walk is much
less.  Further, we also plot the scaling of the approximate mixing time $\hat k_{\text{mix}}$ (defined below) for this simulation as a function of the number of constraints $\obs$ in Figure~\ref{fig:square_mix}.
For \ref{exp:square-highd}, we plot $\hat k_{\text{mix}}$ as a function of the dimensions $\dims$ in Figures \ref{fig:square_d}-\ref{fig:square_d3}, for the random walks on $[-1, 1]^\dims$ where the hypercube is parametrized by different number of constraints $\obs \in \{2\dims, 2\dims^2, 2\dims^3\}$.
The approximate mixing time is defined with respect to the set $\set_\dims = \{x \in \realdim \vert \abss{x_i}  \geq c_\dims \ \forall i \in [\dims]\}$ where $c_\dims$ is chosen such that $\target(\set_\dims) = 1/2$.
In particular, for a fixed value of $\obs$,
let $\hat \transition^k$ denote the empirical measure after
$k$-iterations across $2000$ experiments.
The approximate mixing time $\hat k_{\text{mix}}$ is defined as
\begin{align}
  \hat k_{\text{mix}} \defn \min\braces{ k \bigg\vert
    {\target(\set_\dims)-\hat
      \transition^k(\set_\dims)}\leq
    \frac{1}{20} },
  \label{eq:kmix}
\end{align}
We choose such a set since the set covers the regions near to the boundary of
the polytope which are not covered well by the chosen initial distribution.
We make the following observations:
\begin{enumerate}
  \item The slopes of the best-fit lines, for $\hat k_{\text{mix}}$ versus $\obs$ in the log-log plot in Figure~\ref{fig:square_mix}, are $0.88$ and $0.45$ for Dikin and Vaidya walks respectively. This observation reflects a near-linear
  and sub-linear dependence on $\obs$ for a fixed $\dims$ for the mixing time of the Dikin walk and the Vaidya walk respectively.
  \item In Figures~\ref{fig:square_d}-\ref{fig:square_d3}, once again we observe a more significant effect of increasing the number of constraints on the approximate mixing time $\hat k_{\text{mix}}$.
  We list the slopes of the best fit lines on these log-log plots in Table~\ref{tab:exponents}. These slopes correspond to the exponents for $\dims$ for the approximate mixing time.
  From the table, we can observe that these experiments agree with the
  mixing time bounds of $\order{\obs\dims}$ for the Dikin walk and $\order{\obs^{0.5}\dims^{1.5}}$ for the Vaidya  walk.
\end{enumerate}
\begin{table}[h]
  \centering
  \resizebox{\columnwidth}{!}{
  \begin{tabular}{ccccc}
    \hline
    \toprule
    {\bf No. of Constraints} & {\bf DW Theoretical} & {\bf VW Theoretical} &
    {\bf DW Experiments}
    & {\bf VW Experiments} \\
    \midrule
    $\obs = 2\dims$ & $2.0$ & $2.0$ & $1.58$ & $1.72$ \\
    $\obs = 2\dims^2$ & $3.0$ & $2.5$ & $2.80$ & $2.48$ \\
    $\obs = 2\dims^3$ & $4.0$ & $3.0$ & $3.84$ & $2.75$ \\
    \bottomrule
    \hline
  \end{tabular}
  }
  \caption{Value of the exponent of dimensions $\dims$ for the theoretical bounds on mixing time and the observed approximate mixing time of the Dikin walk (DW) and the Vaidya walk (VW) for $[-1, 1]^\dims$ described by $\obs = 2\dims, 2\dims^2, 2\dims^3$ constraints. The theoretical exponents are based on the mixing time bounds of $\order{\obs\dims}$ for the Dikin walk and $\order{\obs^{0.5}\dims^{1.5}}$ for the Vaidya walk. The experimental exponents are based on the results from the simulations described in \ref{exp:square-highd} in Section~\ref{sub:numerical_experiments}. Clearly, the exponents observed in practice are in agreement with the theoretical rates and imply the faster convergence of the Vaidya walk compared to the Dikin walk for large number of constraints.
  }
  \label{tab:exponents}
\end{table}

In~\ref{exp:random}, we compare the plots of the empirical distribution of $200$ runs of the Dikin walk and the Vaidya walk for different values of $k$, for symmetric polytopes in $\real^2$ with $\obs$-randomly-generated-constraints.
We fix $b_i = 1$.
To generate $a_i$, first we draw two uniform random variables
from $[0, 1]$ and then flip the sign of both of them with probability
$1/2$ and assign these values to the vector $a_i$.  The resulting
polytope is always a subset of the square $\Pspace = [-1, 1]^2$ and
contains the diagonal line connecting the points $(-1, 1)$ and $(1,
-1)$.
From Figure~\ref{fig:random_64}-\ref{fig:random_2048}, we observe that while there is no clear winner for the case
$\obs = 64$, the Vaidya walk mixes mixes significantly faster than the
Dikin walk for the polytope defined by $2048$ constraints.

In~\ref{exp:circle}, the constraint set is the regular $\obs$-polygons
inscribed in the unit circle.  A similar observation as
in~\ref{exp:random} can be made from
Figure~\ref{fig:circle_64}-\ref{fig:circle_2048}: the Vaidya walk
mixes at least as fast as the Dikin walk and mixes significantly
faster for large $\obs$.

In~\ref{exp:random_high}, we examine the performance of the Dikin walk
and the Vaidya walk on hyper-cube $[-1, 1]^\dims$ for $\dims = 10, 50$.
We plot the one dimensional projections onto a random normal direction of all the samples from a single run up to $10,000$ steps.
The Vaidya sequential plot looks more jagged than that of the Dikin walk for $\dims= 10, \obs = 5120$.  For other cases, we do not have a clear winner.  Such an observation is consistent with the $\mathcal{O}({\sqrt{\obs/\dims}})$ speed up of the Vaidya
walk which is apparent when the ratio $\obs/\dims$ is large.

\begin{figure}[h]
  \centering
  \begin{subfigure}{0.52\linewidth}
    \centering
    \includegraphics[width=1.\linewidth]{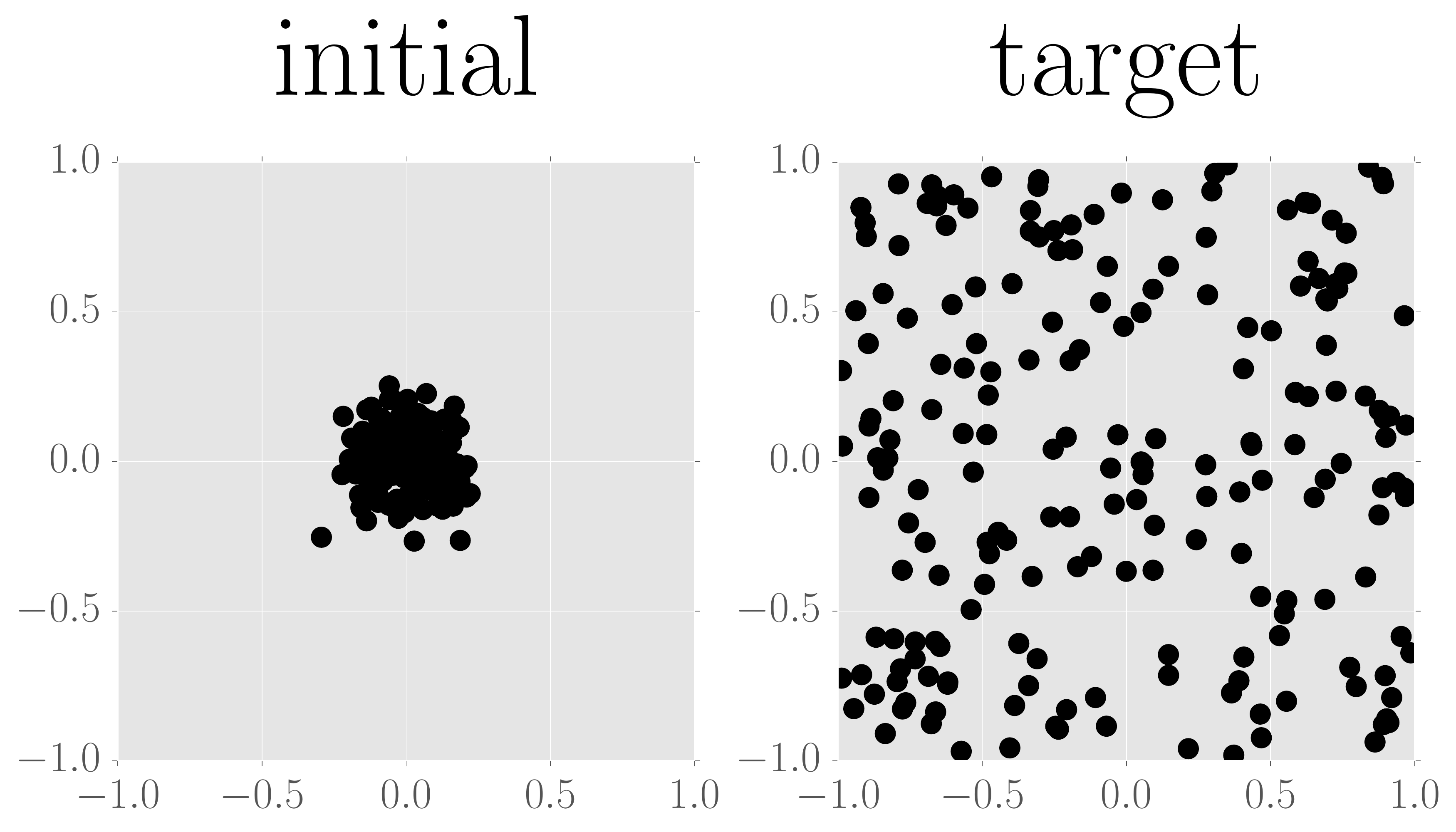}
    \caption{}
    \label{fig:tar_int}
  \end{subfigure}\vspace{3mm}
  \begin{subfigure}{0.38\linewidth}
    \centering
    \includegraphics[width=.9\linewidth]{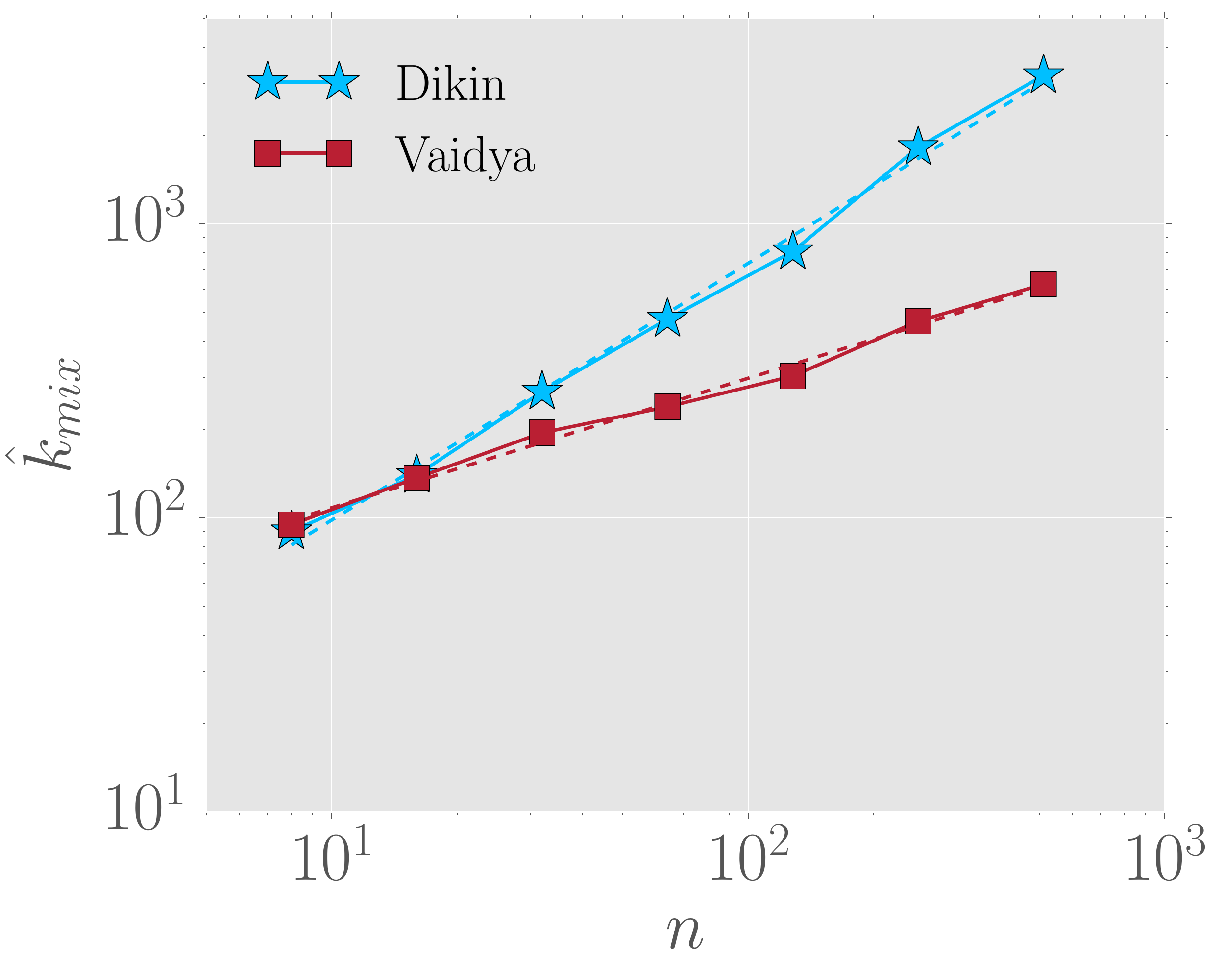}
    \caption{}
    \label{fig:square_mix}
  \end{subfigure}\vspace{3mm}

  \begin{subfigure}{0.48\linewidth}
    \centering
    \includegraphics[width=1\linewidth]{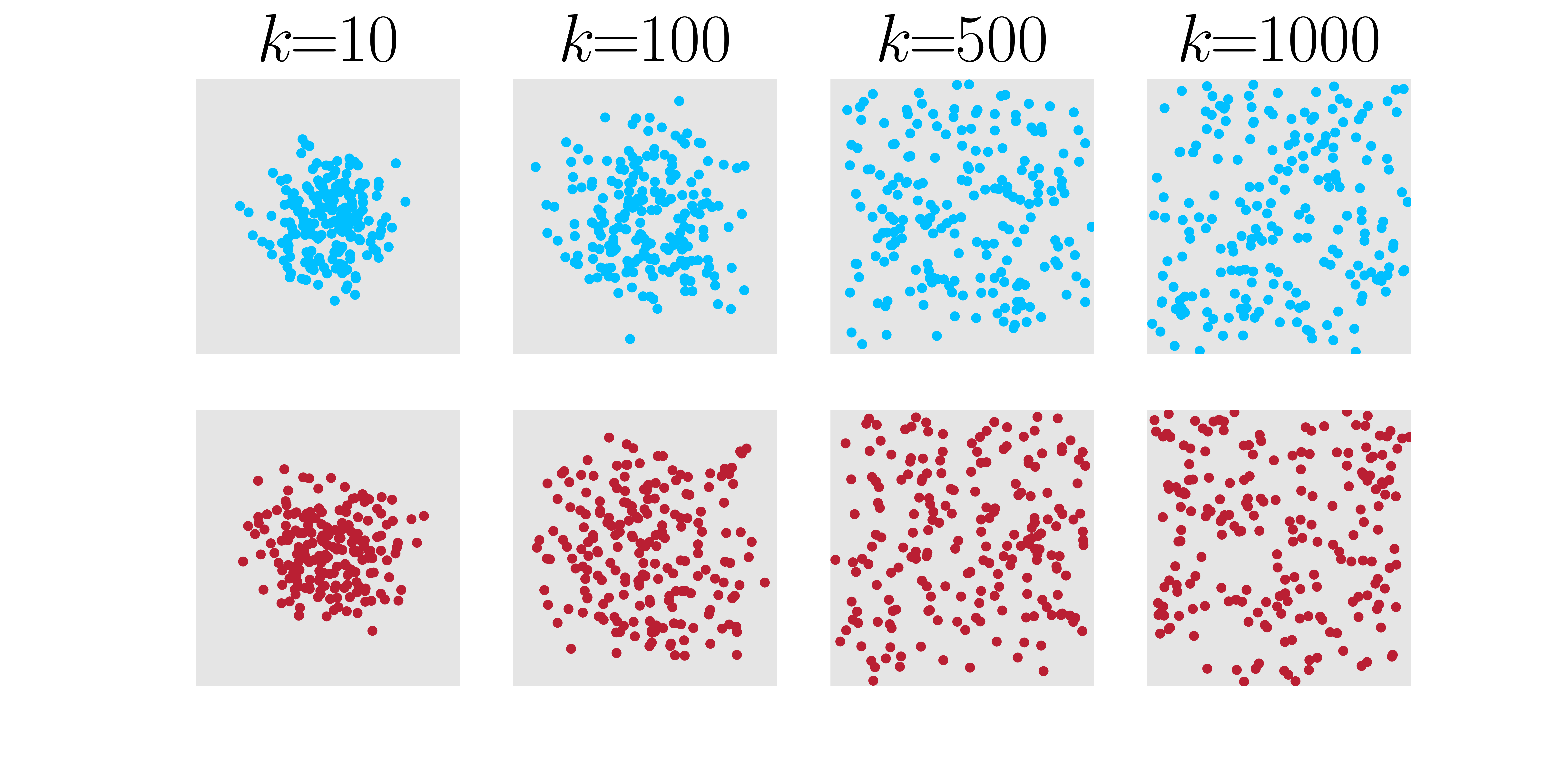}
    \vspace{-1.5\baselineskip}
    \caption{$\obs = 64$}
    \label{fig:square_64}
  \end{subfigure}
  \begin{subfigure}{0.48\linewidth}
    \centering
    \includegraphics[width=1\linewidth]{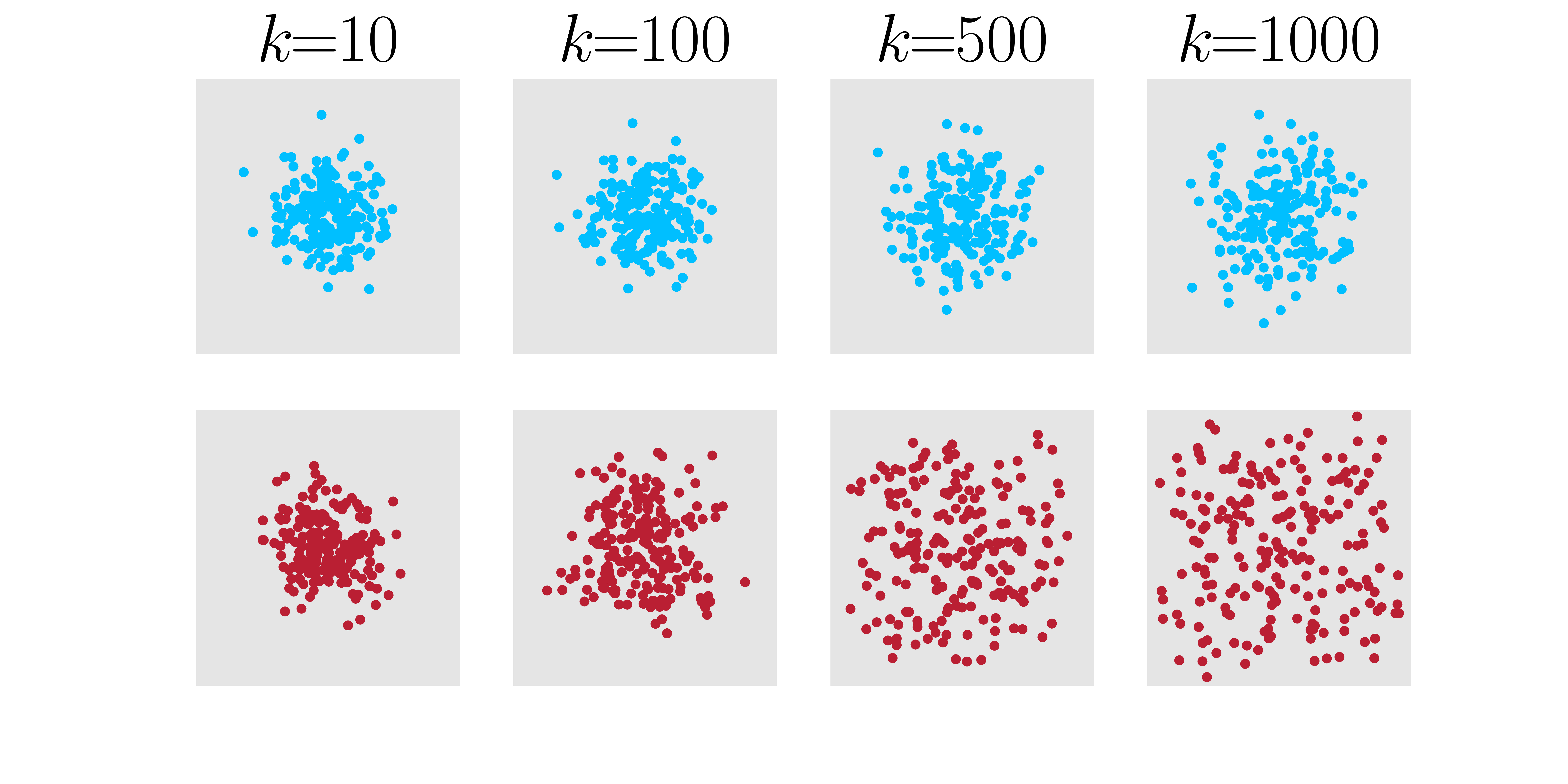}
    \vspace{-1.5\baselineskip}
    \caption{$\obs = 2048$}
    \label{fig:square_2048}
  \end{subfigure}\vspace{3mm}

   \begin{subfigure}{0.32\linewidth}
    \centering
    \includegraphics[width=1\linewidth]{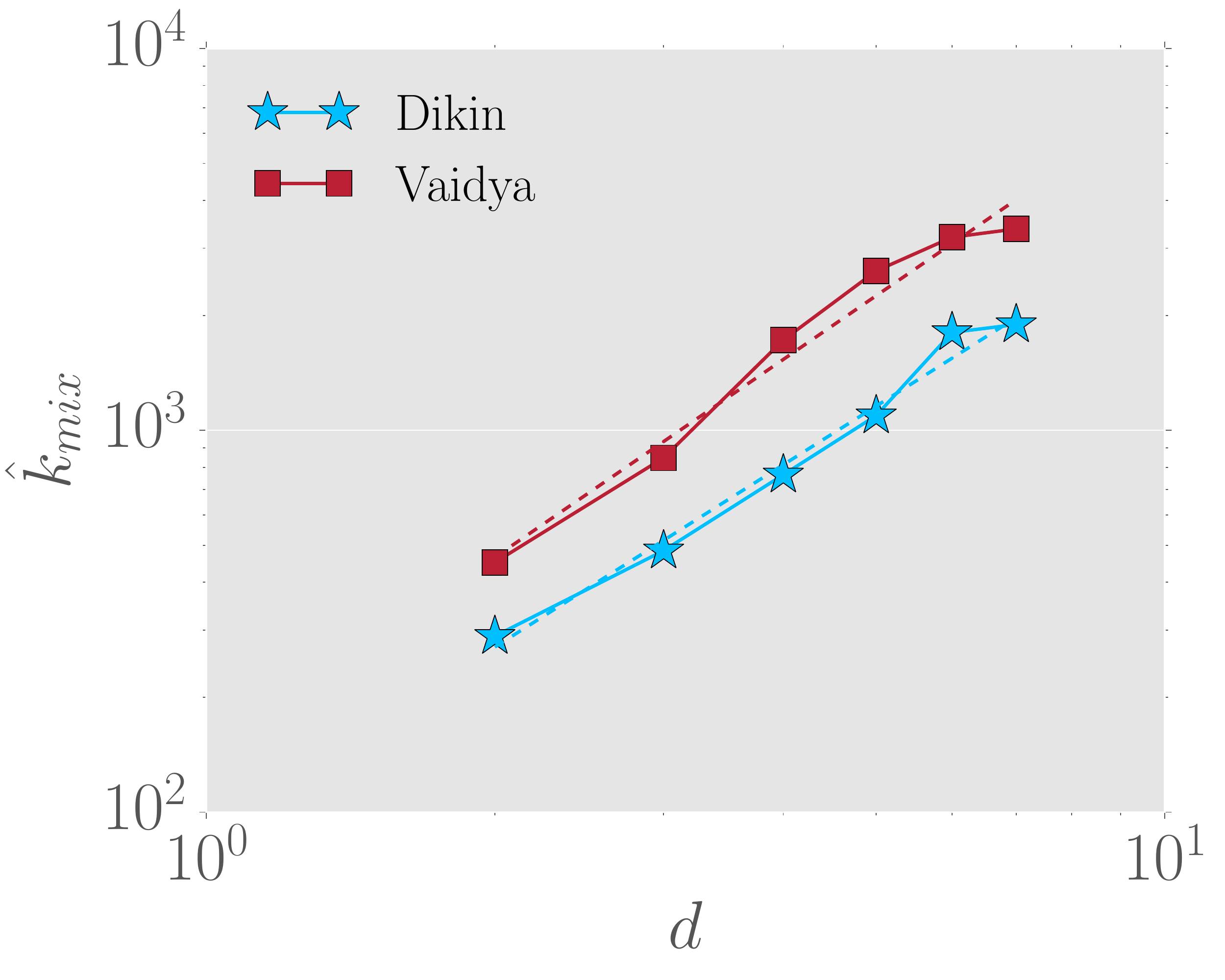}
    \caption{$\obs = 2\dims$}
    \label{fig:square_d}
  \end{subfigure}
  \begin{subfigure}{0.312\linewidth}
    \centering
    \includegraphics[width=1\linewidth]{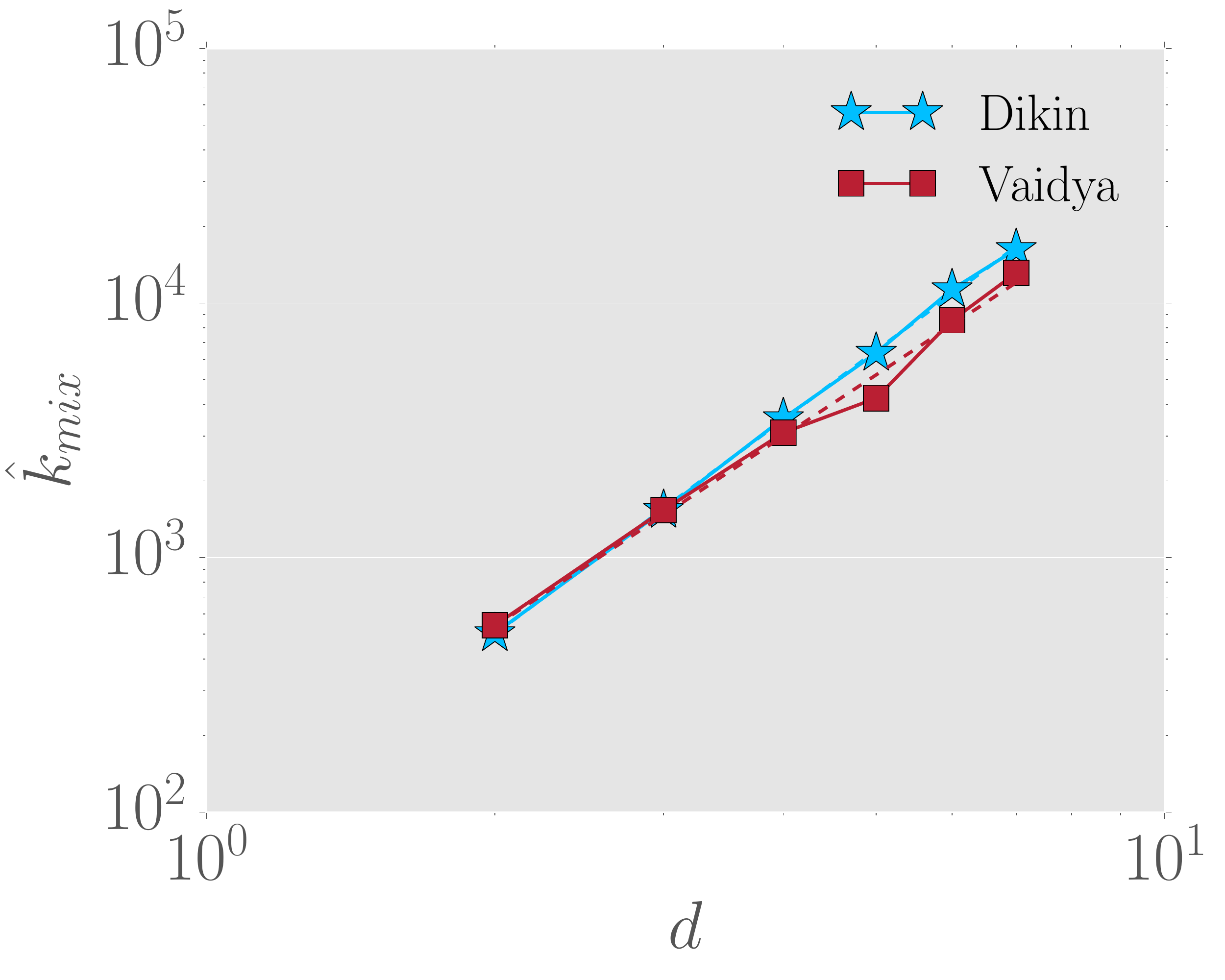}
    \caption{$\obs = 2\dims^2$}
    \label{fig:square_d2}
  \end{subfigure}
  \begin{subfigure}{0.32\linewidth}
    \centering
    \includegraphics[width=1\linewidth]{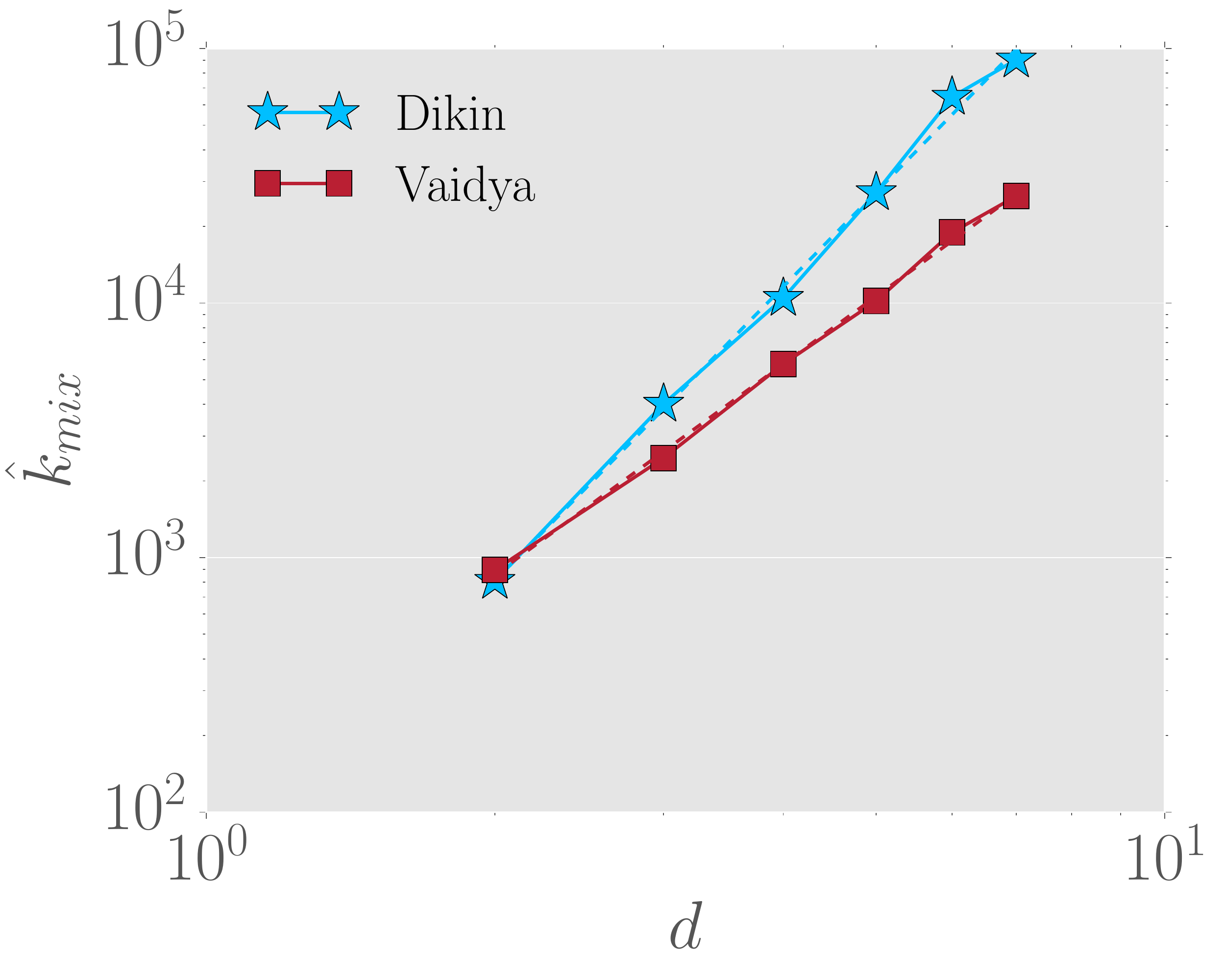}
    \caption{$\obs = 2\dims^3$}
    \label{fig:square_d3}
  \end{subfigure}\vspace{5mm}

  \caption{ Comparison of the Dikin and Vaidya walks on the polytope
    $\Pspace = [-1, 1]^2$.  {\textbf {(a)}} Samples from the initial
    distribution $\initial = \NORMAL(0, 0.04\, \Ind_2)$ and the
    uniform distribution on $[-1, 1]^2$.  {\textbf {(b)}} Log-log plot of
    $\hat k_{\text{mix}}$~\eqref{eq:kmix} versus the number of constraints
    ($\obs$) for a fixed dimension $\dims = 2$.
     {\textbf {(c, d)}} Empirical distribution of the samples for the Dikin walk (blue/top rows) and the Vaidya walk
    (red/bottom rows) for different values of $\obs$ at iteration $k =
    10, 100, 500$ and $1000$.
    \mbox{{\textbf {(e, f, g)}} Log-log} plot of $\hat k_{\text{mix}}$ vs the dimension $\dims$,
    for $\obs \in \{2\dims, 2\dims^2, 2\dims^3\}$ for \mbox{$\dims \in \{2, 3, 4, 5, 6, 7\}$}.
    The exponents from these plots are summarized in Table~\ref{tab:exponents}.
     Note that increasing the number of constraints
    $\obs$ has more profound effect on the Dikin walk in almost all the cases.
    }
  \label{fig:square}
\end{figure}

\begin{figure}
  \centering
  \begin{subfigure}{0.49\linewidth}
    \centering
    \includegraphics[width=1\linewidth]{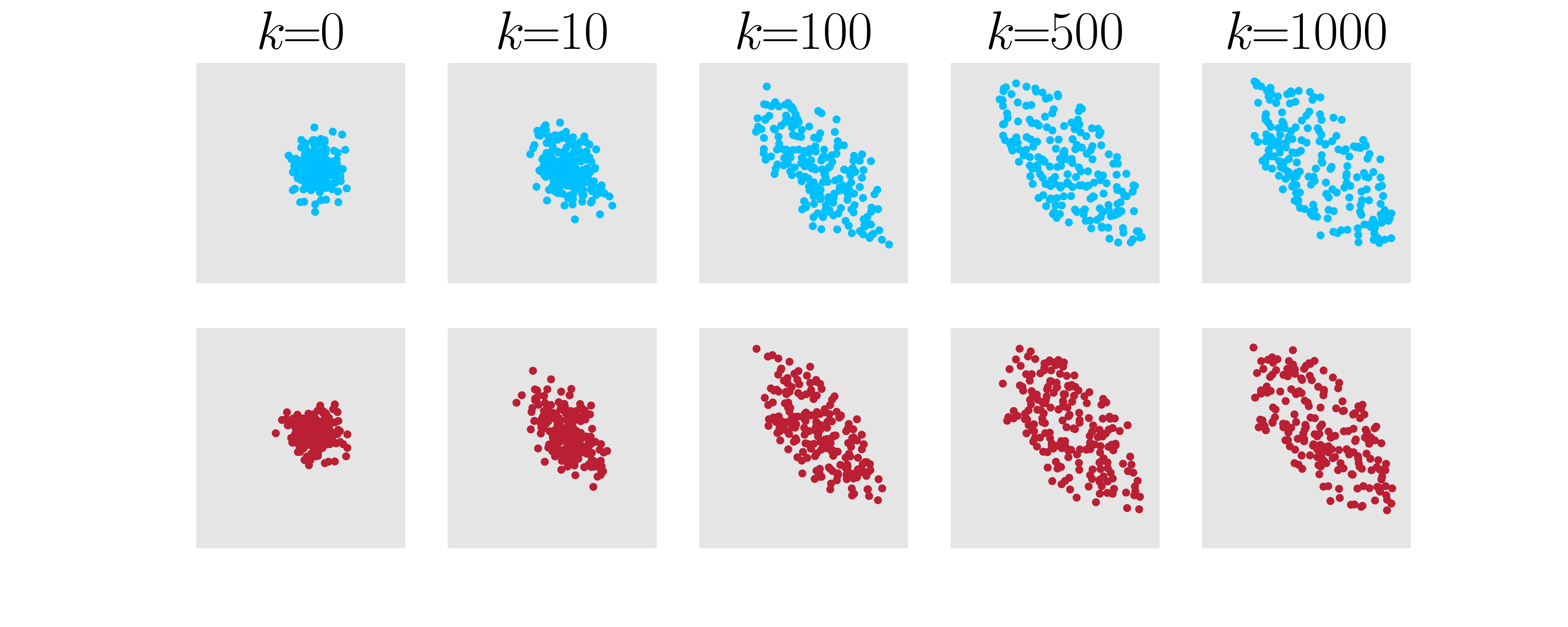}
    \vspace{-1\baselineskip}
    \caption{$\obs = 64$}
    \label{fig:random_64}
  \end{subfigure}\vspace{3mm}
  \begin{subfigure}{0.49\linewidth}
    \centering
    \includegraphics[width=1\linewidth]{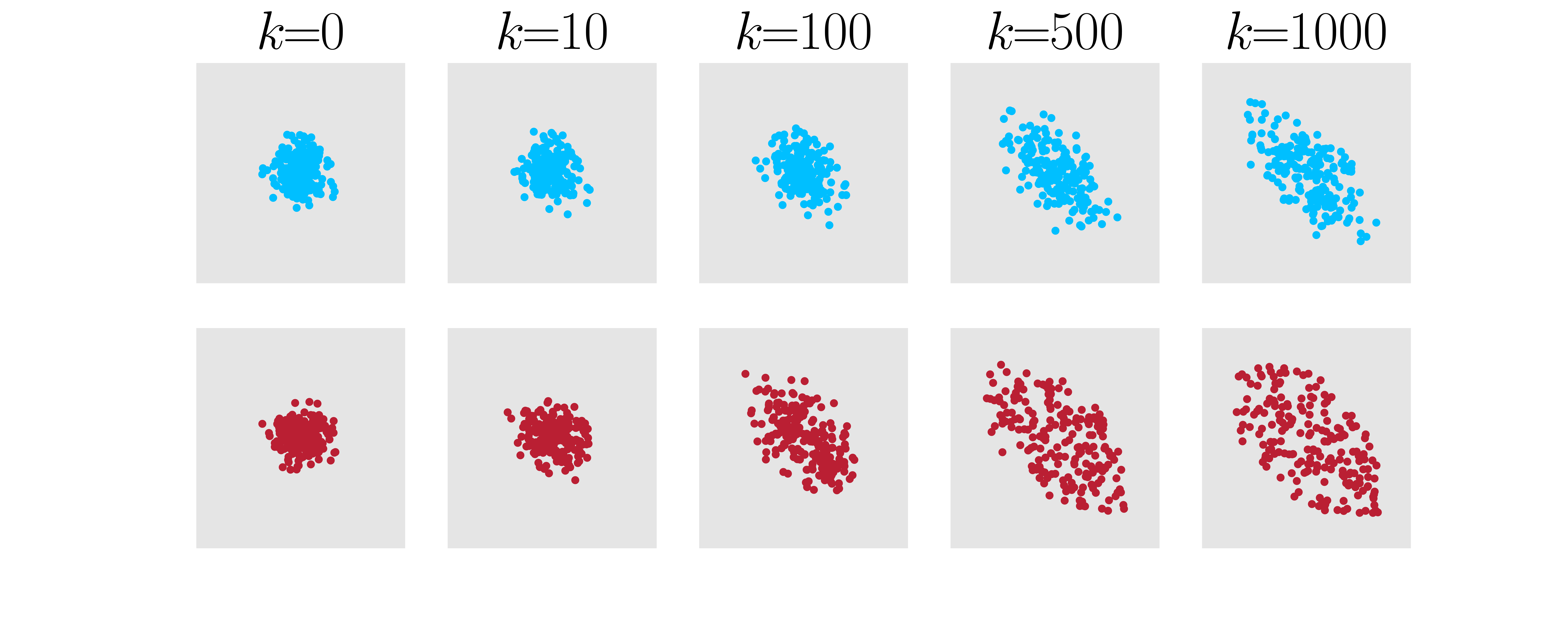}
    \vspace{-1\baselineskip}
    \caption{$\obs = 2048$}
    \label{fig:random_2048}
  \end{subfigure}\vspace{3mm}

  \begin{subfigure}{0.49\linewidth}
    \centering
    \includegraphics[width=1\linewidth]{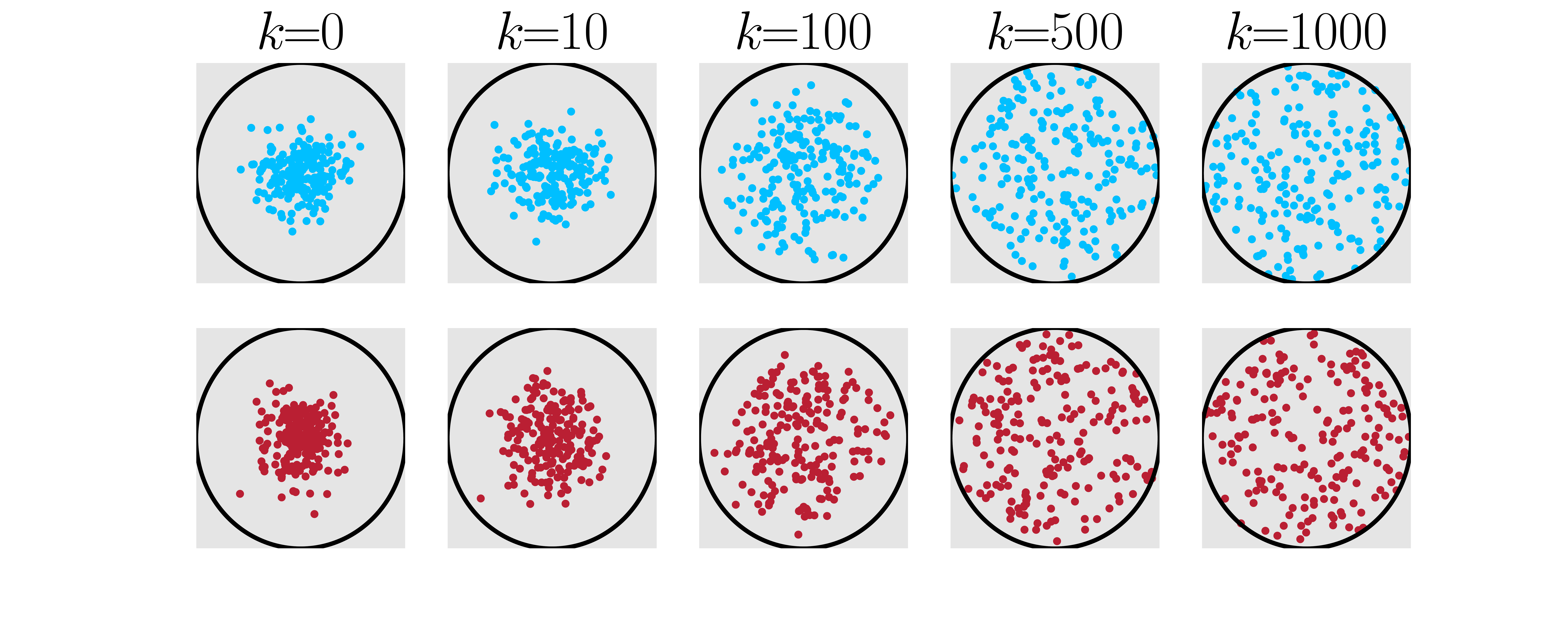}
    \vspace{-1\baselineskip}
    \caption{$\obs = 64$}
    \label{fig:circle_64}
  \end{subfigure}
  \begin{subfigure}{0.49\linewidth}
    \centering
    \includegraphics[width=1\linewidth]{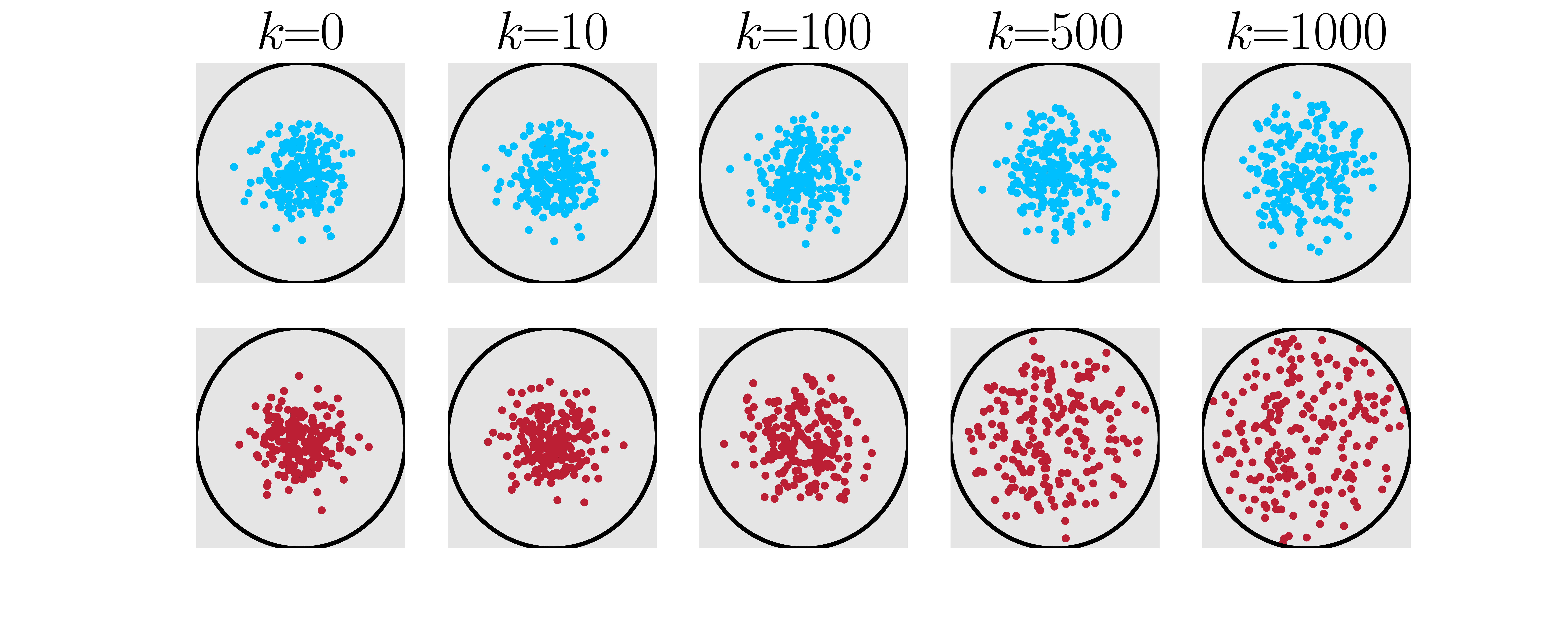}
    \vspace{-1\baselineskip}
    \caption{$\obs = 2048$}
    \label{fig:circle_2048}
  \end{subfigure}

  \caption{Empirical distribution of the samples from $200$ runs for the Dikin walk (blue/top rows) and the
    Vaidya walk (red/bottom rows) at different iterations $k$.
    The $2$-dimensional polytopes considered are: \mbox{\textbf{(a, b)}} random polytopes
    with $\obs$-constraints, and \textbf{(c, d)} regular
    $\obs$-polygons inscribed in the unit circle.  For both sets of cases, we
    observe that higher $\obs$ slows down the walks, with visibly more
    effect on the Dikin walk compared to the Vaidya walk.  }
  \label{fig:new_circle_random}
\end{figure}

\begin{figure}
  \centering
  \begin{subfigure}{0.49\linewidth}
    \centering
    \includegraphics[width=1\linewidth]{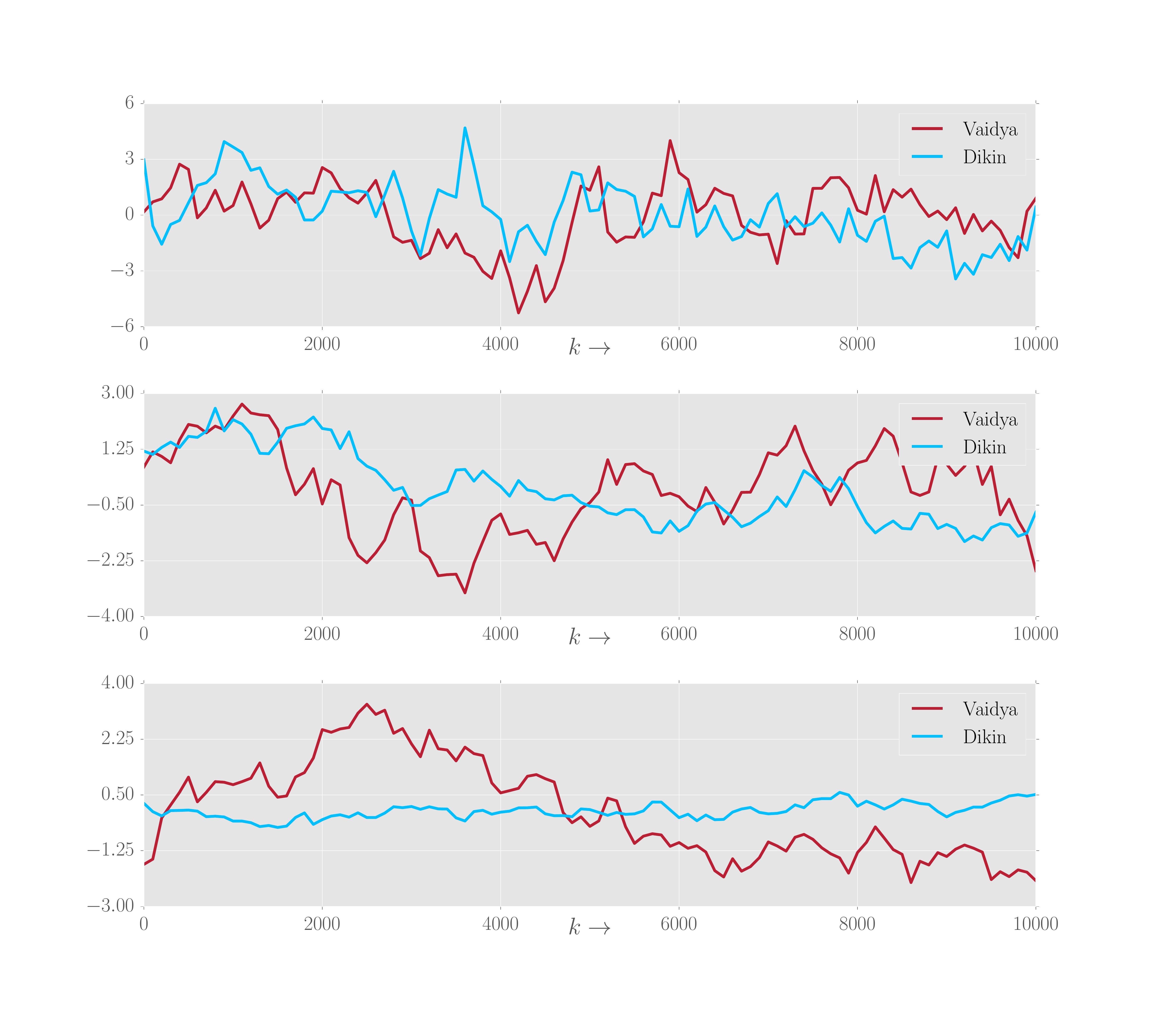}
    \vspace{-2\baselineskip}
    \caption{$\dims = 10$}
    \label{fig:one_d_projection}
  \end{subfigure}
  \begin{subfigure}{0.49\linewidth}
    \centering
    \includegraphics[width=1\linewidth]{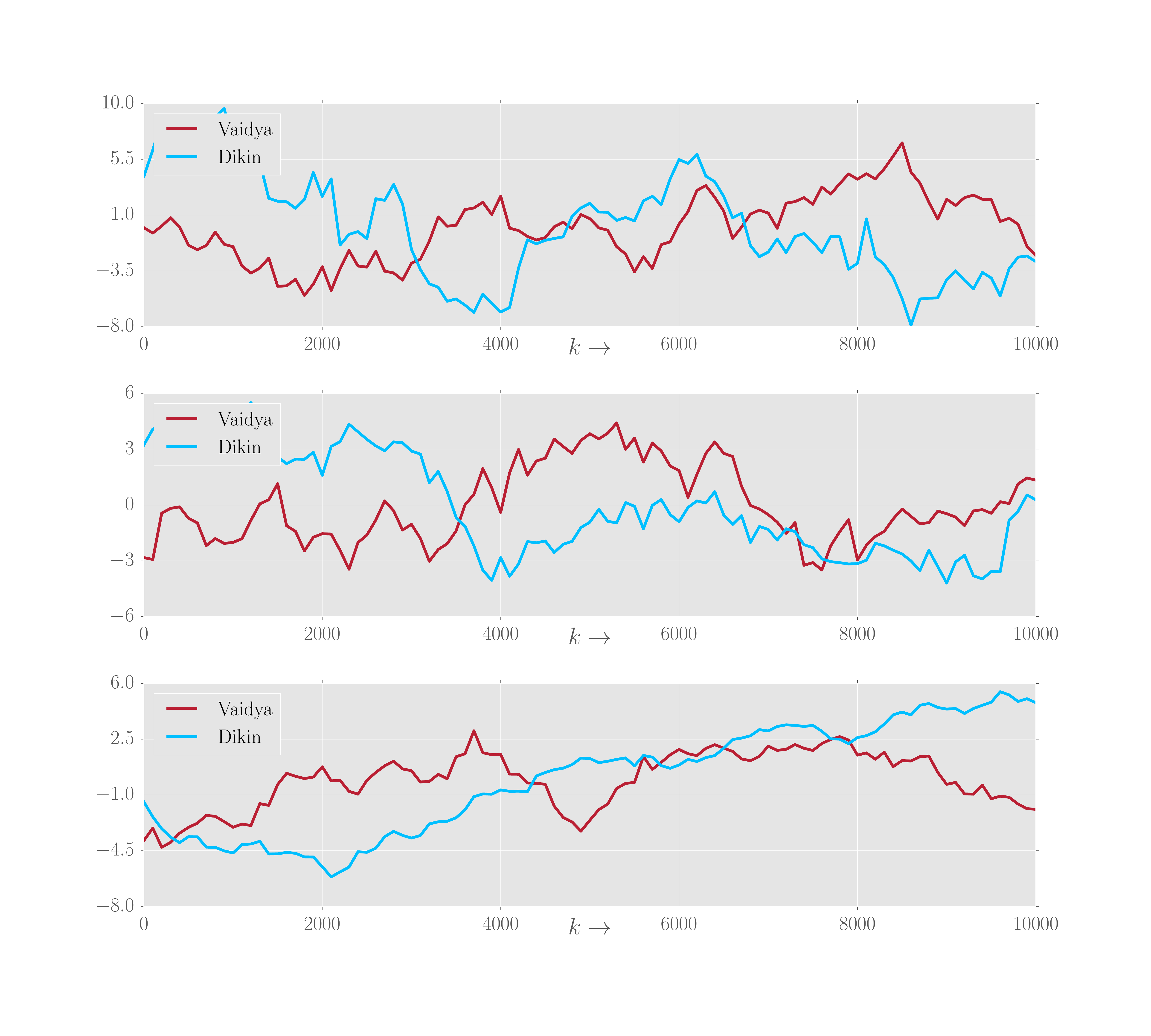}
    \vspace{-2\baselineskip}
    \caption{$\dims = 50$}
    \label{fig:normal_projection}
  \end{subfigure}

  \caption{Sequential plots of a one-dimensional random projection  of the samples on the hyperbox \mbox{$\Pspace = [-1, 1]^\dims$}, defined
    by $\obs$ constraints.  Each plot corresponds to one long run of
    the Dikin and Vaidya walks, and the projection is taken in a
    direction chosen randomly from the sphere.  \textbf{(a)} Plots for
    $\dims = 10$ and $\obs \in \{20, 640, 5120 \}$.  \textbf{(b)}
    Plots for $\dims = 50$ and $\obs \in \{100, 400, 1600\}$.
    Relative to the Dikin walk, the Vaidya walk has a more jagged plot
    for pairs $(\obs, \dims)$ in which the ratio $\obs/\dims$ is
    relatively large: for instance, see the plots corresponding to
    $(\obs, \dims)=(640, 10)$ and $(5120, 10)$. The same claim cannot
    be made for pairs $(\obs, \dims)$ for which the ratio $\obs/\dims$
    is relatively small; e.g., the plot with $(\obs, \dims) =
    (20,10)$.  These observations are consistent with our results that
    the Vaidya walk mixes more quickly by a factor of order
    $\mathcal{O}({{\sqrt{\obs/\dims}}})$ over the Dikin walk.}
  \label{fig:high_d_cube}
\end{figure}


\section{Proofs}
\label{sec:proof}

We begin with auxiliary results in
Section~\ref{sub:some_auxiliary_results} which we use then to prove
Theorem~\ref{thm:mixing_time_bound} in
Section~\ref{sub:proof_of_theorem_thm:mixing_time_bound}.  Proofs of the auxiliary results are in
Sections~\ref{ssub:proof_of_lemma_lemma:first_bounds} and
\ref{ssub:proof_of_lemma_lemma:vaidya_walk_close_kernel}, and we defer
other technical results to appendices.


\subsection{Auxiliary results}
\label{sub:some_auxiliary_results}

Our proof proceeds by formally establishing the following property for
the Vaidya walk: if two points are close, then their one-step
transition distribution are also close.  Consequently, we need to
quantify the closeness between two points and the associated
transition distributions.  We measure the distance between two points
in terms of the cross ratio that we define next.
For a given pair of points $x, y
\in \Pspace$, let $e(x), e(y) \in \partial \Pspace$ denote the
intersection of the chord joining $x$ and $y$ with $\Pspace$ such that
$e(x), x, y, e(y)$ are in order (see Figure~\ref{fig:basic_polytope}).
The cross-ratio $d_\Pspace(x, y)$ is given by
\begin{align}
\label{eq:def_cross_ratio}
d_\Pspace(x, y) & \defn \frac{ \vecnorm{e(x)-e(y)}{2} \vecnorm{x-y}{2}
}{ \vecnorm{e(x)-x}{2} \vecnorm{e(y)-y}{2} }.
\end{align}
The ratio $d_\Pspace(x, y)$ is related to the Hilbert metric on
$\Pspace$, which is given by $\log \parenth{ 1+d_\Pspace(x, y)}$; see
the paper by~\cite{bushell1973hilbert} for more details.

\begin{figure}
  \centering
  \begin{subfigure}[h]{0.45\linewidth}
    \centering
    \widgraph{0.8\linewidth}{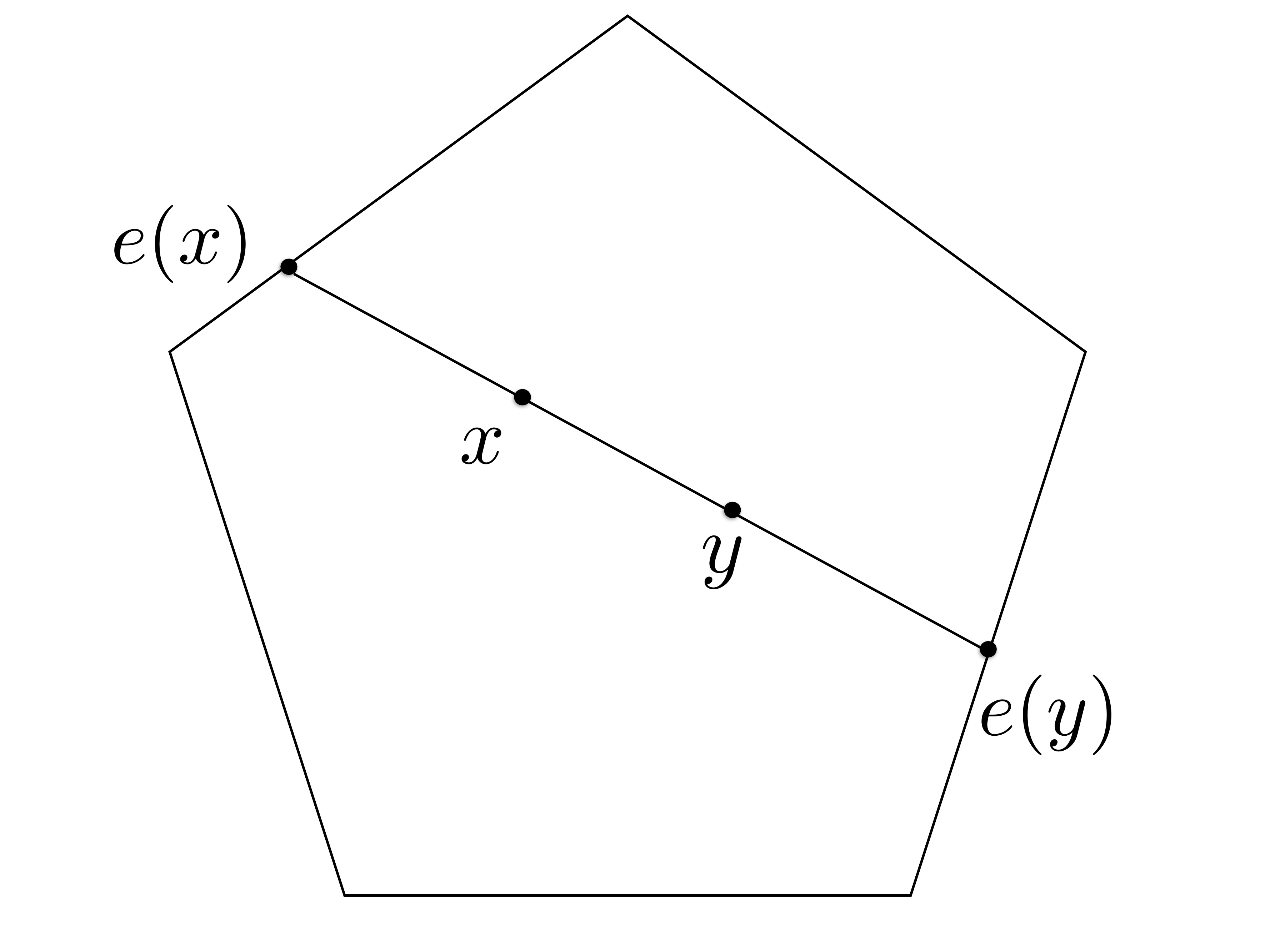}
    \caption{}
    \label{fig:basic_polytope}
  \end{subfigure}
  \begin{subfigure}[h]{0.45\linewidth}
    \centering
    \widgraph{0.8\linewidth}{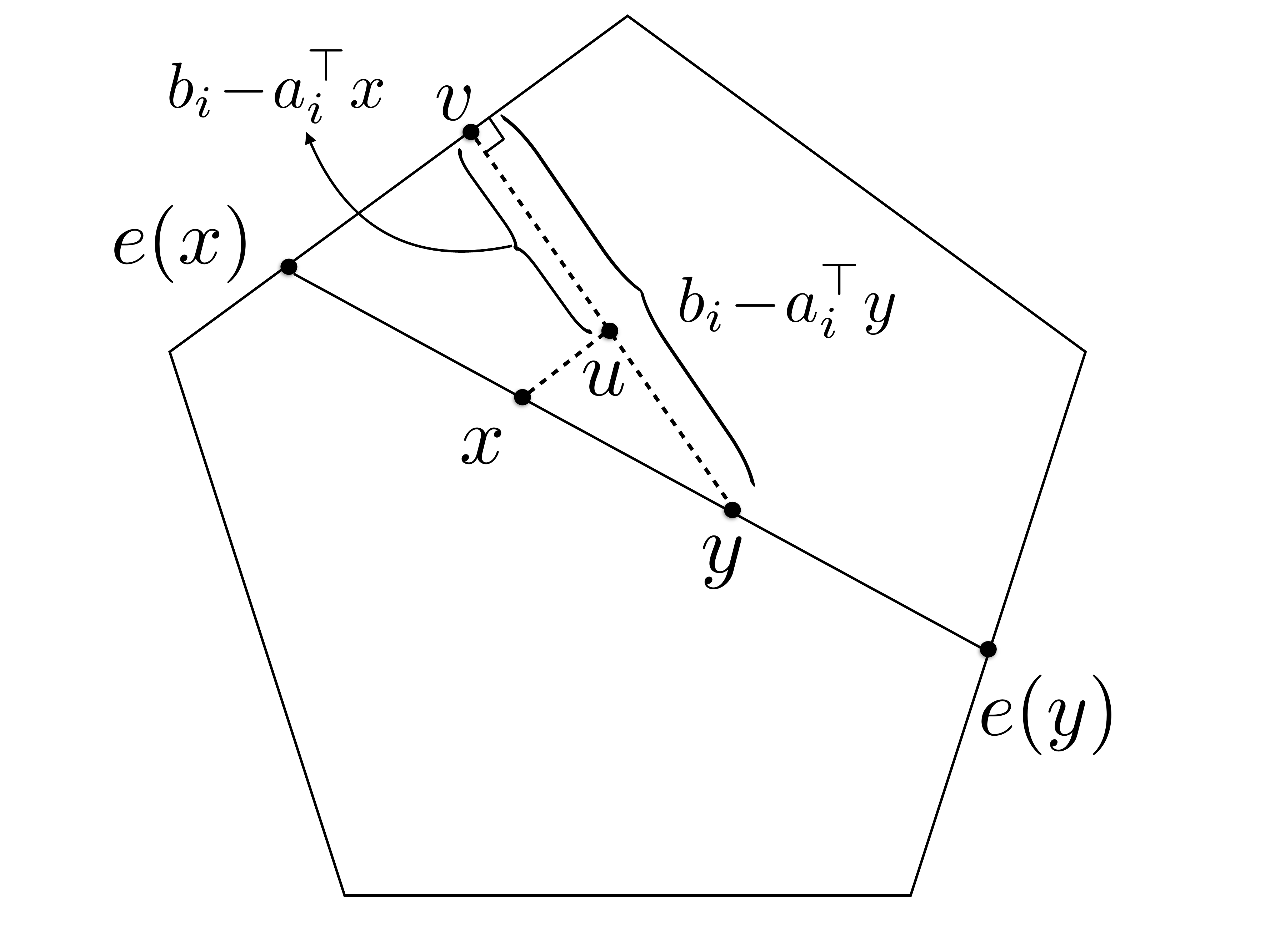}
    \caption{}
    \label{fig:cross_ratio}
  \end{subfigure}
  \caption{Polytope $\Pspace = \{x\in \realdim\vert Ax \leq b \}$.
    (a)~The points $e(x)$ and $e(y)$ denote the intersection points of
    the chord joining $x$ and $y$ with $\Pspace$ such that $e(x), x,
    y, e(y)$ are in order.  (b)~A geometric illustration of the
    argument~\eqref{eq:relate_cross_ratio}. It is straightforward to
    observe that ${\vecnorm{x-y}{2}}/{\vecnorm{e(x)-x}{2}} =
    {\vecnorm{u-y}{2}}/{\vecnorm{u-v}{2}} =
    {\abss{a_i\tp(y-x)}}/\parenth{b_i-a_i\tp x}$.  }
  \label{fig:polytope_and_end_points}
\end{figure}

Consider a lazy reversible random walk on a bounded convex set
$\Pspace$ with transition operator $\lazytrans$ defined via the
mapping $\initial \mapsto \initial/2 + \nolazytrans(\initial)/2$ and
stationary with respect to the uniform distribution on $\Pspace$
(denoted by $\target$).  (Recall that $\diracdelta_x$ denote the
dirac-delta distribution with unit mass at $x$.)  The following lemma
gives a bound on the mixing-time of the Markov chain.

\begin{lemma}[Lov{\'a}sz's Lemma]
\label{lemma:Lovasz_theorem}
Suppose that there exist scalars $\lovone, \lovtwo \in (0, 1)$ such
that
\begin{subequations}
\begin{align}
\label{EqnLovaszHyp}
\tvnorm{\nolazytrans(\diracdelta_x)- \nolazytrans(\diracdelta_y)} \leq
1 - \lovone \qquad \mbox{for all $x, y \in \intP$ with $d_\Pspace(x,
  y) < \lovtwo $.}
\end{align}
Then for every distribution $\initial$ that is $M$-warm with respect
to $\target$, the lazy transition operator $\lazytrans$ satisfies
\begin{align}
  \label{EqnLovaszResult}
\tvnorm{\lazytrans^k(\initial) -\target} \leq
\displaystyle\sqrt{M}\exp\parenth{-k\,\frac{\lovtwo^2 \lovone^2}{4096}}
\quad \forall \ \ k = 1, 2, \ldots.
\end{align}
\end{subequations}
\end{lemma}
This result is implicit in the paper by \cite{lovasz1999hit}, though not explicitly stated.  In
order to keep the paper self-contained, we provide a proof of this
result in Appendix~\ref{sec:proof_of_lemma_lemma:lovasz_theorem}.

Our proof of Theorem~\ref{thm:mixing_time_bound} is based on applying
Lov{\'a}sz's Lemma; the main challenge in our work is to establish
that our random walks satisfy the condition~\eqref{EqnLovaszHyp} with
suitable choices of $\Delta$ and $\lovone$.  In order to proceed with the
proof, we require a few additional notations.  Recall that the slackness
at $x$ was defined as $\slack_{x} \defn (b_1-a_1\tp x, \ldots, b_\obs
- a_\obs\tp x)\tp$.  For all $x\in \intP$, define the \emph{Vaidya
  local norm of $v$ at $x$} as
\begin{subequations}
  \begin{align}
    \label{eq:def_local_norm}
\vecnorm{v}{\vaidya_x} & \defn \vecnorm{\vaidya_x^{1/2}v}{2} =
\sqrt{\sum_{i=1}^\obs (\levvaidya_{x, i}+\vaidyabeta)\frac{(a_i^\top
    v)^2}{\slack_{x, i}^2}},
\end{align}
and the \emph{Vaidya slack sensitivity at $x$} as
\begin{align}
  \label{eq:theta_Definition}
\localslack_{\vaidya_x} \defn \parenth{\vecnorm{\frac{a_1}{\slack_{x,
        1}}}{\vaidya_x}^2, \ldots, \vecnorm{\frac{a_\obs}{\slack_{x,
        \obs}}}{\vaidya_x}^2}\tp = \parenth{\frac{a_1^\top
    \vaidya_x^{-1}a_1}{\slack_{x, 1}^2} , \ldots, \frac{a_\obs^\top
    \vaidya_x^{-1}a_\obs}{\slack_{x, \obs}^2} }\tp.
\end{align}
Similarly, we define the \emph{John local norm of $v$ at $x$} and the
\emph{John slack sensitivity at $x$} as
\begin{align}
\vecnorm{v}{\john_x} & \defn \vecnorm{\john_x^{1/2}v}{2}
\quad\text{and}\quad \localslackvaidya_{\john_x} \defn
\parenth{\vecnorm{\frac{a_1}{\slack_{x, 1}}}{\john_x}^2, \ldots,
  \vecnorm{\frac{a_\obs}{\slack_{x, \obs}}}{\john_x}^2}\tp.
  \label{eq:johntheta}
\end{align}
\end{subequations}
\noindent The following lemma provides useful properties of the
leverage scores $\levvaidya_{x}$ from
equation~\eqref{eq:defn_lev_scores}, the weights $\johnweights_x$
obtained from solving the program~\eqref{eq:john_weights}, and the
slack sensitivities $\localslack_{\vaidya_x}$ and
$\localslack_{\john_x}$.
\begin{lemma}
  \label{lemma:first_bounds}
  For any $x \in \intP$, the following properties hold:
  \begin{enumerate}[label=(\alph*)]
  \item \label{item:sigma_bound} $\levvaidya_{x, i} \in [0, 1]$ for
    all $i \in [\obs]$,
\item \label{item:sigma_sum} $\sum_{i=1}^\obs\levvaidya_{x, i} =
  \dims$,
\item \label{item:theta_bound}
  $\localslack_{\vaidya_x, i} \in \brackets{0,
  \sqrt{\obs/\dims}}$ for all $i \in [\obs]$,
\item \label{item:john_sigma_bound_1}
$\johnweights_{x, i} \in [\johnbeta, 1+\johnbeta]$  for
    all $i \in [\obs]$,
\item \label{item:john_sigma_sum_1}
$\sum_{i=1}^\obs \johnweights_{x, i} = 3\dims/2$, and
\item \label{item:john_theta_bound_1}
  $\localslack_{\john_x, i} \in \brackets{0, 4}$ for
  all $i \in [\obs]$.
  \end{enumerate}
\end{lemma}
\noindent We prove this lemma in
Section~\ref{ssub:proof_of_lemma_lemma:first_bounds}.\\

Let $\proposal_x^\tagvaidya$ to denote the proposal distribution of the random walk $\VW{r}$ at state $x$.
Next, we state a lemma that shows that if two points $x, y \in \intP$ are close in
Vaidya local norm at $x$, then for a suitable choice of the parameter $r$, the
proposal distributions $\proposal^\tagvaidya_x$ and $\proposal^\tagvaidya_y$ are close.
In addition, we show that the proposals are accepted with high probability at any point $x \in \intP$.
To establish the latter result, we now define the non-lazy transition operator
of the Vaidya
walk. Since the Vaidya walk is lazy with probability $1/2$, there exists
a valid (non-lazy) transition operator $\nolazytrans_{\fulltagvaidya(\rparam)}$
such that for any distribution $\initial$,  we have
\begin{align*}
  \lazytrans_{\fulltagvaidya(\rparam)}(\initial)
= \initial/2 + \nolazytrans_{\fulltagvaidya(\rparam)}(\initial)/2.
\end{align*}
We call $\nolazytrans_{\fulltagvaidya}$ the non-lazy transition operator
for the Vaidya walk. Note that the one-step non-lazy transition distribution $\nolazytrans_{\fulltagvaidya
(\rparam)}(\diracdelta_x)$ denotes the distribution of proposals after the
accept-reject step if the chain was not lazy.
Thus to establish that proposals are accepted with high probability, it
suffices to establish that the transition distribution $\nolazytrans_{\fulltagvaidya
(\rparam)}(\diracdelta_x)$ at any point $x \in \Poly$
is close to the proposal distribution $\proposal_x^\tagvaidya$.
We now state these two results formally:
\begin{lemma}
  \label{lemma:vaidya_walk_close_kernel}
  \begin{subequations}
There exists a continuous non-decreasing function $\vaidyaradiusbound: [0, 1/4]
\rightarrow \real_+$ with \mbox{$f(\vaidyaepsilonconst) \geq \vaidyaradiusconst$} such that for any
$\epsilon \in (0, \vaidyaepsilonconst]$, the random walk \VW{r} with $r \in [0,
    \vaidyaradiusbound(\epsilon)]$ satisfies
\begin{align}
  \tvnorm{\proposal^\tagvaidya_x-\proposal^\tagvaidya_y} &\leq
  \epsilon \quad \mbox{ $\forall\ x, y \in \intP$ s.t.
    $\vecnorm{x-y}{\vaidya_x} \leq \displaystyle\frac{\epsilon
      r}{2(\obs\dims)^{1/4}}$ } ,\quad \text{and}
  \label{eq:delta_p}\\
  \tvnorm{\nolazytrans_{\fulltagvaidya(\rparam)}(\diracdelta_x)-\proposal^\tagvaidya_x} &\leq
  5\epsilon \, \,\, \mbox{$\forall\ x \in \intP$.}
  \label{eq:delta_q_p}
\end{align}
\end{subequations}
\end{lemma}
\noindent See
Section~\ref{ssub:proof_of_lemma_lemma:vaidya_walk_close_kernel} for
the proof of this lemma. \\

\vspace*{.05in}

\noindent With these lemmas in hand, we are now equipped to prove
Theorem~\ref{thm:mixing_time_bound}.
To simplify notation, for the rest of this section,
we adopt the shorthands
$\transition_x = \nolazytrans_{\fulltagvaidya(\rparam)}(\diracdelta_x)$, $\proposal_x =
\proposal_x^\tagvaidya$ and \mbox{$\vecnorm{\cdot}{\vaidya_x} = \vecnorm{\cdot}{x}$.}


\subsection{Proof of Theorem~\ref{thm:mixing_time_bound}}
\label{sub:proof_of_theorem_thm:mixing_time_bound}

In order to invoke Lov{\'a}sz's Lemma for
the random walk $\VW{\vaidyaradiusconst}$, we need to verify the
condition~\eqref{EqnLovaszHyp} for suitable choices of $\lovone$ and
$\lovtwo$.  Doing so involves two main steps:
\begin{description}
  \item[(A):] First, we relate the cross-ratio $d_\Pspace(x, y)$ to
    the local norm \eqref{eq:def_local_norm} at $x$.
  \item[(B):] Second, we use
    Lemma~\ref{lemma:vaidya_walk_close_kernel} to show that if $x, y
    \in \intP$ are close in local-norm, then the transition distributions
    $\transition_x$ and $\transition_y$ are close in TV-distance.
\end{description}

\paragraph{Step~(A):}
\label{par:step_a}
We claim that for all $x, y \in \intP$, the cross-ratio can be lower
bounded as
\begin{align}
  d_\Pspace(x, y) \geq \frac{1}{\sqrt{2\dims}} \vecnorm{x-y}{x}.
  \label{eq:hilbert_to_local}
\end{align}
Note that we have
\begin{align*}
d_\Pspace(x, y) = \frac{ \vecnorm{e(x)-e(y)}{2} \vecnorm{x-y}{2} }{
  \vecnorm{e(x)-x}{2} \vecnorm{e(y)-y}{2} } & \stackrel{(i)}{\geq}
\max \braces{ \frac{\vecnorm{x-y}{2} }{\vecnorm{e(x)-x}{2}},
  \frac{\vecnorm{x-y}{2}
  }{\vecnorm{e(y)-y}{2}}}\\
& \stackrel{(ii)}{\geq} \max \braces{ \frac{\vecnorm{x-y}{2}
  }{\vecnorm{e(x)-x}{2}}, \frac{\vecnorm{x-y}{2}
  }{\vecnorm{e(y)-x}{2}}},
\end{align*}
where step (i) follows from the inequality $\vecnorm{e(x)-e(y)}{2}
\geq \max\braces{\vecnorm{e(y)-y}{2}, \vecnorm{e(x)-x}{2}}$; and step
(ii) follows from the inequality $ \vecnorm{e(x)-x}{2} \leq
     {\vecnorm{e(y)-x}{2}}$.  Furthermore, from
     Figure~\ref{fig:cross_ratio}, we observe that
\begin{align}
  \max \braces{ \frac{\vecnorm{x-y}{2} }{\vecnorm{e(x)-x}{2}},
    \frac{\vecnorm{x-y}{2} }{\vecnorm{e(y)-x}{2}}} = \max_{i \in
    [\obs]} \left\vert\frac{a_i^\top (x - y)}{\slack_{x, i}}
  \right\vert.
  \label{eq:relate_cross_ratio}
\end{align}
This argument of equation~\eqref{eq:repeated_square} has also been used~\cite
[lemma 9]{sachdeva2016mixing}.
Note that maximum of a set of non-negative numbers is greater than the
mean of the numbers.  Combining this fact with properties
\ref{item:sigma_bound} and \ref{item:sigma_sum} from
Lemma~\ref{lemma:first_bounds}, we find that
\begin{align*}
d_\Pspace(x, y) \geq \sqrt{ \frac{1}{\sum_{i=1}^\obs
    \parenth{\levvaidya_{x, i}+ \vaidyabeta}} \sum_{i=1}^\obs
  (\levvaidya_{x, i}+ \vaidyabeta) \frac{(a_i^\top(x-y))^2}
  {\slack_{x, i}^2} } = \frac{\vecnorm{x-y}{x}}{\sqrt{2\dims}},
\end{align*}
thereby proving the claim~\eqref{eq:hilbert_to_local}.


\paragraph{Step~(B):}
\label{par:step_b}
By the triangle inequality, we have
\begin{align*}
\tvnorm{\transition_x - \transition_y} \leq
\tvnorm{\transition_x-\proposal_x} + \tvnorm{\proposal_x-\proposal_y}
+ \tvnorm{\proposal_y - \transition_y}.
\end{align*}
Thus, for any $(r, \epsilon)$ such that $\epsilon \in [0, \vaidyaepsilonconst]$ and $r
\leq \vaidyaradiusbound(\epsilon)$, Lemma~\ref{lemma:vaidya_walk_close_kernel} implies
that
\begin{align*}
\tvnorm{\transition_x - \transition_y} & \leq 11 \epsilon, \quad
\mbox{$\forall x, y \in \intP$ such that $ \vecnorm{x-y}{x} \leq
  \displaystyle \frac{r\epsilon}{2 (\obs\dims)^{1/4}}$}.
\end{align*}
\\

\noindent Consequently, the walk \VW{r} satisfies the assumptions of
Lov{\'a}sz's Lemma with
\begin{align*}
\lovtwo \defn \frac{1}{\sqrt{2\dims}} \cdot \frac{r\epsilon} {2
  (\obs\dims)^{1/4}} \quad \mbox{and} \quad \lovone \defn 1 - 11
\epsilon.
\end{align*}
Since $\vaidyaradiusbound(\vaidyaepsilonconst) \geq \vaidyaradiusconst$, we can set $\epsilon = \vaidyaepsilonconst$ and $r = \vaidyaradiusconst$, whence
\begin{align*}
\lovtwo^2\lovone^2 = \frac{(1-11\epsilon)^2 \epsilon^2 r^2}{8 \dims
  \sqrt{\obs \dims}} = \frac{4^2}{15^2}\frac{1}{15^2} \frac{1}{10^{-8}}
\cdot\frac{1}{\dims \sqrt{\obs\dims}} \geq 10^{-12}
\, \frac{1}{\dims \sqrt{\obs\dims}}.
\end{align*}
Observing that $\lovtwo <1$ yields the claimed upper bound for the
mixing time of Vaidya Walk.


\subsection{Proof of Lemma~\ref{lemma:first_bounds}}
\label{ssub:proof_of_lemma_lemma:first_bounds}

In order to prove part~\ref{item:sigma_bound}, observe that for any $x
\in \intP$, the Hessian $\hesslogbarr_x \defn \sum_{i=1}^\obs a_i
a_i^\top/\slack_{x, i}^2$ is a sum of rank one positive semidefinite
(PSD) matrices.  Also, we can write $\hesslogbarr_x = A_x\tp A_x$
where
\begin{align*}
A_x \defn
\begin{bmatrix}
  {a_1\tp}/{\slack_{x, 1}} \\ \vdots \\ {a_\obs\tp}/{\slack_{x, \obs}}
\end{bmatrix}.
\end{align*}
Since $\rank(A_x) = \dims$, we conclude that the matrix
$\hesslogbarr_x$ is invertible and thus, both the matrices
$\hesslogbarr_x$ and $\parenth{\hesslogbarr_x}^{-1}$ are PSD.  Since
$\levvaidya_{x, i} = a_i^\top \parenth{\hesslogbarr_x}^{-1} a_i /
\slack_{x, i}^2$, we have $\levvaidya_{x, i} \geq 0$.  Further, the
fact that ${a_i a_i^\top}/{\slack_{x, i}^2} \preceq \hesslogbarr_x $
implies that $\levvaidya_{x, i} \leq 1$.

Turning to the proof of part~\ref{item:sigma_sum}, from the equality
$\trace(AB) = \trace(BA)$, we obtain
\begin{align*}
\sum_{i=1}^\obs \levvaidya_{x, i} = \trace \parenth{\sum_{i=1}^\obs
  \frac{a_i^\top \parenth{\hesslogbarr_x}^{-1} a_i}{\slack_{x, i}^2}}
= \trace \parenth{\parenth{\hesslogbarr_x}^{-1} \sum_{i=1}^\obs
  \frac{a_i a_i^\top}{\slack_{x, i}^2}} = \trace (\Ind_\dims) = \dims.
\end{align*}

Now we prove part~\ref{item:theta_bound}.  Using the fact that
$\levvaidya_{x, i}\geq 0$, and an argument similar to
part~\ref{item:sigma_bound} we find that that the matrices $\vaidya_x$
and $\vaidya_x^{-1}$ are PSD.  Since $\localslackvaidya_{\vaidya_x, i} =
{a_i^\top \vaidya_x^{-1} a_i}/{\slack_{x, i}^2}$, we have
$\localslackvaidya_{\vaidya_x, i}\geq 0$.  It is straightforward to see that
$\vaidyabeta\hesslogbarr_x \preceq \vaidya_x$ which implies that
$\localslackvaidya_{\vaidya_x, i} \leq \levvaidya_{x, i}/\beta$.  Further,
we also have $\parenth{\levvaidya_{x, i}+ \vaidyabeta}\frac{a_i
  a_i^\top}{\slack_{x, i}^2} \preceq \vaidya_x$ and whence $
\localslackvaidya_{\vaidya_x, i} \leq 1/\parenth{\levvaidya_{x, i}+
  \vaidyabeta} $.  Combining the two inequalities yields the claim.

The other parts of the Lemma follow from Lemma~13, 14 and 15 by~\cite{lee2014path}
and are thereby omitted here.


\subsection{Proof of Lemma~\ref{lemma:vaidya_walk_close_kernel}}
\label{ssub:proof_of_lemma_lemma:vaidya_walk_close_kernel}

We prove the lemma for the following function
\begin{align}
  \label{eq:r_2_epsilon}
  \vaidyaradiusbound(\epsilon) \defn \min \braces{\frac{1}{20\parenth{1+\sqrt{2}\log^{\frac{1}{2}}\parenth{\frac{4}{\epsilon}}}}, \frac{\epsilon}{\sqrt{18 \log(2/\epsilon)}}, \sqrt{\frac{\epsilon}{86\sqrt{3}\tailconst_2}}, \frac{\epsilon}{22\sqrt{5/3}\tailconst_3}, \sqrt{\frac{\epsilon}{50\sqrt{105}\tailconst_4}}},
\end{align}
where $\tailconst_\degree = \parenth{2e/\degree \cdot
     \log\parenth{4/\epsilon}}^{\degree/2}$ for \mbox{$\degree = 2,
     3$ and $4$}.
A numerical calculation shows that $f(\vaidyaepsilonconst) \geq \vaidyaradiusconst$.


\subsubsection{Proof of claim~\eqref{eq:delta_p}}
\label{sssub:proof_of_claim_eq:delta_p}

In order to bound the total variation distance
$\tvnorm{\proposal_x-\proposal_y}$, we apply Pinsker's inequality,
which provides an upper bound on the TV-distance in terms of the KL
divergence:
\begin{align*}
\tvnorm{\proposal_x-\proposal_y} & \leq \sqrt{2\kldiv{\proposal_x}{
    \proposal_y}}.
\end{align*}
For Gaussian distributions, the KL divergence has a closed form
expression.  In particular, for two normal-distributions
$\mathcal{G}_1 = \NORMAL\parenth{\mu_1, \vaidyalevmatrix_1}$ and
$\mathcal{G}_2 = \NORMAL\parenth{\mu_2, \vaidyalevmatrix_2}$, the
Kullback-Leibler divergence between the two is given by
    \begin{align*}
\kldiv{\mathcal{G}_1}{\mathcal{G}_2} = \frac{1}{2} \Big({
  \trace(\vaidyalevmatrix_1^{-1/2}\vaidyalevmatrix_2\vaidyalevmatrix_1^{-1/2})
  \!-\!\dims
  \!-\!\log \det (\vaidyalevmatrix_1^{-1/2}\vaidyalevmatrix_2\vaidyalevmatrix_1^{-1/2})
  \!+\!\parenth{\mu_1\!-\!\mu_2}^\top
  \vaidyalevmatrix_1^{-1}\parenth{\mu_1\!-\!\mu_2} }\Big).
    \end{align*}
Recall from
equation~\eqref{eq:proposal_density} that the proposal distribution for
Vaidya walk is Gaussian, i.e., $\proposal_x = \NORMAL\left(x, \frac{r}{\sqrt{\obs\usedim}}\vaidya_x^{-1}\right)$.
Substituting $\mathcal{G}_1 = \proposal_x$ and $\mathcal{G}_2 = \proposal_y$
into the above expression and applying Pinsker's
inequality, we find that
\begin{align}
\tvnorm{\proposal_x-\proposal_y}^2 \leq 2 \kldiv{\proposal_y}{
  \proposal_x} & =
  \trace(\vaidya_x^{-1/2} \vaidya_y \vaidya_x^{-1/2})
  \!-\!\dims\!-\!\log\det(\vaidya_x^{-1/2}
  \vaidya_y \vaidya_x^{-1/2}) + \frac{\sqrt{\obs
    \dims}}{r^2} \vecnorm{x\!-\!y}{x}^2 \notag \\
& = \braces{\sum_{i=1}^\dims \parenth{\lambda_i - 1 + \log\frac{1}{\lambda_i}}}
+ \frac{\sqrt{\obs \dims}}{r^2} \vecnorm{x-y}{x}^2,
\label{eq:final_KL_expression}
\end{align}
where $\lambda_1, \ldots, \lambda_\dims >0$ denote the eigenvalues of
the matrix $\vaidya_x^{-1/2} \vaidya_y \vaidya_x^{-1/2}$, and we have
used the facts that $\det(\vaidya_x^{-1/2} \vaidya_y \vaidya_x^{-1/2})
= \prod_{i=1}^\dims \lambda_i$ and $\trace({\vaidya_x^{-1/2} \vaidya_y
  \vaidya_x^{-1/2}}) = \sum_{i=1}^\dims \lambda_i$.  The following
lemma is useful in bounding expression~\eqref{eq:final_KL_expression}.
\begin{lemma}
 \label{lemma:close_hessian_eigenvalues}
 For any scalar $t \in [0, 1/12]$ and any pair $x, y \in
 \intP$ such that \mbox{$\vecnorm{x-y}{x} \leq t / (\obs\dims)^{1/4}$,}
 we have
\begin{align*}
\parenth{1 - \frac{8 t}{\sqrt{\dims}}}\Ind_\dims
\preceq \vaidya_x^{-1/2} \vaidya_y \vaidya_x^{-1/2} \preceq \parenth{1
  + \frac{8 t}{\sqrt{\dims}}} \Ind_\dims,
\end{align*}
where $\preceq$ denotes ordering in the PSD cone, and $\Ind_\dims$ is
the $\dims$-dimensional identity matrix.
\end{lemma}
\noindent See
Appendix~\ref{sec:proof_of_lemma_lemma:close_hessian_eigenvalues} for
the proof of this lemma. \\

For $\epsilon \in (0, \vaidyaepsilonconst]$ and $\rparam \in \brackets{0, 1/12} $, we have $t = \epsilon r
  /2 \leq 1/12$, whence the eigenvalues $\{\lambda_i, i \in [\dims]\}$
  can be sandwiched as
  \begin{align}
  \label{eq:lambda_bounds}
  \frac{1}{2} \leq 1 - \frac{4 \epsilon r}{\sqrt{\dims}}
  \leq \lambda_i \leq 1 + \frac{4 \epsilon r}{\sqrt{\dims}} \quad \mbox{for all $i \in \dims$}.
\end{align}
We are now ready to bound the TV distance between $\proposal_x$ and
$\proposal_y$.  Using the bound~\eqref{eq:final_KL_expression} and the
inequality $\log \omega \leq \omega - 1$, valid for $ \omega > 0$, we
obtain
\begin{align*}
\tvnorm{\proposal_x-\proposal_y}^2 \leq \sum_{i=1}^\dims
\parenth{\lambda_i - 2 + \frac{1}{\lambda_i}} + \frac{\sqrt{\obs
    \dims}}{r^2} \vecnorm{x-y}{x}^2.
\end{align*}
Using the assumption that $\vecnorm{x-y}{x}\leq {\epsilon
  r}/\parenth{2(\obs\dims)^{1/4}}$, and plugging in the
bounds~\eqref{eq:lambda_bounds} for the eigenvalues $\{\lambda_i, i
\in [\dims]\}$, we find that
\begin{align*}
\sum_{i=1}^\dims \parenth{\lambda_i - 2 + \frac{1}{\lambda_i}} +
\frac{\sqrt{\obs {}\dims}}{r^2} \vecnorm{x-y}{x}^2 \leq 32
  \epsilon^2 r^2 + \frac{\epsilon^2}{4}.
\end{align*}
In asserting this inequality, we have used the facts that according to equation~\eqref{eq:lambda_bounds}, for any $i \in [\dims]$,
\begin{align*}
  \lambda_i - 2 + \frac{1}{\lambda_i}
  = \frac{\parenth{\lambda_i-1}^2}{\lambda_i}
  \leq 2 \cdot \parenth{\frac{4 \epsilon \rparam}{\sqrt{\dims}}}^2.
\end{align*}
Note that for any $\rparam \in \brackets{0, 1/12}$ we have that $32 r^2
\leq 1/2$.  Putting the pieces together yields
$\tvnorm{\proposal_x-\proposal_y} \leq \epsilon$, as claimed.


\subsubsection{Proof of claim~\eqref{eq:delta_q_p}}
\label{ssub:proof_of_claim_eq:delta_q_p}
Note that
\begin{align}
  \transition_x(\braces{x}) = \proposal_x(\Pspace^c) +
  \int_\Pspace \parenth{1 -  \min\braces{1,
    \frac{\density_z(x)}{\density_x(z)}}} \density_x(z) dz,
  \label{eq:remain_at_x}
\end{align}
where $\Pspace^c$ denotes the complement of $\Pspace$.  Consequently,
we find that
\begin{align}
 \tvnorm{\proposal_x - \transition_x} &= \frac{1}{2}
 \parenth{\transition_x(\braces{x}) + \int_{\realdim} \density_x(z) dz
   - \int_{\Pspace} \min\braces{1,
     \frac{\density_z(x)}{\density_x(z)}} \density_x(z) dz} \notag\\
 & = \frac{1}{2} \parenth{2 - 2
   \int_{\realdim} \min\braces{1, \frac{\density_z(x)}{\density_x(z)}}
   \density_x(z) dz + 2 \int_{\Pspace^c} \min\braces{1,
     \frac{\density_z(x)}{\density_x(z)}} \density_x(z) dz
 }\notag \\
 &\leq \underbrace{\proposal_x(\Pspace^c)}_{= : \;
   \MyTerm_1} + \underbrace{1-\Exs_{z \sim \proposal_x}
   \brackets{\min\braces{1, \frac{\density_z(x)}{\density_x(z)}}}}_{=
   : \; \MyTerm_2},
  \label{eq:bound_on_jumping_out}
\end{align}
Consequently, it suffices to show that both $\MyTerm_1$ and $\MyTerm_2$ are small,
where the probability is taken over the randomness in the proposal
$z$.  In particular, we show that $\MyTerm_1 \leq \epsilon$ and $\MyTerm_2 \leq
4\epsilon$.  \\

\noindent \textbf{Bounding the term $\MyTerm_1$:} Since $z$ is
multivariate Gaussian with mean $x$ and covariance
$\frac{r^2}{\sqrt{\obs \dims}} \vaidya_x^{-1}$, we can write
\begin{align}
 z \stackrel{d}{=} x + \frac{r}{(\obs\dims)^{1/4}}
 \vaidya_x^{-{1}/{2}} \rvg,
    \label{eq:z_x_relation}
\end{align}
where $\rvg \sim \NORMAL\parenth{0, \Ind_\dims}$ and $\stackrel{d}{=}$
denotes equality in distribution.  Using
equation~\eqref{eq:z_x_relation} and
definition~\eqref{eq:theta_Definition} of $\localslackvaidya_{\vaidya_x,
  i}$, we obtain the bound
\begin{align}
 \frac{\parenth{a_i^\top \parenth{z-x}}^2 }{\slack_{x, i}^2} & =
 \frac{r^2}{\parenth{\obs \dims}^{\frac{1}{2}}}
 \brackets{\frac{a_i^\top \vaidya_x^{-{1}/{2}} \rvg }{\slack_{x,
       i}}}^2 \stackrel{(i)}{\leq} \frac{r^2}{\parenth{\obs
     \dims}^{\frac{1}{2}}} \localslackvaidya_{\vaidya_x, i}
 \vecnorm{\rvg}{2}^2 \stackrel{(ii)}{\leq} \frac{r^2}{\dims}
 \vecnorm{\rvg}{2}^2,
 \label{eq:bound_on_deviation}
\end{align}
where step (i) follows from Cauchy-Schwarz inequality, and step (ii)
from the bound on $\localslackvaidya_{\vaidya_x, i}$ from
Lemma~\ref{lemma:first_bounds}\ref{item:theta_bound}.  Define the
events
\begin{align*}
  \mathcal{E} \defn \braces{\frac{r^2}{\dims} \vecnorm{\rvg}{2}^2 < 1}
  \quad \text{and} \quad
  \mathcal{E}'\defn \braces{z \in \intP}.
\end{align*}
Inequality~\eqref{eq:bound_on_deviation} implies that $\mathcal{E}
\subseteq \mathcal{E}'$ and hence $\Prob\brackets{\mathcal{E}'} \geq
\Prob\brackets{\mathcal{E}}$.  Using a standard Gaussian tail bound
and noting that $r \leq \frac{1}{1 + \sqrt{2/\dims \log
    (1/\epsilon)}}$, we obtain $ \Prob\brackets{\mathcal{E}} \geq
1-\epsilon$ and whence $ \Prob\brackets{\mathcal{E}'} \geq
1-\epsilon$.  Thus, we have shown that $\Prob\brackets{z
  \notin \Pspace} \leq \epsilon$ which implies that $\MyTerm_1 \leq
\epsilon$.  \\

\noindent \textbf{Bounding the term $\MyTerm_2$:}
By Markov's inequality, we have
\begin{align}
  \Exs_{z \sim \proposal_x} \brackets{\min\braces{1,
            \frac{\density_z(x)}{\density_x(z)}}} \geq \alpha
        \Prob\brackets{\density_z(x) \geq \alpha \density_x(z)} \quad
        \mbox{for all $\alpha \in (0,
          1]$}. \label{eq:markov_inequality}
\end{align}
By definition~\eqref{eq:proposal_density} of $\density_x$, we obtain
\begin{align*}
\frac{\density_z(x)}{\density_x(z)} = \exp
\parenth{-\frac{\sqrt{\obs\dims}}{2r^2} \parenth{\vecnorm{z-x}{z}^2-
    \vecnorm{z-x}{x}^2 }+ \frac{1}{2} \parenth{\log \det \vaidya_z -
    \log \det \vaidya_x}}.
\end{align*}
The following lemma provides us with useful bounds on the two terms in
this expression, valid for any $x \in \intP$.
 \begin{lemma}
 \label{lemma:change_in_log_det_and_local_norm}
 For any $\epsilon \in (0, \vaidyaepsilonconst]$ and $r \in (0, \vaidyaradiusbound(\epsilon)]$,
     we have
 \begin{subequations}
   \begin{align}
     \label{eq:vaidya_whp_log_det_filter}
     \Prob_{z\sim \proposal_x}\brackets{\frac{1}{2}\log\det \vaidya_z
       - \frac{1}{2}\log\det \vaidya_x \geq -\epsilon } &\geq 1 -
     \epsilon,
     \quad  \text{and}
     \\
\label{eq:vaidya_whp_local_norm}
\Prob_{z\sim \proposal_x}\brackets{\vecnorm{z - x}{z}^2 - \vecnorm{z -
    x}{x}^2 \leq 2 \epsilon \frac{r^2}{\sqrt{\numobs\usedim}} } &\geq
1 - \epsilon.
     \end{align}
   \end{subequations}
 \end{lemma}
 \noindent See
 Appendix~\ref{sec:proof_of_lemma_lemma:change_in_log_det_and_local_norm}
 for the proof of this claim.\\

 Using Lemma~\ref{lemma:change_in_log_det_and_local_norm}, we now
 complete the proof.  For $r \leq \vaidyaradiusbound(\epsilon)$, we obtain
 \begin{align*}
 \frac{\density_z(x)}{\density_x(z)} & \geq \exp\parenth{-2\epsilon}
 \geq 1- 2 \epsilon
 \end{align*}
 with probability at least $1-2\epsilon$.  Substituting $\alpha = 1- 2
 \epsilon$ in inequality~\eqref{eq:markov_inequality} yields that
 $\MyTerm_2 \leq 4 \epsilon$, as claimed.

{\comment{
 \subsection{Mixing time of the John walk}
\label{sub:proof_john_walk}

Here we provide a brief outline the key steps in the analysis of John
walk, leaving the details to the supplementary material.  Proof of
Theorem~\ref{thm:john_walk} is also decomposed in two steps analogous
to the step (A) and (B) in
section~\ref{sub:proof_of_theorem_thm:mixing_time_bound}.  From
Lemma~\ref{lemma:first_bounds}, we can see that while the sum of
weights $\johnweights_x$ is within a factor of $3/2$ to the sum of
leverage scores $\levvaidya_x$, the John slack sensitivity
$\localslackvaidya_{\john_x, i}$~\eqref{eq:johntheta} is bounded by a
constant as compared to the $\sqrt{\obs/\dims}$ bound for Vaidya slack
sensitivity $\localslackvaidya_{\vaidya_x, i}$~\eqref{eq:theta_Definition}.
Along the lines of step (A) in
Section~\ref{sub:proof_of_theorem_thm:mixing_time_bound}, the former
property directly establishes that $d_\Pspace(x, y) \geq
\vecnorm{x-y}{\john_x}/{\sqrt{3\dims/2}}$.  To complete the proof, we
establish an analog of Lemma~\ref{lemma:vaidya_walk_close_kernel} for
the John walk which can then be used to derive a bound for its mixing
time.  We believe that our analysis while deriving the
Lemma~\ref{lemma:vaidya_walk_close_kernel} analog for the John walk is
loose and possibly a tighter analysis would allow us to prove
Conjecture~\ref{conj:john_walk}.\footnote{\mjwcomment{Conjecture
    details To be written} -- \rrdcomment{Done.}}
}
}
\section{Discussion}
\label{sec:discussion}

In this paper, we focused on improving mixing rate of MCMC sampling algorithms for polytopes by building on the advancements in the field of interior
point methods.  We proposed and analyzed two different barrier based
MCMC sampling algorithms for polytopes that outperforms the existing
sampling algorithms like the ball walk, the hit-and-run and the Dikin
walk for a large class of polytopes.  We provably demonstrated the
fast mixing of the Vaidya walk, $\order{\obs^{0.5}\dims^{1.5}}$ and the John walk, $\order{\dims^{2.5}\polylog(\obs/\dims)}$ from a warm start.  Our numerical
experiments, albeit simple, corroborated with our theoretical claims:
the Vaidya walk mixes at least as fast the Dikin walk and
significantly faster when the number of constraints is quite large
compared to the dimension of the underlying space.  For the John walk,
the logarithmic factors were dominant in all our experiments and
thereby we deemed the result of importance only for set-ups with
polytopes in very high dimensions with number of constraints
overwhelmingly larger than the dimensions.  Besides, proving the
mixing time guarantees for the improved John walk
(Conjecture~\ref{conj:john_walk}) is still an open question.

\cite{narayanan2016randomized} analyzed a generalized version
of the Dikin walk for arbitrary convex sets equipped with self-concordant
barrier. From his results, we were able to derive mixing time bounds of
$\order{\obs\dims^4}$ and $\order{\dims^5\polylog(\obs/\dims)}$ from a warm start
for the Vaidya walk and the John walk respectively. Our proof takes advantage of the specific structure of the Vaidya and John walk, resulting a better mixing rate upper bound the the general analysis provided by \cite{narayanan2016randomized}.


While our paper has mainly focused on sampling algorithms on polytopes, the idea of using logarithmic barrier to guide sampling can be extended to more general convex sets. The self-concordance property of the logarithmic barrier for polytopes is extended by \cite{anstreicher2000volumetric} to more general convex
sets defined by semidefinite constraints, namely, linear matrix
inequality (LMI) constraints.  Moreover,
\cite{narayanan2016randomized} showed that for a convex set
in $\realdim$ defined by $\obs$ LMI constraints and equipped with the
log-determinant barrier---the semidefinite analog of the logarithmic
barrier for polytopes---the mixing time of the Dikin walk from a warm
start is $\order{\obs\dims^2}$.  It is possible that an appropriate
Vaidya walk on such sets would have a speed-up over the Dikin walk.
\cite{narayanan2013efficient} used the Dikin
walk to generate samples from time varying log-concave distributions
with appropriate scaling of the radius for different class of
distributions.  We believe that suitable adaptations of the Vaidya and
John walks for such cases would provide significant gains.

{\comment{
  Another possible extension is a new random walk on Riemannian
  manifolds based on the matrix $\vaidya_x$, in contrast to the geodesic
  walk~\cite{lee2016geodesic} where the manifold is based on the Hessian
  $\hesslogbarr_x$.  In contrast to the Dikin walk's $\order{\obs\dims}$
  mixing time, the geodesic walk has an \order{\obs\dims^{3/4}} mixing
  time.  It would be interesting to see whether a geodesic version of
  the Vaidya walk has a convergence rate of
  $\order{\obs^{1/2}\dims^{5/4}}$.  A random walk based on the Hessian
  sketch~\cite{pilanci2015newton,pilanci2015randomized} is another
  possible line of work, that can reduce the per iteration complexity
  and thereby speed up the random walk.

  Our work can be seen as another step to unify optimization and
  sampling algorithms and provide theoretical insights to inform the
  practice of MCMC algorithms.
  }
}


\subsection*{Acknowledgements}

This research was supported by Office of Naval Research grant DOD
ONR-N00014 to MJW and in part by ARO grant W911NF1710005, NSF-DMS 1613002, the
Center for Science of Information (CSoI), a US NSF Science and
Technology Center, under grant agreement CCF-0939370 and the Miller Professorship (2016-2017) at UC Berkeley to BY.  In
addition, MJW was partially supported by National Science Foundation
grant NSF-DMS-1612948 and RD was partially supported by the Berkeley Fellowship.

\newpage

\appendix
\etoctocstyle{1}{Appendix}
\etocdepthtag.toc{mtappendix}
\etocsettagdepth{mtchapter}{none}
\etocsettagdepth{mtappendix}{subsection}
\tableofcontents

\section{Auxiliary results for the Vaidya walk}
\label{sec:technical_results_and_useful_properties}

In this appendix, we first summarize a few notations used in the proofs related to Theorem~\ref{thm:mixing_time_bound},
and collect the auxiliary results for the later proofs.


\subsection{Notation}

We begin with introducing the notation.  Recall $A \in \real^{\obs
  \times\dims}$ is a matrix with $a_i^\top$ as its $i$-th row.
For any positive integer $p$ and any vector $v =(v_1, \ldots,
v_p)^\top$, \mbox{$\diag(v) = \diag(v_1, \ldots, v_p)$} denotes a $p
\times p$ diagonal matrix with the $i$-th diagonal entry equal to
$v_i$.  Recall the definition of $\slackmatrix_x$:
\begin{align}
  \label{eq:\slackmatrix_definition}
\slackmatrix_x &= \diag\parenth{\slack_{x, 1}, \ldots, \slack_{x,
    \obs}} \text{ where } \slack_{x, i} = b_i - a_i^\top x\, \mbox{
  for each $i\in[\obs]$}.
\end{align}
Furthermore, define $A_x = \slackmatrix_x^{-1}A$ for all $x \in
\intP$, and let $\vaidyaprojmatrix_x$ denote the projection matrix for the
column space of $A_x$, i.e.,
\begin{align}
  \label{eq:projection_matrix}
    \vaidyaprojmatrix_x \defn A_x (A_x^\top A_x)^{-1} A_x^\top =
    A_x\hesslogbarr_x^{-1} A_x\tp.
\end{align}
Note that for the scores $\levvaidya_x$~\eqref{eq:defn_lev_scores}, we
have $\levvaidya_{x, i} = (\vaidyaprojmatrix_x)_{ii}$ for each $i \in
[\obs]$.  Let $\vaidyalevmatrix_{x}$ be an $\obs \times \obs$ diagonal
matrix defined as
\begin{align}
    \label{eq:Leverage_definition}
    \vaidyalevmatrix_{x} &= \diag\parenth{\levvaidya_{x, 1}, \ldots,
      \levvaidya_{x, \obs}}.
\end{align}
Let $\levvaidya_{x, i, j} \defn (\vaidyaprojmatrix_x)_{ij}$, and let
$\vaidyaprojmatrix_x^{(2)}$ denote the Hadamard product of $\vaidyaprojmatrix_x$
with itself, i.e.,
\begin{align}
 (\vaidyaprojmatrix_x^{(2)})_{ij} = \levvaidya^2_{x, i, j} =
  \frac{\parenth{a_i^\top \hesslogbarr_x^{-1}a_j}^2}{\slack_{x,
      i}^2\slack_{x, j}^2}\quad \mbox{ for all $i, j \in[\obs]$}.
\end{align}
Using the shorthand $\localslackvaidya_x \defn
\localslackvaidya_{\vaidya_x}$, we define
\begin{align*}
\vaidyalocalslackmatrix_x & \defn \diag\parenth{\localslackvaidya_{x, 1}, \ldots,
  \localslackvaidya_{x, m}} \, \text{ where }\, \localslackvaidya_{x,
  i} = \frac{a_i^\top \vaidya_x^{-1}a_i}{\slack_{x, i}^2} \quad \text{
  for } i\in[\obs], \text{ and } \\
\vaidyasecondthetamatrix_x &\defn (\localslackvaidya^2_{x, i, j})\, \text{
  where }\, \localslackvaidya^2_{x, i, j} = \frac{\parenth{a_i^\top
    \vaidya_x^{-1}a_j}^2}{\slack_{x, i}^2\slack_{x, j}^2} \quad \text{
  for } i, j \in[\obs].
\end{align*}
In our new notation, we can re-write the Vaidya matrix $\vaidya_x$ defined in equation~\eqref{eq:vaidya_covariance} as $\vaidya_x = A_x\tp
\parenth{\vaidyalevmatrix_x + \vaidyabeta \Ind } A_x$, where $\vaidyabeta
= \dims/\obs$.

\subsection{Basic Properties}
\label{sub:basic_properties}

We begin by summarizing some key properties of various terms involved
in our analysis.
\begin{lemma}
  \label{ppt:all_properties}
  For any vector $x \in \intP$, the following properties hold:
  \begin{enumerate}[label=(\alph*)]
        \item\label{item:sigma_properties} $\levvaidya_{x, i} =
          \sum_{j=1}^\obs \levvaidya^2_{x, i, j} = \sum_{j, k=1}^\obs
          \levvaidya_{x, i, j} \levvaidya_{x, j, k} \levvaidya_{x, k,
            i}$ for each $i \in [\obs]$,
\item\label{item:Sigma_dominates_P} $\vaidyalevmatrix_x
          \succeq \vaidyaprojmatrix_x^{(2)}$,
\item\label{item:theta_sum} $\sum_{i=1}^\obs \localslackvaidya_{x, i}
  \parenth{\levvaidya_{x, i}+\vaidyabeta} = \dims$,
\item\label{item:theta_properties} $\forall i \in [\obs],
  \ \localslackvaidya_{x, i} = \sum_{j=1}^\obs \parenth{\levvaidya_{x,
      j}+\vaidyabeta} \localslackvaidya_{x, i, j}^2$, for each $i \in
                     [\obs]$,
\item\label{item:theta_square_sigma_sum_bound}
  $\localslackvaidya_x^\top \parenth{\vaidyalevmatrix_x +
  \vaidyabeta\Ind} \localslackvaidya_x = \sum_{i=1}^\obs
  \localslackvaidya_{x, i}^2 \parenth{\levvaidya_{x, i}+\vaidyabeta}
  \leq \sqrt{\obs\dims}$, and
\item \label{item:W_sandwich} $\displaystyle\vaidyabeta\,
  \hesslogbarr_x \preceq \vaidya_x \preceq \parenth{1 + \vaidyabeta}
  \hesslogbarr_x$.
    \end{enumerate}
    where $\vaidyabeta=\dims/\obs$ was defined in equation~\eqref{eq:defn_lev_scores}.
\end{lemma}
\begin{proof}    We prove each property separately.

\paragraph*{Part~\ref{item:sigma_properties}:}
Using $\Ind_\dims = \hesslogbarr_x \parenth{\hesslogbarr_x}^{-1}$, we
find that
\begin{align*}
  \levvaidya_{x, i} = \frac{a_i^\top \parenth{\hesslogbarr_x}^{-1}
    \hesslogbarr_x \parenth{\hesslogbarr_x}^{-1} a_i}{\slack_{x, i}^2}
  = \frac{a_i^\top \parenth{\hesslogbarr_x}^{-1} \nabla^2
    \sum_{j=1}^\obs \frac{a_j^\top a_j}{\slack_{x, j}^2}
    \parenth{\hesslogbarr_x}^{-1} a_i}{\slack_{x, i}^2} = \sum_{i,
    j=1}^\obs \levvaidya_{x, i, j}^2.
\end{align*}
Applying a similar trick twice and performing some algebra, we obtain
\begin{align*}
  \levvaidya_{x, i} & = \frac{a_i^\top
    \parenth{\hesslogbarr_x}^{-1} \hesslogbarr_x
    \parenth{\hesslogbarr_x}^{-1} \hesslogbarr_x
    \parenth{\hesslogbarr_x}^{-1} a_i}{\slack_{x, i}^2} =
  \sum_{i, j, k=1}^\obs \levvaidya_{x, i, j} \levvaidya_{x, j,
    k} \levvaidya_{x, k, i}.
\end{align*}

    \paragraph*{Part~\ref{item:Sigma_dominates_P}:}
    From part~\ref{item:sigma_properties}, we have that
    $\vaidyalevmatrix_x - \vaidyaprojmatrix_x^{(2)}$ is a symmetric and
    diagonally dominant matrix with non-negative entries on the
    diagonal. Applying Gershgorin's theorem~\citep{bhatia2013matrix,
      horn2012matrix}, we conclude that it is PSD.


\paragraph*{Part~\ref{item:theta_sum}:}

Since $\trace(AB) = \trace(BA)$, we have
\begin{align*}
  \sum_{i=1}^\obs  \localslackvaidya_{x, i}\parenth{\levvaidya_{x, i}+\vaidyabeta}
  = \trace \parenth{\vaidya_x^{-1}  \sum_{i=1}^\obs  \parenth{\levvaidya_{x, i}+\vaidyabeta} \frac{a_i a_i^\top}{\slack_{x, i}^2}} = \trace \parenth{\Ind_\dims} = \dims.
\end{align*}

\paragraph*{Part~\ref{item:theta_properties}:}
An argument similar to part~\ref{item:sigma_properties} implies
that
\begin{align*}
  \localslackvaidya_{x, i} &
  = \frac{a_i^\top \vaidya_x^{-1} \vaidya_x \vaidya_x^{-1} a_i}{\slack_{x, i}^2}
  = \frac{a_i^\top \vaidya_x^{-1} \sum_{j=1}^\obs  \parenth{\levvaidya_{x, i} + \vaidyabeta} \frac{a_j^\top a_j}{\slack_{x, j}^2} \vaidya_x^{-1} a_i}{\slack_{x, i}^2}
  = \sum_{i, j=1}^\obs  \parenth{\levvaidya_{x, i} + \vaidyabeta} \localslackvaidya_{x, i, j}^2.
\end{align*}

\paragraph{Part~\ref{item:theta_square_sigma_sum_bound}:}
Using part~\ref{item:theta_sum} and Lemma~\ref{lemma:first_bounds}\ref{item:theta_bound} yields the claim.

\paragraph*{Part~\ref{item:W_sandwich}:}
The left inequality is by the definition of $\vaidya_x$. The right
inequality uses the fact that $\vaidyalevmatrix_x \preceq
\Ind_\dims$.
\end{proof}

We now prove an important result that relates the \emph{slackness} $s_x$ and $s_y$ at two points, in terms of $ \vecnorm{x - y}{x}$.
\begin{lemma}
  \label{lemma:closeness_of_slackness}
  For all $x, y \in \intP$, we have
  \begin{align*}
    \abss{1 - \frac{\slack_{y, i}}{\slack_{x, i}}} \leq
    \parenth{\frac{\obs }{\dims}}^{\frac{1}{4}} \vecnorm{x - y}{x}
    \quad \mbox{for each $i \in [\obs]$}.
  \end{align*}
\end{lemma}

\begin{proof}
  For any pair $x, y \in \intP$ and index $i \in [\obs]$, we have
 \begin{align*}
   \parenth{a_i^\top \parenth{x-y}}^2 = \parenth{
     (\vaidya_x^{-\frac{1}{2}} a_i)^\top \vaidya_x^{\frac{1}{2}}
     (x-y)}^2
   & \stackrel{(i)}{\leq} \| \vaidya_x^{-\frac{1}{2}} a_i\|^2_2 \; \|
   \vaidya_x^{\frac{1}{2}} (x-y) \|^2_2\\
   & = a_i^T \vaidya_x^{-1} a_i \;  \vecnorm{x - y}{x}^2\\
   & = \localslackvaidya_{x, i} \slack_{x, i}^2 \; \vecnorm{x -
     y}{x}^2\\
   & \stackrel{(ii)}{\leq} \sqrt{\frac{\obs }{\dims}} \slack_{x, i}^2
   \; \vecnorm{x - y}{x}^2,
 \end{align*}
 where step (i) follows from the Cauchy-Schwarz inequality, and step
 (ii) uses the bound $\localslackvaidya_{x, i}$ from
 Lemma~\ref{lemma:first_bounds}\ref{item:theta_bound}.  Noting the
 fact that $a_i\tp(x-y) = \slack_{y, i} - \slack_{x, i} $, the claim
 follows after simple algebra.
\end{proof}


\section{Proof of Lemma~\ref{lemma:close_hessian_eigenvalues}}
\label{sec:proof_of_lemma_lemma:close_hessian_eigenvalues}

In this appendix section, we prove Lemma~\ref{lemma:close_hessian_eigenvalues}
using results from the previous appendix.  As a direct consequence of
Lemma~\ref{lemma:closeness_of_slackness}, we find that
\begin{align*}
  \abss{1 - \frac{\slack_{y, i}}{\slack_{x, i}}} \leq
  \frac{t}{\sqrt{\dims}}, \quad \mbox{for any $x, y \in \intP$ such
    that $\vecnorm{x-y}{x} \leq \displaystyle
    \frac{t}{(\obs\dims)^{1/4}}$.}
\end{align*}
The Hessian $\hesslogbarr_y$ is thus sandwiched in terms of the
Hessian $\hesslogbarr_x$ as
\begin{align*}
\parenth{1 - \frac{t}{\sqrt{\dims}}}^2 \hesslogbarr_x \preceq
\hesslogbarr_y \preceq \parenth{1 + \frac{t}{\sqrt{\dims}}}^2
\hesslogbarr_x.
\end{align*}
By the definition of $\levvaidya_{x, i}$ and $\levvaidya_{y, i}$, we
have
\begin{align}
\frac{\parenth{1- \frac{t}{\sqrt{\dims}}}^2}{\parenth{1+
    \frac{t}{\sqrt{\dims}}}^2} \levvaidya_{x, i} \leq \levvaidya_{y,
  i} \leq \frac{\parenth{1+ \frac{t}{\sqrt{\dims}}}^2}{\parenth{1-
    \frac{t}{\sqrt{\dims}}}^2} \levvaidya_{x, i} \quad \mbox{for all
  $i \in [\obs]$}.
\label{eq:sigma_closeness}
\end{align}
Consequently, we find that
\begin{align*}
\frac{\parenth{1- \frac{t}{\sqrt{\dims}}}^2}{\parenth{1+
    \frac{t}{\sqrt{\dims}}}^4} \vaidya_x \preceq \vaidya_y \preceq
\frac{\parenth{1+ \frac{t}{\sqrt{\dims}}}^2}{\parenth{1-
    \frac{t}{\sqrt{\dims}}}^4} \vaidya_x.
\end{align*}
Note that
\begin{align*}
\frac{\parenth{1-\omega}^2}{\parenth{1+\omega}^4}
\geq 1 - 8\omega \quad
\text{ and } \quad
\frac{\parenth{1+\omega}^2}{\parenth{1-\omega}^4}
\leq 1 + 8\omega \quad \mbox{for any $\omega \in \brackets{0,\frac{1}{12}}$}.
\end{align*}
Applying this sandwiching pair of inequalities with $\omega =
t/\sqrt{\dims}$ yields the claim.


\section{Proof of Lemma~\ref{lemma:change_in_log_det_and_local_norm}}
\label{sec:proof_of_lemma_lemma:change_in_log_det_and_local_norm}

We begin by defining
\begin{align}
\vaidyaphi_{x,i} \defn \displaystyle\frac{\levvaidya_{x, i} +
  \vaidyabeta}{\slack_{x, i}^2} \text{ for } i \in [\obs], \quad
\text{and} \quad \logdetvaidya_x \defn \frac{1}{2} \,\log \det
\vaidya_x,\quad \text{ for all } x \in \intP.
\end{align}
Further, for any two points $x$ and $z$, let $\overline{xz}$ denote
the set of points on the line segment joining $x$ and $z$.  The proof
of Lemma~\ref{lemma:change_in_log_det_and_local_norm} is based on a
Taylor series expansion, and so requires careful handling of
$\levvaidya, \vaidyaphi, \logdetvaidya$ and their derivatives.  At a high level,
the proof involves the following steps: (1) perform a Taylor series
expansion around $x$ and along the line segment $\overline{xz}$; (2)
transfer the bounds of terms involving some point $y \in
\overline{xz}$ to terms involving only $x$ and $z$; and then (3) use
concentration of Gaussian polynomials to obtain high probability
bounds.

\subsection{Auxiliary results for the proof of Lemma~\ref{lemma:change_in_log_det_and_local_norm}}
\label{sub:auxiliary_results_for_the_proof_of_lemma_lemma:change_in_log_det_and_local_norm}

We now introduce some auxiliary results involved in these three steps.
The following lemma provides expressions for gradients of $\levvaidya,
\vaidyaphi$ and $\logdetvaidya$ and bounds for directional Hessian of $\vaidyaphi$
and $\logdetvaidya$.  Let ${e_i} \in \realdim$ denote a vector with $1$ in
the $i$-th position and $0$ otherwise.  For any $h \in \realdim$ and
$x \in \intP$, define $\projD_{x, h, i}=\projD_{x, i} \defn a_i^\top h
/ \slack_{x, i}$ for each $i \in [\obs]$.
\begin{lemma}
\label{lemma:gradient_and_hessian_and_bounds}
The following relations hold;
    \begin{enumerate}[label=(\alph*)]
 \item\label{item:gradient_sigma} Gradient of $\levvaidya$: $\nabla
   \levvaidya_{x, i} = 2 A_x\tp (\vaidyalevmatrix_x-\vaidyaprojmatrix_x^{(2)}) {e_i}$
   for each $i \in [\obs]$.
 \item\label{item:gradient_beta} Gradient of $\vaidyaphi$: $\nabla
   \vaidyaphi_{x, i} = \displaystyle\frac{2}{\slack_{x, i}^2} A_x\tp
   \brackets{2\vaidyalevmatrix_x+\vaidyabeta\, \Ind-\vaidyaprojmatrix_x^{(2)}} e_i$
   for each $i \in [\obs]$;
 \item \label{item:gradient_L} Gradient of $\logdetvaidya$: $\gradlogdetV_x =
   A_x\tp \parenth{ 2 \, \vaidyalevmatrix_x + \displaystyle\vaidyabeta
     \, \Ind - \vaidyaprojmatrix_x^{(2)} }\localslackvaidya_x$;
\item \label{item:beta_hessian_bound} Bound on $\nabla^2\vaidyaphi$:
  $\slack_{x, i}^2 \abss{\frac{1}{2} h^\top \nabla^2\vaidyaphi_{x, i} h}
  \leq 14\,\parenth{\levvaidya_{x, i}+\vaidyabeta}\projD_{x, i}^2 +
  11\,\sum_{j=1}^\obs \levvaidya_{x, i, j}^2 \projD_{x, j}^2$ for $i
  \in [\obs]$;
\item\label{item:hessian_log_det_bound} Bound on $\hesslogdetV$: $
  \abss{\frac{1}{2}h^\top \parenth{\hesslogdetV_x} h} \leq 13\,
  \sum_{i=1}^\obs \parenth{\levvaidya_{x, i}+\vaidyabeta}
  \localslackvaidya_{x, i} \projD_{x, i}^2 + \frac{17}{2}\, \sum_{i,
    j=1}^\obs \levvaidya_{x, i, j}^2 \localslackvaidya_{x, i}
  \projD_{x, j}^2$.
\end{enumerate}
\end{lemma}
\noindent See Section~\ref{sub:proofs_of_derivatives} for the proof of
this claim. \\

The following lemma that shows that for a random variable $z \sim
\proposal_x$, the slackness $\slack_{z, i}$ is close to $\slack_{x,
  i}$ with high probability.
\begin{lemma}
\label{lemma:whp_slackness}
For any $\epsilon \in (0, 1/4], r \in (0, 1)$ and $x \in \intP$, we
  have
  \begin{align*}
    \Prob_{z \sim \proposal_x} \brackets{ \forall i \in [\obs],
      \forall v \in \overline{xz},\ {\frac{\slack_{x, i}}{\slack_{v,
            i}}} \in \parenth{1-\rparam\parenth{1+\delta},
        1+\rparam\parenth{1+\delta}} } \geq 1 - \epsilon/4,
    \end{align*}
  where $\delta = \sqrt{\frac{2\log(4/\epsilon)}{\dims}}$.  Thus for any $\dims \geq 1$ and
    $\rparam \leq 1/\brackets{{20\parenth{1 + \sqrt{2
    \log\parenth{\frac{4}{\epsilon}}}}}}$, we have
    \begin{align*}
        \Prob_{z \sim \proposal_x} \brackets{ \forall i \in [\obs],
          \forall v \in \overline{xz},\ {\frac{\slack_{x,
                i}}{\slack_{v, i}}} \in \parenth{0.95, 1.05}}\geq 1 -
        \epsilon/4.
    \end{align*}
\end{lemma}
\noindent See
Section~\ref{sub:proof_of_lemma_lemma:whp_slackness} for the proof which
is based on combining the bound on $\frac{\slack_{x, i}}{\slack_{v,i}}$
from Lemma~\ref{lemma:closeness_of_slackness} with standard Gaussian tail bounds.

This result comes in handy for transferring bounds for different
expressions in Taylor expansion involving an arbitrary $y$ on
$\overline{xz}$ to bounds on terms involving simply $x$.  The proof
follows from Lemma~\ref{lemma:closeness_of_slackness} and a simple
application of the standard Gaussian tail bounds and is thereby
omitted.  For brevity, we define the shorthand
\begin{align}
    \hat{a}_{x,i} = \frac{1}{\slack_{x, i}} \vaidya_x^{-{1}/{2}} a_i \quad
    \mbox{for each $i \in [\obs]$}.
    \label{eq:hat_a_i}
\end{align}
In the following lemma, we state some tail bounds for particular Gaussian
polynomials that arise in our analysis.

\begin{lemma}
 \label{lemma:gaussian_moment_bounds}
 For any $\epsilon \in (0, \vaidyaepsilonconst]$, define $\tailconst_\degree = \parenth{2e/\degree \cdot
     \log\parenth{4/\epsilon}}^{\degree/2}$ for \mbox{$\degree = 2,
     3$ and $4$}.  Then for $\rvg \sim \NORMAL(0, \Ind_\dims) $ and
   any $x \in \intP$ the following high probability bounds hold:
\begin{subequations}
\begin{align}
  \Prob\brackets{\sum_{i=1}^\obs \parenth{\levvaidya_{x, i} +
      \vaidyabeta} \parenth{\hat{a}_{x,i}^\top \rvg}^2 \leq \tailconst_2 \sqrt{3}\dims } &\geq 1-\frac{\epsilon}{4},
  \label{eq:quadratic} \\
  \Prob\brackets{ \abss{\sum_{i=1}^\obs
      \parenth{\levvaidya_{x, i}+\vaidyabeta}
      \parenth{\hat{a}_{x,i}^\top \rvg}^3} \leq
    \tailconst_3 \sqrt{15} \parenth{\obs\dims}^{1/4} }
  &\geq 1-\frac{\epsilon}{4},
  \label{eq:cube1}\\
\Prob\brackets{\abss{\sum_{i, j=1}^\obs \levvaidya_{x, i, j}^2
    \bigg(\bigg(\frac{\hat{a}_{x,i} + \hat{a}_{x,j}}{2}\bigg)^\top
    \rvg\bigg)^3} \leq \tailconst_3 \sqrt{15} \parenth{\obs\dims}^{1/4} }
  &\geq 1-\frac{\epsilon}{4},
\label{eq:cube2}\\
\Prob\brackets{\sum_{i=1}^\obs \parenth{\levvaidya_{x, i} +
    \vaidyabeta} \parenth{ \hat{a}_{x,i}^\top \rvg}^4\leq
  \tailconst_4 \sqrt{105} \parenth{\obs\dims}^{1/2}}&\geq 1-\frac{\epsilon}{4}.
\label{eq:fourth}
\end{align}
\end{subequations}
\end{lemma}
\noindent See
Section~\ref{sub:proof_of_lemma_lemma:gaussian_moment_bounds} for the
proof of these claims.

Now we summarize the final ingredients needed for our proofs.  Recall
that the Gaussian proposal $z$ is related to the current state $x$ via
the equation
\begin{align}
 \label{eq:z_x_relation2}
  z \stackrel{d}{=} x + \frac{r}{(\obs\dims)^{1/4}}
  \vaidya_x^{-{1}/{2}} \rvg,
\end{align}
where $\rvg \sim \NORMAL\parenth{0, \Ind_\dims}$.  We also use the
following elementary inequalities:
\begin{align}
&\text{Cauchy-Schwarz inequality:} &\vert u^\top v\vert &\leq
  \vecnorm{u}{2} \vecnorm{v}{2} \tag{C-S}\\
&\text{AM-GM inequality:} &\inverserate \kappa &\leq \frac{1}{2}(\inverserate^2 +
  \kappa^2). \tag{AM-GM} \\
&\text{Sum of squares inequality:} & \frac{1}{2} \vecnorm{a+b}{2}^2 &\leq
    \vecnorm{a}{2}^2 + \vecnorm{b}{2}^2 , \tag{SSI}
\end{align}
Note that the sum-of-squares inequality is simply a vectorized version
of the AM-GM inequality.  With these tools, we turn to the proof of
Lemma~\ref{lemma:change_in_log_det_and_local_norm}.  We split our
analysis into parts.


\subsection{Proof of claim~\texorpdfstring{\eqref{eq:vaidya_whp_log_det_filter}}{vaidyawhp}}
\label{sub:proof_of_claim_eq:vaidya_whp_log_det_filter}
Using the second degree Taylor expansion, we have
\begin{align*}
    \logdetvaidya_z - \logdetvaidya_x
    = \parenth{z - x}^\top \gradlogdetV_x + \frac{1}{2} \parenth{z - x}^\top \hesslogdetV_y \parenth{z - x},
    \quad \mbox{for some $y \in \overline{xz}$}.
\end{align*}
We claim that for $r\leq \vaidyaradiusbound(\epsilon)$, we have
\begin{subequations}
  \begin{align}
    \label{eq:gradient_L_tail}
        \Prob_z \brackets{ \parenth{z-x}^\top \gradlogdetV_x \geq -
          \epsilon/2} &\geq 1 - \epsilon/2, \quad \text{and} \\
\label{eq:hessian_L_tail}
\Prob_z \brackets{ {\frac{1}{2} \parenth{z-x} \hesslogdetV_y
    \parenth{z-x}} \geq -\epsilon/2} &\geq 1 - \epsilon/2.
        \end{align}
\end{subequations}
Note that the claim~\eqref{eq:vaidya_whp_log_det_filter} is a consequence of
these two auxiliary claims, which we now prove.


\subsubsection{Proof of bound~\eqref{eq:gradient_L_tail}}
\label{ssub:proof_of_bound_eq:gradient_l_tail}

Equation~\eqref{eq:z_x_relation2} implies that $\parenth{z - x}^\top
\gradlogdetV_x \sim \mathcal{N}\parenth{0, \frac{r^2}{\sqrt{\obs\dims}}
  \gradlogdetV_x^\top \vaidya_x^{-1} \gradlogdetV_x}$.  We claim that
\begin{align}
 \gradlogdetV_x^\top \vaidya_x^{-1} \gradlogdetV_x \leq 9 \sqrt{\obs\dims}
 \quad \text{for all } x \in \intP.
    \label{eq:variance_bound}
\end{align}
We prove this inequality at the end of this subsection.  Taking it as
given for now, let \mbox{$\rvg' \sim \NORMAL(0, 9r^2)$}.  Then using
inequality~\eqref{eq:variance_bound} and a standard Gaussian tail
bound, we find that
\begin{align*}
\Prob \brackets{ \parenth{z-x}^\top \gradlogdetV_x \geq - \omega} \geq
\Prob \brackets{ \rvg' \geq - \omega} \geq 1 -
\exp(-\omega^2/(18r^2)), \quad \mbox{valid for all $\omega \geq 0$}.
\end{align*}
Setting $\omega = \epsilon/2$ and noting that $\rparam \leq \frac{\epsilon}{\sqrt{18 \log(2/\epsilon)}}$
completes the claim.


\subsubsection{Proof of bound~\eqref{eq:hessian_L_tail}}
\label{ssub:proof_of_bound_eq:hessian_l_tail}

Let $\projD_{x, i} = \frac{a_i^\top \parenth{z-x}}{\slack_{x, i}} =
\frac{r}{\parenth{mn}^{\frac{1}{4}}} \hat{a}_{x,i}^\top \rvg$.  Using
Lemma~\ref{lemma:gradient_and_hessian_and_bounds}\ref{item:hessian_log_det_bound},
we have
\begin{align}
\abss{\frac{1}{2} \parenth{z-x}\tp \hesslogdetV_y \parenth{z-x}} &\leq 13
\sum_{i=1}^\obs \parenth{\levvaidya_{y, i} + \vaidyabeta}
\localslackvaidya_{y, i} \frac{\slack_{x, i}^2}{\slack_{y, i}^2}
\projD_{x, i}^2 + \frac{17}{2} \sum_{i,j=1}^\obs \levvaidya_{y, i, j}^2
\localslackvaidya_{y, i} \frac{\slack_{x, j}^2}{\slack_{y, j}^2}
\projD_{x, j}^2\notag\\
& \leq \frac{43}{2} \sqrt{\frac{\obs}{\dims}} \sum_{i=1}^\obs
\parenth{\levvaidya_{x, i} + \vaidyabeta}
\frac{\parenth{\levvaidya_{y, i} +
    \vaidyabeta}}{\parenth{\levvaidya_{x, i} + \vaidyabeta}}
\frac{\slack_{x, i}^2}{\slack_{y, i}^2} \projD_{x, i}^2.
    \label{eq:final_bound_hessian_log_det}
\end{align}
The last inequality comes from Lemma~\ref{lemma:first_bounds}\ref{item:theta_bound} and Lemma~\ref{ppt:all_properties}\ref{item:sigma_properties}.
Setting $\tau = 1.05$, we define the events $\mathcal{E}_1$ and
$\mathcal{E}_2$ as follows:
\begin{subequations}
  \begin{align}
\label{eq:s_whp_bound}
\mathcal{E}_1 &= \braces{\forall i \in[\obs],\, \frac{\slack_{x,
      i}}{\slack_{y, i}} \in \brackets{2-\tau, \tau} },\quad
\text{and}\\
\mathcal{E}_2 &= \braces{\forall i \in[\obs],\, \frac{\levvaidya_{x,
      i}}{\levvaidya_{y, i}} \in \brackets{0,
    \frac{\tau^2}{(2-\tau)^2}} }
\label{eq:sigma_whp_bound}.
  \end{align}
\end{subequations}
It is straightforward to see that $\mathcal{E}_1 \subseteq
\mathcal{E}_2$ following a similar argument we used to obtain
equation~\eqref{eq:sigma_closeness} in the proof of Lemma~\ref{lemma:close_hessian_eigenvalues}.
Since $\rparam \leq 1/\brackets{{20\parenth{1 + \sqrt{2}
    \log^{1/2}\parenth{\frac{4}{\epsilon}}}}}$, Lemma~\ref{lemma:whp_slackness}
implies that $\Prob\brackets{\mathcal{E}_1} \geq 1-\epsilon/4 $ whence
$\Prob\brackets{\mathcal{E}_2} \geq 1-\epsilon/4$.  Using these high
probability bounds and the setting $\tau = 1.05$, we obtain that with
probability at least $1-\epsilon/4$
\begin{align}
      \label{eq:with_tau}
\sqrt{\frac{\obs}{\dims}} \sum_{i=1}^\obs \parenth{\levvaidya_{x, i} +
  \vaidyabeta} \frac{\parenth{\levvaidya_{y, i} +
    \vaidyabeta}}{\parenth{\levvaidya_{x, i} + \vaidyabeta}}
\frac{\slack_{x, i}^2}{\slack_{y, i}^2} \projD_{x, i}^2 \leq
2\sqrt{\frac{\obs}{\dims}} \sum_{i=1}^\obs \parenth{\levvaidya_{x, i} +
  \vaidyabeta} \projD_{x, i}^2 =\frac{2r^2}{\dims} \sum_{i=1}^\obs
\parenth{\levvaidya_{x, i} + \vaidyabeta} (\hat{a}_{x,i}\tp \rvg)^2.
\end{align}
Applying the high probability bound Lemma~\ref{lemma:gaussian_moment_bounds}~\eqref{eq:quadratic} and the condition
\begin{align}
  \label{eq:hessian_l_tail_r_condition}
  r \leq \sqrt{\frac{\epsilon}{86\sqrt{3}\tailconst_2}},
\end{align}
we obtain that with probability at least $1-\epsilon/2$,
\begin{align*}
\frac{1}{2} \parenth{z-x}\tp \hesslogdetV_y \parenth{z-x} \geq -\epsilon/2,
\end{align*}
as claimed.


\subsubsection{Proof of bound~\eqref{eq:variance_bound}}
\label{ssub:proof_of_bound_eq:variance_bound}

We now return to prove our earlier
inequality~\eqref{eq:variance_bound}.  Using the expression for the
gradient $\gradlogdetV_x$ from
Lemma~\ref{lemma:gradient_and_hessian_and_bounds}\ref{item:gradient_L},
we have that for any vector $u \in \real^{n}$
\begin{align}
u^\top \gradlogdetV_x \gradlogdetV_x^\top u & = \angles{u, A_x\tp
  \parenth{2\vaidyalevmatrix_x - \vaidyaprojmatrix_x^{(2)} + \vaidyabeta
    \Ind} \localslackvaidya_x }^2 \notag \\
& = \angles{A_x u, \parenth{2\vaidyalevmatrix_x - \vaidyaprojmatrix_x^{(2)}
    + \vaidyabeta \Ind} \localslackvaidya_x }^2\notag \\
& = \angles{\parenth{\vaidyalevmatrix_x + \vaidyabeta
    \Ind}^{\frac{1}{2}} A_x u, \parenth{\vaidyalevmatrix_x +
    \vaidyabeta \Ind}^{-{1}/{2}}\parenth{2\vaidyalevmatrix_x -
    \vaidyaprojmatrix_x^{(2)} + \vaidyabeta \Ind} \localslackvaidya_x }^2
\notag\\
\label{eq:final_variance_bound}
& \leq u\tp \vaidya_x u \cdot \localslackvaidya_x^\top
\parenth{2\vaidyalevmatrix_x - \vaidyaprojmatrix_x^{(2)} + \vaidyabeta \Ind}
\parenth{\vaidyalevmatrix_x + \vaidyabeta
  \Ind}^{-1}\parenth{2\vaidyalevmatrix_x - \vaidyaprojmatrix_x^{(2)} +
  \vaidyabeta \Ind}\localslackvaidya_x
\end{align}
where the last step follows from the Cauchy-Schwarz inequality.  As a
consequence of
Lemma~\ref{ppt:all_properties}\ref{item:Sigma_dominates_P}, the matrix
$\vaidyalevmatrix_x - \vaidyaprojmatrix_x^{(2)}$ is PSD.
Thus, we have
\begin{align*}
0 \preceq 2 \vaidyalevmatrix_x - \vaidyaprojmatrix_x^{(2)} +
\vaidyabeta \Ind \preceq 3 \parenth{\vaidyalevmatrix_x +
  \vaidyabeta\Ind}.
\end{align*}
Consequently, we find that
\begin{align*}
    0 \preceq \underbrace{\parenth{3 \vaidyalevmatrix_x + 3
        \vaidyabeta\Ind}^{-{1}/{2}} \parenth{2 \vaidyalevmatrix_x -
        \vaidyaprojmatrix_x^{(2)} + \vaidyabeta \Ind} \parenth{3
        \vaidyalevmatrix_x + 3\vaidyabeta\Ind}^{-{1}/{2}}}_{=:{L}}
    \preceq \Ind.
\end{align*}
We deduce that all eigenvalues of the matrix ${L}$ lie in the interval $[0, 1]$
and hence all the eigenvalues of the matrix ${L}^2$
belong to the interval $[0, 1]$. As a result, we have
\begin{align*}
\parenth{2 \vaidyalevmatrix_x - \vaidyaprojmatrix_x^{(2)} + \vaidyabeta
  \Ind} \parenth{3 \vaidyalevmatrix_x + 3\vaidyabeta\Ind}^{-1}
\parenth{2 \vaidyalevmatrix_x - \vaidyaprojmatrix_x^{(2)} + \vaidyabeta
  \Ind} \preceq \parenth{3 \vaidyalevmatrix_x + 3\vaidyabeta\Ind}.
\end{align*}
Thus, we obtain
\begin{align}
   \label{eq:final_bound_variance_2}
  \localslackvaidya_x^\top \parenth{2\vaidyalevmatrix_x -
    \vaidyaprojmatrix_x^{(2)} + \vaidyabeta \Ind}
  \parenth{\vaidyalevmatrix_x + \vaidyabeta
    \Ind}^{-1}\parenth{2\vaidyalevmatrix_x - \vaidyaprojmatrix_x^{(2)} +
    \vaidyabeta \Ind}\localslackvaidya_x \leq 9
  \localslackvaidya_x^\top \parenth{\vaidyalevmatrix_x +
    \vaidyabeta\Ind} \localslackvaidya_x.
\end{align}
Finally, applying Lemma~\ref{ppt:all_properties} and combining
bounds~\eqref{eq:final_variance_bound} and
\eqref{eq:final_bound_variance_2} yields the claim.


\subsection{Proof of claim~\eqref{eq:vaidya_whp_local_norm}}
\label{sub:proof_of_claim_eq:vaidya_whp_local_norm}

The quantity of interest can be written as
\begin{align*}
\vecnorm{z-x}{z}^2 - \vecnorm{z-x}{x}^2 = \sum_{i=1}^\obs
\parenth{a_i^\top \parenth{z- x}}^2 \parenth{\vaidyaphi_{z, i} - \vaidyaphi_{x,
    i}}.
\end{align*}
We can write $z = x + \alpha u$, where $\alpha$ is a scalar and $u$ is
a unit vector in \realdim.  Then we have
\begin{align*}
\vecnorm{z-x}{z}^2 - \vecnorm{z-x}{x}^2 = \alpha^2 \sum_{i=1}^\obs
\parenth{a_i^\top u}^2 \parenth{\vaidyaphi_{z, i} - \vaidyaphi_{x, i}}.
\end{align*}
We apply a Taylor series expansion for $\sum_{i=1}^\obs
\parenth{a_i^\top u}^2 \parenth{\vaidyaphi_{z, i} - \vaidyaphi_{x,
    i}}$ around the point $x$, along the line $u$.  There exists a
point $y \in \overline{xz}$ such that
\begin{align*}
\sum_{i=1}^\obs \parenth{a_i^\top u}^2 \parenth{\vaidyaphi_{z, i} -
  \vaidyaphi_{x, i}} = \sum_{i=1}^\obs \parenth{a_i^\top u}^2 \parenth{
  \parenth{z - x}^\top \nabla \vaidyaphi_{x, i}+ \frac{1}{2}
  \parenth{z-x}^\top \nabla^2 \vaidyaphi_{y, i} \parenth{z-x}}.
\end{align*}
Multiplying both sides by $\alpha^2$, and using the shorthand
$\projD_{x, i} = \frac{a_i^\top (z-x)}{\slack_{x,i}}$, we obtain
\begin{align}
  \label{eq:vaidya_z_x_two_terms}
\vecnorm{z \!-\! x}{z}^2 \!-\! \vecnorm{z\!-\!x}{x}^2 &=
\sum_{i=1}^\obs \projD_{x, i}^2 \slack_{x, i}^2 \parenth{z\!-\!x}^\top
\nabla \vaidyaphi_{x, i} + \sum_{i=1}^\obs \projD_{x, i}^2 \slack_{x, i}^2
\frac{1}{2} \parenth{z\!-\!x}^\top \nabla^2 \vaidyaphi_{y, i}
\parenth{z\!-\!x}.
\end{align}
Substituting the expression for $\nabla \vaidyaphi_{x, i}$ from
Lemma~\ref{lemma:gradient_and_hessian_and_bounds}\ref{item:gradient_beta}
in equation~\eqref{eq:vaidya_z_x_two_terms} and performing some algebra, the
first term on the RHS of equation~\eqref{eq:vaidya_z_x_two_terms} can be
written as
\begin{align}
\sum_{i=1}^\obs \projD_{x, i}^2 \slack_{x, i}^2 (z-x)^\top \nabla
\vaidyaphi_{x, i} = 2 \sum_{i=1}^\obs \parenth{\frac{7}{3}\levvaidya_{x, i}
  + \vaidyabeta}\projD_{x, i}^3 - \frac{1}{3}\sum_{i, j=1}^\obs
\levvaidya_{x, i, j}^2 \parenth{\projD_{x, i}+\projD_{x, j}}^3.
 \label{eq:z_x_first_term}
\end{align}
On the other hand, using Lemma~\ref{lemma:gradient_and_hessian_and_bounds}~\ref{item:beta_hessian_bound}, we have
\begin{align}
\label{eq:hessian_beta_first_bound}
\frac{1}{2} \slack_{x, i}^2 \abss{\parenth{z-x}^\top \nabla^2
  \vaidyaphi_{y, i} \parenth{z-x}} &\leq \frac{\slack_{x, i}^2}{\slack_{y,
    i}^2} \brackets{14 \parenth{\levvaidya_{y, i} + \vaidyabeta}
  \frac{\slack_{x, i}^2}{\slack_{y, i}^2} \projD_{x, i}^2 + 11
  \parenth{\sum_{j=1}^\obs \levvaidya_{y, i, j}^2 \projD_{x, j}^2
    \frac{\slack_{x, j}^2}{\slack_{y, j}^2}}}.
\end{align}
Now, we use a fourth degree Gaussian polynomial to bound both the
terms on the RHS of inequality~\eqref{eq:hessian_beta_first_bound}.
To do so, we use high probability bound for $\slack_{x, i}/\slack_{y,
  i}$.  In particular, we use the high probability bounds for the
events $\mathcal{E}_1$ and $\mathcal{E}_2$ defined in
equations~\eqref{eq:s_whp_bound} and \eqref{eq:sigma_whp_bound}.
Multiplying both sides of
inequality~\eqref{eq:hessian_beta_first_bound} by $\projD_{x, i}^2$
and summing over the index $i$, we obtain that with probability at
least $1-\epsilon/4$, we have
\begin{align}
  \sum_{i=1}^\obs \projD_{x, i}^2 \slack_{x, i}^2 \abss{\frac{1}{2}
    \parenth{z-x}^\top \nabla^2 \vaidyaphi_{y, i} \parenth{z-x}}
  &\stackrel{\mathmakebox[\widthof{====}]{ }}{\leq} \brackets{14
    \sum_{i=1}^\obs \parenth{\levvaidya_{y, i} + \vaidyabeta}
    \frac{\slack_{x, i}^4}{\slack_{y, i}^4} \projD_{x, i}^4 + 11
         {\sum_{i, j=1}^\obs \levvaidya_{y, i, j}^2 \projD_{x,
             i}^2\projD_{x, j}^2 \frac{\slack_{x, i}^2\slack_{x,
               j}^2}{\slack_{y, i}^2\slack_{y, j}^2}}} \notag \\
&
  \stackrel{\mathmakebox[\widthof{====}]{(\mathrm{hpb.}\eqref{eq:s_whp_bound})
  }}{\leq} \tau^4 \bigg[14 \sum_{i=1}^\obs\parenth{\levvaidya_{y, i} +
      \vaidyabeta} \projD_{x, i}^4 + 11 \sum_{i, j=1}^\obs
    \levvaidya_{y, i, j}^2 \projD_{x, i}^2\projD_{x, j}^2 \bigg]
  \notag \\
& \stackrel{\mathmakebox[\widthof{====}]{ (\mathrm{AM-GM}) }}{\leq}
  \tau^4 \bigg[14 \sum_{i=1}^\obs\parenth{\levvaidya_{y, i} +
      \vaidyabeta} \projD_{x, i}^4 + \frac{11}{2} \sum_{i, j=1}^\obs
    \levvaidya_{y, i, j}^2 (\projD_{x, i}^4+\projD_{x, j}^4) \bigg]
  \notag \\
  & \stackrel{\mathmakebox[\widthof{====}]{
      (\mathrm{Lem.}~\ref{ppt:all_properties}\ref{item:sigma_properties})
  }}{\leq} 25 \tau^4\sum_{i=1}^\obs \parenth{\levvaidya_{y, i} +
    \vaidyabeta} \projD_{x, i}^4 \notag\\
& \stackrel{\mathmakebox[\widthof{====}]{(\mathrm{hpb.}
      \eqref{eq:sigma_whp_bound})}}{\leq} 50 \sum_{i=1}^\obs
  \parenth{\levvaidya_{x, i} + \vaidyabeta} \projD_{x, i}^4,
        \label{eq:z_x_second_term}
\end{align}
where ``hpb'' stands for high probability bound for events
$\mathcal{E}_1$ and $\mathcal{E}_2$.  In the last step, we have used
the fact that $\tau^6/(2-\tau)^2 \leq 2$ for $\tau = 1.05$.  Combining
equations~\eqref{eq:vaidya_z_x_two_terms}, \eqref{eq:z_x_first_term} and
\eqref{eq:z_x_second_term} and noting that $\projD_{x, i}= r \hat
a_i\tp \rvg/(\obs\dims)^{1/4}$, we find that
\begin{align}
\abss{\vecnorm{z - x}{z}^2 - \vecnorm{z-x}{x}^2} &\leq \frac{14}{3}
\abss{\sum_{i=1}^\obs \parenth{\levvaidya_{x, i} +
    \vaidyabeta}\projD_{x, i}^3} + \frac{8}{3} \abss{\sum_{i,
    j=1}^\obs \levvaidya_{x, i, j}^2 \parenth{\parenth{\projD_{x,
        i}+\projD_{x, j}}/2}^3} + 38 \sum_{i=1}^\obs \levvaidya_{x, i}
\projD_{x, i}^4 \notag \\
& \leq \frac{14}{3} \frac{r^3}{(\obs\dims)^{3/4}}\abss{\sum_{i=1}^\obs
  \parenth{\levvaidya_{x, i}+\vaidyabeta} \parenth{\hat{a}_{x,i}^\top
    \rvg}^3} + \frac{8}{3} \frac{r^3}{(\obs\dims)^{3/4}}
\abss{\sum_{i, j=1}^\obs \levvaidya_{x, i, j}^2
  \parenth{\frac{1}{2}(\hat{a}_{x,i} + \hat{a}_{x,j})\tp\rvg}^3} \notag
\\ &\quad + 50\, \frac{r^4}{\obs\dims} \sum_{i=1}^\obs
\parenth{\levvaidya_{x, i}+\vaidyabeta} (\hat{a}_{x,i}\tp \rvg)^4,
    \label{eq:gaussian_polynomials}
\end{align}
where the last step follows from the fact that $0\leq \levvaidya_{x,
  i} \leq \levvaidya_{x, i} + \vaidyabeta$.  In order to show that
$\abss{\vecnorm{z - x}{z}^2 - \vecnorm{z-x}{x}^2}$ is bounded as
$\order{1/\sqrt{\obs\dims}}$ with high probability, it suffices to
show that with high probability, the third and fourth degree
polynomials of $\hat{a}_{x,i}\tp\rvg$, that appear in
bound~\eqref{eq:gaussian_polynomials}, are bounded by
$\order{(\obs\dims)^{1/4}}$ and $\order{\sqrt{\obs\dims}}$
respectively.

Applying the bounds~\eqref{eq:cube1},~\eqref{eq:cube2} and~\eqref{eq:fourth} from Lemma~\ref{lemma:gaussian_moment_bounds}, we have with probability at least $1-\epsilon$,

\begin{align*}
  \vecnorm{z - x}{z}^2 - \vecnorm{z-x}{x}^2 \leq \frac{r^3}{\sqrt{\obs \dims}} \parenth{\frac{22\sqrt{15}\tailconst_3}{3}} + \frac{r^4}{\sqrt{\obs \dims}} \parenth{50\sqrt{105}\tailconst_4}.
\end{align*}

Using the condition
\begin{align}
  \label{eq:local_norm_tail_r_condition}
  r \leq \min \braces{\frac{\epsilon}{22\sqrt{5/3}\tailconst_3}, \sqrt{\frac{\epsilon}{50\sqrt{105}\tailconst_4}}},
\end{align}
completes our proof of claim~\eqref{eq:vaidya_whp_local_norm}.

\subsection{Proof of Lemma~\ref{lemma:whp_slackness}} 
\label{sub:proof_of_lemma_lemma:whp_slackness}

The proof is based on Lemma~\ref{lemma:closeness_of_slackness} and a simple
application of the standard chi-square tail bounds. According to Lemma~\ref{lemma:closeness_of_slackness}, we have that for $v \in \overline{xz}$,
\begin{align*}
  \abss{1 - \frac{\slack_{v, i}}{\slack_{x, i}}} \leq \parenth{\frac{\obs}{\dims}}^{\frac{1}{4}} \vecnorm{x - v}{x} \leq \parenth{\frac{\obs}{\dims}}^{\frac{1}{4}} \vecnorm{x - z}{x}.
\end{align*}
According to equation~\eqref{eq:z_x_relation2}, the proposal follows Gaussian distribution
\begin{align*}
  \parenth{\frac{\obs}{\dims}}^{\frac{1}{4}} \vecnorm{x - z}{x} = \frac{\rparam}{\dims^{1/2}} \vecnorm{\rvg}{2},
\end{align*}
where $\rvg \sim \NORMAL\parenth{0, \Ind_\dims}$.
Using the standard chi-square tail bound we have that for $\delta > 0$,
\begin{align*}
  \Prob\brackets{\vecnorm{\rvg}{2}/\sqrt{\dims} \geq 1 + \delta} \leq \exp\parenth{- \dims\delta^2/2}.
\end{align*}
Plugging in $\delta = \sqrt{\frac{2}{\dims}}\log^{\frac{1}{2}}\parenth{\frac{4}{\epsilon}}$ concludes the lemma.



\subsection{Proof of Lemma~\ref{lemma:gaussian_moment_bounds}}
\label{sub:proof_of_lemma_lemma:gaussian_moment_bounds}

The proof relies on the classical fact that the tails of a polynomial
in Gaussian random variables decay exponentially independently of
dimension.  In particular, Theorem 6.7 by \cite{janson1997gaussian} ensures that for any integers $\dims,
\degree \geq 1$, any polynomial $\MYPOLY:\realdim \rightarrow \real$
of degree $\degree$, and any scalar $t \geq (2e)^{{\degree}/{2}}$, we
have
\begin{align}
  \label{vaidya_EqnJansonBound}
    \Prob\brackets{\abss{\MYPOLY(\rvg)} \geq t
      \parenth{\Exs{\MYPOLY(\rvg)}^2}^{\frac{1}{2}}} \leq
    \exp\parenth{- \frac{\degree}{2e} t^{{2}/{\degree}}},
\end{align}
where $\rvg \sim \NORMAL(0, \Ind_n)$ denotes a standard Gaussian
vector in $n$ dimensions.
Also, the following observations on the behavior of the vectors $\hat{a}_
{x,i}$ defined in equation~\eqref{eq:hat_a_i} are useful:
\begin{subequations}
  \begin{align}
	\vecnorm{\hat{a}_{x,i}}{2}^2 &= \localslackvaidya_{x, i}
        \stackrel{(i)}{\leq}\sqrt{\frac{\obs}{\dims}} \quad \text{for all
        } i \in
                 [\obs], \quad \text{and} \label{eq:hat_a_i_norm}\\
  (\hat{a}_{x,i}^\top \hat{a}_{x,j})^2 &= \localslackvaidya^2_{x, i, j} \quad
                 \text{for all } i, j \in [\obs], \label{eq:hat_a_i_j}
	\end{align}
  where inequality~(i) follows from Lemma~\ref{lemma:first_bounds}
  \ref{item:theta_bound}.
\end{subequations}


\subsubsection{Proof of bound~\eqref{eq:quadratic}}
\label{ssub:proof_of_part_item:quadratic}
We have
\begin{align*}
  \Exs \parenth{\sum_{i=1}^\obs \parenth{\levvaidya_{x,
        i}+\vaidyabeta} \parenth{\hat{a}_{x,i}^\top \rvg}^2}^2 &=
  \sum_{i,j=1}^\obs \parenth{\levvaidya_{x, i} +
    \vaidyabeta}\parenth{\levvaidya_{x, j}+\vaidyabeta}\Exs
  \parenth{\hat{a}_{x,i}^\top \rvg}^2 \parenth{\hat{a}_{x,j}^\top \rvg}^2 \\
& = \sum_{i,j=1}^\obs \parenth{\levvaidya_{x,
      i}+\vaidyabeta}\parenth{\levvaidya_{x, j}+\vaidyabeta}\parenth{
    \vecnorm{\hat{a}_{x,i}}{2}^2\vecnorm{\hat{a}_{x,j}}{2}^2 + 2
    \parenth{\hat{a}_{x,i}^\top \hat{a}_{x,j}}^2 } \\
& = \sum_{i,j=1}^\obs \parenth{\levvaidya_{x,
      i}+\vaidyabeta}\parenth{\levvaidya_{x, j}+\vaidyabeta}
  \parenth{\localslackvaidya_{x, i}\localslackvaidya_{x, j} + 2
    \localslackvaidya_{x, i, j}^2 } \\
& \stackrel{(i)}{=} \dims^2 + 2 \dims \\ & \leq 3 \dims^2,
\end{align*}
where step~(i) follows from properties~\ref{item:theta_sum}
and~\ref{item:theta_properties} from Lemma~\ref{ppt:all_properties}.
Applying the bound~\eqref{vaidya_EqnJansonBound} with $\degree = 2, t =
e\log(\frac{4}{\epsilon})$ yields the claim. We verify that for $\epsilon \in (0, \vaidyaepsilonconst]$, $t \geq 2e$.


\subsubsection{Proof of bound~\eqref{eq:cube1}}
\label{ssub:proof_of_part_item:cube1}

Using Isserlis' theorem~\citep{isserlis1918formula} for Gaussian
moments, we obtain
\begin{align}
\Exs\parenth{\sum_{i=1}^\obs \parenth{\levvaidya_{x, i}+\vaidyabeta}
  \parenth{\hat{a}_{x,i}^\top \rvg }^3}^2 &= \sum_{i,j=1}^\obs
\parenth{\levvaidya_{x, i}+\vaidyabeta}\parenth{\levvaidya_{x,
    i}+\vaidyabeta} \Exs\parenth{\hat{a}_{x,i}^\top \rvg}^3
\parenth{\hat{a}_{x,j}^\top \rvg}^3 \notag \\
& = 9 \underbrace{\sum_{i,j=1}^\obs \parenth{\levvaidya_{x,
      i}+\vaidyabeta}\parenth{\levvaidya_{x,
      j}+\vaidyabeta}\vecnorm{\hat{a}_{x,i}}{2}^2 \vecnorm{\hat{a}_{x,j}}{2}^2
  \parenth{\hat{a}_{x,i}^\top \hat{a}_{x,j}}}_{=:N_1}\notag \\
& \quad + 6 \underbrace{\sum_{i,j=1}^\obs \parenth{\levvaidya_{x,
      i}+\vaidyabeta}\parenth{\levvaidya_{x,
      j}+\vaidyabeta}\parenth{\hat{a}_{x,i}^\top \hat{a}_{x,j}}^3}_{=:N_2}.
	\label{eq:decomp_third_moment}
\end{align}
We claim that the two terms in this sum are bounded as $N_1 \leq
\sqrt{\obs\dims}$ and $N_2 \leq \sqrt{\obs\dims}$.  Assuming the
claims as given, we now complete the proof.  Plugging in the bounds
for $N_1$ and $N_2$ in equation~\eqref{eq:decomp_third_moment} we find
that $\Exs\big(\sum_{i=1}^\obs \parenth{\levvaidya_{x, i}+\vaidyabeta}
\big(\hat{a}_{x,i}^\top \rvg \big)^3\big)^2 \leq 15\sqrt{\obs\dims}$.
Applying the bound~\eqref{vaidya_EqnJansonBound} with $\degree = 3, t =
\parenth{\frac{2e}{3}\log (4/\epsilon)}^{3/2}$ yields the claim. We also verify that for $\epsilon \in (0, \vaidyaepsilonconst]$, $t \geq \parenth{2e}^{3/2}$.  We now turn to proving the bounds on $N_1$ and
$N_2$.


\paragraph{Bounding $N_1$:}
\label{par:bounding_n1}
Let $B$ be an $\obs \times \dims $ matrix with its $i$-th row given by
$\sqrt{\parenth{\levvaidya_{x, i}+\vaidyabeta}} \hat{a}_{x,i}^\top$.
Observe that
\begin{align}
	\sum_{i=1}^\obs \parenth{\levvaidya_{x, i}+\vaidyabeta}\hat
        a_i \hat{a}_{x,i}\tp = \vaidya_x^{-1/2} \parenth{\sum_{i=1}^\obs
          \parenth{\levvaidya_{x, i}+\vaidyabeta} \frac{a_i
            a_i\tp}{\slack_{x, i}^2}}\vaidya_x^{-1/2} =
        \vaidya_x^{-1/2} \vaidya_x \vaidya_x^{-1/2} = \Ind_\dims.
	\label{eq:hat_a_hat_a_Ind}
\end{align}
Thus we have $B^\top B = \Ind_\dims $, which implies that $BB^\top$ is
an orthogonal projection matrix.  Letting $v \in \real^\obs $ be a
vector such that $v_i = \sqrt{\parenth{\levvaidya_{x,
      i}+\vaidyabeta}}\vecnorm{\hat{a}_{x,i}}{2}^2$, we then have
\begin{align*}
  \sum_{i,j=1}^\obs \!\parenth{\levvaidya_{x,
      i}\!+\!\vaidyabeta}\!{\vecnorm{\hat{a}_{x,i}}{2}^2 \hat{a}_{x,i}^\top
    \parenth{\levvaidya_{x, j}\!+\!\vaidyabeta}\!
    \vecnorm{\hat{a}_{x,j}}{2}^2 \hat{a}_{x,j}} &= \!\vecnorm{\sum_{i=1}^\obs
    \! \parenth{\levvaidya_{x,
        i}\!+\!\vaidyabeta}\!\vecnorm{\hat{a}_{x,i}}{2}^2 \hat{a}_{x,i} }{2}^2
  = \vecnorm{B^\top v}{2}^2 \stackrel{(i)}{\leq} \vecnorm{v}{2}^2,
\end{align*}
where inequality $(i)$ follows from the fact that $v\tp P v \leq
\vecnorm{v}{2}^2$ for any orthogonal projection matrix $P$.
Equation~\eqref{eq:hat_a_i_norm} implies that $v_i^2 =
\parenth{\levvaidya_{x, i} + \vaidyabeta}\localslackvaidya_{x, i}^2$.
Using
Lemma~\ref{ppt:all_properties}\ref{item:theta_square_sigma_sum_bound},
we find that
\begin{align*}
	\vecnorm{v}{2}^2 = \sum_{i=1}^\obs \parenth{\levvaidya_{x, i}
          + \vaidyabeta} \localslackvaidya_{x, i}^2 \leq \sqrt{\obs
          \dims}.
\end{align*}


\paragraph{Bounding $N_2$:}
\label{par:bounding_n2}
We see that
\begin{align*}
  \sum_{i,j=1}^\obs \parenth{\levvaidya_{x,
      i}+\vaidyabeta}\parenth{\levvaidya_{x, j}+\vaidyabeta}
  \parenth{\hat{a}_{x,i}^\top \hat{a}_{x,j}}^3
  &\stackrel{\mathmakebox[\widthof{=======}]{ (\mathrm{C-S}) }}{\leq}
  \sum_{i,j=1}^\obs \parenth{\levvaidya_{x,
      i}+\vaidyabeta}\parenth{\levvaidya_{x, j}+\vaidyabeta}
  \parenth{\hat{a}_{x,i}^\top \hat{a}_{x,j}}^2
  \vecnorm{\hat{a}_{x,i}}{2}\vecnorm{\hat{a}_{x,j}}{2} \\
& \stackrel{\mathmakebox[\widthof{=======}]{
      (\mathrm{eqns.}\eqref{eq:hat_a_i_norm},\eqref{eq:hat_a_i_j})
  }}{\leq} \sum_{i,j=1}^\obs \parenth{\levvaidya_{x,
      i}+\vaidyabeta}\parenth{\levvaidya_{x, j}+\vaidyabeta}
  \localslackvaidya_{x, i, j}^2 \sqrt{\localslackvaidya_{x,
      i}\localslackvaidya_{x, j}} \\
& \stackrel{\mathmakebox[\widthof{=======}]{
      (\mathrm{Lem.}~\ref{lemma:first_bounds}\ref{item:theta_bound})
  }}{\leq} \sqrt{\frac{\obs }{\dims}} \sum_{i,j=1}^\obs
  \parenth{\levvaidya_{x, i}+\vaidyabeta}\parenth{\levvaidya_{x,
      j}+\vaidyabeta}\localslackvaidya_{x, i, j}^2.
\end{align*}
We now apply Lemma~\ref{ppt:all_properties}\ref{item:theta_properties}
followed by Lemma~\ref{ppt:all_properties}\ref{item:theta_sum} to
obtain the claimed bound on $N_2$.


\subsubsection{Proof of bound~\eqref{eq:cube2}}
\label{ssub:proof_of_part_item:cube2}

Let $c_{i, j} = \displaystyle\frac{\parenth{\hat{a}_{x,i}+\hat{a}_{x,j}}}{2}$
for $i, j \in [\obs]$. Using Isserlis' theorem for Gaussian moments,
we obtain
\begin{align*}
  \Exs\parenth{\sum_{i,j=1}^\obs \levvaidya_{x, i, j}^2
    \parenth{c_{i,j}^\top \rvg }^3}^2 &= \sum_{i,j,k,l=1}^\obs
  \levvaidya_{x, i, j}^2\levvaidya_{x, k, l}^2
  \Exs\parenth{c_{i,j}^\top \rvg }^3 \parenth{c_{k,l}^\top \rvg }^3 \\
& = 9 \underbrace{\sum_{i,j,k,l=1}^\obs \levvaidya_{x, i,
      j}^2\levvaidya_{x, k, l}^2 {\vecnorm{c_{i,j}}{2}^2
      \vecnorm{c_{k,l}}{2}^2 \parenth{c_{i,j}^\top c_{k,l}}}}_{= : \;
    C_1} + 6 \underbrace{\sum_{i,j,k,l=1}^\obs \levvaidya_{x, i,
      j}^2\levvaidya_{x, k, l}^2 {\parenth{c_{i,j}^\top
        c_{k,l}}^3}}_{= : \; C_2}
\end{align*}
We claim that $C_1 \leq \sqrt{\obs\dims}$ and $C_2 \leq
\sqrt{\obs\dims}$.  Assuming the claims as given, the result follows
using similar arguments as in the previous part.  We now bound $C_i,
i=1, 2$, using arguments similar to the ones used in
Section~\ref{ssub:proof_of_part_item:cube1} to bound $N_i, i=1, 2$,
respectively.  The following bounds on $\vecnorm{c_{i,j}}{2}^2$ are
used in the arguments that follow:
\begin{subequations}
	\begin{align}
		\vecnorm{c_{i,j}}{2}^2
                \stackrel{\mathmakebox[\widthof{====}]{ \mathrm{SSI}
                }}{\leq} \frac{1}{2}\parenth{\vecnorm{\hat
                    a_{i}}{2}^2+\vecnorm{\hat a_{j}}{2}^2}
                &\stackrel{\mathmakebox[\widthof{====}]{}}{=}
                \frac{1}{2}\parenth{\localslackvaidya_{x,
                    i}+\localslackvaidya_{x, j} }
			\label{eq:norm_c_theta}\\
		&\stackrel{\mathmakebox[\widthof{====}]{
                            \mathrm{Lem.}~\ref{lemma:first_bounds}\ref{item:theta_bound}
                        }}{\leq} \sqrt{\frac{\obs}{\dims}}.
			\label{eq:norm_c}
	\end{align}
\end{subequations}


\paragraph{Bounding $C_1$:}
\label{par:bounding_c1}

Let $B$ be the same $\obs \times \dims $ matrix as in the proof of
previous part with its $i$-th row given by
$\sqrt{\parenth{\levvaidya_{x, i} + \vaidyabeta}} \hat{a}_{x,i}^\top$.
Define the vector $u \in \realdim$ with entries given by \mbox{$u_i =
  {\sum_{j=1}^\obs \levvaidya_{x, i, j}^2
    \vecnorm{c_{i,j}}{2}^2}/{{\parenth{\levvaidya_{x, i} +
        \vaidyabeta}^{1/2}}}$.} We have
\begin{align*}
  \sum_{i,j,k,l=1}^\obs \levvaidya_{x, i, j}^2\levvaidya_{x, k, l}^2
  \vecnorm{c_{i,j}}{2}^2 \vecnorm{c_{k,l}}{2}^2 \parenth{c_{i,j}^\top
    c_{k,l}} &\stackrel{\mathmakebox[\widthof{==}]{}} {\leq}
  \vecnorm{\sum_{i,j=1}^\obs \levvaidya_{x, i, j}^2
    \vecnorm{c_{i,j}}{2}^2 c_{i,j} }{2}^2 \\
  & \stackrel{\mathmakebox[\widthof{==}]{ (\mathrm{SSI}) }}{\leq}
  \frac{1}{2} \parenth{\vecnorm{\sum_{i,j=1}^\obs \levvaidya_{x, i,
        j}^2 \vecnorm{c_{i,j}}{2}^2 \hat{a}_{x,i} }{2}^2 +
    \vecnorm{\sum_{i,j=1}^\obs \levvaidya_{x, i, j}^2
      \vecnorm{c_{i,j}}{2}^2 \hat{a}_{x,j} }{2}^2 } \\
& \stackrel{\mathmakebox[\widthof{==}]{}} {=} \vecnorm{B\tp u}{2}^2 \\
& \stackrel{\mathmakebox[\widthof{==}]{(i)}} {\leq} \vecnorm{u}{2}^2,
\end{align*}
where inequality $(i)$ follows from the fact that $v\tp P v \leq
\vecnorm{v}{2}^2$ for any orthogonal projection matrix $P$.  It is
left to bound the term $u_i^2$.  We see that
\begin{align*}
u_i^2 = \frac{1}{\levvaidya_{x, i} + \vaidyabeta}\sum_{j,k=1}^\obs
\levvaidya_{x, i, j}^2\levvaidya_{x, i, k}^2 \vecnorm{c_{i,j}}{2}^2
\vecnorm{c_{i,k}}{2}^2\,\, &\stackrel{\mathmakebox[\widthof{======}]{
    (\mathrm{bnd.}~\eqref{eq:norm_c}) }}
        {\leq}\sqrt{\frac{\obs}{\dims}} \frac{1}{\levvaidya_{x, i} +
          \vaidyabeta}\sum_{j,k=1}^\obs \levvaidya_{x, i, j}^2
        \levvaidya_{x, i, k}^2 \vecnorm{c_{i,j}}{2}^2 \\
& \stackrel{\mathmakebox[\widthof{======}]{
            (\mathrm{Lem.}~\ref{ppt:all_properties}\ref{item:sigma_properties})
        }} {\leq}\sqrt{\frac{\obs}{\dims}} \frac{\levvaidya_{x,
            i}}{\levvaidya_{x, i} + \vaidyabeta}\sum_{j=1}^\obs
        \levvaidya_{x, i, j}^2 \vecnorm{c_{i,j}}{2}^2 \\
& \stackrel{\mathmakebox[\widthof{======}]{
            (\mathrm{bnd.}~\eqref{eq:norm_c_theta}) }}
                  {\leq}\sqrt{\frac{\obs}{\dims}} \sum_{j=1}^\obs
                  \levvaidya_{x, i, j}^2 \frac{\localslackvaidya_{x,i}
                    + \localslackvaidya_{x, j}}{2}.
\end{align*}
Now, summing over $i$ and using symmetry of indices $i, j$, we find
that
\begin{align*}
\vecnorm{u}{2}^2 \leq \sqrt{\frac{\obs }{\dims}} \sum_{i=1}^\obs
\sum_{j=1}^\obs \levvaidya_{x, i, j}^2 \localslackvaidya_{x, i}
\stackrel{\mathmakebox[\widthof{====}]{
    (\mathrm{Lem.}~\ref{ppt:all_properties}\ref{item:sigma_properties})
}} {=} \sqrt{\frac{\obs}{\dims}} \sum_{i=1}^\obs \levvaidya_{x,
  i}\localslackvaidya_{x, i} \stackrel{\mathmakebox[\widthof{====}]{
    (\mathrm{Lem.}~\ref{ppt:all_properties}\ref{item:theta_sum}) }}
         {\leq} \sqrt{\obs \dims},
\end{align*}
thereby implying that $C_1 \leq \sqrt{\obs\dims}$.


\paragraph{Bounding $C_2$:}
\label{par:bounding_c2}
Using the Cauchy-Schwarz inequality and the bound~\eqref{eq:norm_c},
we find that
\begin{align*}
\sum_{i,j,k,l=1}^\obs \levvaidya_{x, i, j}^2\levvaidya_{x, k, l}^2
\parenth{c_{i,j}^\top c_{k,l}}^3 & \leq \sum_{i,j,k,l=1}^\obs
\levvaidya_{x, i, j}^2\levvaidya_{x, k, l}^2 \parenth{c_{i,j}^\top
  c_{k,l}}^2 \vecnorm{c_{i,j}}{2}\vecnorm{c_{k,l}}{2} \\
& \leq \sqrt{\frac{\obs }{\dims}} \sum_{i,j,k,l=1}^\obs \levvaidya_{x,
  i, j}^2\levvaidya_{x, k, l}^2 \parenth{c_{i,j}^\top c_{k,l}}^2.
\end{align*}
Using SSI and the symmetry of pairs of indices $(i,j)$ and $(k, l)$,
we obtain
\begin{align*}
  \sum_{i,j,k,l=1}^\obs \levvaidya_{x, i, j}^2\levvaidya_{x, k, l}^2
  \parenth{c_{i,j}^\top c_{k,l}}^2 \leq \sum_{i,j,k,l=1}^\obs
  \levvaidya_{x, i, j}^2\levvaidya_{x, k, l}^2 \parenth{\hat{a}_{x,i}^\top
    \hat{a}_k}^2 = \sum_{i,k=1}^\obs \levvaidya_{x, i}\levvaidya_{x,
    k} \parenth{\hat{a}_{x,i}^\top \hat{a}_k}^2.
\end{align*}
The resulting expression can be bounded as follows:
\begin{align*}
\sum_{i,k=1}^\obs \levvaidya_{x, i}\levvaidya_{x, k}
\parenth{\hat{a}_{x,i}^\top \hat{a}_k}^2\,\,
\stackrel{\mathmakebox[\widthof{====}]{(\mathrm{eqn.}\eqref{eq:hat_a_i_j}
    ) }} {=} \, \sum_{i,k=1}^\obs \levvaidya_{x, i}\levvaidya_{x, k}
\localslackvaidya_{x, i, k}^2
\stackrel{\mathmakebox[\widthof{====}]{(\mathrm{Lem.}~\ref{ppt:all_properties}\ref{item:theta_properties}
    ) }} {\leq}\,\, \sum_{i=1}^\obs \levvaidya_{x, i}
\localslackvaidya_{x, i}
\stackrel{\mathmakebox[\widthof{====}]{(\mathrm{Lem.}~\ref{ppt:all_properties}\ref{item:theta_sum}
    )}} {\leq} n.
\end{align*}
Putting the pieces together yields the claimed bound on $C_2$.


\subsubsection{Proof of bound~\eqref{eq:fourth}}
\label{ssub:proof_of_part_item:fourth}

Observe that $\hat{a}_{x,i}^\top \rvg\sim \NORMAL\parenth{0,
  \localslackvaidya_{x, i}}$ and hence $\Exs\parenth{\hat{a}_{x,i}^\top
  \rvg}^8 = 105\, \localslackvaidya_{x, i}^4$.  Thus we have
\begin{align*}
\Exs \parenth{\sum_{i=1}^\obs \levvaidya_{x, i}
  \parenth{\hat{a}_{x,i}^\top \rvg}^4}^2 &
\stackrel{\mathmakebox[\widthof{====}]{\mathrm{C-S} }} {\leq}
\sum_{i,j=1}^\obs \levvaidya_{x, i} \levvaidya_{x, j}
\parenth{\Exs\parenth{\hat{a}_{x,i}^\top \rvg}^8}^{\frac{1}{2}}
\parenth{\Exs\parenth{\hat{a}_{x,j}^\top \rvg}^8}^{\frac{1}{2}} \\
& \stackrel{\mathmakebox[\widthof{====}]{}} {=} 105 \sum_{i,j=1}^\obs
\levvaidya_{x, i} \levvaidya_{x, j} \localslackvaidya_{x, i}^2
\localslackvaidya_{x, j}^2 \\ &
\stackrel{\mathmakebox[\widthof{====}]{}} {=} 105
\parenth{\sum_{i=1}^\obs \levvaidya_{x, i} \localslackvaidya_{x,
    i}^2}^2 \\
&
\stackrel{\mathmakebox[\widthof{====}]{(\mathrm{Lem.}~\ref{ppt:all_properties}\ref{item:theta_square_sigma_sum_bound}
    )}} {\leq} \, 105 \obs\dims.
\end{align*}
Applying the bound~\eqref{vaidya_EqnJansonBound} with $\degree = 4, t =
\parenth{\frac{e}{2} \log(4/\epsilon)}^2$ yields the result. We also verify that for $\epsilon \in (0, \vaidyaepsilonconst]$, we have $t \geq \parenth{2e}^2$


\subsection{Proof of Lemma~\ref{lemma:gradient_and_hessian_and_bounds}}
\label{sub:proofs_of_derivatives}
We now derive the different expressions for derivatives and prove the
bounds for Hessians of $x\mapsto \vaidyaphi_{x, i}$, $i \in [\obs]$ and $x
\mapsto \logdetvaidya_x$.  In this section we use the simpler notation
$\hessbarr_x \defn \hesslogbarr_x$.

\subsubsection{Gradient of $\levvaidya$}
\label{ssub:gradient_of_lev_}

Using $ \slack_{x+h, i} = (b_i- a_i^\top (x+h )) = \slack_{x, i} -
a_i^\top h $, we define the Hessian difference matrix
\begin{align}
\label{eq:hessian_difference}
\Delta^\hessbarr_{x, h} \defn \hessbarr_{x+h}-\hessbarr_x =
\sum_{i=1}^\obs a_i a_i^\top \parenth{ \frac{1}{(\slack_{x,
      i}-a_i^\top h)^2} - \frac{1}{\slack_{x, i}^2}}.
\end{align}
Up to second order terms, we have
\begin{subequations}
\begin{align}
\frac{1}{\slack_{x+ h, i}^2} &= \frac{1}{\slack_{x,
    i}^2}\brackets{1+\frac{2a_i^\top h}{\slack_{x, i} } +
  \frac{3(a_i^\top h)^2}{\slack_{x, i}^2} } +
\order{\vecnorm{h}{2}^3},
\label{eq:vaidya_expansion_of_s2}\\
\Delta^\hessbarr_{x, h} &= \sum_{i=1}^\obs
\frac{a_ia_i^\top}{\slack_{x, i}^2}\brackets{\frac{2a_i^\top
    h}{\slack_{x, i} } + \frac{3(a_i^\top h)^2}{\slack_{x, i}^2} } +
\order{\vecnorm{h}{2}^3},
\label{eq:expansion_of_H} \\
a_i^T \hessbarr_{x+ h}^{-1} a_i &= a_i^\top \hessbarr_x^{-1}a_i -
a_i^\top \hessbarr_x^{-1} \Delta^\hessbarr_{x, h} \hessbarr_x^{-1} a_i
+ a_i^\top \hessbarr_x^{-1} \Delta^\hessbarr_{x, h} \hessbarr_x^{-1}
\Delta^\hessbarr_{x, h} \hessbarr_x^{-1} a_i +
\order{\vecnorm{h}{2}^3}.
\label{eq:expansion_of_aH}
\end{align}
\end{subequations}
Collecting different first order terms in $\levvaidya_{x+ h, i} -
\levvaidya_{x, i}$, we obtain
\begin{align*}
\levvaidya_{x+ h, i} - \levvaidya_{x, i} & = 2\, \frac{a_i^\top
  \hessbarr_x^{-1}a_i}{\slack_{x, i}^2} \frac{a_i^\top h}{\slack_{x,
    i}} -2\, \frac{a_i^\top \hessbarr_x^{-1} \parenth{\sum_{j=1}^\obs
    \frac{a_ja_j^\top}{\slack_{x, j}^2} \frac{a_j^\top h}{\slack_{x,
        j}}} \hessbarr_x^{-1} a_i}{\slack_{x, i}^2} +
\order{\vecnorm{h}{2}^2} \\
& = 2\, \brackets{\levvaidya_{x, i} \frac{a_i^\top h}{\slack_{x, i}} -
  \sum_{j=1}^\obs \levvaidya_{x, i, j}^2 \, \frac{a_j^\top
    h}{\slack_{x, j}} } + \order{\vecnorm{h}{2}^2}\\
& = 2 \, [(\vaidyalevmatrix_x -
  \vaidyaprojmatrix_x^{(2)})\slackmatrix_x^{-1}A]_i\, h +
\order{\vecnorm{h}{2}^2}.
\end{align*}
Dividing both sides by $h$ and letting $ h \rightarrow 0$ yields the
claim.


\subsubsection{Gradient of $\vaidyaphi$}
\label{ssub:gradient_of_phiw_}

Using the chain rule and the fact that $\nabla \slack_{x, i} = -a_i$,
we find that
\begin{align*}
\nabla\vaidyaphi_{x, i} &= \frac{\nabla \levvaidya_{x, i}}{\slack_{x, i}^2}
- 2\, (\levvaidya_{x, i}+\vaidyabeta) \frac{\nabla \slack_{x,
    i}}{\slack_{x, i}^3} \\
& = \frac{2}{\slack_{x, i}^2} A^\top \slackmatrix_x^{-1}
\brackets{2\vaidyalevmatrix_x+\vaidyabeta\, \Ind-\vaidyaprojmatrix_x^{(2)}}
e_i,
\end{align*}
as claimed.


\subsubsection{Gradient of $\logdetvaidya$}
\label{ssub:gradient_of_logdetw_x_}

For convenience, let us restate equations~\eqref{eq:hat_a_i}
and~\eqref{eq:hat_a_hat_a_Ind}:
\begin{align*}
  \hat{a}_{x,i} = \frac{1}{\slack_{x, i}} \vaidya_x^{-{1}/{2}} a_i, \quad
  \text{and} \quad \sum_{i=1}^\obs \parenth{\levvaidya_{x, i} +
    \vaidyabeta} \hat{a}_{x,i} \hat{a}_{x,i}^\top = \Ind_\dims.
\end{align*}
For a unit vector $h$, we have
\begin{align}
h^\top \nabla \log \det \vaidya_x & = \lim_{\delta \rightarrow 0}
\frac{1}{\delta} \brackets{\trace \log \parenth{\sum_{i=1}^\obs
    \frac{\parenth{\levvaidya_{x+\delta h,
          i}+\vaidyabeta}}{\parenth{1-\delta a_i^\top h /\slack_{x,
          i}}^2} {\hat{a}_{x,i} \hat{a}_{x,i}^\top}} - \trace \log
  \parenth{\sum_{i=1}^\obs \parenth{\levvaidya_{x,
        i}+\vaidyabeta}\hat{a}_{x,i} \hat{a}_{x,i}^\top}}.
	\label{eq:grad_expression_W}
\end{align}
Let $\log L$ denote the logarithm of the matrix $L$.  Keeping track of
the first order terms on RHS of equation~\eqref{eq:grad_expression_W},
we find that
\begin{align*}
	& \trace \brackets{\log \parenth{\sum_{i=1}^\obs
      \parenth{\levvaidya_{x+\delta h, i}+\vaidyabeta} \frac{\hat{a}_{x,i}
        \hat{a}_{x,i}^\top}{\parenth{1-\delta a_i^\top h /\slack_{x,
            i}}^2}}} - \trace \brackets{\log \parenth{\sum_{i=1}^\obs
      \parenth{\levvaidya_{x, i}+\vaidyabeta}\hat{a}_{x,i}
      \hat{a}_{x,i}^\top}} \\
& = \trace \brackets{\log \parenth{\sum_{i=1}^\obs
      \parenth{\levvaidya_{x+\delta h, i}+\vaidyabeta + \delta h^\top
        \nabla \levvaidya_{x, i} } \parenth{1 + 2\delta \frac{a_i^\top
          h}{\slack_{x, i}^2}}}} - \trace \brackets{\log
    \parenth{\sum_{i=1}^\obs \parenth{\levvaidya_{x,
          i}+\vaidyabeta}\hat{a}_{x,i} \hat{a}_{x,i}^\top}} + \order{\delta^2}
  \\
& = \trace \brackets{\sum_{i=1}^\obs \delta \parenth{2
      \parenth{\levvaidya_{x, i}+\vaidyabeta}\frac{a_i^\top
        h}{\slack_{x, i}^2}+ h^\top \nabla \levvaidya_{x, i}}
    \hat{a}_{x,i} \hat{a}_{x,i}^\top} + \order{\delta^2} \\
& = \delta \parenth{\sum_{i=1}^\obs \parenth{2 \parenth{\levvaidya_{x,
          i}+\vaidyabeta}\frac{a_i^\top h}{\slack_{x, i}^2}+ h^\top
      \nabla \levvaidya_{x, i}} \localslackvaidya_i} +
  \order{\delta^2},
\end{align*}
where we have used the fact $\trace (\log \Ind) = 0$.  Letting $\delta
\rightarrow 0$ and substituting expression of $h^\top \nabla
\levvaidya_{x}$ from part~\ref{item:gradient_sigma}, we obtain
\begin{align*}
h^\top \nabla \log \det \vaidya_x = A_x\tp \parenth{4
  \vaidyalevmatrix_x + 2\vaidyabeta \Ind - 2 \vaidyaprojmatrix_x^{(2)}}
\vaidyalocalslackmatrix_x h.
\end{align*}


\subsubsection{Bound on Hessian $\nabla^2\vaidyaphi$}
\label{ssub:bound_on_nabla_2phiw}

In terms of the shorthand $E_{ii} = e_ie_i\tp$, we claim that for any
$h \in \realdim$,
\begin{align}
h^\top \nabla^2 \vaidyaphi_{x, i}h = \displaystyle\frac{2}{\slack_{x, i}^2}
h^\top A_x\tp \bigg[&
  E_{ii}\parenth{3\parenth{\vaidyalevmatrix_x+\vaidyabeta\Ind} +
    7\vaidyalevmatrix_x -8\diag(\vaidyaprojmatrix_x^{(2)}e_i)} E_{ii}\notag
  \\
  & + \diag(\vaidyaprojmatrix_xe_i) (4\vaidyaprojmatrix_x-3\Ind)
  \diag(\vaidyaprojmatrix_xe_i)\bigg] A_xh.
	\label{eq:hessian_beta}
\end{align}
Note that
\begin{align}
\vaidyaphi_{x+ h, i} - \vaidyaphi_{x, i} = \underbrace{\parenth{\frac{a_i^\top
      \hessbarr_{x+ h, i}^{-1}a_i}{\slack_{x+ h, i}^4} -
    \frac{a_i^\top \hessbarr_{x, i}^{-1} a_i}{\slack_{x, i}^4}
}}_{=:A_1} + \underbrace{\vaidyabeta \parenth{\frac{1}{\slack_{x+ h,
        i}^2} - \frac{1}{\slack_{x, i}^2} }}_{=:A_2}.
	\label{eq:delta_beta}
\end{align}
The second order Taylor expansion of $1/\slack^4_{x, i}$ is given by
\begin{align*}
\frac{1}{\slack_{x+ h, i}^4} &= \frac{1}{\slack_{x, i}^4}\brackets{1 +
  \frac{4 a_i^\top h}{\slack_{x, i} } + \frac{10 (a_i^\top
    h)^2}{\slack_{x, i}^2}} + \order{\vecnorm{h}{2}^3}.
\end{align*}
Let $B_1$ and $B_2$ denote the second order terms, i.e., the terms
that are of order $\order{\vecnorm{ h}{2}^2}$, in Taylor expansion of
$A_1$ and $A_2$ around $x$, respectively.
Borrowing terms from
equations~\eqref{eq:vaidya_expansion_of_s2}-\eqref{eq:expansion_of_aH} and
simplifying we obtain
\begin{align*}
B_1 &= 10 {\levvaidya_{x, i}} \frac{(a_i^\top h)^2}{\slack_{x, i}^2}
\!-\! 8 \frac{a_i^\top h}{\slack_{x, i}}\sum_{j=1}^\obs
\frac{\levvaidya_{x, i, j}^2}{\slack_{x, i}^2} \frac{a_j^\top
  h}{\slack_{x, j}} \!-\! 3\sum_{j=1}^\obs \frac{\levvaidya_{x, i,
    j}^2}{\slack_{x, i}^2} \frac{(a_j^\top h)^2}{\slack_{x, j}^2}
\!+\!  4 \sum_{j=1}^\obs \sum_{l=1}^\obs \frac{\levvaidya_{x, i,
    j}}{\slack_{x, i}} \levvaidya_{x, j, l} \frac{\levvaidya_{x, l,
    i}}{\slack_{x, i}} \frac{a_j^\top h}{\slack_{x, j}} \frac{a_l^\top
  h}{\slack_{x, l}},\\ \text{and } B_2 &= 3 \vaidyabeta
\frac{(a_i^\top h)^2}{\slack_{x, i}^2}.
\end{align*}
Observing that the second order term in the Taylor expansion of
$\vaidyaphi_{x+h, i}$ around $x$, is exactly $\frac{1}{2} h^\top \nabla^2
\vaidyaphi_{x, i}h$ yields the claim~\eqref{eq:hessian_beta}.  We now turn
to prove the bound on the directional Hessian.  Recall $\projD_{x, i}
= a_i^\top h/\slack_{x, i}$.  We have
\begin{align*}
&\slack_{y,i}^2\abss{\frac{1}{2} h^\top \nabla^2 \vaidyaphi_{x, i} h} \\
&\quad = \abss{3 \parenth{\levvaidya_{x,i} \!+\!
      \vaidyabeta}\projD_{x, i}^2 \!+\! 7 \levvaidya_{x, i} \projD_{x,
      i}^2 \!-\! 8 \sum_{j=1}^\obs \levvaidya_{x, i, j}^2 \projD_{x,
      j} \projD_{x, i} \!-\! 3 \sum_{j=1}^\obs \levvaidya_{x, i, j}^2
    \projD_{x, j}^2 \!+\! 4 \sum_{j,k=1}^\obs \levvaidya_{x,
      i,j}\levvaidya_{x, j, k}\levvaidya_{x, k, i} \projD_{x, j}
    \projD_{x, k}} \\ &\quad \stackrel{(i)}{\leq} 10
  \parenth{\levvaidya_{x,i} + \vaidyabeta}\projD_{x, i}^2 + 8
  \sum_{j=1}^\obs \levvaidya_{x, i, j}^2 \abss{\projD_{x, i}
    \projD_{x, j}} + 7 \sum_{j=1}^\obs \levvaidya_{x, i, j}^2
  \projD_{x, j}^2 \\
  & \quad \stackrel{(ii)}{\leq} 10 \parenth{\levvaidya_{x,i} +
    \vaidyabeta}\projD_{x, i}^2 + 4 \sum_{j=1}^\obs \levvaidya_{x, i,
    j}^2 \parenth{\projD_{x, i}^2 + \projD_{x, j}^2} + 7
  \sum_{j=1}^\obs \levvaidya_{x, i, j}^2 \projD_{x, j}^2 \\
  & \quad \stackrel{(iii)}{\leq} 10 \parenth{\levvaidya_{x,i} +
    \vaidyabeta}\projD_{x, i}^2 + 4 \sum_{j=1}^\obs \levvaidya_{x,
    i}\projD_{x, i}^2 +4 \sum_{j=1}^\obs \levvaidya_{x, i, j}^2
  \projD_{x, j}^2 + 7 \sum_{j=1}^\obs \levvaidya_{x, i, j}^2
  \projD_{x, j}^2,\\
  & \quad \stackrel{(iv)}{\leq} 14 \parenth{\levvaidya_{x,i} +
    \vaidyabeta}\projD_{x, i}^2 + 11 \sum_{j=1}^\obs \levvaidya_{x, i,
    j}^2 \projD_{x, j}^2,
\end{align*}
 where step (i) follows from the fact that $\diag(\vaidyaprojmatrix_y e_i)
 \vaidyaprojmatrix_y \diag(\vaidyaprojmatrix_y e_i) \preceq \diag(\vaidyaprojmatrix_y
 e_i)\diag(\vaidyaprojmatrix_y e_i)$ since $\vaidyaprojmatrix_y$ is an orthogonal
 projection matrix; step $(ii)$ follows from AM-GM inequality; step
 $(iii)$ follows from the symmetry of indices $i$ and $j$ and
 Lemma~\ref{ppt:all_properties}\ref{item:sigma_properties}, and step
 $(iv)$ from the fact that $\levvaidya_{x, i} \leq \levvaidya_{x, i} +
 \vaidyabeta$.

\subsubsection{Bound on Hessian $\hesslogdetV$}
\label{ssub:bound_on_hessdet_}

We have
\begin{align}
\frac{1}{2} h^\top \parenth{\nabla^2 \log \det \vaidya_x} h = & \frac{1}{2}\lim_{\delta
  \rightarrow 0} \frac{1}{\delta^2} \Bigg[\trace \log
  \parenth{\sum_{i=1}^\obs \frac{\parenth{\levvaidya_{x+\delta h,
          i}+\displaystyle\vaidyabeta}}{\parenth{1-\delta a_i^\top h
        /\slack_{x, i}}^2} {\hat{a}_{x,i} \hat{a}_{x,i}^\top}} \notag\\ & +\trace \log
  \parenth{\sum_{i=1}^\obs \frac{\parenth{\levvaidya_{x-\delta h,
          i}+\displaystyle\vaidyabeta}}{\parenth{1+\delta a_i^\top h
        /\slack_{x, i}}^2} {\hat{a}_{x,i} \hat{a}_{x,i}^\top}} \notag\\ &-
  2\trace \log \parenth{\sum_{i=1}^\obs
    \parenth{\levvaidya_{x}+\vaidyabeta} \hat{a}_{x,i} \hat{a}_{x,i}^\top}
  \Bigg].
\label{eq:hessian_L_exp}
\end{align}
Up to second order terms, we have
\begin{align*}
& \trace \brackets{\log \parenth{\sum_{i=1}^\obs
      \parenth{\levvaidya_{x+\delta h, i}+\vaidyabeta} \frac{\hat{a}_{x,i}
        \hat{a}_{x,i}^\top}{\parenth{1-\delta a_i^\top h /\slack_{x, i}
        }^2}}} \\
& = \trace \brackets{\log \parenth{\sum_{i=1}^\obs
      \parenth{\levvaidya_{x, i}+\vaidyabeta + \delta h^\top \nabla
        \levvaidya_{x, i} + \frac{1}{2} \delta^2 h^\top \nabla^2
        \levvaidya_{x, i} h} \parenth{1+ 2 \delta \frac{a_i^\top
          h}{\slack_{x, i}}+3 \delta^2 \parenth{\frac{a_i^\top
            h}{\slack_{x, i}}}^2} \hat{a}_{x,i} \hat{a}_{x,i}^\top} } \\ & =
  \trace \brackets{\sum_{i=1}^\obs \parenth{\levvaidya_{x,
        i}+\vaidyabeta + \delta h^\top \nabla \levvaidya_{x, i} +
      \frac{1}{2} \delta^2 h^\top \nabla^2 \levvaidya_{x, i} h}
    \parenth{1+ 2 \delta \frac{a_i^\top h}{\slack_{x, i}}+3 \delta^2
      \parenth{\frac{a_i^\top h}{\slack_{x, i}}}^2} \hat{a}_{x,i}
    \hat{a}_{x,i}^\top} \\ & \quad- \trace \brackets{\frac{1}{2}
    \parenth{\sum_{i=1}^\obs \parenth{\levvaidya_{x, i}+\vaidyabeta +
        \delta h^\top \nabla \levvaidya_{x, i} + \frac{1}{2} \delta^2
        h^\top \nabla^2 \levvaidya_{x, i} h} \parenth{1+ 2 \delta
        \frac{a_i^\top h}{\slack_{x, i}}+3 \delta^2
        \parenth{\frac{a_i^\top h}{\slack_{x, i}}}^2} \hat{a}_{x,i}
      \hat{a}_{x,i}^\top}^2 }.
\end{align*}
We can similarly obtain the second order expansion of the term $\trace
\log \parenth{\sum_{i=1}^\obs \frac{\parenth{\levvaidya_{x-\delta h,
        i} +\vaidyabeta}}{\parenth{1+\delta a_i^\top h /\slack_{x,
        i}}^2} {\hat{a}_{x,i} \hat{a}_{x,i}^\top}}$.  Recall $\projD_{x, i} =
\frac{a_i^\top h}{\slack_{x, i}}$.  Using
part~\ref{item:gradient_sigma} to substitute $h^\top \nabla
\levvaidya_{x, i}$, we obtain
\begin{align}
 \frac{1}{2} h^\top \parenth{\nabla^2 \log \det \vaidya_x} h & = \sum_{i=1}^\obs
 \parenth{3 \parenth{\levvaidya_{x,i} +\vaidyabeta} \projD_{x, i}^2 +
   4 \parenth{\levvaidya_{x, i} \projD_{x, i}^2 - \sum_{j=1}^\obs
     \levvaidya_{x, i,j}^2 \projD_{x, i} \projD_{x, j}} + \frac{1}{2}
   h^\top \nabla^2 \levvaidya_{x,i} h } \localslackvaidya_i \notag \\
 & - 2 \bigg[ \sum_{i, j=1}^\obs \parenth{2\levvaidya_{x,
       i}+\vaidyabeta} \parenth{2\levvaidya_{x, j}+\vaidyabeta}
   \projD_{x, i} \projD_{x, j} \localslackvaidya_{x, i, j}^2 -
   2\sum_{i,j,k = 1}^\obs \parenth{2\levvaidya_{x, i}+\vaidyabeta}
   \levvaidya_{x, j, k}^2 \localslackvaidya_{x, i, k}^2 \projD_{x, i}
   \projD_{x, j} \notag \\ &+ \sum_{i, j, k, l = 1}^\obs
   \levvaidya_{x, i,l}^2 \levvaidya_{x, j, k}^2 \localslackvaidya_{x,
     k, l}^2 \projD_{x, i} \projD_{x, j}\bigg]. \label{eq:hessian_L}
\end{align}
We claim that the directional Hessian $h^\top \nabla^2 \levvaidya_{x,
  i} h$ is given by
\begin{align}
h^\top \nabla^2 \levvaidya_{x, i} h = 2\, h^\top A_x\tp
\brackets{E_{ii}(3\vaidyalevmatrix_x-4\diag(\vaidyaprojmatrix_x^{(2)}e_i))
  E_{ii}+ \diag(\vaidyaprojmatrix_xe_i) (4\vaidyaprojmatrix_x-3\Ind)
  \diag(\vaidyaprojmatrix_xe_i)} A_xh.
	\label{eq:hessian_sigma}
\end{align}
Assuming the claim at the moment we now bound $ \abss{h^\top \nabla^2
  \logdetvaidya_x h} $.  To shorten the notation, we drop the
$x$-dependence of the terms $\levvaidya_{x, i}, \levvaidya_{x, i, j},
\localslackvaidya_{x, i}$ and $\projD_{x, i}$.  Since $\vaidyaprojmatrix_x$
is an orthogonal projection matrix, we have
\begin{align*}
  \diag(\vaidyaprojmatrix_x e_i) \vaidyaprojmatrix_x \diag(\vaidyaprojmatrix_x e_i)
  \preceq \diag(\vaidyaprojmatrix_x e_i)\diag(\vaidyaprojmatrix_x e_i).
\end{align*}
Using this fact and substituting the expression for $h^\top \nabla^2
\levvaidya_{x, i} h$ from equation~\eqref{eq:hessian_sigma} in
equation~\eqref{eq:hessian_L}, we obtain
\begin{align*}
  &\abss{h^\top \hesslogdetV_x h}\\ &\begin{aligned} \mathllap{}
    &\stackrel{}{\leq} \sum_{i=1}^\obs \bigg[3
      \bigg(\levvaidya_i+\vaidyabeta\bigg) \projD_{i}^2 + 4
      \big(\levvaidya_i \projD_{i}^2 + \sum_{j=1}^\obs
      \levvaidya_{i,j}^2 \projD_{i} \projD_{j}\big) + 3 \levvaidya_i
      \projD_{i}^2 + 4 \sum_{j=1}^\obs \levvaidya_{i, j}^2 \projD_{i}
      \projD_{j} + 7\sum_{j= 1}^\obs \levvaidya_{i,j}^2 \projD_{j}^2
      \bigg] \localslackvaidya_i\\ &\qquad + \bigg[ 8 \sum_{i,
        j=1}^\obs \parenth{\levvaidya_i+\vaidyabeta}
      \parenth{\levvaidya_j+\vaidyabeta} \projD_{i} \projD_{j}
      \localslackvaidya_{i, j}^2 + 8 \sum_{i,j,k = 1}^\obs
      \parenth{\levvaidya_i+\vaidyabeta} \levvaidya_{j, k}^2
      \localslackvaidya_{i, k}^2 \projD_{i} \projD_{j} + 2 \sum_{i, j,
        k, l = 1}^\obs \levvaidya_{i,l}^2 \levvaidya_{j, k}^2
      \localslackvaidya_{k, l}^2 \projD_{i} \projD_{j}\bigg].
  \end{aligned}
\end{align*}
Rearranging terms, we find that
\begin{align*}
	&\abss{h^\top \hesslogdetV_x h}\\ &\begin{aligned} \mathllap{}
    &\leq \sum_{i=1}^\obs \Bigg[10 \parenth{\levvaidya_i+\vaidyabeta}
      \projD_{i}^2 + 8 \sum_{j=1}^\obs \levvaidya_{i, j}^2 \projD_{i}
      \projD_{j} + 7 \sum_{j= 1}^\obs \levvaidya_{i,j}^2 \projD_{j}^2
      \Bigg] \localslackvaidya_i\\ &\qquad + \Bigg[ 8 \sum_{i,
        j=1}^\obs \parenth{\levvaidya_i+\vaidyabeta}
      \parenth{\levvaidya_j+\vaidyabeta} \projD_{i} \projD_{j}
      \localslackvaidya_{i, j}^2 + 8 \sum_{i,j,k = 1}^\obs
      \parenth{\levvaidya_i+\vaidyabeta} \levvaidya_{j, k}^2
      \localslackvaidya_{i, k}^2 \projD_{i} \projD_{j} + 2 \sum_{i, j,
        k, l = 1}^\obs \levvaidya_{i,l}^2 \levvaidya_{j, k}^2
      \localslackvaidya_{k, l}^2 \projD_{i} \projD_{j} \Bigg]
  \end{aligned}
  \\ &\begin{aligned} \mathllap{} &\stackrel{(i)}{\leq}
    \sum_{i=1}^\obs \Bigg[10 \parenth{\levvaidya_i+\vaidyabeta}
      \projD_{i}^2 + 4 \sum_{j=1}^\obs \levvaidya_{i, j}^2
      \parenth{\projD_{i}^2+\projD_{j}^2} + 7\sum_{j= 1}^\obs
      \levvaidya_{i,j}^2 \projD_{j}^2 \Bigg]
    \localslackvaidya_i\\ &\qquad + \Bigg[ 4\! \sum_{i, j=1}^\obs
      \!\!\big(\levvaidya_i+\vaidyabeta\big)
      \big(\levvaidya_j+\vaidyabeta\big) \localslackvaidya_{i, j}^2
      (\projD_{i}^2+\projD_{j}^2) + 4 \!\!\!\sum_{i,j,k = 1}^\obs\!
      \big(\levvaidya_i+\vaidyabeta\big) \levvaidya_{j, k}^2
      \localslackvaidya_{i, k}^2 (\projD_{i}^2+\projD_{j}^2) +\!\!\!\!
      \sum_{i, j, k, l = 1}^\obs \!\!\! \levvaidya_{i,l}^2
      \levvaidya_{j, k}^2 \localslackvaidya_{k, l}^2
      (\projD_{i}^2+\projD_{j}^2)\Bigg]
  \end{aligned}
 \end{align*}
where in step (i) we have used the AM-GM inequality.  Simplifying
further, we obtain
\begin{align*}
\abss{h^\top \hesslogdetV_y h} & \begin{aligned} \mathllap{} &\leq
  \sum_{i=1}^\obs \Bigg[14 \parenth{\levvaidya_i+\vaidyabeta}
    \projD_{i}^2 + 11 \sum_{j=1}^\obs \levvaidya_{i, j}^2 \projD_{j}^2
    \Bigg] \localslackvaidya_i + \Bigg[ \sum_{i=1}^\obs 12
    \parenth{\levvaidya_i+\vaidyabeta} \localslackvaidya_i
    \projD_{i}^2 + \sum_{i, j=1}^\obs 6 \levvaidya_{i, j}^2
    \localslackvaidya_i \projD_{j}^2 \Bigg]
  		\end{aligned}\\
  	&\begin{aligned} \mathllap{} &= 26 \sum_{i=1}^\obs
           \parenth{\levvaidya_i+\vaidyabeta} \localslackvaidya_i
           \projD_{i}^2 + 17 \sum_{i,j=1}^\obs \levvaidya_{i, j}^2
           \localslackvaidya_i \projD_{j}^2.
  	\end{aligned}
\end{align*}
Dividing both sides by two completes the proof.


\paragraph{Proof of claim~\eqref{eq:hessian_sigma}:}
\label{par:proof_of_claim_eq:hessian_sigma}

In order to compute the directional Hessian of $x\mapsto
\levvaidya_{x, i}$, we need to track the second order terms in
equations~\eqref{eq:vaidya_expansion_of_s2}-\eqref{eq:expansion_of_aH}.
Collecting the second order terms (denoted by $\levvaidya^{(2)}_h$) in
the expansion of $\levvaidya_{x+ h, i} - \levvaidya_{x, i}$, we obtain
\begin{align*}
 \levvaidya^{(2)}_h = 3\, \frac{a_i^\top
   \hessbarr_x^{-1}a_i}{\slack_{x, i}^2} \frac{(a_i^\top
   h)^2}{\slack_{x, i}^2} &- 4\, \frac{a_i^\top \hessbarr_x^{-1}
   \parenth{\sum_{j=1}^\obs \frac{a_ja_j^\top}{\slack_{x, j}^2}
     \frac{a_j^\top h}{\slack_{x, j}}} \hessbarr_x^{-1}
   a_i}{\slack_{x, i}^2} \frac{a_i^\top h}{\slack_{x, i}} \\ & - 3\,
 \frac{a_i^\top \hessbarr_x^{-1} \parenth{\sum_{j=1}^\obs
     \frac{a_ja_j^\top}{\slack_{x, j}^2} \frac{(a_j^\top
       h)^2}{\slack_{x, j}^2}} \hessbarr_x^{-1} a_i}{\slack_{x, i}^2}
 \\
 & + 4\, \frac{a_i^\top \hessbarr_x^{-1} \parenth{\sum_{j=1}^\obs
     \frac{a_j a_j^\top}{\slack_{x, j}^2} \frac{a_j^\top h}{\slack_{x,
         j}}} \hessbarr_x^{-1} \parenth{\sum_{l=1}^\obs \frac{a_l
       a_l^\top}{\slack_{x, l}^2} \frac{a_l^\top h}{\slack_{x, l}}}
   a_i}{\slack_{x, i}^2}.
\end{align*}

We simply each term on the RHS one by one.  Simplifying the first
term, we obtain
\begin{align*}
3\, \frac{a_i^\top \hessbarr_x^{-1}a_i}{\slack_{x, i}^2}
\frac{(a_i^\top h)^2}{\slack_{x, i}^2} = 3\, \levvaidya_{x, i}
\projD_{x, i}^2 = h^\top 3\, A_x\tp E_{ii} \vaidyalevmatrix_x E_{ii}
A_x \, h.
\end{align*}
For the second term, we have
\begin{align*}
4 \, \frac{a_i^\top \hessbarr_x^{-1} \parenth{\sum_{j=1}^\obs
    \frac{a_ja_j^\top}{\slack_{x, j}^2} \frac{a_j^\top h}{\slack_{x,
        j}}} \hessbarr_x^{-1} a_i}{\slack_{x, i}^2} \frac{a_i^\top
  h}{\slack_{x, i}} &= 4 \, \projD_{x, i} \sum_{j=1}^\obs
\levvaidya_{x, i, j}^2 \, \projD_{x, j}\\
&= 4\, h^\top A_x\tp E_{ii} \diag\parenth{\vaidyaprojmatrix_x^{(2)}e_i}
E_{ii} A_x h.
\end{align*}
The third term can be simplified as follows:
\begin{align*}
	3\, \frac{a_i^\top \hessbarr_x^{-1} \parenth{\sum_{j=1}^\obs
            \frac{a_ja_j^\top}{\slack_{x, j}^2} \frac{(a_j^\top
              h)^2}{\slack_{x, j}^2}} \hessbarr_x^{-1} a_i}{\slack_{x,
            i}^2} &= 3 \, \sum_{j=1}^\obs \levvaidya_{x, i, j}^2
        \projD_{x, j}^2 \\
        & = 3\, h^\top A_x\tp \diag\parenth{\vaidyaprojmatrix_x e_i}
        \diag\parenth{\vaidyaprojmatrix_x e_i} A_x h
\end{align*}
For the last term, we find that
\begin{align*}
	4\, \frac{a_i^\top \hessbarr_x^{-1} \parenth{\sum_{j=1}^\obs
            \frac{a_ja_j^\top}{\slack_{x, j}^2} \frac{a_j^\top
              h}{\slack_{x, j}}} \hessbarr_x^{-1}
          \parenth{\sum_{l=1}^\obs \frac{a_la_l^\top}{\slack_{x, l}^2}
            \frac{a_l^\top h}{\slack_{x, l}}} a_i}{\slack_{x, i}^2} &=
        4\, \sum_{j,l=1}^\obs \levvaidya_{x, i, j}\, \levvaidya_{x, j,
          l}\, \levvaidya_{x, l, i}\, \projD_{x, j}\, \projD_{x, l} \\
        & = 4\, h^\top A_x\tp \diag\parenth{\vaidyaprojmatrix_x
          e_i}\vaidyaprojmatrix_x \diag\parenth{\vaidyaprojmatrix_x e_i} A_x h.
\end{align*}
Putting together the pieces yields the
expression~\eqref{eq:hessian_sigma}.

\section{Analysis of the John walk}
\label{sec:proof_john_walk}

We recap the key ideas of the John walk for convenience.
We have designed a new proposal distribution by making use of an \emph{optimal set of weights}
to define the new covariance structure for the Gaussian proposals,
where optimality is defined with respect to the convex program defined
below~
\eqref{eq:john_convex_john_weights}.
The optimality condition is closely related to the problem of finding the largest ellipsoid at
any interior point of the polytope, such that the ellipsoid is contained within the polytope.
This problem of finding the largest ellipsoid was first studied by~\cite{joh48} who showed that each convex body in $\realdim$ contains a unique ellipsoid of maximal volume.
More recently,~\cite{lee2014path} make use of approximate John Ellipsoids to improve
the convergence rate of interior point methods for linear programming.
We refer the readers to their paper for more discussion about the use of John Ellipsoids
for optimization problems.
In this work, we make use of these ellipsoids for designing sampling algorithms with better theoretical
bounds on the mixing times.

The vector $\johnweights_{x} = \parenth{\johnweights_{x, 1}, \ldots, \johnweights_{x, \obs}}\tp$ defined in the John walk's inverse covariance matrix~\eqref{eq:john_covariance} is
computed by solving the following optimization problem:
\begin{align}
  \johnweights_x = \arg\min_{\weights \in \real^\obs} c_x\parenth{\weights} := \sum_{i=1}^\obs \weights_i - \frac{1}{\johnalpha} \log \det \parenth{A\tp \slackmatrix_x^{-1} \weightmatrix^{\johnalpha} \slackmatrix_x^{-1} A} - \johnbeta \sum_{i=1}^\obs \log \weights_i,
  \label{eq:john_convex_john_weights}
\end{align}
where the parameters $\johnalpha, \johnbeta$ are given by
\begin{align*}
  \johnalpha = 1 - \frac{1}{\log_2 \parenth{2\obs/\dims}}
   \quad\text{ and } \quad
   \johnbeta = \frac{\dims}{2\obs},
\end{align*}
and $\weightmatrix$ denotes an $\obs \times \obs$ diagonal matrix with $\weightmatrix_{ii} = \weights_i$ for each $i \in [\obs]$.
In particular, for our proposal the inverse covariance matrix is proportional to $\john_x$, where
\begin{align}
  \john_x =\sum_{i=1}^\obs \johnweights_{x, i}\frac{a_i a_i\tp} {(b_i -a_i\tp x)^2}.
  \label{eq:john_define_Jx}
\end{align}
where $\johnkappa \defn \johnkappa_{\obs, \dims} = \log_2(2\obs/\dims) = (1-\johnalpha)^{-1}$.

Recall that for John walk with parameter $\frac{r}{\dims^{3/4} \johnkappa^2}$, the proposals at state $x$ are
drawn from the multivariate Gaussian distribution given by $\NORMAL\parenth{x, \frac{r^2}{\dims^{3/2} \johnkappa^4} \john_x^{-1}}$,
which we denote by $\proposal_x^\tagjohn$.
In particular, the proposal density at point $x \in \intP$ is given by
\begin{align}
  \label{eq:john_proposal_density}
  \density_x(z)
  := \density(x, z)
  = \sqrt{\det{\john_x}} \parenth{\frac{\johnkappa^4 \dims^{3/2}}{2\pi r^2}}^{\dims/2} \exp\parenth{-\frac{\johnkappa^4 \dims^{3/2}}{2r^2} \ (z-x)\tp \john_x (z-x) }.
\end{align}
Here we restate our result for the mixing time of the John walk.
\setcounter{theorem}{1}
\begin{theorem}
\label{thm:john_walk_1}
 Let $\initial$ be any distribution that is $M$-warm with respect to
$\stationary$ and let \mbox{$\obs < \exp(\sqrt{\dims})$}.  Then for any $\delta \in (0, 1]$, the John walk with parameter
  $\rparam_{\fulltagjohn} = \johnradiusconst$ satisfies
\begin{align*}
    \tvnorm{\plaintransition_{\fulltagjohn(\rparam)}^k(\initial) -\stationary} \leq
    \delta \qquad \text{ for all } k \geq C \; \dims^{2.5}\,
    \log_2^4\parenth{\frac{2\obs}{\dims}} \, \log\parenth{\frac{\sqrt{M}}{\delta}}.
  \end{align*}
\end{theorem}

\subsection{Auxiliary results}
We begin by proving basic properties of the weights~$\johnweights_x$ which are used throughout the paper. For $x \in \intP, w \in \real^\obs_{++}$, define the projection matrix $\johnprojmatrix_{x, w}$ as follows
\begin{align}
  \johnprojmatrix_{x, w} = \weightmatrix^{\alpha/2} A_x (A_x^\top \weightmatrix^\alpha A_x)^{-1} A_x^\top \weightmatrix^{\alpha/2},
  \label{eq:john_projection_matrix}
\end{align}
where $A_x = \slackmatrix_x^{-1} A$ and $\weightmatrix$ is the $\obs \times \obs$ diagonal matrix with $i$-th diagonal entry given by $\weights_i$.
Also, let
\begin{align}
\levjohn_{x, i} \defn \parenth{\johnprojmatrix_{x, \johnweights_x}}_{ii} \quad \text{for } x \in \intP \text{ and } i \in [\obs].
\label{eq:john_first_Leverage_definition}
\end{align}
Define the \emph{John slack sensitivity} $\johnlocalslack_x^\tagjohn$ as
  \begin{align}
    \johnlocalslack_{x} \defn \johnlocalslack_x^\tagjohn \defn \parenth{\frac{a_1\tp \john_x^{-1} a_1}{\slack_{x, 1}^2}, \ldots, \frac{a_\obs\tp \john_x^{-1} a_\obs}{\slack_{x, \obs}^2}}\tp \quad \text{for all } x \in \intP.
    \label{eq:john_john_theta_defn}
  \end{align}
Further, for any $x \in \intP$,  define the \emph{John local norm at $x$} as
\begin{align}
  \label{eq:john_def_john_local_norm}
  \vecnorm{\cdot}{\john_x}: v \mapsto \vecnorm{\john_x^{1/2}v}{2} =
  \sqrt{\sum_{i=1}^\obs \johnweights_{x, i} \frac{(a_i\tp
      v)^2}{\slack_{x, i}^2}}.
\end{align}
We now collect some basic properties of the weights~$\johnweights_x$ and the local sensitivity~$\johnlocalslack_x$ and restate parts of Lemma~\ref{lemma:first_bounds} for clarity here.
\begin{lemma}
  \label{lemma:john_first_bounds}
  For any $x \in \intP$, the following properties are true:
  \begin{enumerate}[label=(\alph*)]
  \item (Implicit weight formula) \label{item:john_weight_equality} $\johnweights_{x, i} = \levjohn_{x, i} + \johnbeta$ for all $i \in [\obs]$,
  \item (Uniformity)\label{item:john_sigma_bound}
  $\johnweights_{x, i} \in [\johnbeta, 1+\johnbeta]$  for
      all $i \in [\obs]$,
  \item (Total size)\label{item:john_weight_sum}
  $\sum_{i=1}^\obs \johnweights_{x, i} = 3\dims/2$, and
  \item (Slack sensitivity)\label{item:john_theta_bound} $\johnlocalslack_{x, i} \in \brackets{0, 4}$ for
  all $i \in [\obs]$.
  \end{enumerate}
\end{lemma}
\noindent Lemma~\ref{lemma:john_first_bounds} follows from Lemmas~14 and 15 by~\cite{lee2014path} and thereby we omit its proof.

Next, we state a key lemma that is crucial for proving the convergence rate of John walk.
In this lemma, we provide bounds on difference in total variation norm between the proposal distributions of two nearby points.
\begin{lemma}
  \label{lemma:john_walk_close_kernel}
  \begin{subequations}
    There exists a continuous non-decreasing function $\johnradiusbound: [0, 1/30] \rightarrow \real_+$ with $\johnradiusbound(\johnepsilonconst) \geq \johnradiusconst$, such that for any
    $\epsilon \in (0, \johnepsilonconst]$, the John walk with $r \in [0, \johnradiusbound(\epsilon)]$ satisfies
    \begin{align}
      \tvnorm{\proposal^\tagjohn_x-\proposal^\tagjohn_y} &\leq
      \epsilon, \quad \mbox{for all $x, y \in \intP$ such that
        $\vecnorm{x-y}{\john_x} \leq \displaystyle\frac{\epsilon
          r}{2\johnkappa^2\dims^{3/4}}$ } ,\quad \text{and}
      \label{eq:john_john_delta_p}\\
      \tvnorm{\plaintransition_{\fulltagjohn(\rparam)}(\diracdelta_x)-\proposal^\tagjohn_x} &\leq
      5\epsilon, \, \,\, \mbox{for all $x \in \intP$.}
      \label{eq:john_john_delta_q_p}
    \end{align}
  \end{subequations}
\end{lemma}
\noindent See Section~\ref{ssub:proof_of_lemma_lemma:john_walk_close_kernel} for its proof.

With these lemmas in hand, we are now ready to prove Theorem~\ref{thm:john_walk_1}.

\subsection{Proof of Theorem~\ref{thm:john_walk_1}}
\label{sub:proof_of_theorem_thm:john_walk}
The proof is similar to the proof of Theorem~1, and relies on the Lov{\'a}sz's Lemma.
Here onwards, we use the following simplified notation
\begin{align*}
  \transition_x = \plaintransition_{\fulltagjohn(\rparam)}(\diracdelta_x), \proposal_x = \proposal_x^\tagjohn \text{ and }\vecnorm{\cdot}{x} = \vecnorm{\cdot}{\john_x}.
\end{align*}
In order to invoke Lov{\'a}sz's Lemma, we need to show that for any two points $x, y \in \intP$ with small cross-ratio $d_\Pspace(x, y)$, the TV-distance $\tvnorm{\transition_x-\transition_y}$ is also small.

We proceed with the proof in two steps: (A) first, we relate the cross-ratio $d_\Pspace(x, y)$ to the John local norm of $x-y$ at $x$, and (B) we then use Lemma~\ref{lemma:john_walk_close_kernel} to show that if $x, y \in \intP$ are close in the John local-norm, then the transition kernels $\transition_x$ and $\transition_y$ are close in TV-distance.

\paragraph{Step~(A):}
\label{par:step_a_john}
We claim that for all $x, y \in \intP$, the cross-ratio can be lower
bounded as
\begin{align}
  d_\Pspace(x, y) \geq \frac{1}{\sqrt{3\dims/2}} \vecnorm{x-y}{x}.
  \label{eq:john_hilbert_to_local_john}
\end{align}
From the arguments in the proof of Theorem~1 (proof for the Vaidya Walk), we have
\begin{align}
d_\Pspace(x, y) \geq \max_{i \in [\obs]} \left\vert\frac{a_i\tp (x - y)}{\slack_{x, i}}
  \right\vert.
  \label{eq:john_relate_cross_ratio_john}
\end{align}
Using the fact that maximum of a set of non-negative numbers is greater than the
weighted mean of the numbers and Lemma~\ref{lemma:john_first_bounds}, we find that
\begin{align*}
  d_\Pspace(x, y) \geq \sqrt{ \frac{1}{\sum_{i=1}^\obs
    \johnweights_{x, i}} \sum_{i=1}^\obs
  \johnweights_{x, i} \frac{(a_i \tp (x-y))^2}
  {\slack_{x, i}^2}} = \frac{\vecnorm{x-y}{x}}{\sqrt{3\dims/2}},
\end{align*}
thereby proving the claim~\eqref{eq:john_hilbert_to_local_john}.


\paragraph{Step~(B):}
\label{par:step_b_john}
By the triangle inequality, we have
\begin{align*}
  \tvnorm{\transition_x - \transition_y} \leq
  \tvnorm{\transition_x-\proposal_x} + \tvnorm{\proposal_x-\proposal_y}
  + \tvnorm{\proposal_y - \transition_y}.
\end{align*}
Using Lemma~\ref{lemma:john_walk_close_kernel}, we obtain that
\begin{align*}
\tvnorm{\transition_x - \transition_y} & \leq 11 \epsilon, \quad
\mbox{$\forall x, y \in \intP$ such that $ \vecnorm{x-y}{x} \leq
  \displaystyle \frac{\epsilon
          r}{2\johnkappa^2\dims^{3/4}} $}.
\end{align*}
Consequently, the John walk satisfies the assumptions of
Lov{\'a}sz's Lemma with
\begin{align*}
  \Delta \defn \frac{1}{\sqrt{3\dims/2}} \cdot \frac{\epsilon
          r}{2\johnkappa^2\dims^{3/4}} \quad \mbox{and} \quad \rho \defn 1 - 11\epsilon.
\end{align*}
Plugging in $\epsilon = 1/30$, $r = 10^{-5}$, we obtain the claimed upper bound of $\order{\johnkappa^4 \dims^{5/2}}$
on the mixing time of the random walk.

\subsection{Proof of Lemma~\ref{lemma:john_walk_close_kernel}}
\label{ssub:proof_of_lemma_lemma:john_walk_close_kernel}
We prove the lemma for the following function,
\begin{align*}
  \johnradiusbound(\epsilon) &= \min \braces{\frac{1}{25\sqrt{1+\sqrt{2}\log(4/\epsilon)}}, \frac{\epsilon}{\parenth{2\sqrt{32}\tailconstjohn_{1, \epsilon}}}, \sqrt{\frac{\epsilon}{386\sqrt{24}\tailconstjohn_{2, \epsilon}}}, \frac{\epsilon}{5\sqrt{60} \tailconstjohn_{3, \epsilon}}, \right. \\
  &\left. \sqrt{\frac{\epsilon}{8\sqrt{1680} \tailconstjohn_{4, \epsilon}}}, \sqrt{\frac{\epsilon}{40\parenth{\tailconstjohn_{2, \epsilon}\tailconstjohn_{6, \epsilon} \sqrt{24}\sqrt{15120}}^{1/2}}}, \sqrt{\frac{\epsilon}{204800\tailconstjohn_{2, \epsilon}\sqrt{24}\log(32/\epsilon)}}}.
\end{align*}
where $\tailconstjohn_{1, \epsilon} = \log(2/\epsilon)$and $\tailconstjohn_{\degree, \epsilon} = \parenth{2e/\degree \cdot \log\parenth{16/\epsilon}}^{\degree/2}$ for \mbox{$\degree = 2,
     3, 4$ and $6$}. A numerical calculation shows that $\johnradiusbound(\johnepsilonconst) \geq \johnradiusconst$.


We now prove the two parts~\eqref{eq:john_john_delta_p}~\eqref{eq:john_john_delta_q_p} of the Lemma separately.

\subsubsection{Proof of claim~(\ref{eq:john_john_delta_p})}
\label{sssub:proof_of_claim_eq:john_john_delta_p}

Applying Pinsker's inequality, and plugging in the closed formed expression for the KL divergence between two Gaussian distributions we find that
\begin{align}
  \label{eq:john_final_KL_expression_john}
  \tvnorm{\proposal_x-\proposal_y}^2 \leq 2 \kldiv{\proposal_y}{\proposal_x}
  & = \trace(\john_x^{-1/2} \john_y \john_x^{-1/2})
    \!-\!\dims\!-\!\log\det(\john_x^{-1/2}
    \john_y \john_x^{-1/2}) + \frac{\johnkappa^4\dims^{3/2}}{r^2} \vecnorm{x\!-\!y}{x}^2 \notag \\
  & = \sum_{i=1}^\dims \parenth{\lambda_i - 1 + \log\frac{1}{\lambda_i}}
    + \frac{\johnkappa^4\dims^{3/2}}{r^2} \vecnorm{x-y}{x}^2,
\end{align}
where $\lambda_1, \ldots, \lambda_\dims >0$ denote the eigenvalues of the matrix $\john_x^{-1/2} \john_y \john_x^{-1/2}$.
To bound the expression~\eqref{eq:john_final_KL_expression_john}, we make use of the following lemma:

\begin{lemma}
  \label{lemma:john_close_hessian_eigenvalues}
  For any scalar $t \in [0, 1/64]$ and  pair of points $x, y \in
  \intP$ such that \mbox{$\vecnorm{x-y}{x} \leq t / \johnkappa^2$,}
  we have
  \begin{align*}
    \parenth{1-48t+4t^2} \Ind_\dims
    \preceq \john_x^{-1/2} \john_y \john_x^{-1/2} \preceq \parenth{1+48t+4t^2},
  \end{align*}
  where $\preceq$ denotes ordering in the PSD cone and $\Ind_\dims$ denotes the $\dims$-dimensional identity matrix.
\end{lemma}
\noindent See Section~\ref{sec:proof_of_lemma_lemma:john_close_hessian_eigenvalues}
 for the proof of this lemma.

For $\epsilon \in (0, \johnepsilonconst]$ and $r = \johnradiusconst$, we have $t = \epsilon r /(2\dims^{3/4}) \leq 1/64$, whence the eigenvalues $\{\lambda_i, i \in [\dims]\}$ can be sandwiched as
\begin{align}
  \label{eq:john_lambda_bounds_john}
  1 - \frac{24 \epsilon r}{\dims^{3/4}} + \frac{\epsilon^2r^2}{\dims^{3/2}}
  \leq \lambda_i
  \leq  1 + \frac{24 \epsilon r}{\dims^{3/4}} + \frac{\epsilon^2r^2}{\dims^{3/2}}
  \quad \mbox{for all $i \in \dims$}.
\end{align}
We are now ready to bound the TV distance between $\proposal_x$ and $\proposal_y$.
Using the bound~\eqref{eq:john_final_KL_expression_john} and the inequality $\log \omega \leq \omega - 1$, valid for $ \omega > 0$, we obtain
\begin{align*}
\tvnorm{\proposal_x-\proposal_y}^2 \leq \sum_{i=1}^\dims
\parenth{\lambda_i - 2 + \frac{1}{\lambda_i}} + \frac{\johnkappa^4\dims^{3/2}}{r^2} \vecnorm{x-y}{x}^2.
\end{align*}
Using the assumption that $\vecnorm{x-y}{x}\leq {\epsilon r}/\parenth{2\johnkappa^2\dims^{3/4}}$, and plugging in the bounds~\eqref{eq:john_lambda_bounds_john} for the eigenvalues $\{\lambda_i, i \in [\dims]\}$, we find that
\begin{align*}
  \sum_{i=1}^\dims \parenth{\lambda_i - 2 + \frac{1}{\lambda_i}}
  + \frac{\johnkappa^4\dims^{3/2}}{r^2} \vecnorm{x-y}{x}^2 \leq \frac{2000
    \epsilon^2 r^2}{\sqrt{\dims}} + \frac{\epsilon^2}{4}.
\end{align*}
In asserting this inequality, we have used the facts that
\begin{align*}
  \frac{1}{1 - 24\omega + \omega^2} \leq 1 + 24 \omega + 1000 \omega^2,
  \quad \text{and}\quad \frac{1}{1 + 24\omega + \omega^2} \leq 1 -
  24\omega + 1000 \omega^2 \quad \mbox{for all $\omega \in \brackets{0,
      \frac{1}{100}}$}.
\end{align*}
Note that for any $r \in [0, 1/100]$, we have that $2000 r^2/\sqrt{\dims} \leq 1/2$.
Putting the pieces together yields $\tvnorm{\proposal_x-\proposal_y} \leq \epsilon$, as claimed.

\subsubsection{Proof of claim~(\ref{eq:john_john_delta_q_p})}
\label{sssub:proof_of_claim_eq:john_john_delta_q_p}
We have
\begin{align}
  \label{eq:john_bound_on_jumping_out}
  \tvnorm{\proposal_x - \transition_x}
  &\leq \underbrace{\frac{3}{2} \proposal_x(\Pspace^c)}_{= : \; S_1} +
  \underbrace{1-\Exs_{z \sim \proposal_x} \brackets{\min\braces{1,
       \frac{\density_z(x)}{\density_x(z)}}}}_{= : \; S_2},
\end{align}
where $\Pspace^c$ denotes the complement of $\Pspace$.
We now show that $S_1 \leq \epsilon$ and $S_2 \leq 4\epsilon$, from which the claim follows.  \\

\noindent \textbf{Bounding the term $S_1$:}
Note that for $z \sim \NORMAL(x, \frac{r^2}{\johnkappa^2\dims^{3/2}} \john_x^{-1})$, we can write
\begin{align}
  \label{eq:john_z_x_relation}
  z \stackrel{d}{=} x + \frac{r}{\johnkappa\dims^{3/4}}
  \john_x^{-{1}/{2}} \rvg,
\end{align}
where $\rvg \sim \NORMAL\parenth{0, \Ind_\dims}$ and $\stackrel{d}{=}$
denotes equality in distribution.
Using equation~\eqref{eq:john_z_x_relation} and definition~\eqref{eq:john_john_theta_defn} of $\localslack_{x, i}$, we obtain the bound
\begin{align}
  \label{eq:john_bound_on_deviation}
  \frac{\parenth{a_i \tp \parenth{z-x}}^2 }{\slack_{x, i}^2} & =
  \frac{r^2}{\johnkappa^2\dims^{3/2}}
  \brackets{\frac{a_i \tp \john_x^{-{1}/{2}} \rvg }{\slack_{x,
       i}}}^2 \stackrel{(i)}{\leq} \frac{r^2}{\johnkappa^2\dims^{3/2}} \localslack_{x, i}
  \vecnorm{\rvg}{2}^2 \stackrel{(ii)}{\leq} \frac{4r^2}{\dims}
  \vecnorm{\rvg}{2}^2,
\end{align}
where step $(i)$ follows from Cauchy-Schwarz inequality, and step $(ii)$ from part~\ref{item:john_theta_bound} of Lemma~\ref{lemma:john_first_bounds}.
Define the events
\begin{align*}
  \mathcal{E} \defn \braces{\frac{r^2}{\dims} \vecnorm{\rvg}{2}^2 < \frac{1}{4}}
  \quad \text{and} \quad
  \mathcal{E}'\defn \braces{z \in \intP}.
\end{align*}
Inequality~\eqref{eq:john_bound_on_deviation} implies that $\mathcal{E}
\subseteq \mathcal{E}'$ and hence $\Prob\brackets{\mathcal{E}'} \geq
\Prob\brackets{\mathcal{E}}$.  Using a standard Gaussian tail bound
and noting that $r \leq \frac{1/2}{1 + \sqrt{2/\dims \log
    (2/\epsilon)}}$, we obtain $ \Prob\brackets{\mathcal{E}} \geq
1-\epsilon/2$ and whence $ \Prob\brackets{\mathcal{E}'} \geq
1-\epsilon/2$.  Thus, we have shown that $\Prob\brackets{z
  \notin \Pspace} \leq \epsilon/2$ which implies that $S_1 \leq
\epsilon$.  \\

\noindent \textbf{Bounding the term $S_2$:}
By Markov's inequality, we have
\begin{align}
  \label{eq:john_markov_inequality}
  \Exs_{z \sim \proposal_x} \brackets{\min\braces{1,
            \frac{\density_z(x)}{\density_x(z)}}} \geq \alpha
        \Prob\brackets{\density_z(x) \geq \alpha \density_x(z)} \quad
        \mbox{for all $\alpha \in (0,
          1]$}.
\end{align}
By definition~\eqref{eq:john_proposal_density} of $\density_x$, we obtain
\begin{align*}
\frac{\density_z(x)}{\density_x(z)} = \exp
\parenth{-\frac{\dims^{3/2}\johnkappa^4}{2r^2} \parenth{\vecnorm{z-x}{z}^2-
    \vecnorm{z-x}{x}^2 }+ \frac{1}{2} \parenth{\log \det \john_z -
    \log \det \john_x}}.
\end{align*}
The following lemma provides us with useful bounds on the two terms in
this expression, valid for any $x \in \intP$.
\begin{lemma}
 \label{lemma:john_change_in_log_det_and_local_norm}
  For any $\epsilon \in (0, \frac{1}{4}]$ and $r \in (0, \johnradiusbound(\epsilon)]$,
     we have
  \begin{subequations}
    \begin{align}
      \label{eq:john_whp_log_det_filter}
      \Prob_{z\sim \proposal_x}\brackets{\frac{1}{2}\log\det \john_z
       - \frac{1}{2}\log\det \john_x \geq -\epsilon } &\geq 1 -
      \epsilon,
      \quad  \text{and}
      \\
      \label{eq:john_whp_local_norm}
      \Prob_{z\sim \proposal_x}\brackets{\vecnorm{z - x}{z}^2 - \vecnorm{z -
      x}{x}^2 \leq 2 \epsilon \frac{r^2}{\johnkappa^4 \dims^{3/2}}} &\geq
      1 - \epsilon.
    \end{align}
   \end{subequations}
\end{lemma}
\noindent We provide the  of this lemma in Section~\ref{sec:proof_of_lemma_lemma:john_change_in_log_det_and_local_norm}.

Using Lemma~\ref{lemma:john_change_in_log_det_and_local_norm}, we now
 complete the proof of the Theorem~\ref{thm:john_walk_1}.  For $r \leq \johnradiusbound(\epsilon)$, we obtain
\begin{align*}
  \frac{\density_z(x)}{\density_x(z)} & \geq \exp\parenth{-2\epsilon}
  \geq 1- 2 \epsilon
\end{align*}
 with probability at least $1-2\epsilon$.
 Substituting $\alpha = 1- 2 \epsilon$ in inequality~\eqref{eq:john_markov_inequality} yields that $S_2 \leq 4 \epsilon$, as claimed.

\section{Technical Lemmas for the John walk} 
\label{sec:technical_lemmas}

We begin by summarizing a few key properties of various terms involved
in our analysis.

\begin{subequations}
Let $\johnlevmatrix_{x, w}$ be an $\obs \times \obs$ diagonal matrix defined as
\begin{align}
    \johnlevmatrix_{x, w} &= \diag\parenth{\levjohn_{x, w, i}, \ldots, \levjohn_{x, w, \obs}} \text{ where } \levjohn_{x, \johnweights_x, w, i} = (\johnprojmatrix_{x, w})_{ii}, i\in[\obs].
    \label{eq:john_john_Leverage_definition}
\end{align}Let $\johnprojmatrix_{x, w}^{(2)}$ denote the hadamard product of $\johnprojmatrix_{x, w}$ with itself.
Further define
\begin{align}
    \johnlaplacian_{x, w} := \johnlevmatrix_{x, w} - \johnprojmatrix_{x, w}^{(2)}.
    \label{eq:john_lambda}
\end{align}
\end{subequations}
\begin{subequations}
\cite{lee2014path} proved that the weight vector $\johnweights_x$ is the unique solution of the following fixed point equation:
\begin{align}
    \weights_i = \levjohn_{x, \weights, i} + \johnbeta, i \in [\obs].
    \label{eq:john_fixed_point_equation_john_weights}
\end{align}
To simplify notation, we use the following shorthands:
\begin{align}
    \levjohn_x = \levjohn_{x, \johnweights_x}, \quad \johnprojmatrix_x = \johnprojmatrix_{x, \johnweights_x}, \quad \johnprojmatrix_x^{(2)} = \johnprojmatrix_{x, \johnweights_x}^{(2)}, \quad \johnlevmatrix_x = \johnlevmatrix_{x, \johnweights_x}, \quad \johnlaplacian_x = \johnlaplacian_{x, \johnweights_x}.
\end{align}
Thus, we have the following relation:
\begin{align}
    \johnweights_x = \levjohn_{x, \johnweights_x} + \johnbeta \mathbf{1} = \levjohn_x + \johnbeta \mathbf{1}.
    \label{eq:john_normal_equation}
\end{align}
\end{subequations}

\subsection{Deterministic expressions and bounds} 
\label{sub:deterministic_bounds}

We now collect some properties of various terms defined above.
\begin{lemma}
  \label{ppt:john_all_properties}
  For any $x \in \intP$, the following properties hold:
  \begin{enumerate}[label=(\alph*)]
    \item\label{item:john_sigma_properties} $\levjohn_{x, i} =
    \sum_{j=1}^\obs \levjohn^2_{x, i, j} = \sum_{j, k=1}^\obs
    \levjohn_{x, i, j} \levjohn_{x, j, k} \levjohn_{x, k,
      i}$ for each $i \in [\obs]$,
    \item\label{item:john_Sigma_dominates_P} $\johnlevmatrix_x
    \succeq \johnprojmatrix_x^{(2)}$,
    \item\label{item:john_theta_sum} $\sum_{i=1}^\obs \johnweights_{x, i} \johnlocalslack_{x, i}
    = \dims$,
    \item\label{item:john_theta_properties} $\johnlocalslack_{x, i} = \sum_{j=1}^\obs \johnweights_{x, i} \johnlocalslack_{x, i, j}^2$, for each $i \in
               [\obs]$,
    \item\label{item:john_theta_square_sigma_sum_bound}
    $\johnlocalslack_x \tp \johnlevmatrix_x \johnlocalslack_x = \sum_{i=1}^\obs
    \johnlocalslack_{x, i}^2 \johnweights_{x, i}
    \leq 4 \dims$, and
    \item \label{item:john_W_sandwich} $\displaystyle\johnbeta\,
    \hesslogbarr_x \preceq \john_x \preceq \parenth{1 + \johnbeta}
    \hesslogbarr_x$.
  \end{enumerate}
\end{lemma}
The proof is based on the ideas similar to Lemma 5 in the proof of the Vaidya walk and is thereby omitted.

The next lemma relates the change in \emph{slackness} $\slack_{x, i} = b_i-a_i\tp x$ to the John-local norm at $x$.
\begin{lemma}
  \label{lemma:john_closeness_of_slackness}
  For all $x, y \in \intP$, we have
  \begin{align*}
    \max_{i\in[\obs]}\abss{1 - \frac{\slack_{y,i}}{\slack_{x,i}}} \leq 2 \vecnorm{x - y}{x}.
  \end{align*}
\end{lemma}
\begin{proof}
  For any pair $x, y \in \intP$ and index $i \in [\obs]$, we have
  \begin{align*}
   \parenth{a_i \tp \parenth{x-y}}^2
   \stackrel{(i)}{\leq} \| \john_x^{-\frac{1}{2}} a_i\|^2_2 \; \|
   \john_x^{\frac{1}{2}} (x-y) \|^2_2
   = \johnlocalslack_{x, i} \slack_{x, i}^2 \; \vecnorm{x -
     y}{x}^2
   \stackrel{(ii)}{\leq} 4 \slack_{x, i}^2
   \; \vecnorm{x - y}{x}^2,
  \end{align*}
  where step (i) follows from the Cauchy-Schwarz inequality, and step
  (ii) uses the bound $\johnlocalslack_{x, i}$ from
  Lemma~\ref{lemma:john_first_bounds}\ref{item:john_theta_bound}.  Noting the
  fact that $a_i\tp(x-y) = \slack_{y, i} - \slack_{x, i} $, the claim
  follows after simple algebra.
\end{proof}

We now state various expressions and bounds for the first and second order derivatives of the different terms.
To lighten notation, we introduce some shorthand notation.
For any $y \in \intP$ and $h \in \realdim$, define the following terms:
\begin{subequations}
  \begin{align}
    \projd_{y, i} &= \frac{a_i\tp h}{\slack_{y, i}}, \ i \in [\obs]
    &\projdmatrix_y &= \diag(\projd_{y, 1}, \ldots, \projd_{y, \obs}),
    \label{eq:john_def_D}\\
    \projf_{y, i} &= \frac{\nabla \johnweights_{y, i}\tp h}{\johnweights_{y, i}}, \ i \in [\obs]
    &\projfmatrix_y &= \diag(\projf_{y, 1}, \ldots, \projf_{y, \obs}),
    \label{eq:john_def_F}\\
    \projz_{y, i}& = \frac{1}{2} h\tp \nabla^2 \johnweights_{y, i} h / \johnweights_{y, i}, \ i \in [\obs]
    &\projzmatrix_y &= \diag(\projz_{y, 1}, \ldots, \projz_{y, \obs}),
    \label{eq:john_def_Z}\\
    \projzsub_y &\defn \parenth{\johnweightmatrix_y - \alpha \johnlaplacian_y}\begin{bmatrix} \projz_{y, 1} \\ \vdots \\ \projz_{y, \obs} \end{bmatrix}, \label{eq:john_def_rho}
  \end{align}
\end{subequations}
where for brevity in our notation we have omitted the dependence on $h$.
The choice of $h$ is specified as per the context.
Further, we define for each $x \in \intP$ and $i \in [\obs]$
\begin{align}
  \johnphi_{x,i} &\defn \displaystyle\frac{\johnweights_{x, i}}{\slack_{x, i}^2}, \quad
  &\text{and}& \quad &\logdetJ_x &\defn\frac{1}{2}\log \det
  \john_x,\\
  \hat{a}_{x, i} &\defn \frac{\john_x^{-1/2} a_{x, i}}{s_{x, i}^2}, \quad&\text{and}&\quad &\hat{b}_{x, i} &\defn\john_x^{-1/2}A_x\Lambda_x \parenth{G_x - \alpha \Lambda_x}^{-1} e_i.
  \label{eq:john_hat_a_hat_b}
\end{align}
Next, we state expressions for gradients of $\johnweights, \johnphi$ and
$\logdetJ$ and bounds for directional Hessian of $\levjohn$, $\johnphi$ and
$\logdetJ$ which are used in various Taylor series expansions and bounds in our proof.

\begin{lemma}[Calculus]
  \label{lemma:john_gradient_and_hessian_and_bounds}
  For any $y \in \intP$ and $h \in \real^\obs$, the following relations hold;
  \begin{enumerate}[label=(\alph*)]
    \item\label{item:gradient_g} Gradient of $\johnweights$:
    $(\projf_{y, 1}, \ldots, \projf_{y, \obs})\tp = 2 \parenth{\johnweightmatrix_y - \alpha \johnlaplacian_y}^{-1} \johnlaplacian_y A_y h$;
    \item\label{item:hessian_g} Hessian of $\johnweights$:
    \begin{align}
    \label{eq:john_l1_norm_on_weights_hessian}
      \vecnorm{\projzsub_y}{1}\leq 56 \johnkappa^2 \sum_{i=1}^\obs \johnweights_{y, i} \projd_{y, i}^2.
    \end{align}
    \item \label{item:gradient_log_det} Gradient of $\logdetJ$:
    $\nabla \logdetJ \tp h = \johnlocalslack_y\tp \johnweightmatrix_y \parenth{\Ind_\obs + \parenth{\johnweightmatrix_y - \alpha \johnlaplacian_y}^{-1}\johnlaplacian_y} A_y h$.

%
%
    \item\label{item:gradient_phi} Gradient of $\johnphi$: $\nabla \johnphi_{y, i}\tp h = \johnphi_{y, i} \parenth{2 \projd_{y, i} + \projf_{y, i}}$.
%
%
    \item\label{item:john_hessian_log_det_bound} Bound on $\hesslogdetJ$:
    $\frac{1}{2} \abss{h\tp (\nabla^2 \logdetJ) h}
    \leq \frac{1}{2}\brackets{\sum_{i=1}^\obs \johnweights_{y, i}\, \johnlocalslack_{y, i} \brackets{ 9\, \projd_{y, i} ^2 + 4 \projf_{y, i}^2}   +  \abss{ \sum_{i=1}^\obs \johnweights_{y, i}\, \johnlocalslack_{y, i} \projz_{y, i}}}$
    \item \label{item:hessian_beta_bound} Bound on $\nabla^2\johnphi$:
    \begin{align*}
      \abss{\sum_{i=1}^\obs \projd_{y, i}^2 \slack_{y, i}^2 \frac{1}{2} h\tp \nabla^2 \johnphi_{y, i} h} \leq 3 \sum_{i=1}^\obs \johnweights_{y, i} \projd_{y, i}^4 + 2\abss{\sum_{i=1}^\obs \johnweights_{y,i} \projd_{y, i}^3 \projf_{y, i}} + \abss{\sum_{i=1}^\obs \johnweights_{y, i} \projd_{y, i}^2 \projz_{y, i}}.
    \end{align*}
  \end{enumerate}
\end{lemma}
The proof is provided in Section~\ref{sub:proof_of_lemma_lemma:john_gradient_and_hessian_and_bounds}.

Next, we state some results that would be useful to provide explicit bounds for various terms like $\projf_y, \projz_y$ and $\projzsub_y$ that appear in the statements of the previous lemma.
Note that the following results do not have a corresponding analog in our analysis of the Vaidya walk.
\begin{lemma}
  \label{lemma:f_to_d_l2}
  For any $c_1, c_2 \geq 0$, $y\in \intP$, we have
  \begin{align*}
    \parenth{c_1 \Ind_\obs + c_2 \johnlaplacian_y \parenth{\johnweightmatrix_y-\alpha\johnlaplacian_y}^{-1}} \johnweightmatrix_y \parenth{c_1 \Ind_\obs + c_2 \parenth{\johnweightmatrix_y-\alpha \johnlaplacian_y}^{-1} \johnlaplacian_y} \preceq \parenth{c_1 + c_2}^2 \johnkappa^2 \johnweightmatrix_y,
  \end{align*}
  where $\preceq$ denotes the ordering in the PSD cone.
\end{lemma}
\begin{lemma}
  \label{lemma:f_to_d_matrix}
    Let $\johnftdinvmatrix_{y}$ denote the $\obs \times \obs$ matrix $\parenth{\johnweightmatrix_y-\alpha\johnlaplacian_y}^{-1} \johnweightmatrix_y$, and let $\johnftdinvmatrix_{y, i, j}$ denote its $ij$-th entry.
    Then for each $i \in [\obs]$ and $y \in \intP$, we have
    \begin{subequations}
    \begin{align}
      \johnftdinvmatrix_{y, i, i} &\in [0, \johnkappa], \quad \text{ and, } \label{eq:john_mu_first_coordinate_bound}\\
      \sum_{j \neq i, j \in [\obs]} \frac{\johnftdinvmatrix_{y, i, j}^2}{\johnweights_{y, j}} &\leq \johnkappa^3.\label{eq:john_mu_sum_of_terms_bound}
    \end{align}
    \end{subequations}
\end{lemma}
\begin{corollary}
  \label{corollary:f_to_d_l1}
    Let $e_i \in \real^\obs$ denote the unit vector along $i$-th axis. Then for any $y\in \intP$, we have
    \begin{align}
    \label{eq:john_mu_1_bound}
      \vecnorm{\johnweightmatrix_y \parenth{\johnweightmatrix_y-\alpha\johnlaplacian_y}^{-1} e_i}{1} \leq 3 \sqrt{\dims} \johnkappa^{3/2}, \quad \text{ for all } i \in [\obs].
    \end{align}
    Consequently, we also have $\matsnorm{\parenth{\johnweightmatrix_y-\alpha\johnlaplacian_y}^{-1} \johnweightmatrix_y}{\infty} \leq 3 \sqrt{\dims} \johnkappa^{3/2}.$
\end{corollary}
See Section~\ref{sub:proof_of_lemma_lemma:f_to_d_l2}, \ref{sub:proof_of_lemma_lemma:f_to_d_matrix} and \ref{sub:proof_of_corollary_corollary:f_to_d_l1} for the proofs of Lemma~\ref{lemma:f_to_d_l2}, Lemma~\ref{lemma:f_to_d_matrix} and Corollary~\ref{corollary:f_to_d_l1} respectively.


\subsection{Tail Bounds} 
\label{sub:tail_bounds}
We now collect lemmas that provide us with useful tail bounds.

We start with a result that shows that for a random variable $z \sim
\proposal_x$, the slackness $\slack_{z, i}$ is close to $\slack_{x,
  i}$ with high probability and consequently the weights $\johnweights_{z, i}$ are also close to $\johnweights_{x, i}$.
This result comes in handy for transferring the remainder terms in Taylor expansions to the reference point (around which the series is being expanded).

\begin{lemma}
  \label{lemma:john_whp_slackness}
  For any point $x \in \intP$ and $r \leq  \frac{1}{25 \cdot \sqrt{1+\sqrt2 \log(4/\epsilon)}}$, we have
  \begin{subequations}
    \begin{align}
      \label{eq:john_slack_whp}
      \Prob_{z \sim \proposal_x} \brackets{ \forall i \in [\obs],
        \forall v \in \overline{xz},\ {\frac{\slack_{x,
              i}}{\slack_{v, i}}} \in \brackets{0.99, 1.01}\text{ and }{\frac{\johnweights_{x, i}}{\johnweights_{v, i}}} \in \brackets{0.96, 1.04}}
              &\geq 1 - \epsilon/4
      \end{align}
  \end{subequations}
\end{lemma}
See Section~\ref{sub:proof_of_lemma_whp:slackness} for the proof of this lemma.

Next, we state high probability results for some Gaussian polynomials.
These results are useful to bound various polynomials of the form $\sum_{i=1}^\obs \johnweights_{x, i} \projd_{x, i}^k$, where $\projd_{x, i} = a_i\tp(z-x)/\slack_{x, i}$ and $z$ is drawn from the transition distribution for the John walk at point $x$.

\begin{lemma}[Gaussian moment bounds]
  \label{lemma:john_gaussian_moment_bounds}
  To simplify notations, all subscripts on $x$ are omitted in the following statements.
  For any $\epsilon \in (0, 1/30]$, define $\tailconstjohn_k \defn \tailconstjohn_{k, \epsilon} = \parenth{2e/k \cdot \log \parenth{16/\epsilon}}^{k/2}$, for $k=2,3,4\text{ and }6$, then we have
  \begin{subequations}
  \begin{align}
    \Prob\brackets{\sum_{i=1}^\obs \johnweights_i \parenth{\hat{a}_i\tp \rvg}^2 \leq \tailconstjohn_2 \sqrt{24}\dims} &\geq 1-\frac{\epsilon}{16},
    \label{eq:john_john_quadratic} \\
    \Prob\brackets{\sum_{i=1}^\obs \johnweights_i \parenth{\hat{a}_i\tp \rvg}^3 \leq \tailconstjohn_3 \sqrt{60}\dims^{1/2}} &\geq 1-\frac{\epsilon}{16},
    \label{eq:john_john_cubic} \\
    \Prob\brackets{\sum_{i=1}^\obs \johnweights_i \parenth{\hat{a}_i\tp \rvg}^2 \parenth{\hat{b}_i\tp \rvg}\leq \tailconstjohn_3 \sqrt{240}\johnkappa\dims^{1/2}} &\geq 1-\frac{\epsilon}{16},
    \label{eq:john_john_cubic_f} \\
    \Prob\brackets{\sum_{i=1}^\obs \johnweights_i \parenth{\hat{a}_i\tp \rvg}^4\leq \tailconstjohn_4 \sqrt{1680}\dims} &\geq 1-\frac{\epsilon}{16},
    \label{eq:john_john_fourth} \\
    \Prob\brackets{\sum_{i=1}^\obs \johnweights_i \parenth{\hat{a}_i\tp \rvg}^6 \leq \tailconstjohn_6 \sqrt{15120}\dims} &\geq 1-\frac{\epsilon}{16}.
    \label{eq:john_john_sixth}
  \end{align}
  \end{subequations}
\end{lemma}
See Section~\ref{sub:proof_of_lemma_lemma:john_gaussian_moment_bounds} for the proof.


\section{Proof of Lemma~\ref{lemma:john_close_hessian_eigenvalues}}
\label{sec:proof_of_lemma_lemma:john_close_hessian_eigenvalues}

  As a direct consequence of Lemma~\ref{lemma:john_closeness_of_slackness}, for any $x, y \in \intP$ such that \mbox{$\vecnorm{x-y}{x} \leq t/\johnkappa^2$}, we have
  \begin{align}
    \label{eq:john_slackness_tbound}
    \max_{i\in[\obs]}\abss{1 - \frac{\slack_{y, i}}{\slack_{x, i}}}
    \leq \frac{2t}{\johnkappa^2}.
  \end{align}
  Bounding the terms in $\hesslogbarr_x$ one by one, we obtain
  \begin{align*}
    \parenth{1 - \frac{2t}{\johnkappa^2}}^2 \hesslogbarr_y
    \preceq \hesslogbarr_x
    \preceq \parenth{1 + \frac{2t}{\johnkappa^2}}^2 \hesslogbarr_y.
  \end{align*}
  We claim that
  \begin{align}
    \label{eq:john_log_weight_final_bound}
    \vecnorm{\log \johnweights_{y} - \log \johnweights_{x}}{\infty} \leq 16 t.
  \end{align}
  Assuming the claim as given at the moment, we now complete the proof.
  Putting the result~\eqref{eq:john_log_weight_final_bound} in matrix form, we obtain that $\exp \parenth{-16t} \Ind_\obs
    \preceq \johnweightmatrix_x^{-1} \johnweightmatrix_y \preceq
    \exp \parenth{ 16t} \Ind_\obs$,
  and hence
  \begin{align}
    \label{eq:john_bound_on_ratio_of_johnweights}
    \exp \parenth{-16t} \johnweights_{x, i}
    \leq \johnweights_{y, i}
    \leq \exp \parenth{16t} \johnweights_{x, i}.
  \end{align}
  Consequently, using the definition of $\john_x$ we have,
  \begin{align*}
    \underbrace{ \parenth{1 - \frac{2t}{\johnkappa^2}}^2 \exp\parenth{-16t}}_{\omega_\ell} \john_x
    \leq \john_y
    \leq \underbrace{ \parenth{1 + \frac{2t}{\johnkappa^2}}^2 \exp \parenth{16t} }_{\omega_u} \john_y.
  \end{align*}
  Letting $\omega = 2t$, we obtain
  \begin{align*}
    \omega_\ell
     \geq (1-\omega)^2 \cdot \exp\parenth{-8\omega}
    \stackrel{(i)}{\geq} 1-24\omega + \omega^2, \quad \text{ and } \quad
    \omega_u
    \leq (1+\omega)^2 \cdot \exp\parenth{8\omega}
    \stackrel{(ii)}{\leq} 1+24\omega + \omega^2,
  \end{align*}
  where inequalities $(i)$ and $(ii)$ hold since $\omega \leq 1/24$.
  Putting the pieces together, we find that
  \begin{align*}
    \parenth{1-48t+4t^2} \john_x
    \preceq \john_y
    \preceq \parenth{1-48t+4t^2} \john_x
  \end{align*}
  for $t \in [0, 1/48]$.

Now, we return to the proof of our earlier claim~\eqref{eq:john_log_weight_final_bound}.
We use an argument based on the continuity of the function $x \mapsto \log \johnweights_x$.
(Such an argument appeared in a similar scenario in~\cite{lee2014path}.)
For $\lambda \in [0, 1]$, define \mbox{$u_\lambda = \lambda y + \parenth{1- \lambda}x$}.
Let
  \begin{align}
    \label{eq:john_log_weight_initial_bound}
    \lambda^{\mathrm {max}} \defn \sup\braces{\lambda \in [0, 1] \bigg\vert \vecnorm{\log \johnweights_{u_\lambda} - \log \johnweights_x}{\infty} \leq 16 t}.
  \end{align}
  It suffices to establish that $\lambda^{\mathrm {max}} = 1$.
  Note that $\lambda = 0$ is feasible on the RHS of equation~\eqref{eq:john_log_weight_initial_bound} and hence $\lambda^{\mathrm {max}}$ exists.
  Now for any $\lambda \in [0, \lambda^{\mathrm {max}}]$ and $i \in \{1, \ldots, \obs\}$, there exists $v$ on the segment $\overline{u_\lambda x}$ such that
  \begin{align*}
    \abss{\log \johnweights_{u_\lambda, i} - \log \johnweights_{x, i}} & = \abss{\parenth{ \frac{\nabla \johnweights_{v, i}}{\johnweights_{v, i}}} \tp \parenth{u_\lambda - x}}
    \stackrel{(i)}{\leq} \vecnorm{ \johnweightmatrix_v^{-1} \johnweightmatrix'_v \parenth{y - x} }{\infty}
    = 2 \vecnorm{\parenth{\johnweightmatrix_v - \alpha\johnlaplacian_v}^{-1} \johnlaplacian_v A_v \parenth{y - x} }{\infty}.
  \end{align*}
  where in step $(i)$ we have used the fact that $u_\lambda - x = \lambda (y-x)$ and $\lambda \in [0, 1]$.
  We claim that
  \begin{align}
    \label{eq:john_step_consistency_infty}
    \vecnorm{\parenth{\johnweightmatrix_v - \alpha\johnlaplacian_v}^{-1} \johnlaplacian_v \myvec_1}{\infty} \leq \johnkappa  \vecnorm{\myvec_1}{\infty} +  2 \johnkappa^2 \vecnorm{\johnweightmatrix_v^{1/2}\myvec_1}{2}
    \quad \mbox{for any $\myvec_1 \in \real^\obs$}.
  \end{align}
  We prove the claim at the end of this section.
  We now derive bounds for the two terms on the RHS of the equation~\eqref{eq:john_step_consistency_infty} for $\myvec_1 = A_v (y-x)$.
  Note that
  \begin{align*}
    \vecnorm{A_v \parenth{y - x}}{\infty} = \max_i  \abss{\frac{\slack_{y, i} - \slack_{x, i}}{\slack_{v, i}}} = \max_i  \abss{\frac{\slack_{y, i} - \slack_{x, i}}{\slack_{x, i}}} \abss{\frac{\slack_{x, i}}{\slack_{v, i}}} \stackrel{(i)}{\leq} \frac{2t}{\johnkappa^2 \parenth{1 - 2t/\johnkappa^2}} \stackrel{(ii)}{\leq} \frac{3t}{\johnkappa^2}.
  \end{align*}
  Inequality $(i)$ uses bound~\eqref{eq:john_slackness_tbound} and inequality $(ii)$ follows by plugging in $t\leq 1/64$.
  Next, we have
  \begin{align*}
    \vecnorm{\johnweightmatrix_v^{1/2}A_v \parenth{y - x}}{2}^2
    = \sum_{i=1}^\obs \johnweights_{x,i} \frac{\parenth{a_i \tp \parenth{y - x}}^2}{\slack_{x,i}^2} \frac{\johnweights_{v, i}}{\johnweights_{x, i}} \frac{\slack_{v, i}^2}{\slack_{x, i}^2}
    & \stackrel{(i)}{\leq} \vecnorm{x - y}{x}^2 \ \max_{i \in [\obs]} \frac{\johnweights_{v, i}}{\johnweights_{x, i}} \frac{\slack_{v, i}^2}{\slack_{x, i}^2} \\
    & \stackrel{(ii)}{\leq} \frac{t^2}{\johnkappa^4} \parenth{1 + (16t) +(16t)^2} \parenth{1 + \frac{2t}{\johnkappa^2}}^2 \\
    & \stackrel{(iii)}{\leq} \frac{1.5 t}{\johnkappa^4},
  \end{align*}
  where step $(i)$ follows from the definition of the local norm; step $(ii)$ follows from bounds~\eqref{eq:john_slackness_tbound} and \eqref{eq:john_log_weight_initial_bound} and the fact that $e^x \leq 1+ x + x^2 $ for all $x \in [0, 1/4]$; and inequality $(iii)$ follows by plugging in $t\leq 1/64$.
  Putting the pieces together, we obtain
  \begin{align*}
    \vecnorm{\log \johnweights_{u_\lambda} - \log \johnweights_{x}}{\infty} \leq 2(\johnkappa\cdot 3t/\johnkappa^2 + 2 \johnkappa^2 \cdot 1.5t/\johnkappa^4 ) \leq 12 t < 16 t.
  \end{align*}
  The strict inequality is valid for $\lambda = \lambda^{\mathrm {max}}$.
  Consequently, using the continuity of $x \mapsto \log \johnweights_x$, we conclude that $\lambda^{\mathrm {max}} = 1$.

 It is left to prove claim~\eqref{eq:john_step_consistency_infty}.
 Let $\myvectwo \defn \parenth{\johnweightmatrix_v - \alpha\johnlaplacian_v}^{-1} \johnlaplacian_v \myvec_1$.
  which implies $\parenth{\johnweightmatrix_v - \alpha\johnlaplacian_v} \myvectwo = \johnlaplacian_v \myvec_1$.
  Plugging the expression of $\johnweightmatrix_v$ and $\johnlaplacian_v$, we have
  \begin{align*}
    \parenth{(1-\alpha)\johnlevmatrix_v + \johnbeta\Ind_\obs + \alpha \johnprojmatrix_v^{(2)}} \myvectwo = \parenth{\johnlevmatrix_v - \johnprojmatrix_v^{(2)}} \myvec_1.
  \end{align*}
  Writing component wise, we find that for any $i \in [\obs]$, we have
  \begin{align}
    \abss{\parenth{(1-\alpha)\levjohn_{v,i} + \johnbeta} \myvectwo_i}
    &\leq \alpha \abss{e_i\tp\johnprojmatrix_v^{(2)} \myvectwo } +  \levjohn_{v,i} \abss{\myvec_{1, i}} + \abss{e_i\tp\johnprojmatrix_v^{(2)} \myvec_1} \notag\\
    & \stackrel{(i)}{\leq} \alpha \levjohn_{v, i} \vecnorm{\johnlevmatrix_v^{1/2} \myvectwo}{2} +  \levjohn_{v,i} \vecnorm{\myvec_1}{\infty} + \levjohn_{v, i} \vecnorm{\johnlevmatrix_v^{1/2} \myvec_1}{2} \notag\\
    & \stackrel{(ii)}{\leq} \alpha \levjohn_{v, i} \vecnorm{\johnweightmatrix_v^{1/2} \myvectwo}{2} +  \levjohn_{v,i} \vecnorm{\myvec_1}{\infty} + \levjohn_{v, i} \vecnorm{\johnweightmatrix_v^{1/2} \myvec_1}{2} \notag\\
    & \stackrel{(iii)}{\leq} \alpha \levjohn_{v, i} \johnkappa \vecnorm{\johnweightmatrix_v^{1/2} \myvec_1}{2} +  \levjohn_{v,i} \vecnorm{\myvec_1}{\infty} + \levjohn_{v, i} \vecnorm{\johnweightmatrix_v^{1/2} \myvec_1}{2},
    \label{eq:john_final_infty_two_bound}
  \end{align}
  where inequality $(ii)$ from the fact that $\johnlevmatrix_y \preceq \johnweightmatrix_y$ and inequality $(iii)$ from Lemma~\ref{lemma:f_to_d_l2} with $c_1=0, c_2 = 1$.
  To assert inequality $(i)$, observe the following
  \begin{align*}
    \abss{\sum_{j=1}^\obs \levjohn_{y, i, j}^2 \myvectwo_j}
    \leq \sum_{j=1}^\obs \levjohn_{y, i, j}^2 \abss{\myvectwo_j}
    \stackrel{(a)}{\leq} \levjohn_{y, i} \sum_{j=1}^\obs \levjohn_{y, j} \abss{\myvectwo_j}
    \stackrel{(b)}{\leq} \levjohn_{y, i} \sum_{j=1}^\obs \sqrt{\levjohn_{y, j}} \abss{\myvectwo_j}
    = \levjohn_{y, i} \vecnorm{\johnlevmatrix_v^{1/2}\myvectwo}{2},
  \end{align*}
  where step $(a)$ follows from the fact that $\levjohn_{y, i, j}^2 \leq \levjohn_{y, i} \levjohn_{y, j}$, and step $(b)$ from the fac that $\levjohn_{y, i} \in [0 ,1]$.
  Dividing both sides of inequality~\eqref{eq:john_final_infty_two_bound} by $\parenth{(1-\alpha)\levjohn_{v,i} + \johnbeta}$ and observing that $\levjohn_{v, i}/\parenth{(1-\alpha)\levjohn_{v,i} + \johnbeta} \leq\johnkappa$, and $\alpha \in [0, 1]$, yields the claim.

\section{Proof of Lemma~\ref{lemma:john_change_in_log_det_and_local_norm}} 
\label{sec:proof_of_lemma_lemma:john_change_in_log_det_and_local_norm}

We prove Lemma~\ref{lemma:john_change_in_log_det_and_local_norm} in two parts: claim~\eqref{eq:john_whp_log_det_filter} in Section~\ref{ssec:proof_of_claim_whp_log_det_filter} and claim~\eqref{eq:john_whp_local_norm} in Section~\ref{ssec:proof_of_claim_whp_local_norm}.
\subsection{Proof of claim~\eqref{eq:john_whp_log_det_filter}}
\label{ssec:proof_of_claim_whp_log_det_filter}
Using the second order Taylor expansion, we have
\begin{align*}
  \logdetJ_z - \logdetJ_x
  = \parenth{z - x}\tp \gradlogdetJ_x + \frac{1}{2} \parenth{z - x}\tp \hesslogdetJ_y \parenth{z - x},
  \quad \mbox{for some $y \in \overline{xz}$}.
\end{align*}
We claim that for $r\leq \johnradiusbound(\epsilon)$, we have
\begin{subequations}
  \label{eq:john_log_det_two_tails}
  \begin{align}
    \Prob \brackets{ \parenth{z-x}\tp \gradlogdetJ_x \geq -
      \epsilon/2} &\geq 1 - \epsilon/2,
    \ \text{and} \label{eq:john_gradient_log_det_tail} \\
    \quad \Prob
    \brackets{ {\frac{1}{2} \parenth{z-x} \hesslogdetJ_y
        \parenth{z-x}} \geq -\epsilon/2} &\geq 1 - \epsilon/2. \label{eq:john_hessian_log_det_tail}
  \end{align}
\end{subequations}
Note that the claim~\eqref{eq:john_whp_log_det_filter} follows from the above two claims.


\subsubsection{Proof of bound~(\ref{eq:john_gradient_log_det_tail})}
\label{sssec:proof_of_claim_eq_gradient_log_det_tail}

We observe that
\begin{align*}
  \parenth{z - x}\tp \nabla \logdetJ_x \sim \NORMAL \parenth{0, \frac{r^2}{\johnkappa^2 n} \nabla \logdetJ_x\tp \john_x^{-1} \nabla \logdetJ_x}.
\end{align*}
Let $\theinverse_x = \Ind_\obs + \parenth{\johnweightmatrix_x - \alpha\johnlaplacian_x}^{-1} \johnlaplacian_x$.
Substituting the expression of $\gradlogdetJ_x$ from Lemma~\ref{lemma:john_gradient_and_hessian_and_bounds}~\ref{item:gradient_log_det} and applying Cauchy-Schwarz inequality, we have that for any vector $\myvec \in \real^\dims$
\begin{align}
  \myvec\tp \nabla \logdetJ_x \nabla \logdetJ_x \tp \myvec
  = (\johnlocalslack_x\tp \johnweightmatrix_x \theinverse_x A_x \myvec)^2
  \leq \parenth{\myvec \tp A_x\tp \johnweightmatrix_x A_x \myvec} \cdot
  \parenth{\johnlocalslack_x \tp \johnweightmatrix_x  \theinverse_x \johnweightmatrix_x^{-1} \theinverse_x \johnweightmatrix_x \johnlocalslack_x}.  \label{eq:john_gradient_log_det_variance_bound}
\end{align}
Observe that
\begin{align*}
  \johnweightmatrix_x^{1/2}\theinverse_x\johnweightmatrix_x^{-1/2} = \Ind_\obs + (\Ind_\obs - \alpha \johnweightmatrix_x^{-1/2}\johnlaplacian_x\johnweightmatrix_x^{-1/2})^{-1}(\johnweightmatrix_x^{-1/2}\johnlaplacian_x\johnweightmatrix_x^{-1/2}).
\end{align*}
Now, using the intermediate bound~\eqref{eq:john_G_laplace_G_bound} from the proof of Lemma~\ref{lemma:f_to_d_l2}, we obtain that
\begin{align*}
  \Ind_\obs \preceq \johnweightmatrix_x^{1/2}\theinverse_x\johnweightmatrix_x^{-1/2} \preceq 2\johnkappa\Ind_\obs,
\end{align*}
and hence $\johnweightmatrix_x \preceq \johnweightmatrix_x  \theinverse_x \johnweightmatrix_x^{-1} \theinverse_x \johnweightmatrix_x \preceq 4 \johnkappa^2 \johnweightmatrix_x$.
Consequently, we have
\begin{align*}
  \johnlocalslack_x \tp \johnweightmatrix_x  \theinverse_x \johnweightmatrix_x^{-1} \theinverse_x \johnweightmatrix_x \johnlocalslack_x
  \leq 4\johnkappa^2 \johnlocalslack_x\tp \johnweightmatrix_x \johnlocalslack_x = 4\johnkappa^2 \sum_{i=1}^\obs \johnweights_{x, i} \johnlocalslack_{x, i}^2 \leq 16 \johnkappa^2 \dims,
\end{align*}
where the last step follows from Lemma~\ref{ppt:john_all_properties}.
Putting the pieces together into equation~\eqref{eq:john_gradient_log_det_variance_bound}, we obtain $\gradlogdetJ_x \gradlogdetJ_x\tp \preceq 16 \johnkappa^2 \dims \john_x$ whence
$\john_x^{-1/2} \gradlogdetJ_x \gradlogdetJ_x\tp \john_x^{-1/2} \preceq 16 \johnkappa^2 \dims \Ind_\dims$.
Noting that the matrix $\john_x^{-1/2} \gradlogdetJ_x \gradlogdetJ_x\tp \john_x^{-1/2}$ has rank one, we have
\begin{align*}
  \gradlogdetJ_x\tp \john_x^{-1} \gradlogdetJ_x = \trace \parenth{\john_x^{-1/2} \gradlogdetJ_x  \gradlogdetJ_x\tp \john_x^{-1/2}} \leq 16 \johnkappa^2 \dims.
\end{align*}
Using standard Gaussian tail bound, we have $\Prob\parenth{\parenth{z-x}\tp \gradlogdetJ_x  \geq -\sqrt{32}\tailconstjohn_1 r} \geq 1 - \exp\parenth{-\tailconstjohn_1^2}.$

Choosing $\tailconstjohn_1 = \log(2/\epsilon)$, and observing that
\begin{align}
  \label{eq:john_gradient_log_det_tail_r_condition}
  r \leq \frac{\epsilon}{\parenth{2\sqrt{32}\tailconstjohn_1}},
\end{align}
yields the claim.

\subsubsection{Proof of bound~(\ref{eq:john_hessian_log_det_tail})}
\label{sssec:proof_of_claim_eq_hessian_log_det_tail}
In the following proof, we use $h = z-x$ for definitions~\eqref{eq:john_def_D}-\eqref{eq:john_def_rho}.
According to Lemma~\ref{lemma:john_gradient_and_hessian_and_bounds}\ref{item:john_hessian_log_det_bound}, we have
\begin{align*}
  \abss{\frac{1}{2} \parenth{z - x}\tp \hesslogdetJ_y \parenth{z - x}}
  \leq  \sum_{i=1}^\obs \johnweights_{y, i}\, \johnlocalslack_{y, i} \brackets{ \frac{9}{2} \, \projd_{y, i}^2 + 2 \projf_{y, i}^2}   +  \frac{1}{2}\abss{ \sum_{i=1}^\obs \johnweights_{y, i}\, \johnlocalslack_{y, i} \projz_{y, i}}
\end{align*}
We claim that
\begin{align}
  \sum_{i=1}^\obs \johnweights_{y, i}\, \johnlocalslack_{y, i} \brackets{ \frac{9}{2} \, \projd_{y, i} ^2 + 2 \projf_{y, i}^2}   +  \frac{1}{2}\abss{ \sum_{i=1}^\obs \johnweights_{y, i}\, \johnlocalslack_{y, i} \projz_{y, i}}
  \leq 386 \sqrt{\dims} \johnkappa^4  \sum_{i=1}^\obs \johnweights_{y, i}  \projd_{y, i}^2.
  \label{eq:john_hess_logdet_in_terms_of_d}
\end{align}
Assuming the claim as given at the moment, we now complete the proof.
Note that $y$ is some particular point on $\overline{xz}$ and its dependence on $z$ is hard to characterize.
Consequently, we transfer all the terms with dependence on $y$, to terms with dependence on $x$ only.
We have
\begin{align*}
  \sum_{i=1}^\obs \johnweights_{y, i} \projd_{y, i}^2 = \sum_{i=1}^\obs \johnweights_{x, i} \projd_{x, i}^2 \underbrace{\frac{\johnweights_{y, i}}{\johnweights_{x, i}} \frac{s_{x, i}^2}{s_{y, i}^2}}_{\tau_{y, i}}.
\end{align*}
We now invoke the following high probability bounds implied by Lemma~\ref{lemma:john_whp_slackness} and Lemma~\ref{lemma:john_gaussian_moment_bounds}~\eqref{eq:john_john_quadratic} respectively
\begin{align}
  \Prob\brackets{\sup_{y \in\overline{xz}, i \in [\obs]}\tau_{y, i} \leq 1.1} \geq 1-\epsilon/4,
  \quad \text{and}, \quad
  \Prob\brackets{\sum_{i=1}^\obs \johnweights_{x, i}  \parenth{\hat{a}_{x, i}\tp \rvg}^2 \leq \tailconstjohn_2 \sqrt{24}\dims} \geq 1-\epsilon/16.
  \label{eq:john_hbp_logdetjohn}
\end{align}
Since $h = z-x$, we have that $\projd_{x, i}^2 = \frac{r^2}{\johnkappa^2 \dims^{3/2}}\parenth{\hat{a}_{x,i}\tp \rvg}^2$.
Consequently, for
\begin{align}
  \label{eq:john_hessian_log_det_tail_r_condition}
  r \leq \sqrt{\frac{\epsilon}{386\sqrt{24}\tailconstjohn_2}},
\end{align}
with probability at least $1-\epsilon/2$, we have
\begin{align*}
  \abss{\frac{1}{2} \parenth{z - x}\tp \hesslogdetJ_y \parenth{z - x}}
  \stackrel{\mathrm{eqn.}~\eqref{eq:john_hess_logdet_in_terms_of_d}}{\leq} 386 \sqrt{\dims} \johnkappa^4  \sum_{i=1}^\obs \johnweights_{y, i}  \projd_{y, i}^2
  \stackrel{\mathrm{hpb}~\eqref{eq:john_hbp_logdetjohn}}{\leq} \epsilon,
\end{align*}
which completes the proof.

We now turn to the proof of claim~\eqref{eq:john_hess_logdet_in_terms_of_d}.
First we observe the following relationship between the terms $\projd_{y, i}$ and  $\projf_{y, i}$:
\begin{align}
  \sum_{i=1}^\obs\! \johnweights_{y, i} \projf_{y, i}^2
  \!\stackrel{(i)}{=} \! 4 h\tp A_y\tp \!\Lambda_y \parenth{\johnweightmatrix_y \!-\! \alpha \johnlaplacian_y}^{-1}\! \johnweightmatrix_y \parenth{\johnweightmatrix_y\! -\! \alpha \johnlaplacian_y}^{-1} \!\johnlaplacian_y A_y h \!
  \stackrel{(ii)}{\leq}\! 4 \johnkappa^2 h\tp\! A_y\tp\!  \johnweightmatrix_y A_y h
   \!=\! 4 \johnkappa^2 \sum_{i=1}^\obs \johnweights_{y, i} \projd_{y, i}^2,
  \label{eq:john_exact_f_to_d_bound}
\end{align}
where step $(i)$ follows by plugging in the definition of $\projf_{y, i}$~\eqref{eq:john_def_F} and step~$(ii)$ by invoking Lemma~\ref{lemma:f_to_d_l2} with $c_1 =0 $ and $c_2 = 1$.
Next, we relate the term on the LHS of equation~\eqref{eq:john_hess_logdet_in_terms_of_d} involving $\projz_{y, i}$ to a polynomial in $\projd_{y, i}$.
Using Lemma~\ref{lemma:john_gradient_and_hessian_and_bounds}, we find that
\begin{align*}
\abss{ \sum_{i=1}^\obs \johnweights_{y, i}\, \johnlocalslack_{y, i} \projz_{y, i}}
&= \abss{\parenth{\parenth{\johnweightmatrix_y - \alpha \johnlaplacian_y}^{-1} \johnweightmatrix_y \theta_y}\tp \parenth{\johnweightmatrix_y - \alpha \johnlaplacian_y} \projz_y}
\leq \vecnorm{\underbrace{\parenth{\johnweightmatrix_y - \alpha \johnlaplacian_y}^{-1} \johnweightmatrix_y \theta_y}_{v_1}}{\infty}
\vecnorm{\underbrace{\parenth{\johnweightmatrix_y - \alpha \johnlaplacian_y} \projz_y}_{\projzsub_y}}{1},
\end{align*}
where the last step follows from the Holder's inequality: for any two vectors $u, v \in \realdim$, we have that $u\tp v \leq \vecnorm{u}{\infty} \vecnorm{v}{1}$.
Substituting the bound for the norm $\vecnorm{v_1}{\infty}$ from Corollary~\ref{corollary:f_to_d_l1} and the bound on $\projzsub_{y, i}$ from Lemma~\ref{lemma:john_gradient_and_hessian_and_bounds}\ref{item:hessian_g}, we obtain that
\begin{align*}
  \abss{ \sum_{i=1}^\obs \johnweights_{y, i}\, \johnlocalslack_{y, i} \projz_{y, i}}
  \!\leq 12 \sqrt{n} \johnkappa^{3/2} \sum_{i=1}^\obs \bigg[{7 \johnweights_{y, i} \projd_{y, i}^2 \!+\! 3 \johnweights_{y, i} \projf_{y, i}^2 \!+\! \sum_{j=1}^\obs \parenth{ 13\projd_{y, j}^2 \!+\! 6 \projf_{y, j}^2} \johnprojmatrix_{y, i, j}^2}\bigg]
  \!\leq\! 672\sqrt{n} \johnkappa^{4} \sum_{i=1}^\obs \johnweights_{y, i} \projd_{y, i}^2,
\end{align*}
where the last step follows from Lemma~\ref{ppt:john_all_properties}\ref{item:john_sigma_properties} and the bound~\eqref{eq:john_exact_f_to_d_bound}.
The claim now follows.

\subsection{Proof of claim~(\ref{eq:john_whp_local_norm})}
\label{ssec:proof_of_claim_whp_local_norm}

Writing $z = x + t u$, where $t$ is a scalar and $u$ is
a unit vector in \realdim, we obtain
\begin{align*}
\vecnorm{z-x}{z}^2 - \vecnorm{z-x}{x}^2 = t^2 \sum_{i=1}^\obs
\parenth{a_i\tp u}^2 \parenth{\johnphi_{z, i} - \johnphi_{x, i}}.
\end{align*}
Now, we use a Taylor series expansion for $\sum_{i=1}^\obs
\parenth{a_i\tp u}^2 \parenth{\johnphi_{z, i} - \johnphi_{x, i}}$ around
the point $x$, along the line $u$.  There exists a point $y \in
\overline{xz}$ such that
\begin{align*}
\sum_{i=1}^\obs \parenth{a_i\tp u}^2 \parenth{\johnphi_{z, i} -
  \johnphi_{x, i}} = \sum_{i=1}^\obs \parenth{a_i\tp u}^2 \parenth{
  \parenth{z - x}\tp \nabla \johnphi_{x, i}+ \frac{1}{2}
  \parenth{z-x}\tp \nabla^2 \johnphi_{y, i} \parenth{z-x}}.
\end{align*}
Note that the point $y$ in this discussion is not the same as the point $y$ used in previous proofs, in particular in Section~\ref{ssec:proof_of_claim_whp_log_det_filter}.
Multiplying both sides by $t^2$, and using the shorthand
$\projd_{x, i} = \frac{a_i\tp (z-x)}{\slack_{x,i}}$, we obtain
\begin{align}
  \label{eq:john_z_x_two_terms}
\vecnorm{z \!-\! x}{z}^2 \!-\! \vecnorm{z\!-\!x}{x}^2 &=
\sum_{i=1}^\obs \projd_{x, i}^2 \slack_{x, i}^2 \parenth{z\!-\!x}\tp
\nabla \johnphi_{x, i} + \sum_{i=1}^\obs \projd_{x, i}^2 \slack_{x, i}^2
\frac{1}{2} \parenth{z\!-\!x}\tp \nabla^2 \johnphi_{y, i}
\parenth{z\!-\!x}.
\end{align}
We claim that for $r\leq \johnradiusbound(\epsilon)$, we have
\begin{subequations}
    \label{eq:john_local_norm_two_tails}
    \begin{align}
        \Prob_{z \sim \transition^\tagjohn_x} \brackets{ \sum_{i=1}^\obs \projd_{x, i}^2 \slack_{x, i}^2 \parenth{z\!-\!x}\tp
\nabla \johnphi_{x, i} \leq
          \epsilon \frac{r^2}{\johnkappa^4 \dims^{3/2}}} &\geq 1 - \epsilon/2,
        \ \text{and} \label{eq:john_gradient_local_norm_tail} \\
        \quad \Prob_{z \sim \transition^\tagjohn_x}
        \brackets{ \sup_{y \in \overline{xz}} \parenth{\sum_{i=1}^\obs \projd_{x, i}^2 \slack_{x, i}^2
\frac{1}{2} \parenth{z\!-\!x}\tp \nabla^2 \johnphi_{y, i}
\parenth{z\!-\!x}} \leq \epsilon \frac{r^2}{\johnkappa^4 \dims^{3/2}} } &\geq 1 - \epsilon/2. \label{eq:john_hessian_local_norm_tail}
    \end{align}
\end{subequations}
We now prove each claim separately.

\subsubsection{Proof of bound~(\ref{eq:john_gradient_local_norm_tail})}
\label{sssec:proof_of_claim_eq_gradient_local_norm_tail}
Using Lemma~\ref{lemma:john_gradient_and_hessian_and_bounds}\ref{item:gradient_phi} and using $h = z-x$ where $z$ is given by the relation~\eqref{eq:john_z_x_relation}, we find that
\begin{align}
\sum_{i=1}^\obs \projd_{x, i}^2 \slack_{x, i}^2 \parenth{z\!-\!x}\tp \nabla \johnphi_{x, i}
&= \sum_{i=1}^\obs \johnweights_{x, i} \projd_{x, i}^2  \parenth{2 \projd_{x, i} + \projf_{x, i}}\notag\\
&= \frac{r^3}{\dims^{9/4} \johnkappa^6} \sum_{i=1}^\obs \johnweights_{x, i} \parenth{\hat{a}_{x,i} \tp \xi}^3
  + \frac{2r^3}{\dims^{9/4} \johnkappa^6} \sum_{i=1}^\obs \johnweights_{x, i} \parenth{\hat{a}_{x,i} \tp \xi}^2 \parenth{\hat{b}_{x,i} \tp \xi}
  \label{eq:john_two_terms_gradient_phi}
\end{align}
Using high probability bounds for the two terms in equation~\eqref{eq:john_two_terms_gradient_phi} from Lemma~\ref{lemma:john_gaussian_moment_bounds}, part~\eqref{eq:john_john_cubic} and part~\eqref{eq:john_john_cubic_f},  we obtain that
\begin{align*}
  \abss{\sum_{i=1}^\obs \projd_{x, i}^2 \slack_{x, i}^2 \parenth{z\!-\!x}\tp
  \nabla \johnphi_{x, i}}
  \leq \frac{5 \sqrt{60} \tailconstjohn_3 r^3}{\johnkappa^5\dims^{7/4}}
  \leq \epsilon \frac{r^2}{\johnkappa^4\dims^{3/2}},
\end{align*}
with probability at least $1 - \epsilon/2$.
The last inequality uses the condition that
\begin{align}
  \label{eq:john_gradient_local_norm_tail_r_condition}
  r \leq \frac{\epsilon}{5\sqrt{60} \tailconstjohn_3}.
\end{align}
The claim now follows.
\subsubsection{Proof of bound~(\ref{eq:john_hessian_local_norm_tail})}
\label{sssec:proof_of_claim_eq_hessian_local_norm_tail}
Note that $\projd_{x, i}\slack_{x, i} = a_i\tp h = \projd_{y, i} \slack_{y, i}$ for any $h$.
Using this equality for $h = z-x$, we find that
\begin{align}
    \abss{\sum_{i=1}^\obs \projd_{x, i}^2 \slack_{x, i}^2 \frac{1}{2} h\tp \nabla^2 \johnphi_{y, i} h}
    & = \abss{\sum_{i=1}^\obs \projd_{y, i}^2 \slack_{y, i}^2 \frac{1}{2} h\tp \nabla^2 \johnphi_{y, i} h} \notag\\
    & \stackrel{(i)}{\leq} 3 \underbrace{\sum_{i=1}^\obs \johnweights_{y, i} \projd_{y, i}^4}_{C_1} + 2\underbrace{\abss{\sum_{i=1}^\obs \johnweights_{y,i} \projd_{y, i}^3 \projf_{y, i}}}_{C_2} + \underbrace{\abss{\sum_{i=1}^\obs \johnweights_{y, i} \projd_{y, i}^2 \projz_{y, i}}}_{C_3}, \label{eq:john_hess_phi_decomposition}
\end{align}
where step~$(i)$ follows from Lemma~\ref{lemma:john_gradient_and_hessian_and_bounds}\ref{item:hessian_beta_bound}.
We can write $C_1$ as follows
\begin{align}
\sum_{i=1}^\obs \johnweights_{y, i} \projd_{y, i}^4
 = \sum_{i=1}^\obs \johnweights_{x, i} \projd_{x, i}^4 \frac{\johnweights_{y, i}}{\johnweights_{x, i}} \frac{\projd_{y, i}^4}{\projd_{x, i}^4}
 = \frac{r^4}{n^3 \johnkappa^8}  \sum_{i=1}^\obs \johnweights_{x, i} \parenth{\hat{a}_{x, i}\tp \rvg}^4 \frac{\johnweights_{y, i}}{\johnweights_{x, i}} \frac{\projd_{y, i}^4}{\projd_{x, i}^4}.
 \label{eq:john_hess_phi_c1}
\end{align}
Now, we claim the following:
\begin{subequations}
\begin{align}
  C_2 &\leq 2 \frac{r^4}{n^3 \johnkappa^7}  \cdot \sqrt{\brackets{ \sum_{i=1}^\obs \johnweights_{x, i} \parenth{\hat{a}_{x, i}\tp \rvg}^2 \frac{\johnweights_{y, i}}{\johnweights_{x, i}} \frac{\projd_{y, i}^2}{\projd_{x, i}^2}}
  \cdot
  \brackets{\sum_{i=1}^\obs \johnweights_{x, i} \parenth{\hat{a}_{x,i}\tp \rvg}^6 \frac{\johnweights_{y, i}}{\johnweights_{x, i}} \frac{\projd_{y, i}^6}{\projd_{x, i}^6}}}, \quad \text{and}, \label{eq:john_C2_hess_phi}\\
  C_3 &\leq 56 \frac{r^4}{n^3 \johnkappa^{4.5}} \parenth{\sum_{i=1}^\obs \johnweights_{x, i} \parenth{\hat{a}_{x, i}\tp \rvg}^2 \frac{\johnweights_{y, i}}{\johnweights_{x, i}} \frac{\projd_{y, i}^2}{\projd_{x, i}^2}} \parenth{\max_i \parenth{\hat{a}_{x, i}\tp \rvg}^2 \frac{\projd_{y, i}^2}{\projd_{x, i}^2}
  + \sqrt{ \sum_{i=1}^\obs \johnweights_{x, i} \parenth{\hat{a}_{x, i}\tp \rvg}^4 \frac{\johnweights_{y, i}}{\johnweights_{x, i}} \frac{\projd_{y, i}^4}{\projd_{x, i}^4} } } \label{eq:john_C3_hess_phi}
\end{align}
\end{subequations}
Assuming the claims as given, we now complete the proof.
Using Lemma~\ref{lemma:john_whp_slackness}, we have
\begin{align*}
\Prob\brackets{\frac{\johnweights_{y, i}}{\johnweights_{x, i}} \frac{\projd_{y, i}^6}{\projd_{x, i}^6} \leq 1.2} \geq 1-\epsilon/4,
\end{align*}
and consequently
\begin{align}
  3C_1 \!+\! 2C_2 \!+\! C_3
  &\leq \frac{r^4}{\dims^3 \johnkappa^{4.5}}
        \Bigg[
        4 \cdot \sum_{i=1}^\obs \johnweights_{x, i} (\hat{a}_{x, i}\tp \rvg)^4
        +10  \cdot \bigg(\sum_{i=1}^\obs \johnweights_{x, i} (\hat{a}_{x, i}\tp \rvg)^2
                              \cdot
                              \sum_{i=1}^\obs \johnweights_{x, i} (\hat{a}_{x,i}\tp \rvg)^6
                    \bigg)^{1/2}
            \notag        \\
  & \quad \quad \quad
        + 100
        \cdot
        {\sum_{i=1}^\obs \johnweights_{x, i} \parenth{\hat{a}_{x, i}\tp \rvg}^2}
        \cdot
        \bigg({
        \max_i ({\hat{a}_{x, i}\tp \rvg})^2
        + \big({ \sum_{i=1}^\obs \johnweights_{x, i} ({\hat{a}_{x, i}\tp \rvg})^4}\big)^{1/2}
        }\bigg)
        \Bigg],
        \label{eq:john_3c1_c2_c3}
\end{align}
with probability at least $1-\epsilon/4$.
Now, we observe that for all $i \in [\obs]$ and $x \in \intP$, we have
\begin{align*}
  \parenth{\hat{a}_{x, i}\tp \rvg} \sim \NORMAL(0, \johnlocalslack_{x, i}) \quad \text{and} \quad \johnlocalslack_{x, i} \leq 4.
\end{align*}
Invoking the standard tail bound for maximum of Gaussian random variables, we obtain
\begin{align*}
  \Prob \brackets{\max_i \abss{\parenth{\hat{a}_{x, i}\tp \rvg}} \leq 8 \cdot \parenth{\sqrt{\log \obs}+\sqrt{\log(32/\epsilon)}} } \geq 1 - \epsilon/16.
\end{align*}
Using the fact that $2c_1c_2 \geq c_1 + c_2$ for all $c_1, c_2 \geq 1$, we obtain
\begin{align*}
   \Prob \brackets{\max_i \abss{\parenth{\hat{a}_{x, i}\tp \rvg}} \leq 16 \cdot{\sqrt{\log \obs} \cdot \sqrt{\log(32/\epsilon)}} } \geq 1 - \epsilon/16.
\end{align*}
Combining this bound with the tail bounds for various Gaussian polynomials~\eqref{eq:john_john_quadratic}, \eqref{eq:john_john_fourth}, \eqref{eq:john_john_sixth} from  Lemma~\ref{lemma:john_gaussian_moment_bounds}, and substituting in inequality~\eqref{eq:john_3c1_c2_c3}, we obtain that
\begin{multline*}
  \abss{\sum_{i=1}^\obs \projd_{x, i}^2 \slack_{x, i}^2 \frac{1}{2} h\tp \nabla^2 \johnphi_{y, i} h}
  \leq \frac{r^4}{\johnkappa^{6.5} \dims^3} \Bigg[ 4\cdot  \tailconstjohn_4\sqrt{1680}\dims + 10 \parenth{\tailconstjohn_2 \sqrt{24}\dims \cdot \tailconstjohn_6 \sqrt{15120} \dims}^{1/2} \\
  + 100 \cdot \tailconstjohn_2 \sqrt{24}\dims\cdot \parenth{256 \cdot \log\obs \cdot \log(32/\epsilon)  + \parenth{\tailconstjohn_4\sqrt{1680}\dims}^{1/2}}\Bigg]
\end{multline*}
with probability at least $1-\epsilon/2$.
In the above expression, the terms $\tailconstjohn_i$ are a function of $\epsilon$ as defined in Lemma~\ref{lemma:john_gaussian_moment_bounds}.
In particular, $\tailconstjohn_{i} = \tailconstjohn_{i, \epsilon} =  (2e/i \cdot \log(16/\epsilon))^{i/2}$ for $i \in \{2, 3, 4, 6\}$.
Observing that $256 \log(32/\epsilon) \geq \parenth{\tailconstjohn_4 \sqrt{1680}}^{1/2}$, and that our choice of $r$ satisfies
\begin{align}
  \label{eq:john_hessian_local_norm_tail_r_condition}
  r^2 \leq \min\braces{{\frac{\epsilon}{8\sqrt{1680} \tailconstjohn_4}}, {\frac{\epsilon}{40\parenth{\tailconstjohn_2\tailconstjohn_6 \sqrt{24}\sqrt{15120}}^{1/2}}}, {\frac{\epsilon}{204800\tailconstjohn_2\sqrt{24}\log(32/\epsilon)}}},
\end{align}
we obtain
\begin{align*}
  \abss{\sum_{i=1}^\obs \projd_{x, i}^2 \slack_{x, i}^2 \frac{1}{2} h\tp \nabla^2 \johnphi_{y, i} h}
  \leq
  \frac{r^2}{\johnkappa^4 \dims^{3/2}}\brackets{\frac{\epsilon}{2}+ \frac{\epsilon}{4} + \frac{\epsilon}{8}\parenth{ \frac{\log \obs}{\sqrt{\dims}}+1}}.
\end{align*}
Asserting the additional condition $\sqrt{\dims} \geq \log{\obs}$, yields the claim.

It is now left to prove the bounds~\eqref{eq:john_C2_hess_phi} and \eqref{eq:john_C3_hess_phi}.
We prove these bounds separately.
\paragraph{Bounding $C_2$:} 
\label{par:bounding_c_2}

Applying Cauchy-Schwarz inequality, we have
\begin{align*}
\abss{\sum_{i=1}^\obs \johnweights_{y,i} \projd_{y, i}^3 \projf_{y, i}} \leq \parenth{ \sum_{i=1}^\obs \johnweights_{y, i} \projf_{y, i}^2 \cdot\sum_{i=1}^\obs \johnweights_{y, i} \projd_{y, i}^6 }^{1/2}
\end{align*}
Using the bound~\eqref{eq:john_exact_f_to_d_bound}, we obtain
\begin{align*}
\sum_{i=1}^\obs \johnweights_{y, i} \projf_{y, i}^2
\leq 4 \johnkappa^2 \sum_{i=1}^\obs \johnweights_{y, i} \projd_{y, i}^2
= 4 \johnkappa^2 \sum_{i=1}^\obs \johnweights_{x, i} \projd_{x, i}^2 \frac{\johnweights_{y, i}}{\johnweights_{x, i}} \frac{\projd_{y, i}^2}{\projd_{x, i}^2}.
\end{align*}
Substituting $h = z-x$ where $z$ is given by relation~\eqref{eq:john_z_x_relation}, we obtain that \mbox{$\projd_{x, i} =  \frac{r}{\dims^{3/4}\johnkappa} \hat a_{x, i}\tp\rvg$},
and thereby
\begin{align*}
  \sum_{i=1}^\obs \johnweights_{y, i} \projf_{y, i}^2 \leq 4\johnkappa^2 \frac{r^2}{\dims^{3/2}\johnkappa^4} \sum_{i=1}^\obs \johnweights_{x, i} (\hat a_{x, i}\tp\rvg)^2 \frac{\johnweights_{y, i}}{\johnweights_{x, i}} \frac{\projd_{y, i}^2}{\projd_{x, i}^2}.
\end{align*}
Doing similar algebra, we obtain $\sum_{i=1}^\obs \johnweights_{y, i} \projd_{y, i}^6
 = \frac{r^6}{\dims^{9/2}\johnkappa^{12}} \sum_{i=1}^\obs \johnweights_{x, i} \parenth{\hat{a}_{x,i}\tp \rvg}^6 \frac{\johnweights_{y, i}}{\johnweights_{x, i}} \frac{\projd_{y, i}^6}{\projd_{x, i}^6}.$
Putting the pieces together yields the claim.

\paragraph{Bounding $C_3$:} 
\label{par:bounding_c_3}

Recall that $\projzsub_y = (\johnweightmatrix_y-\alpha \johnlaplacian_y)\projz_y$ (Lemma~\ref{lemma:john_gradient_and_hessian_and_bounds}) and $\johnftdinvmatrix_y = \parenth{\johnweightmatrix_y - \alpha \johnlaplacian_y}^{-1} \johnweightmatrix_y$ (Lemma~\ref{lemma:f_to_d_matrix}).
We have
\begin{align*}
  \abss{\sum_{i=1}^\obs \johnweights_{y, i} \projd_{y, i}^2 \projz_{y, i}}
  = \mathbf{1}  \projdmatrix^2_y \johnweightmatrix_y \projz_y
  = \underbrace{\mathbf{1}  \projdmatrix^2_y \johnweightmatrix_y (\johnweightmatrix_y-\alpha \johnlaplacian_y)^{-1}}_{=:\myvec_y\tp}\
   \underbrace{(\johnweightmatrix_y-\alpha \johnlaplacian_y)\projz_y}_{\projzsub_y}.
\end{align*}
Using the definition of $\myvec_y$ and $\johnftdinvmatrix_y$, we obtain
\begin{align*}
  \myvec_{y, i} \defn e_i\tp \myvec_y = e_i\tp\parenth{\johnweightmatrix_y - \alpha \johnlaplacian_y}^{-1} \johnweightmatrix_y \projdmatrix_y^2 \mathbf{1} = e_i\tp\johnftdinvmatrix_y \projdmatrix_y^2 \mathbf{1}
  = \johnftdinvmatrix_{y, i, i} \projd_{y, i}^2 + \sum_{j\in[\obs], j\neq i} \johnftdinvmatrix_{y, i, j} \projd_{y, j}^2.
\end{align*}
Consequently, we have
\begin{align*}
  \abss{\sum_{i=1}^\obs \myvec_{y, i} \projzsub_{y, i}}
  \leq
  \overbrace{\sum_{i=1}^\obs \abss{\projzsub_{y, i}} \cdot \abss{\johnftdinvmatrix_{y, i, i}\projd_{y, i}^2}}^{=:C_4}
  + \overbrace{\sum_{i=1}^\obs \abss{\projzsub_{y, i}} \cdot \parenth{
    \sum_{j\in[\obs], j\neq i} \abss{\johnftdinvmatrix_{y, i, j} \projd_{y, j}^2} }}^{=:C_5}
\end{align*}
From Lemma~\ref{lemma:f_to_d_matrix}, we have that $\johnftdinvmatrix_{y, i, i} \in [0, \johnkappa]$.
Hence, we have $C_4 \leq \vecnorm{\projzsub_y}{1} \cdot \johnkappa \cdot \max_{i \in [\obs]} \projd_{y, i}^2$.
To bound $C_5$, we note that
\begin{align*}
  \sum_{j\in[\obs], j\neq i} \abss{\johnftdinvmatrix_{y, i, j} \projd_{y, j}^2}
  \stackrel{(i)}{\leq} \parenth{\sum_{j\in[\obs], j\neq i}\frac{\johnftdinvmatrix_{y, i, j}^2}{\johnweights_{y, j}}
  \cdot \sum_{j=1}^\obs \johnweights_{y, j} \projd_{y, j}^4}^{1/2}
  \stackrel{(ii)}{\leq} \parenth{\johnkappa^3 \cdot \sum_{j=1}^\obs \johnweights_{x, j} \projd_{x, j}^4 \frac{\johnweights_{y, j}}{\johnweights_{x, j}} \frac{\projd_{y, j}^4}{\projd_{x, j}^4}}^{1/2},
\end{align*}
where step~$(i)$ follows from Cauchy-Schwarz inequality and step~$(ii)$ from Lemma~\ref{lemma:f_to_d_matrix}.
Putting the pieces together, we obtain that
\begin{align*}
      \abss{\sum_{i=1}^\obs \johnweights_{y, i} \projd_{y, i}^2 \projz_{y, i}}
      \leq \vecnorm{\projzsub_y}{1} \cdot \brackets{\johnkappa \cdot \max_{i \in [\obs]} \projd_{y, i}^2 + \johnkappa^{3/2} \parenth{\sum_{j=1}^\obs \johnweights_{x, j} \projd_{x, j}^4 \frac{\johnweights_{y, j}}{\johnweights_{x, j}} \frac{\projd_{y, j}^4}{\projd_{x, j}^4}}^{1/2}}.
\end{align*}
Using the bound on $\vecnorm{\projzsub_y}{1}$ from Lemma~\ref{lemma:john_gradient_and_hessian_and_bounds}, we have
\begin{align*}
  \abss{\sum_{i=1}^\obs \johnweights_{y, i} \projd_{y, i}^2 \projz_{y, i}}
  \leq \parenth{56 \johnkappa^2 \sum_{i=1}^\obs \johnweights_{y, i}\projd_{y, i}^2} \cdot
  \brackets{\johnkappa \cdot \max_{i \in [\obs]} \projd_{y, i}^2 + \johnkappa^{3/2} \parenth{\sum_{j=1}^\obs \johnweights_{x, j} \projd_{x, j}^4 \frac{\johnweights_{y, j}}{\johnweights_{x, j}} \frac{\projd_{y, j}^4}{\projd_{x, j}^4}}^{1/2}}.
\end{align*}
Substituting the expression for $\projd_{x, i}= \frac{r}{\johnkappa^2 \dims^{3/4}}\parenth{\hat{a}_{x,i}\tp \rvg}$ yields the claim.



\section{Proofs of Lemmas from Section~\ref{sub:deterministic_bounds}} 
\label{sec:proofs_of_technical_lemmas_from_section_}

In this section we collect proofs of lemmas from Section~\ref{sub:deterministic_bounds}.
Each lemma is proved in a different subsection.


\subsection{Proof of Lemma~\ref{lemma:john_gradient_and_hessian_and_bounds}} 
\label{sub:proof_of_lemma_lemma:john_gradient_and_hessian_and_bounds}

\begin{subequations}
Up to second order terms, we have
\begin{align}
\frac{1}{\slack_{x+ h, i}^2} &= \frac{1}{\slack_{x,
    i}^2}\brackets{1+\frac{2a_i^\top h}{\slack_{x, i} } +
  \frac{3(a_i^\top h)^2}{\slack_{x, i}^2}  } + \order{\vecnorm{h}{2}^3},
\label{eq:john_expansion_of_s2}\\
    \johnweights_{y+h, i} &= \johnweights_{y, i} + h^\top \nabla \johnweights_{y, i} + \frac{1}{2} h^\top \nabla^2 \johnweights_{y, i} h + \order{\vecnorm{h}{2}^3}, \label{eq:john_expansion_of_g}\\
    \johnweights_{y+h, i}^\alpha &= \johnweights_{y, i}^\alpha + \alpha \johnweights_{y, i}^{\alpha-1}\parenth{h^\top \nabla \johnweights_{y, i}+\frac{1}{2} h^\top \nabla^2 \johnweights_{y, i} h} + \frac{\alpha\parenth{\alpha-1}}{2} \johnweights_{y, i}^{\alpha-2} \parenth{h^\top \nabla \johnweights_{y, i}}^2 + \order{\vecnorm{h}{2}^3}, \label{eq:john_expansion_of_g_alpha}
\end{align}
Further, let
\begin{align}
    \johnmod_y \defn A_y\tp \johnweightmatrix_y^\alpha A_y = \sum_{i=1}^\obs \johnweights_{y, i}^\alpha \frac{a_ia_i\tp}{\slack_{y, i}^2}.
    \label{eq:john_johnmod}
\end{align}
Using equations~\eqref{eq:john_expansion_of_s2} and \eqref{eq:john_expansion_of_g_alpha}, and substituting $\projd_{y, i} = a_i\tp h/\slack_{y, i}, \projf_{y, i} = h\tp \nabla \johnweights_{y, i}/\johnweights_{y, i}$ and $\projz_{y, i} = \frac{1}{2} h\tp \nabla^2 \johnweights_{y, i} h /\johnweights_{y, i}$, we find that
\begin{align*}
    \johnmod_{y+h}
    & = \sum_{i=1}^\obs \brackets{1 + \alpha \projf_{y, i} +
    \alpha \projz_{y, i} + \frac{\alpha\parenth{\alpha-1}}{2} \projf_{y, i}^2} \brackets{1+2\projd_{y, i}+3\projd_{y, i}^2 } \johnweights_{y, i}^\alpha \frac{a_ia_i^\top}{s_{y, i}^2}
    + \order{\vecnorm{h}{2}^3}.
\end{align*}
Note that $\projd_{y, i}$ and $\projf_{y, i}$ are first order terms in $\vecnorm{h}{2}$ and $\projz_{y, i}$ is a second order term in $\vecnorm{h}{2}$.

Thus we obtain
\begin{align*}
    \johnmod_{y+h} - \johnmod_y &= \underbrace{\sum_{i=1}^\obs \parenth{ 2 \projd_{y, i} + \alpha \projf_{y, i} } \johnweights_{y, i}^\alpha \frac{a_ia_i^\top}{s_{y, i}^2}}_{=:\Delta_{y, h}^{(1)}}\\
     &+  \underbrace{\sum_{i=1}^\obs \brackets{3\projd_{y, i}^2 + 2 \alpha \projd_{y, i} \projf_{y, i} + \alpha \projz_{y, i}  + \frac{\alpha(\alpha-1)}{2} \projf_{y, i}^2 }\johnweights_{y, i}^\alpha \frac{a_ia_i^\top}{s_{y, i}^2}}_{=:\Delta_{y, h}^{(2)}} + \order{\vecnorm{h}{2}^3}.
\end{align*}
Let $\Delta_{y, h}  \defn \Delta_{y, h}^{(1)} + \Delta_{y, h}^{(2)}$.
Note that $\Delta_{y, h}^{(i)}$ denotes the $i$-th order term in $\vecnorm{h}{2}$.
Finally, the following expansion also comes in handy for our derivations:
\begin{align}
a_i^T \johnmod_{y+ h}^{-1} a_i &= a_i^\top \johnmod_y^{-1}a_i -
a_i^\top \johnmod_y^{-1} \Delta_{y, h} \johnmod_y^{-1} a_i
+ a_i^\top \johnmod_y^{-1} \Delta_{y, h} \johnmod_y^{-1}
\Delta_{y, h} \johnmod_y^{-1} a_i
+ \order{\vecnorm{h}{2}^3}.
\label{eq:john_expansion_of_johnmod}
\end{align}
\end{subequations}
\subsubsection{Proof of part~\ref{item:gradient_g}: Gradient of weights} 
\label{ssub:proof_of_part_item:gradient_g}
The expression for the gradient $\nabla \johnweights_{y, i}$ is derived in Lemma~14 of the paper~\citep{lee2014path} and is thereby omitted.


\subsubsection{Proof of part~\ref{item:hessian_g}: Hessian of weights} 
\label{ssub:proof_of_part_item:hessian_g}
We claim that
\begin{align}
    \projzsub_{y} = \parenth{\Ind - \alpha \johnlaplacian_y \johnweightmatrix_y^{-1}}
    \begin{bmatrix}
        \frac{1}{2} h^\top \nabla^2 \johnweights_{y, 1} h \\ \cdots \\ \frac{1}{2} h^\top \nabla^2 \johnweights_{y, m} h
    \end{bmatrix}
    & = (2 \projdmatrix_y + \alpha \projfmatrix_y) \johnprojmatrix_y^{(2)}  (2 \projdmatrix_y + \alpha \projfmatrix_y) \mathbf{1} \nonumber \\
    & + \parenth{\johnlevmatrix_y - \johnprojmatrix_y^{(2)}} \brackets{2 \alpha \projdmatrix_y\projfmatrix_y + 3 \projdmatrix_y^2 + \alphaprod \projfmatrix_y^2 } \mathbf{1} \nonumber \\
    & + \diag \parenth{ \johnprojmatrix_y (2\projdmatrix_y + \alpha \projfmatrix_y)   \johnprojmatrix_y (2\projdmatrix_y + \alpha \projfmatrix_y) \johnprojmatrix_y },
\label{eq:john_hessians_of_g}
\end{align}
where we have used $\diag(\mymat)$ to denote the diagonal vector $(\mymat_{1, 1}, \ldots, \mymat_{\obs, \obs})$ of the matrix $\mymat$.
Deferring the proof of this expression for the moment, we now derive a bound on the $\ell_1$ norm of $\projzsub_y$.
Expanding the $i$-th term of $\projzsub_{y, i}$ from equation~\eqref{eq:john_hessians_of_g}, we obtain
 \begin{align*}
    \projzsub_{y, i} &= (2\projd_{y, i} + \alpha \projf_{y, i} ) \sum_{j=1}^\obs (2 \projd_{y, j} + \alpha \projf_{y, j}) \johnprojmatrix_{y, i, j}^2
   + \brackets{2 \alpha \projd_{y, i} \projf_{y, i} +  3\projd_{y, i}^2 +
   \alphaprod \projf_{y, i}^2} \levjohn_{y, i} \\
   &\quad - \sum_{j=1}^\obs \brackets{2 \alpha \projd_{y, j} \projf_{y, j} + 3\projd_{y, j}^2 + \alphaprod \projf_{y, j}^2 } \johnprojmatrix_{y, i, j}^2
   + \sum_{j, l=1}^\obs (2 \projd_{y, j} + \alpha \projf_{y, j}) (2 \projd_{y, l}  + \alpha \projf_{y, l} ) \johnprojmatrix_{y, i, j} \johnprojmatrix_{y, j, l} \johnprojmatrix_{y, l, i}.
  \end{align*}
  Recall that $\alpha = 1-1/\log_2(2\obs/\dims)$.
  Since $\obs \geq \dims$ for polytopes, we have $\alpha \in [0, 1]$ and consequently $\vert \alphaprod \vert = \vert \alpha (\alpha-1)/2\vert \in [0, 1]$.
  Further note that $\johnprojmatrix_x$ is an orthogonal projection matrix, and hence we have
\begin{align*}
  \diag(\johnprojmatrix_x e_i) \johnprojmatrix_x \diag(\johnprojmatrix_x e_i)
  \preceq \diag(\johnprojmatrix_x e_i)\diag(\johnprojmatrix_x e_i).
\end{align*}
  Combining these observations with the AM-GM inequality, we have
  \begin{align*}
  \abss{\projzsub_{y, i}}
  &\leq 7 \levjohn_{y, i} \projd_{y, i}^2 + 3 \levjohn_{y, i} \projf_{y, i}^2 + \sum_{j=1}^\obs \parenth{ 13\projd_{y, j}^2 + 6 \projf_{y, j}^2} \johnprojmatrix_{y, i, j}^2.
\end{align*}
Summing both sides over the index $i$, we find that
\begin{align*}
\sum_{i=1}^\obs  \abss{\projzsub_{y, i}}
\stackrel{(i)}{\leq} \sum_{i=1}^\obs 20 \levjohn_{y, i} \projd_{y, i}^2 + 9 \levjohn_{y, i} \projf_{y, i}^2
\stackrel{(ii)}{\leq} \sum_{i=1}^\obs 20 \johnweights_{y, i} \projd_{y, i}^2 + 9 \johnweights_{y, i} \projf_{y, i}^2
\stackrel{(iii)}{\leq} 56\johnkappa^2 \sum_{i=1}^\obs \johnweights_{y, i} \projd_{y, i}^2,
\end{align*}
where step $(i)$ follows from Lemma~\ref{ppt:john_all_properties}~\ref{item:john_sigma_properties}, step $(ii)$ from Lemma~\ref{lemma:john_first_bounds}~\ref{item:john_weight_equality} and step $(iii)$ from the bound~\eqref{eq:john_exact_f_to_d_bound}.

We now return to the proof of expression~\eqref{eq:john_hessians_of_g}.
Using equation~\eqref{eq:john_normal_equation}, we find that
\begin{align}
    \frac{1}{2} h\tp \nabla^2 \johnweights_{y, i} h = \frac{1}{2} h \tp \nabla^2 \levjohn_{y, i} h \quad \text{ for all } i \in [\obs].
    \label{eq:john_equality_of_hessians}
\end{align}
Next, we derive the Taylor series expansion of $\levjohn_{y, i}$.
Using the definition of $\johnmod_x$~\eqref{eq:john_johnmod} in equation~\eqref{eq:john_projection_matrix}, we find that $\levjohn_{y, i} = \johnweights_{y, i}^\alpha\frac{a_i\tp \johnmod_{y}^{-1}a_i}{\slack_{y, i}^2}$.
To compute the difference $\levjohn_{y+h, i} - \levjohn_{y, i}$, we use the expansions~\eqref{eq:john_expansion_of_s2}, \eqref{eq:john_expansion_of_g_alpha} and \eqref{eq:john_expansion_of_johnmod}.
Letting $\alphaprod = \alpha(\alpha-1)/2$, we have
\begin{align*}
    \levjohn_{y+h, i}
    & = \johnweights_{y+h, i}^\alpha \frac{a_i^\top \johnmod_{y+h}^{-1} a_i}{\slack_{y+h,i}^2}\\
    & = \johnweights_{y, i}^\alpha \frac{a_i^\top \johnmod_{y+h}^{-1} a_i}{\slack_{y,i}^2} \brackets{1 + \alpha \projf_{y, i}
        +\alpha \projz_{y, i}
        +\alphaprod \projf_{y, i}^2} \brackets{1+2\projd_{y, i} + 3\projd_{y, i}^2} + \order{\vecnorm{h}{2}^3}\\
    & = \levjohn_{y, i} + (2\projd_{y, i} + \alpha \projf_{y, i}) \levjohn_{y, i} - \sum_{j=1}^\obs (2 \projd_{y, j} + \alpha \projf_{y, j}) \johnprojmatrix_{y, i, j}^2 + (2\projd_{y, i} + \alpha \projf_{y, i} ) \sum_{j=1}^\obs (2 \projd_{y, j} + \alpha \projf_{y, j}) \johnprojmatrix_{y, i, j}^2 \\
    & \quad + 2 \alpha \projd_{y, i} \projf_{y, i} \levjohn_{y, i} + \brackets{\alpha \projz_{y, i}+ \alphaprod \projf_{y, i}^2 + 3\projd_{y, i}^2} \levjohn_{y, i}
    - \sum_{j=1}^\obs \brackets{3\projd_{y, j}^2 + 2 \alpha \projd_{y, j} \projf_{y, j} + \alpha \projz_{y, j} + \alphaprod \projf_{y, j}^2 } \johnprojmatrix_{y, i, j}^2 \\
    & \quad + \sum_{j, l=1}^\obs (2 \projd_{y, j} + \alpha \projf_{y, j}) (2 \projd_{y, l}  + \alpha \projf_{y, l} ) \johnprojmatrix_{y, i, j} \johnprojmatrix_{y, j, l} \johnprojmatrix_{y, l, i} + \order{\vecnorm{h}{2}^3}.
\end{align*}
We identify the second order (in \order{\vecnorm{h}{2}^2}) terms in the previous expression.
Using the equation~\eqref{eq:john_equality_of_hessians}, these are indeed the terms that correspond to the terms $\frac{1}{2} h^\top \nabla^2 \johnweights_{y, i} h$, $i \in [\obs]$.
Substituting $\projz_{y, i} = \frac{1}{2}h\tp \nabla^2 \johnweights_{y, i} h / \johnweights_{y, i}$, we have
\begin{align*}
    &\frac{1}{2} h^\top \nabla^2 \johnweights_{y, i} h \\
    &= (2\projd_{y, i} + \alpha \projf_{y, i} ) \sum_{j=1}^\obs (2 \projd_{y, j} + \alpha \projf_{y, j}) \johnprojmatrix_{y, i, j}^2 + 2 \alpha \projd_{y, i} \projf_{y, i} \levjohn_{y, i}
    + \brackets{\frac{\alpha}{2} \frac{h^\top \nabla^2 \johnweights_{y, i} h}{\johnweights_{y,i}} + \alphaprod \projf_{y, i}^2 + 3\projd_{y, i}^2} \levjohn_{y, i} \\
    & - \sum_{j=1}^\obs \brackets{3\projd_{y, j}^2 + 2 \alpha \projd_{y, j} \projf_{y, j} + \frac{\alpha}{2} \frac{h^\top \nabla^2 \johnweights_{y, j} h}{\johnweights_{y,j}} + \alphaprod \projf_{y, j}^2 } \johnprojmatrix_{y, i, j}^2
    + \sum_{j, l=1}^\obs (2 \projd_{y, j} + \alpha \projf_{y, j}) (2 \projd_{y, l}  + \alpha \projf_{y, l} ) \johnprojmatrix_{y, i, j} \johnprojmatrix_{y, j, l} \johnprojmatrix_{y, l, i}.
\end{align*}

Collecting the different terms and doing some algebra yields the result~\eqref{eq:john_hessians_of_g}.

\subsubsection{Proof of part~\ref{item:gradient_log_det}: Gradient of logdet} 
\label{ssub:proof_of_part_item:gradient_log_det}
For a unit vector $h \in \realdim$, we have
\begin{align*}
    h\tp \log\det J_y &= \lim_{\delta \rightarrow 0}\frac{1}{\delta}(\log\det J_{y+ \delta h} - \log\det J_y)
    = \lim_{\delta \rightarrow 0}\frac{1}{\delta}(\log\det J_y^{-1/2}J_{y+ \delta h}J_y^{-1/2} - \log\det \Ind_\dims)
\end{align*}
Let $\hat a_{y, i} \defn \john_{y, i}^{-1/2}a_i/\slack_{y, i}$ for each $i \in [\obs]$.
Using the property $\log \det B = \trace\log B$, where $\log B$ denotes the logarithm of the matrix and that $\log \det \Ind_\dims = 0$,
we obtain
\begin{align*}
    h\tp \log\det J_y = \lim_{\delta \rightarrow 0} \frac{1}{\delta} \brackets{ \trace\log \parenth{\sum_{i=1}^\obs \frac{\johnweights_{y+\delta h}}{(1-\delta a_i\tp h/\slack_{y, i})} \hat a_{y, i} \hat a_{y, i}\tp} },
\end{align*}
where we have substituted $\slack_{y+\delta h, i} = \slack_{y, i} - \delta a_i\tp h$.
Keeping track of first order terms in $\delta$, and noting that $\sum_{i=1}^\obs \johnweights_{y, i} \hat a_{y, i} \hat a_{y, i}\tp = \Ind_\dims$, we find that
\begin{align*}
    \trace\log \parenth{\sum_{i=1}^\obs \frac{\johnweights_{y+\delta h, i}}{(1-\delta a_i\tp h/\slack_{y, i})} \hat a_{y, i} \hat a_{y, i}\tp}
    &= \trace\log \brackets{\sum_{i=1}^\obs \parenth{\johnweights_{y, i} + \delta h\tp \nabla \johnweights_{y, i} } \parenth{1+\frac{2\delta a_i\tp h}{\slack_{y, i}}} \hat a_{y, i} \hat a_{y, i}\tp} + \order{\delta^2}\\
    & = \trace \brackets{\sum_{i=1}^\obs \delta \parenth{\frac{2a_i\tp h}{\slack_{y, i}} + h\tp \nabla \johnweights_{y, i}} \hat a_{y, i} \hat a_{y, i}\tp} + \order{\delta^2}\\
    & = \sum_{i=1}^\obs \delta \parenth{\frac{2a_i\tp h}{\slack_{y, i}} + h\tp \nabla \johnweights_{y, i}} \johnlocalslack_{y, i} + \order{\delta^2}
\end{align*}
where in the last step we have used the fact that $\trace(\hat a_{y, i} \hat a_{y, i}\tp) = \hat a_{y, i} \tp \hat a_{y, i} = \johnlocalslack_{y, i}$ for each $i \in [\obs]$.
Substituting the expression for $\nabla \johnweights_{y}$ from part~\ref{item:gradient_g}, and rearranging the terms yields the claimed expression in the limit $\delta \rightarrow 0$.

\subsubsection{Proof of part~\ref{item:gradient_phi}: Gradient of $\johnphi$} 
\label{ssub:proof_of_part_item:gradient_phi}
Using the chain rule and the fact that $\nabla \slack_{y, i} = -a_i$, yields the result.
\subsubsection{Proof of part~\ref{item:john_hessian_log_det_bound}} 

\label{ssub:proof_of_part_ref}
We claim that
\begin{align*}
      \frac{1}{2} h \tp \hesslogdetJ_y h
      = \frac{1}{2}\brackets{\sum_{i=1}^\obs \johnweights_{y, i} \johnlocalslack_{y, i} (3\projd_{y, i}^2 + 2 \projd_{y, i}\projf_{y, i} + \projz_{y, i}) - \frac{1}{2} \sum_{i, j=1}^\obs \johnweights_{y, i} \johnweights_{y, j} \johnlocalslack_{y, i, j}^2 \parenth{2\projd_{y, i} + \projf_{y, i}} \parenth{2\projd_{y, j} + \projf_{y, j}}}.
    \end{align*}
The desired bound on $\abss{h \tp \hesslogdetJ_y h}/2$ now follows from an application of AM-GM inequality with Lemma~\ref{ppt:john_all_properties}\ref{item:john_theta_properties}.

We now derive the claimed expression for the directional Hessian of the function $\logdetJ$.
We have
\begin{align*}
    &\frac{1}{2} h\tp \parenth{\nabla^2 \log\det J_y} h
    = \lim_{\delta \rightarrow 0}\frac{1}{2\delta^2}(\log\det J_y^{-1/2}J_{y+ \delta h}J_y^{-1/2} + \log\det J_y^{-1/2}J_{y- \delta h}J_y^{-1/2} - 2\log\det \Ind_\dims)  \\
    &= \frac{1}{2} \lim_{\delta \rightarrow 0} \frac{1}{\delta^2} \brackets{ \trace\log \parenth{\sum_{i=1}^\obs \frac{\johnweights_{y+\delta h}}{(1-\delta a_i\tp h/\slack_{y, i})} \hat a_{y, i} \hat a_{y, i}\tp} +
    \trace\log \parenth{\sum_{i=1}^\obs \frac{\johnweights_{y-\delta h}}{(1+\delta a_i\tp h/\slack_{y, i})} \hat a_{y, i} \hat a_{y, i}\tp} }.
\end{align*}
\noindent Expanding the first term in the above expression, we find that
\begin{align*}
    &\trace\log \parenth{\sum_{i=1}^\obs \frac{\johnweights_{y+\delta h, i}}{(1-\delta a_i\tp h/\slack_{y, i})} \hat a_{y, i} \hat a_{y, i}\tp}\\
    &= \trace\log \underbrace{\brackets{
    \sum_{i=1}^\obs
    \parenth{\johnweights_{y, i} + \delta h\tp \nabla \johnweights_{y, i} + \frac{\delta^2}{2} h\tp \nabla^2 \johnweights_{y, i} h }
    \parenth{1+2\delta\frac{a_i\tp h}{\slack_{y, i}} + 3\delta^2\frac{(a_i\tp h)^2}{\slack_{y, i}^2}}
    \hat a_{y, i} \hat a_{y, i}\tp
     }}_{=:\Ind_\dims + \mymat}
     + \order{\delta^3}.
\end{align*}
Substituting the shorthand notation from equations~\eqref{eq:john_def_D}, \eqref{eq:john_def_F} and \eqref{eq:john_def_Z}, we have
\begin{align*}
    \mymat
    &=\sum_{i=1}^\obs \johnweights_{y, i} \brackets{\delta (2\projd_{y, i} + \projf_{y, i})
    + \delta^2 (3\projd_{y, i}^2 + 2 \projd_{y, i}\projf_{y, i} + \projz_{y, i})} \hat a_{y, i} \hat a_{y, i}\tp + \order{\delta^3}.
\end{align*}
Now we make use of the following facts (1) $\trace\log (\Ind_\dims + \mymat) = \trace\brackets{\mymat - \frac{\mymat^2}{2} + \order{\vecnorm{\mymat}{}^3}}$, (2) for each $i, j \in [\obs]$, we have $\trace(\hat a_{y, i} \hat a_j\tp) = \hat a_{y, i} \tp \hat a_j = \johnlocalslack_{y, i, j}$, and (3) for each $i \in [\obs]$, we have $\johnlocalslack_{y, i, i} = \johnlocalslack_{y, i}$.
Thus we obtain
\begin{align*}
    \trace\log \parenth{\sum_{i=1}^\obs \frac{\johnweights_{y+\delta h, i}}{(1-\delta a_i\tp h/\slack_{y, i})} \hat a_{y, i} \hat a_{y, i}\tp}
    &= \sum_{i=1}^\obs \johnweights_{y, i} \johnlocalslack_{y, i} \brackets{\delta (2\projd_{y, i} + \projf_{y, i})
    + \delta^2 (3\projd_{y, i}^2 + 2 \projd_{y, i}\projf_{y, i} + \projz_{y, i})} \\
    &\quad - \frac{1}{2} \sum_{i, j=1}^\obs \johnweights_{y, i} \johnweights_{y, j} \johnlocalslack_{y, i, j}^2 \delta^2 (2\projd_{y, i} + \projf_{y, i})   (2\projd_{y, j} + \projf_{y, j})
    + \order{\delta^3}.
\end{align*}
Similarly, we can obtain an expression for $\trace\log \parenth{\sum_{i=1}^\obs \frac{\johnweights_{y-\delta h}}{(1+\delta a_i\tp h/\slack_{y, i})} \hat a_{y, i} \hat a_{y, i}\tp}$.
Putting the pieces together, we obtain
\begin{align}
    \frac{1}{2} h\tp \parenth{\nabla^2 \log\det J_y} h = \sum_{i=1}^\obs \johnweights_{y, i} \johnlocalslack_{y, i} (3\projd_{y, i}^2 + 2 \projd_{y, i}\projf_{y, i} + \projz_{y, i})
    - \frac{1}{2} \sum_{i, j=1}^\obs \johnweights_{y, i} \johnweights_{y, j} \johnlocalslack_{y, i, j}^2 (2\projd_{y, i} + \projf_{y, i})   (2\projd_{y, j} + \projf_{y, j}).
    \label{eq:john_final_hess_detj_expression}
\end{align}


\subsubsection{Proof of part~\ref{item:hessian_beta_bound}} 
\label{ssub:proof_of_part_item:hessian_beta_bound}
We claim that
\begin{align}
\frac{1}{2} h\tp \nabla^2 \johnphi_{y, i} h = \johnphi_{y, i} \parenth{2 \projd_{y, i} \projf_{y, i} + 3 \projd_{y, i}^2 + \projz_{y, i}}.
\label{eq:john_hessian_phi}
\end{align}
The claim follows from a straightforward application of chain rule and substitution of the expressions for $\nabla \johnweights_{y, i}$ and $\nabla^2 \johnweights_{y, i}$ in terms of the shorthand notation $\projd_{y, i}, \projf_{y, i}$ and $\projz_{y, i}$.
Multiplying both sides of equation~\eqref{eq:john_hessian_phi} with $\projd_{y, i}^2\slack_{y, i}^2$ and summing over index $i$, we find that
\begin{align*}
    \sum_{i=1}^\obs \projd_{y, i}^2 \slack_{y, i}^2 \frac{1}{2} h\tp \nabla \johnphi_{y, i}^2 h
    = \sum_{i=1}^\obs \projd_{y, i}^2 \slack_{y, i}^2 \johnphi_{y, i}\brackets{\projz_{y, i} + 2\projd_{y, i} \projf_{y, i} + 3 \projd_{y, i}^2}
    &= \sum_{i=1}^\obs \projd_{y, i}^2 \johnweights_{y, i}\brackets{\projz_{y, i} + 2\projd_{y, i} \projf_{y, i} + 3 \projd_{y, i}^2}\\
    &\leq\sum_{i=1}^\obs \projd_{y, i}^2 \johnweights_{y, i}\brackets{\projz_{y, i} + \projf_{y, i}^2 + 4 \projd_{y, i}^2},
\end{align*}
where in the last step we have used the AM-GM inequality.
The claim follows.


\subsection{Proof of Lemma~\ref{lemma:f_to_d_l2}} 
\label{sub:proof_of_lemma_lemma:f_to_d_l2}

  We claim that
  \begin{align}
    \label{eq:john_f_to_d_intermediate}
    0 \preceq \johnweightmatrix_y^{-1/2} \parenth{c_1 \Ind_\obs + c_2 \johnlaplacian_y \parenth{\johnweightmatrix_y-\alpha\johnlaplacian_y}^{-1}} \johnweightmatrix_y^{1/2} \preceq \parenth{c_1 + c_2}\johnkappa \Ind_\obs.
  \end{align}
  The proof of the lemma is immediate from this claim, as for any PSD matrix $H \leq c\Ind_\obs$, we have $H^2 \leq c^2 \Ind_\obs$.

  We now prove claim~\eqref{eq:john_f_to_d_intermediate}.
  Note that
  \begin{align}
     \johnweightmatrix_y^{-1/2}\johnlaplacian_y \parenth{\johnweightmatrix_y-\alpha\johnlaplacian_y}^{-1} \johnweightmatrix_y^{1/2}
     = \underbrace{\johnweightmatrix_y^{-1/2}\johnlaplacian_y \johnweightmatrix_y^{-1/2}}_{:=\mymat_y}
     (\Ind_\obs -\johnalpha \johnweightmatrix_y^{-1/2}\johnlaplacian_y \johnweightmatrix_y^{-1/2})^{-1}.
     \label{eq:john_equality_glg}
  \end{align}
  Note that the RHS is equal to the matrix $\mymat_y (\Ind_\obs - \alpha_\tagjohn \mymat_y)^{-1}$ which is symmetric.
  Observe the following ordering of the matrices in the PSD cone
  \begin{align*}
    \johnlevmatrix_y + \johnbeta \Ind_\obs = \johnweightmatrix_y
    \succeq \johnlevmatrix_y
    \succeq \johnlaplacian_y = \johnlevmatrix_y - \johnprojmatrix_y^{(2)}
    \succeq 0.
  \end{align*}
  For the last step we have used the fact that $ \johnlevmatrix_y - \johnprojmatrix_y^{(2)}$ is a diagonally dominant matrix with non negative entries on the diagonal to conclude that it is a PSD matrix.
  Consequently, we have
  \begin{align}
  \label{eq:john_glg}
    \mymat_y = \johnweightmatrix_y^{-1/2}\johnlaplacian_y \johnweightmatrix_y^{-1/2} \preceq \Ind_\obs.
  \end{align}
  Further, recall that $\johnalpha = (1-1/\johnkappa) \Leftrightarrow \johnkappa = (1-\johnalpha)^{-1}$.
  As s result, we obtain
  \begin{align*}
    0 \preceq (\Ind_\obs-\johnalpha\johnweightmatrix_y^{-1/2}\johnlaplacian_y \johnweightmatrix_y^{-1/2} )^{-1} \preceq \johnkappa \Ind_\obs.
  \end{align*}
  Multiplying both sides by $\mymat_y^{1/2}$ and using the relation~\eqref{eq:john_glg}, we obtain
  \begin{align}
    0 \preceq \mymat_y^{1/2}  (\Ind_\obs-\johnalpha\johnweightmatrix_y^{-1/2}\johnlaplacian_y \johnweightmatrix_y^{-1/2} )^{-1} \mymat_y^{1/2} \preceq \johnkappa \Ind_\obs.
    \label{eq:john_G_laplace_G_bound}
  \end{align}
  Using the fact that $\mymat_y$ commutes with $(\Ind_\obs - \mymat_y)^{-1}$, we obtain $\mymat_y (\Ind_\obs-\johnalpha\mymat_y)^{-1} \preceq \johnkappa \Ind_\obs$.
  Using observation~\eqref{eq:john_equality_glg} now completes the proof.


\subsection{Proof of Lemma~\ref{lemma:f_to_d_matrix}} 
\label{sub:proof_of_lemma_lemma:f_to_d_matrix}


Without loss of generality, we can first prove the result for $i=1$.
Let $\johnftdinvvector \defn \johnftdinvmatrix_y\tp e_1$ denote the first row of the matrix $\johnftdinvmatrix_y$.
Observe that
\begin{align}
    \label{eq:john_F_to_D_infinity_bound}
    e_1
    = \parenth{\johnweightmatrix_y - \alpha \johnlaplacian_y} \johnweightmatrix_y^{-1} \johnftdinvvector
    = \johnftdinvvector - \alpha \johnlevmatrix_y \johnweightmatrix_y^{-1} \johnftdinvvector + \alpha \johnprojmatrix_y^{(2)} \johnweightmatrix_y^{-1} \johnftdinvvector
\end{align}
We now prove bounds~\eqref{eq:john_mu_first_coordinate_bound} and \eqref{eq:john_mu_sum_of_terms_bound} separately.


\paragraph{Proof of bound~\eqref{eq:john_mu_first_coordinate_bound}:} 
\label{par:proof_of_bound_eq:john_mu_first_coordinate_bound}

Multiplying the equation~\eqref{eq:john_F_to_D_infinity_bound} on the left by $\johnftdinvvector \tp \johnweightmatrix_y^{-1}$, we obtain
\begin{align}
  g_1^{-1} \johnftdinvvector_1
  &= \johnftdinvvector \tp \johnweightmatrix_y^{-1} \johnftdinvvector - \alpha \johnftdinvvector \tp \johnweightmatrix_y^{-1} \johnlevmatrix_y \johnweightmatrix_y^{-1} \johnftdinvvector + \alpha \johnftdinvvector \tp \johnweightmatrix_y^{-1} \johnprojmatrix_y^{(2)} \johnweightmatrix_y^{-1} \johnftdinvvector \notag \\
  & \geq \johnftdinvvector \tp \johnweightmatrix_y^{-1} \johnftdinvvector - \alpha \johnftdinvvector \tp \johnweightmatrix_y^{-1} \johnlevmatrix_y \johnweightmatrix_y^{-1} \johnftdinvvector
  \label{eq:john_F_to_D_infinity_bound_upper}
  \\
  & \geq \parenth{g_1^{-1} - \alpha \levjohn_{y, 1}/g_1^2} \johnftdinvvector_1^2. \notag
\end{align}
Rearranging terms, we obtain
\begin{align}
0 \leq \johnftdinvvector_1 \leq \frac{\johnweights_{y, 1}}{\johnweights_{y, 1} -\alpha \levjohn_{y, 1}} \stackrel{(i)}{\leq} \johnkappa,
\label{eq:john_upper_bound_mu1}
\end{align}
where  inequality (i) follows from the facts that $\johnweights_{y, j} \geq \levjohn_{y, j}$ and  $(1-\alpha) = \johnkappa$.


\paragraph{Proof of bound~\eqref{eq:john_mu_sum_of_terms_bound}:} 
\label{par:proof_of_bound_eq:john_mu_sum_of_terms_bound}
In our proof, we  use the following improved lower bound for the term $\johnftdinvmatrix_{y, 1, 1} = \johnftdinvvector_1$.
\begin{align}
\johnftdinvvector_1 \geq \frac{\johnweights_{y, 1}}{\johnweights_{y, 1} - \alpha \levjohn_{y, 1} + \alpha \levjohn_{y, 1}^2},
\label{eq:john_lower_bound_mu1}
\end{align}
Deferring the proof of this claim at the moment, we now complete the proof.

We begin by deriving a weighted $\ell_2$-norm bound for the vector $\tilde\johnftdinvvector = (\johnftdinvvector_2, \ldots, \johnftdinvvector_\obs)\tp$.
Equation~\eqref{eq:john_F_to_D_infinity_bound_upper} implies
\begin{align*}
\johnweights_{y, 1}^{-1} \johnftdinvvector_1 \parenth{1- \johnftdinvvector_1 + \alpha \frac{\levjohn_{y, 1}}{\johnweights_{y, 1}} \johnftdinvvector_1 } \geq \sum_{j=2}^\obs \johnftdinvvector_j^2 \parenth{\johnweights_{y, j}^{-1} - \alpha \johnweights_{y, j}^{-2} \levjohn_{y, j}}
\stackrel{(i)}{\geq} (1-\alpha)
\sum_{j=2}^\obs \frac{\johnftdinvvector_j^2}{\johnweights_{y, j}},
\end{align*}
where step $(i)$ follows from the fact that $\johnweights_{y, i} \geq \levjohn_{y, i}$.
Now, we upper bound the expression on the left hand side of the above inequality using the upper~\eqref{eq:john_upper_bound_mu1} and lower~\eqref{eq:john_lower_bound_mu1} bounds on $\johnftdinvvector_1$:
\begin{align*}
\johnweights_{y, 1}^{-1} \johnftdinvvector_1 \parenth{1- \johnftdinvvector_1 + \alpha \frac{\levjohn_{y, 1}}{\johnweights_{y, 1}} \johnftdinvvector_1 } &\leq \johnweights_{y, 1}^{-1} \frac{\johnweights_{y, 1}}{\johnweights_{y, 1} - \alpha \levjohn_{y, 1}} \parenth{1 - \parenth{1 - \alpha\frac{\levjohn_{y, 1}}{\johnweights_{y, 1}}} \frac{\johnweights_{y, 1}}{\johnweights_{y, 1} - \alpha \levjohn_{y, 1} + \alpha \levjohn_{y, 1}^2 }} \\
& = \frac{\alpha \levjohn_{y, 1}^2}{\parenth{\johnweights_{y, 1} - \alpha \levjohn_{y, 1}}\parenth{\johnweights_{y, 1} - \alpha \levjohn_{y, 1} + \alpha \levjohn_{y, 1}^2}} \\
& \leq \johnkappa^2,
\end{align*}
where in the last step we have used the facts that $\johnweights_{y, 1} \geq \levjohn_{y, 1}$ and $(1-\alpha)^{-1} = \johnkappa$.
Putting the pieces together, we obtain $\sum_{j=2}^\obs \johnftdinvvector_j^2 \johnweights_{y, j}^{-1} \leq \johnkappa^3,$
which is equivalent to our claim~\eqref{eq:john_mu_sum_of_terms_bound} for $i=1$.

It remains to prove our earlier claim~\eqref{eq:john_lower_bound_mu1}.
Writing equation~\eqref{eq:john_F_to_D_infinity_bound} separately for the first coordinate and for the rest of the coordinates, we obtain
\begin{subequations}
\begin{align}
  1 &= \parenth{1 - \alpha \levjohn_{y, 1} \johnweights_{y, 1}^{-1} + \alpha \levjohn_{y, 1, 1}^2 \johnweights_{y, j}^{-1}} \johnftdinvvector_1 + \alpha \sum_{j = 2}^\obs \levjohn_{y, 1, j}^2 \johnweights_{y, j}^{-1} \johnftdinvvector_j \label{eq:john_F_to_D_infinity_bound_dev1}, \quad \text{and} \\
  0 &= \parenth{\Ind_{\obs-1} - \alpha \johnlevmatrix_y' \johnweightmatrix_y'^{-1}} \begin{pmatrix}\johnftdinvvector_2 \\ \vdots \\ \johnftdinvvector_\obs \end{pmatrix} + \alpha \johnprojmatrix_y'^{(2)} \johnweightmatrix_y'^{-1}\begin{pmatrix}\johnftdinvvector_2 \\ \vdots \\ \johnftdinvvector_\obs \end{pmatrix} + \alpha \johnweights_{y, 1}^{-1} \johnftdinvvector_1 \begin{pmatrix} \levjohn_{y, 1, 2}^2 \\ \vdots \\ \levjohn_{y, 1, \obs}^2
\end{pmatrix}, \label{eq:john_F_to_D_infinity_bound_dev2}
\end{align}
\end{subequations}
where $\johnweightmatrix_y'$ (respectively $\johnlevmatrix_y', \johnprojmatrix_y'^{(2)}$) denotes the principal minor of $\johnweightmatrix_y$ (respectively $\johnlevmatrix_y, \johnprojmatrix_y^{(2)}$) obtained by excluding the first column and the first row.
Multiplying both sides of the equation~\eqref{eq:john_F_to_D_infinity_bound_dev2} from the left by $\begin{pmatrix}\johnftdinvvector_2, \cdots, \johnftdinvvector_\obs\end{pmatrix} \johnweightmatrix_y'^{-1}$, we obtain
\begin{align}
  0 = \sum_{j=2}^\obs \underbrace{\frac{1}{\johnweights_{y, j}}\parenth{1- \frac{\alpha\levjohn_{y, j}}{\johnweights_{y, j}}} \johnftdinvvector_{j}^2}_{c_{y, j}}
  + \alpha \underbrace{\begin{pmatrix}\johnftdinvvector_2, \cdots, \johnftdinvvector_\obs\end{pmatrix} \johnweightmatrix_y'^{-1} \johnprojmatrix_y'^{(2)} \johnweightmatrix_y'^{-1}\begin{pmatrix}\johnftdinvvector_2 \\ \vdots \\ \johnftdinvvector_\obs \end{pmatrix}}_{C_{y. 2}}
  + \alpha \frac{\johnftdinvvector_1}{\johnweights_{y, 1}} \sum_{j=2}^\obs \frac{\levjohn_{y, j}^2}{\johnweights_{y, j}} \johnftdinvvector_j.
  \label{eq:john_second_sum_mu}
\end{align}
Observing that $\alpha \in [0, 1]$ and $\johnweights_{y, j} \geq \levjohn_{y, j}$ for all $y \in \intP$ and $j \in [\obs]$, we obtain $c_{y, j} \geq 0$.
Further, note that $\johnweightmatrix_y'^{-1} \johnprojmatrix_y'^{(2)} \johnweightmatrix_y'^{-1}$ is a PSD matrix and hence we have that $C_{y, 2} \geq 0$.
Putting the pieces together, we have
\begin{align*}
  \alpha \frac{\johnftdinvvector_1}{\johnweights_{y, 1}} \sum_{j=2}^\obs \frac{\levjohn_{y, j}^2}{\johnweights_{y, j}} \johnftdinvvector_j \leq 0.
\end{align*}
Combining this inequality with equation~\eqref{eq:john_F_to_D_infinity_bound_dev1} yields the claim.


\subsection{Proof of Corollary~\ref{corollary:f_to_d_l1}} 
\label{sub:proof_of_corollary_corollary:f_to_d_l1}
Without loss of generality, we can prove the result for $i=1$.
Applying Cauchy-Schwarz inequality, we have
\begin{align*}
\vecnorm{\johnftdinvvector}{1} &= \johnftdinvvector_1 + \sum_{j=2}^\obs \abss{\johnftdinvvector_j}
\leq \johnftdinvvector_1 + \sqrt{\sum_{j=2}^\obs \frac{\johnftdinvvector_j^2}{\johnweights_{y, j}} \cdot \sum_{j=2}^\obs \johnweights_{y, j}}
\leq \johnkappa +  \johnkappa^{3/2} \cdot \sqrt{1.5\ \dims}
\leq 3\sqrt{\dims} \johnkappa^{3/2},
\end{align*}
where to assert the last inequality we have used Lemma~\ref{lemma:f_to_d_matrix} and Lemma~\ref{lemma:john_first_bounds}\ref{item:john_weight_sum}.
The claim~\eqref{eq:john_mu_1_bound} follows.
Further, noting that the infinity norm of a matrix is the $\ell_1$-norm of its transpose, we obtain $\matsnorm{\parenth{\johnweightmatrix_y-\alpha\johnlaplacian_y}^{-1} \johnweightmatrix_y}{\infty} \leq 3 \sqrt{\dims} \johnkappa^{3/2}$ as claimed.



\section{Proof of Lemmas from Section~\ref{sub:tail_bounds}} 
\label{sec:proof_of_lemmas_from_section_sub:tail_bounds}

In this section, we collect proofs of auxiliary lemmas from Section~\ref{sub:tail_bounds}.

\subsection{Proof of Lemma~\ref{lemma:john_whp_slackness}} 
\label{sub:proof_of_lemma_whp:slackness}

Using Lemma~\ref{lemma:john_closeness_of_slackness}, and the relation~\eqref{eq:john_z_x_relation} we have
\begin{align}
  \parenth{1 - \frac{\slack_{z, i}}{\slack_{x, i}}}^2  \leq 4\frac{r^2}{\johnkappa^4\dims^{3/2}} \rvg\tp\rvg,
  \label{eq:john_gaussian_bound_on_slackness}
\end{align}
where $\rvg\sim\NORMAL(0, \Ind_\dims)$.
Define
\begin{align}
  \Delta_\slack \defn \max_{i \in [\obs],\  v \in \overline{xz}}\left\vert 1 - \frac{\slack_{v, i}}{\slack_{x, i}}\right\vert.
  \label{eq:john_delta_s}
\end{align}
Using the standard Gaussian tail bound, we observe that $\Prob_{\rvg \sim \NORMAL(0, \Ind_n)} \brackets{\rvg \tp \rvg \geq \dims (1+\delta)} \leq 1 - \epsilon/4$ for $\delta  = \sqrt{\frac{2}{\dims}}$.
Plugging this bound in the inequality~\eqref{eq:john_gaussian_bound_on_slackness} and noting that for all $v \in \overline{xz}$ we have $\vecnorm{v-x}{\john_x} \leq \vecnorm{z-x}{\john_x} $, we obtain that
\begin{align*}
  \Prob_{z\sim \proposal_x}\brackets{\Delta_\slack  \leq \frac{2r^2 (1+\sqrt{2/\dims}\log(4/\epsilon)}{\johnkappa^4\sqrt{\dims}}} \geq 1 - \epsilon/4.
\end{align*}
Setting
\begin{align}
  \label{eq:john_whp_slackness_r_condition}
  r \leq 1/(25\sqrt{1+\sqrt{2}\log(4/\epsilon)}),
\end{align}
and noting that $\johnkappa^4\sqrt{\dims} \geq 1$ implies the claim~\eqref{eq:john_slack_whp}.
Hence, we obtain that $\Delta_\slack < .005/\johnkappa^2$  and consequently $\max_{i \in [\obs], v\in \overline{xz}} \slack_{x, i}/\slack_{v, i} \in (0.99, 1.01) $ with probability at least $1-\epsilon/4$.

We now claim that
\begin{align*}
  \max_{i \in [\obs], v \in \overline{xz}} \frac{\johnweights_{x, i}}{\johnweights_{v, i}} \in \brackets{1-24\johnkappa^2\Delta_\slack, 1+24\johnkappa^2\Delta_\slack}, \quad \text{ if } \Delta_\slack \leq \frac{1}{32\johnkappa^2}.
\end{align*}
The result follows immediately from this claim.
To prove the claim, note that equation~\eqref{eq:john_bound_on_ratio_of_johnweights} implies that if $\Delta_\slack  \leq \frac{1}{32\johnkappa^2}$, then
\begin{align*}
  \frac{\johnweights_{v, i}}{\johnweights_{x, i}} \in (e^{-8\johnkappa^2\Delta_\slack}, e^{8\johnkappa^2\Delta_\slack}) \quad \text{for all } i \in [\obs] \text{ and } v \in \overline{xz},
\end{align*}
which implies that
\begin{align*}
  \max_{i \in [\obs], v \in \overline{xz}} \frac{\johnweights_{x, i}}{\johnweights_{v, i}} \in (e^{-8\johnkappa^2\Delta_\slack}, e^{8\johnkappa^2\Delta_\slack}).
\end{align*}
Asserting the facts that $e^x \leq 1+3x $ and $e^{-x} \geq 1-3x$, for all $x \in [0, 1]$ yields the claim.


\subsection{Proof of Lemma~\ref{lemma:john_gaussian_moment_bounds}} 
\label{sub:proof_of_lemma_lemma:john_gaussian_moment_bounds}

The proof once again makes use of the classical tail bounds for
polynomials in Gaussian random variables.
We restate the classical result stated in equation~\eqref{EqnJansonBound} for convenience.
For any $\dims \geq 1$,
any polynomial $P:\realdim \rightarrow \real$ of degree $\degree$, and
any $t \geq (2e)^{{\degree}/{2}}$, we have
\begin{align}
  \label{EqnJansonBound}
    \Prob\brackets{\abss{P(\rvg)} \geq t
      \parenth{\Exs{P(\rvg)}^2}^{\frac{1}{2}}} \leq \exp\parenth{-
      \frac{\degree}{2e} t^{{2}/{\degree}}},
\end{align}
where $\rvg \sim \NORMAL(0, \Ind_n)$ denotes a standard Gaussian vector in $n$ dimensions.

Recall the notation from equation~\eqref{eq:john_hat_a_hat_b} and observe that
\begin{align}
  \vecnorm{\hat a_{x, i}}{2}^2 = \johnlocalslack_{x, i},
  \quad \text{and} \quad
  \hat a_{x, i}\tp \hat a_{x, j} = \johnlocalslack_{x, i, j}.
  \label{eq:john_hat_a_norm}
\end{align}
We also have
\begin{align}
  \sum_{i=1}^\obs \johnweights_{x, i} \hat a_{x, i} \hat a_{x, i}\tp =
  \john_x^{-1/2} \sum_{i=1}^\obs \johnweights_{x, i} \frac{ a_ia_i\tp}{s_{x, i}^2} \john_x^{-1/2}
  = \Ind_\dims.
  \label{eq:john_hat_a_a_Identity}
\end{align}
Further, using Lemma~\ref{lemma:f_to_d_l2} we obtain
\begin{align}
\sum_{i=1}^\obs \johnweights_{x, i} \hat{b}_{x, i}\hat{b}_{x, i} \tp
= \john_x^{-1/2}A_x\johnlaplacian_x \parenth{\johnweightmatrix_x - \alpha \johnlaplacian_x}^{-1} \johnweightmatrix_x \parenth{\johnweightmatrix_x - \alpha \johnlaplacian_x}^{-1} \johnlaplacian_x A_x\tp  \john_x^{-1/2}
= 4 \johnkappa^2 \Ind_\dims.
\label{eq:john_john_hat_b_hat_b_Ind}
\end{align}
Throughout this section, we consider a fixed point $x \in \intP$.
For brevity in our notation, we drop the dependence on $x$ for terms like $\johnweights_{x, i}, \johnlocalslack_{x, i}, \hat a_{x, i}$ (etc.) and  denote them simply by $\johnweights_i, \johnlocalslack_i, \hat a_i$ respectively.

We introduce some matrices and vectors that would come in handy for our proofs.
\begin{align}
  B = \begin{bmatrix}
    \sqrt{\johnweights_1}\hat a_1\tp\\
    \vdots\\
    \sqrt{\johnweights_\obs}\hat a_\obs\tp
  \end{bmatrix},
  \quad
  B_b = \begin{bmatrix}
    \sqrt{\johnweights_1}\hat b_1\tp\\
    \vdots\\
    \sqrt{\johnweights_\obs}\hat b_\obs\tp
  \end{bmatrix},
  \quad
  v = \begin{bmatrix}
    \sqrt{\johnweights_1} \vecnorm{\hat a_1}{2}^2\\
    \vdots\\
    \sqrt{\johnweights_\obs} \vecnorm{\hat a_\obs}{2}^2
  \end{bmatrix},
  \quad\text{and}\quad
  v^{ab} = \begin{bmatrix}
    \sqrt{\johnweights_1} \hat a_1\tp\hat b_1\\
    \vdots\\
    \sqrt{\johnweights_\obs} \hat a_\obs\tp\hat b_\obs
  \end{bmatrix}.
\end{align}
We claim that
\begin{subequations}
  \begin{align}
    BB\tp \preceq \Ind_\obs, \quad\text{and}\quad
    B_bB_b\tp \preceq 4\johnkappa^2\Ind_\obs \label{eq:john_BB_and_BbBb}.
  \end{align}
To see these claims, note that equation~\eqref{eq:john_hat_a_a_Identity}
implies that $B\tp B = \Ind_\dims$ and consequently, $BB\tp$ is an orthogonal
projection matrix and $BB\tp \preceq \Ind_\obs$.
Next, note that from equation~\eqref{eq:john_john_hat_b_hat_b_Ind}
we have that $B_b\tp B_b \preceq \johnkappa^2\Ind_\dims$, which implies that
$B_b B_b\tp \preceq \johnkappa^2\Ind_\obs$.
In asserting both these arguments, we have used the fact that for any matrix $B$,
the matrices $BB\tp$ and $B\tp B$ are PSD and have same set of eigenvalues.

Next, we bound the $\ell_2$ norm of the vectors $v$ and $v^{ab}$:
\begin{align}
  \vecnorm{v}{2}^2 &= \sum_{i=1}^\obs \johnweights_i \johnlocalslack_i^2 \ \stackrel{\mathrm{Lem.}~\ref{ppt:john_all_properties}~\ref{item:john_theta_square_sigma_sum_bound}}{\leq} \ 4\dims, \quad \text{and} \label{eq:john_v2}\\
    \vecnorm{v^{ab}}{2}^2 &= \sum_{i=1}^\obs \johnweights_i \parenth{\hat{a}_i\tp\hat{b}_i}^2 \leq \sum_{i=1}^\obs \johnweights_i \vecnorm{\hat{a}_i}{2}^2 \vecnorm{\hat{b}_i}{2}^2 \leq 4 \sum_{i=1}^\obs \johnweights_i \vecnorm{\hat{b}_i}{2}^2
    = 4\trace(B_b\tp B_b)
    \stackrel{\mathrm{eqn.}~\eqref{eq:john_BB_and_BbBb}}{\leq} 16 \johnkappa^2 \dims.\label{eq:john_vab}
\end{align}
\end{subequations}
We now prove the five claims of the lemma separately.

\subsubsection{Proof of bound~(\ref{eq:john_john_quadratic})}
\label{ssub:proof_of_part_item:john_quadratic}
Using Isserlis’ theorem~\citep{isserlis1918formula} for fourth order Gaussian moments, we have
\begin{align*}
  \Exs \parenth{\sum_{i=1}^\obs \johnweights_i \parenth{\hat{a}_i\tp \rvg}^2}^2
%
 = \sum_{i,j=1}^\obs \johnweights_i \johnweights_j \parenth{
    \vecnorm{\hat{a}_i}{2}^2\vecnorm{\hat{a}_j}{2}^2 + 2
    \parenth{\hat{a}_i\tp \hat{a}_j}^2 }
 = \sum_{i,j=1}^\obs \johnweights_i \johnweights_j
  \parenth{\johnlocalslack_{i}\johnlocalslack_{j} + 2
    \johnlocalslack_{i, j}^2 }
%
& \leq 24 \dims^2,
\end{align*}
where the last follows from Lemma~\ref{ppt:john_all_properties}. Applying the bound~\eqref{EqnJansonBound} with $k=2$ and $t = e \log(\frac{16}{\epsilon})$. Note that the bound is valid since $t \geq (2e)$ for all $\epsilon \in (0, \johnepsilonconst]$.

\subsubsection{Proof of bound~(\ref{eq:john_john_cubic})}
\label{ssub:proof_of_part_item:john_cubic}
Applying Isserlis’ theorem for Gaussian moments, we obtain
\begin{align*}
\Exs \parenth{\sum_{i=1}^\obs \johnweights_i \parenth{\hat{a}_i \tp \xi}^3}^2 = 9\underbrace{\sum_{i,j=1}^\obs \johnweights_i \johnweights_j \vecnorm{\hat{a}_i}{2}^2 \vecnorm{\hat{a}_j}{2}^2 \parenth{\hat{a}_i\tp \hat{a}_j}}_{=:N_1} + 6 \underbrace{\sum_{i,j=1}^\obs \johnweights_i \johnweights_j \parenth{\hat{a}_i\tp \hat{a}_j}^3}_{=:N_2}.
\end{align*}
We claim that $N_1 \leq 4\dims$ and $N_2 \leq 4\dims$.
Assuming these claims as given at the moment, we now complete the proof.
We have $\Exs \parenth{\sum_{i=1}^\obs \johnweights_i \parenth{\hat{a}_i \tp \xi}^3}^2 \leq 60 \dims$.
Applying the bound~\eqref{EqnJansonBound} with $\degree = 3$ and $t =
\parenth{\frac{2e}{3} \log \parenth{\frac{16}{\epsilon}}}^{3/2}$, and verifying that $t \geq \parenth{2e}^{3/2}$ for $\epsilon \in (0, \johnepsilonconst]$ yields the claim.

We now turn to prove the bounds on $N_1$ and $N_2$.
We have
\begin{align*}
  N_1 = \sum_{i,j=1}^\obs \johnweights_i {\vecnorm{\hat{a}_i}{2}^2 \hat{a}_i\tp
    \johnweights_j
    \vecnorm{\hat{a}_j}{2}^2 \hat{a}_j} &= \vecnorm{\sum_{i=1}^\obs
     \johnweights_i \vecnorm{\hat{a}_i}{2}^2 \hat{a}_i }{2}^2
  = \vecnorm{B\tp v}{2}^2
  \stackrel{\mathrm{eqn.}~\eqref{eq:john_BB_and_BbBb}}{\leq} \vecnorm{v}{2}^2
  \stackrel{\mathrm{eqn.}~\eqref{eq:john_v2}}{\leq} 4\dims.
\end{align*}
Next, applying Cauchy-Schwarz inequality and using equation~\eqref{eq:john_hat_a_norm}, we obtain
\begin{align*}
  N_2 = \sum_{i,j=1}^\obs \johnweights_i \johnweights_j
  \parenth{\hat{a}_i\tp \hat{a}_j}^3
%
& \leq \sum_{i,j=1}^\obs \johnweights_i \johnweights_j
  \johnlocalslack_{i, j}^2 \sqrt{\johnlocalslack_i\johnlocalslack_j}
& \stackrel{\mathmakebox[\widthof{====}]{
      (\mathrm{Lem.}~\ref{lemma:john_first_bounds}~\ref{item:john_theta_bound})
  }}{\leq} 4 \sum_{i,j=1}^\obs
  \johnweights_i \johnweights_j \johnlocalslack_{i, j}^2
& \stackrel{\mathmakebox[\widthof{=====}]{
      (\mathrm{Lem.}~\ref{ppt:john_all_properties}~\ref{item:john_theta_properties})
  }}{\leq} 4 \sum_{i=1}^\obs
  \johnweights_i \johnlocalslack_i
& = 4\dims.
\end{align*}

\subsubsection{Proof of bound~(\ref{eq:john_john_cubic_f})}
\label{ssub:proof_of_part_item:john_cubic_f}

Using Isserlis’ theorem for Gaussian moments, we have
\begin{multline*}
\Exs \parenth{\sum_{i=1}^\obs \johnweights_i \parenth{\hat{a}_i \tp \xi}^2 \parenth{\hat{b}_{x,i} \tp \xi}}^2
= \underbrace{\sum_{i, j=1}^\obs \johnweights_i \johnweights_j \vecnorm{\hat{a}_i}{2}^2 \vecnorm{\hat{a}_j}{2}^2 \parenth{\hat{b}_i\tp \hat{b}_j}}_{:=N_3}
+ 4 \underbrace{\sum_{i, j=1}^\obs \johnweights_i \johnweights_j \parenth{\hat{a}_i \tp \hat{a}_j} \parenth{\hat{a}_i \tp \hat{b}_i}  \parenth{\hat{a}_j\tp \hat{b}_j}}_{:=N_4} \\
+ 4 \underbrace{\sum_{i, j=1}^\obs \johnweights_i \johnweights_j \vecnorm{\hat{a}_i}{2}^2 \parenth{\hat{b}_i\tp \hat{a}_j} \parenth{\hat{a}_j\tp \hat{b}_j}}_{:=N_5}
+ 2 \underbrace{\sum_{i, j=1}^\obs \johnweights_i \johnweights_j \parenth{\hat{a}_i \tp \hat{a}_j}^2 \parenth{\hat{b}_i\tp \hat{b}_j}}_{:=N_6}
+ 4 \underbrace{\sum_{i, j=1}^\obs \johnweights_i \johnweights_j \parenth{\hat{a}_i \tp \hat{a}_j} \parenth{\hat{a}_i \tp \hat{b}_j} \parenth{\hat{b}_i\tp \hat{a}_j}}_{:=N_7}
\end{multline*}
We claim that all terms $N_k \leq 16\johnkappa^2 \dims, k \in \braces{3, 4, 5, 6, 7}$.
Putting the pieces together, we have
\begin{align*}
\Exs \parenth{\sum_{i=1}^\obs \johnweights_i \parenth{\hat{a}_i \tp \xi}^2 \parenth{\hat{b}_{x,i} \tp \xi}}^2 \leq 240 \johnkappa^2 \dims.
\end{align*}
Applying the bound~\eqref{EqnJansonBound} with $\degree = 3$ and $t =
\parenth{\frac{2e}{3} \log \parenth{\frac{16}{\epsilon}}}^{3/2}$ yields the claim. Note that for the given definition of $t$, we have $t \geq \parenth{2e}^{3/2}$ for $\epsilon \in (0, \johnepsilonconst]$ so that the bound~\eqref{EqnJansonBound} is valid.

It is now left to prove the bounds on $N_k$ for $k \in \braces{3, 4, 5, 6, 7}$.
We have
\begin{align*}
  N_3 &= \sum_{i,j=1}^\obs \johnweights_i {\vecnorm{\hat{a}_i}{2}^2 \hat{b}_i\tp
    \johnweights_j
    \vecnorm{\hat{a}_j}{2}^2 \hat{b}_j} = \vecnorm{\sum_{i=1}^\obs
     \johnweights_i \vecnorm{\hat{a}_i}{2}^2 \hat{b}_i }{2}^2
  = \vecnorm{B_b \tp v}{2}^2\
  \stackrel{\mathrm{eqn.}~\eqref{eq:john_BB_and_BbBb}}{\leq} 4 \johnkappa^2 \vecnorm{v}{2}^2\
 = \stackrel{\mathrm{eqn.}~\eqref{eq:john_v2}}{\leq} 16 \johnkappa^2 \dims,\\
 N_4 &= \sum_{i, j=1}^\obs \johnweights_i \johnweights_j \parenth{\hat{a}_i \tp \hat{a}_j} \parenth{\hat{a}_i \tp \hat{b}_i}  \parenth{\hat{a}_j\tp \hat{b}_j} = \vecnorm{B \tp v^{ab}}{2}^2\ \
 \stackrel{\mathrm{eqn.}~\eqref{eq:john_BB_and_BbBb}}{\leq} \vecnorm{v^{ab}}{2}^2\ \
\stackrel{\mathrm{eqn.}~\eqref{eq:john_vab}}{\leq} 16\johnkappa^2\dims, \quad \text{and}\\
N_5 &=\sum_{i, j=1}^\obs \johnweights_i \johnweights_j \vecnorm{\hat{a}_i}{2}^2 \parenth{\hat{b}_i\tp \hat{a}_j} \parenth{\hat{a}_j\tp \hat{b}_j} = \parenth{B\tp v^{ab}}\tp \parenth{B_b \tp v} \
\stackrel{\mathrm{C-S}}{\leq} \vecnorm{B \tp v^{ab}}{2} \vecnorm{B_b \tp v}{2} \ \leq 16 \johnkappa^2 \dims.
\end{align*}
For the term $N_6$, we have
\begin{align*}
N_6 = \sum_{i, j=1}^\obs \johnweights_i \johnweights_j \parenth{\hat{a}_i \tp \hat{a}_j}^2 \parenth{\hat{b}_i\tp \hat{b}_j}
& \stackrel{\mathmakebox[\widthof{=======}]{(\mathrm{C-S})}}{\leq}
\frac{1}{2}\sum_{i, j=1}^\obs \johnweights_i \johnweights_j \parenth{\hat{a}_i \tp \hat{a}_j}^2 \parenth{\vecnorm{\hat{b}_i}{2}^2 +  \vecnorm{\hat{b}_j}{2}^2}\\
& \stackrel{\mathmakebox[\widthof{=======}]{(\mathrm{symm. in}~i, j)}}{=} \sum_{i, j=1}^\obs \johnweights_i \johnweights_j \parenth{\hat{a}_i \tp \hat{a}_j}^2 \vecnorm{\hat{b}_i}{2}^2\\
& \stackrel{\mathmakebox[\widthof{=======}]{(\mathrm{eqn.}~\eqref{eq:john_hat_a_a_Identity})}}{\leq} \sum_{i=1}^\obs \johnweights_i \vecnorm{\hat{a}_i}{2}^2 \vecnorm{\hat{b}_i}{2}^2 \\
& \stackrel{\mathmakebox[\widthof{=======}]{(\mathrm{Lem.}~\ref{lemma:john_first_bounds}\ref{item:john_theta_bound})}}{\leq} 4 \sum_{i=1}^\obs \johnweights_i \vecnorm{\hat{b}_i}{2}^2\\
& \stackrel{\mathmakebox[\widthof{=======}]{(\mathrm{eqn.}~\eqref{eq:john_vab})}}{\leq} 16 \johnkappa^2 \dims.
\end{align*}
The bound on the term $N_7$ can be obtained in a similar fashion.

\subsubsection{Proof of bound~(\ref{eq:john_john_fourth})}
\label{ssub:proof_of_part_item:john_fourth}

Observe that $\hat{a}_i\tp \rvg\sim \NORMAL\parenth{0,
  \johnlocalslack_i}$ and hence $\Exs\parenth{\hat{a}_i\tp
  \rvg}^8 = 105\, \johnlocalslack_i^4$.
  Thus, we have
\begin{align*}
\Exs \parenth{\sum_{i=1}^\obs \johnweights_i
  \parenth{\hat{a}_i\tp \rvg}^4}^2 &
\stackrel{{\mathrm{C-S}}}{\leq}
\sum_{i,j=1}^\obs \johnweights_i \johnweights_j
\parenth{\Exs\parenth{\hat{a}_i\tp \rvg}^8}^{\frac{1}{2}}
\parenth{\Exs\parenth{\hat{a}_j\tp \rvg}^8}^{\frac{1}{2}}
= 105 \sum_{i,j=1}^\obs
\johnweights_i \johnweights_j \johnlocalslack_i^2
\johnlocalslack_j^2
= 105 \parenth{\sum_{i=1}^\obs \johnweights_i \johnlocalslack_i^2}^2.
\end{align*}
Now applying Lemma~\ref{ppt:john_all_properties}, we obtain that $\Exs \parenth{\sum_{i=1}^\obs \johnweights_i \parenth{\hat{a}_i\tp \rvg}^4}^2 \leq 1680 \dims^2$.
Consequently, applying the bound~\eqref{EqnJansonBound} with $\degree = 4$ and $t =
\parenth{\frac{e}{2} \log \parenth{\frac{16}{\epsilon}}}^{2}$ and noting that $t \geq \parenth{2e}^{2}$ for $\epsilon \in (0, \johnepsilonconst]$, yields the claim.

\subsubsection{Proof of bound~(\ref{eq:john_john_sixth})}
\label{ssub:proof_of_part_item:john_sixth}
Using the fact that $\Exs\parenth{\hat{a}_i\tp
  \rvg}^{12} = 945\, \johnlocalslack_i^6$ and an argument similar to the previous part yields that $\Exs \parenth{\sum_{i=1}^\obs \johnweights_i
  \parenth{\hat{a}_i\tp \rvg}^6}^2 \leq 15120 \dims^2$.

Finally, applying the bound~\eqref{EqnJansonBound} with $\degree = 6$ and $t =
\parenth{\frac{e}{3} \log \parenth{\frac{16}{\epsilon}}}^{3}$, and verifying that $t \geq \parenth{2e}^{3}$ for $\epsilon \in (0, \johnepsilonconst]$, yields the claim.



\section{Proof of Lov{\'a}sz's Lemma}
\label{sec:proof_of_lemma_lemma:lovasz_theorem}

We begin by formally defining the conductance ($\conductance$) of a
Markov chain on $(\Pspace, \Ball(\Pspace))$ with arbitrary
transition operator $\lazytrans$ and stationary distribution
$\target$. We assume that the operator $\lazytrans$ is lazy
and  thereby the stationary distribution $\target$ is unique.
Let $\transition_x = \lazytrans(\diracdelta_x)$ denote the transition
distribution at point $x$, then the conductance
$\conductance$ is defined as
\begin{align*}
  \conductance \defn \inf_{\substack{\set \in
      \Ball(\Pspace)\\ \target(\set) \in (0, 1/2)}}
  \frac{\conductance(\set)}{\target(\set)} \quad \text{where}
  \quad \conductance(\set) \defn \int_\set \transition_u(\Pspace \cap
  \set^c) d\target(u) \quad \text{ for any } \set
  \subseteq \Pspace.
\end{align*}
The conductance denotes the measure of the flow from a set to its
complement relative to its own measure, when initialized in the
stationary distribution.  If the conductance is high, the following
result shows that the Markov chain mixes fast.
\begin{lemma}
  \label{lemma:conductance_mixing}{\citep[Theorem~1.4]{lovasz1993random}}
  For any $M$-warm start $\initial$, the mixing time of the Markov
  chain with conductance~$\conductance$ is bounded as
\begin{align*}
  \vecnorm{\lazytrans^k(\initial)-\target}{\text{TV}} \leq \sqrt{M}
  \parenth{1-\frac{\conductance^2}{2}}^k \leq \sqrt{M}
  \exp\parenth{-k\frac{\conductance^2}{2} }.
\end{align*}
\end{lemma}
Note that this result holds for a general distribution $\target$
although we apply for uniform $\target$.  The result can be
derived from Cheeger's inequality for continuous-space discrete-time
Markov chain and elementary results in Calculus.
See, e.g.,
Theorem~1.4 and Corollary~1.5 by~\cite{lovasz1993random} for
a proof.  For ease in notation define $\Pspace \backslash \set
\defn \Pspace \cap \set^c$.  We now state a key isoperimetric
inequality.
\begin{lemma}{\citep[Theorem~6]{lovasz1999hit}}
  \label{lemma:isoperimetry}
  For any measurable sets $\set_1, \set_2 \subseteq \Pspace$, we have
  \begin{align*}
    \vol(\Pspace \backslash \set_1 \backslash \set_2)
    \cdot \vol(\Pspace) \geq d_\Pspace(\set_1, \set_2)
    \cdot \vol(\set_1) \cdot \vol(\set_2),
  \end{align*}
  where $d_\Pspace(\set_1, \set_2) \defn \inf_{x \in \set_1, y \in
    \set_2} d_\Pspace(x, y)$.
\end{lemma}
Since $\target$ is the uniform measure on $\Pspace$, this lemma
implies that
\begin{align}
  \label{eq:isoperimetry}
  \target(\Pspace \backslash \set_1 \backslash \set_2) \geq
  d_\Pspace(\set_1, \set_2) \cdot \target(\set_1) \cdot
  \target(\set_2).
\end{align}
In fact, such an inequality holds for an arbitrary log-concave
distribution~\citep{lovasz2003hit}.  In words, the inequality says that
for a bounded convex set any two subsets which are far apart, can not
have a large volume.  Taking these lemmas as given, we now complete
the proof.

\paragraph{Proof of (Lov{\'a}sz's) Lemma~\ref{lemma:Lovasz_theorem}:} 
\label{par:proof_of_lov_'a_sz_s_lemma}
We first bound the conductance of the Markov chain using the
assumptions of the lemma.  From
Lemma~\ref{lemma:conductance_mixing}, we see that the Markov
chain mixes fast if all the sets $\set$ have a high
conductance $\conductance(\set)$.  We claim that
\begin{align}
  \label{eq:condutance_bound}
  \conductance \geq \frac{\lovone \Delta}{64},
\end{align}
from which the proof follows by applying
Lemma~\ref{lemma:conductance_mixing}.  We now prove the
claim~\eqref{eq:condutance_bound} along the lines of Theorem~11 in the
paper by~\cite{lovasz1999hit}.  In particular, we show that under the
assumptions in the lemma, the sets with bad conductance are far apart
and thereby have a small measure under $\target$, whence the ratio
$\conductance(\set)/\target(\set)$ is not arbitrarily small.
Consider a partition $\set_1, \set_2$ of the set $\Pspace$ such that
$\set_1$ and $\set_2$ are measurable.  To prove
claim~\eqref{eq:condutance_bound}, it suffices to show that
\begin{align}
  \label{eq:to_prove_conductance}
  \frac{1}{\vol(\Pspace)}\int_{\set_1} \lazytrans_u(\set_2) du \geq
  \frac{\lovone \Delta}{64} \cdot \min\braces{\target(\set_1),
    \target(\set_2)},
\end{align}

Define the sets
\begin{align}
 \label{eq:define_three_sets}
  \set_1' &\defn \braces{ u \in \set_1 \bigg\vert
    \nolazytrans_u(\set_2) < \frac{\lovone}{2} },\quad \set_2' &\defn
  \braces{ v \in \set_2 \bigg\vert \nolazytrans_v(\set_1) <
    \frac{\lovone}{2} }, \quad \text{and}\quad \set_3' &\defn \Pspace
  \backslash \set_1' \backslash \set_2'.
\end{align}


\paragraph*{Case 1:}
If we have $\vol(\set_1') \leq \vol(\set_1)/2$ and consequently
$\vol(\Pspace \backslash \set_1') \geq \vol(\set_1)/2$, then
\begin{align*}
  \int_{\set_1} \lazytrans_u(\set_2) du \stackrel{(i)}{\geq}
        \frac{1}{2} \int_{\set_1 \backslash \set_1'}
        \nolazytrans_u(\set_2) du \stackrel{(ii)}{\geq} \frac{\lovone}{4}
        \vol(\set_1) \stackrel{(iii)}{\geq} \frac{\lovone\Delta}{4}
        \cdot \min\braces{\vol(\set_1), \vol(\set_2)},
\end{align*}
which implies the inequality~\eqref{eq:to_prove_conductance} since
$\target$ is the uniform measure on $\Pspace$.  In the above
sequence of inequalities, step $(i)$ follows from the definition of
the kernel $\lazytrans$, step $(ii) $ follows from the definition of
the set $\set_1'$~\eqref{eq:define_three_sets} and step $(iii)$ from
the fact that $\Delta < 1 $.
Dividing both sides by $\vol(\Pspace)$ yields the
inequality~\eqref{eq:to_prove_conductance} and we are done.


\paragraph*{Case 2:}

It remains to establish the inequality~\eqref{eq:to_prove_conductance}
for the case when $\vol(\set_i') \geq \vol(\set_i)/2$ for each $i\in
\braces{1, 2}$.  Now for any $u \in \set_1'$ and $v \in \set_2'$ we
have
\begin{align*}
  \vecnorm{\nolazytrans_u- \nolazytrans_v}{\text{TV}} \geq
  \nolazytrans_u(\set_1) - \nolazytrans_v(\set_1) = 1 -
  \nolazytrans_u(\set_2) - \nolazytrans_v(\set_1) > 1 - \lovone,
\end{align*}
and hence by assumption we have $d_\Pspace(\set_1', \set_2') \geq
\Delta$.  Applying Lemma~\ref{lemma:isoperimetry} and the definition
of $\set_3'$~\eqref{eq:define_three_sets} we find that
\begin{align}
  \label{eq:iso_s3}
  \vol(\set_3') \cdot \vol(\Pspace) \geq \Delta \cdot \vol(\set_1')
  \cdot \vol(\set_2') \geq \frac{\Delta}{4} \cdot \vol(\set_1) \cdot
  \vol(\set_2).
\end{align}
Using this inequality and the fact that for any $x \in [0, 1]$ we have
$x(1-x) \geq \min\braces{x, (1-x)}/2$ we obtain that
\begin{align}
  \label{eq:iso_s4}
  \target(\set_3') \geq \frac{\Delta}{4} \cdot \target(\set_1)
  \cdot \target(\set_2) \geq \frac{\Delta}{8}
  \min\braces{\target(\set_1), \target(\set_2)}.
\end{align}
We claim that
\begin{align}
  \label{eq:equality_of_T1_T2}
  \int_{\set_1} \lazytrans_u(\set_2)du = \int_{\set_2}
  \lazytrans_v(\set_1)dv.
\end{align}
Assuming the claim as given, we now complete the proof.  Using the
equation~\eqref{eq:equality_of_T1_T2}, we have
\begin{align*}
  \frac{1}{\vol(\Pspace)}\int_{\set_1} \lazytrans_u(\set_2) du
        &= \frac{1}{2\vol(\Pspace)}\parenth{ \int_{\set_1}
          \lazytrans_u(\set_2) du + \int_{\set_2}
          \lazytrans_v(\set_1)dv } \\ &\stackrel{(i)}{\geq}
        \frac{1}{2\vol(\Pspace)}\parenth{ \frac{1}{2} \int_{\set_1 \backslash
            \set_1'} \nolazytrans_u(\set_2) du + \frac{1}{2}
          \int_{\set_2 \backslash \set_2'} \nolazytrans_v(\set_2) dv
        }\\ &\stackrel{(ii)}{\geq} \frac{\lovone}{8}
        \frac{\vol(\set_3')}{\vol(\Pspace)}\\ &\stackrel{(iii)}{\geq}
        \frac{\lovone\Delta}{64} \min\braces{\target(\set_1),
          \target(\set_2)},
\end{align*}
where step $(i)$ follows from the definition of the kernel
$\lazytrans$, step $(ii)$ follows from the definition of the set
$\set_3'$~\eqref{eq:define_three_sets} and step $(iii)$ follows from
the inequality~\eqref{eq:iso_s4}.  Putting together the pieces yields
the claim~\eqref{eq:condutance_bound}.

It remains to prove the claim~\eqref{eq:equality_of_T1_T2}.  We make
use of the following result
\begin{align}
  \conductance(\set) = \conductance(\Pspace\backslash \set)
        \quad \mbox{for any measurable $\set \subseteq \Pspace$}.
  \label{eq:equality_of_conductance}
\end{align}
Using equation~\eqref{eq:equality_of_conductance} and noting that
$\set_1 = \Pspace\backslash \set_2$, we have
\begin{align*}
  \frac{1}{\vol(\Pspace)} \int_{\set_1} \lazytrans_u(\set_2) du
        = \int_{\set_1} \lazytrans_u(\set_2) \targetdensity(u) du
        = \conductance(\set_1) =
        \conductance(\Pspace\backslash \set_1) =
        \frac{1}{\vol(\Pspace)} \int_{\set_2} \lazytrans_v(\set_1) dv,
\end{align*}
which yields equation~\eqref{eq:equality_of_T1_T2}.


\paragraph*{Proof of result~\eqref{eq:equality_of_conductance}:}

Note that $\int_{\Pspace} \lazytrans_u(\set) d\target(u) =
\target(\set)$.  Thus, we have
\begin{align*}
  \conductance(\Pspace\backslash \set) = \int_{\Pspace\backslash
          \set} \lazytrans_u(\set) d\target(u) = \int_{\Pspace}
        \lazytrans_u(\set) d\target(u) - \int_{\set}
        \lazytrans_u(\set) d\target(u) &= \target(\set) -
        \int_{\set} \lazytrans_u(\set) d\target(u).
\end{align*}
Using the fact that $1-\lazytrans_u(\set) =
\lazytrans_u(\Pspace\backslash\set)$, we obtain
\begin{align*}
  \target(\set) - \int_{\set} \lazytrans_u(\set)
        d\target(u) = \int_{\set} d\target(u) - \int_{\set}
        \lazytrans_u(\set) d\target(u) =
        \int_{\set}\lazytrans_u(\Pspace\backslash\set) d\target(u)
        = \conductance(\set),
\end{align*}
thereby yielding the claim~\eqref{eq:equality_of_conductance}.

\vskip 0.2in
\bibliography{vaidya_bib}
\end{document}